\documentclass[11pt]{article}
\pdfoutput=1  
\usepackage{etoolbox}
\usepackage{preamble/color-edits}

\usepackage{authblk}

\usepackage[labelfont=bf]{caption}
\usepackage{booktabs}       %
\usepackage{bbm}       %
\newcommand{\twnote}[1]%
    {\textcolor{cyan}{\textbf{TW: #1}}}

\newtoggle{icml}
\togglefalse{icml} 
\newtoggle{arxiv}
\toggletrue{arxiv}

\def\ddefloop#1{\ifx\ddefloop#1\else\ddef{#1}\expandafter\ddefloop\fi}
\def\ddef#1{\expandafter\def\csname bb#1\endcsname{\ensuremath{\mathbb{#1}}}}
\ddefloop ABCDEFGHIJKLMNOPQRSTUVWXYZ\ddefloop

\def\ddefloop#1{\ifx\ddefloop#1\else\ddef{#1}\expandafter\ddefloop\fi}
\def\ddef#1{\expandafter\def\csname sf#1\endcsname{\ensuremath{\mathsf{#1}}}}
\ddefloop ABCDEFGHIJKLMNOPQRSTUVWXYZ\ddefloop

\def\ddefloop#1{\ifx\ddefloop#1\else\ddef{#1}\expandafter\ddefloop\fi}
\def\ddef#1{\expandafter\def\csname frak#1\endcsname{\ensuremath{\mathfrak{#1}}}}
\ddefloop ABCDEFGHIJKLMNOPQRSTUVWXYZ\ddefloop

\def\ddefloop#1{\ifx\ddefloop#1\else\ddef{#1}\expandafter\ddefloop\fi}
\def\ddef#1{\expandafter\def\csname fr#1\endcsname{\ensuremath{\mathfrak{#1}}}}
\ddefloop ABCDEFGHIJKLMNOPQRSTUVWXYZ\ddefloop

\def\ddefloop#1{\ifx\ddefloop#1\else\ddef{#1}\expandafter\ddefloop\fi}
\def\ddef#1{\expandafter\def\csname eul#1\endcsname{\ensuremath{\EuScript{#1}}}}
\ddefloop ABCDEFGHIJKLMNOPQRSTUVWXYZ\ddefloop

\def\ddefloop#1{\ifx\ddefloop#1\else\ddef{#1}\expandafter\ddefloop\fi}
\def\ddef#1{\expandafter\def\csname scr#1\endcsname{\ensuremath{\mathscr{#1}}}}
\ddefloop ABCDEFGHIJKLMNOPQRSTUVWXYZ\ddefloop

\def\ddefloop#1{\ifx\ddefloop#1\else\ddef{#1}\expandafter\ddefloop\fi}
\def\ddef#1{\expandafter\def\csname b#1\endcsname{\ensuremath{\mathbf{#1}}}}
\ddefloop ABCDEFGHIJKLMNOPQRSTUVWXYZ\ddefloop

\def\ddefloop#1{\ifx\ddefloop#1\else\ddef{#1}\expandafter\ddefloop\fi}
\def\ddef#1{\expandafter\def\csname bhat#1\endcsname{\ensuremath{\hat{\mathbf{#1}}}}}
\ddefloop ABCDEFGHIJKLMNOPQRSTUVWXYZ\ddefloop

\def\ddefloop#1{\ifx\ddefloop#1\else\ddef{#1}\expandafter\ddefloop\fi}
\def\ddef#1{\expandafter\def\csname btil#1\endcsname{\ensuremath{\tilde{\mathbf{#1}}}}}
\ddefloop ABCDEFGHIJKLMNOPQRSTUVWXYZ\ddefloop

\def\ddefloop#1{\ifx\ddefloop#1\else\ddef{#1}\expandafter\ddefloop\fi}
\def\ddef#1{\expandafter\def\csname bst#1\endcsname{\ensuremath{\mathbf{#1}^\star}}}
\ddefloop ABCDEFGHIJKLMNOPQRSTUVWXYZ\ddefloop

\def\ddefloop#1{\ifx\ddefloop#1\else\ddef{#1}\expandafter\ddefloop\fi}
\def\ddef#1{\expandafter\def\csname bst#1\endcsname{\ensuremath{\mathbf{#1}^\star}}}
\ddefloop abcdefghijklmnopqrstuvwxyz\ddefloop

\def\ddefloop#1{\ifx\ddefloop#1\else\ddef{#1}\expandafter\ddefloop\fi}
\def\ddef#1{\expandafter\def\csname bhat#1\endcsname{\ensuremath{\hat{\mathbf{#1}}}}}
\ddefloop abcdefghijklmnopqrstuvwxyz\ddefloop

\def\ddefloop#1{\ifx\ddefloop#1\else\ddef{#1}\expandafter\ddefloop\fi}
\def\ddef#1{\expandafter\def\csname b#1\endcsname{\ensuremath{\mathbf{#1}}}}
\ddefloop abcdeghijklnopqrstuvwxyz\ddefloop

\def\ddefloop#1{\ifx\ddefloop#1\else\ddef{#1}\expandafter\ddefloop\fi}
\def\ddef#1{\expandafter\def\csname barb#1\endcsname{\ensuremath{\bar{\mathbf{#1}}}}}
\ddefloop abcdefghijklmnopqrstuvwxyz\ddefloop

\def\ddef#1{\expandafter\def\csname c#1\endcsname{\ensuremath{\mathcal{#1}}}}
\ddefloop ABCDEFGHIJKLMNOPQRSTUVWXYZ\ddefloop
\def\ddef#1{\expandafter\def\csname h#1\endcsname{\ensuremath{\widehat{#1}}}}
\ddefloop ABCDEFGHIJKLMNOPQRSTUVWXYZ\ddefloop
\def\ddef#1{\expandafter\def\csname hc#1\endcsname{\ensuremath{\widehat{\mathcal{#1}}}}}
\ddefloop ABCDEFGHIJKLMNOPQRSTUVWXYZ\ddefloop
\def\ddef#1{\expandafter\def\csname t#1\endcsname{\ensuremath{\widetilde{#1}}}}
\ddefloop ABCDEFGHIJKLMNOPQRSTUVWXYZ\ddefloop
\def\ddef#1{\expandafter\def\csname tc#1\endcsname{\ensuremath{\widetilde{\mathcal{#1}}}}}
\ddefloop ABCDEFGHIJKLMNOPQRSTUVWXYZ\ddefloop

\usepackage{boxedminipage}
\usepackage{multirow,nicefrac}
\usepackage{makecell,upgreek}
\usepackage{footnote}
\usepackage{longtable}
\usepackage[shortlabels]{enumitem}
\usepackage{tablefootnote}
\usepackage[T1]{fontenc}
\usepackage{verbatim} 
\usepackage[utf8]{inputenc}

\usepackage{booktabs}   
\usepackage{float}

\usepackage{xcolor}

\usepackage{amsthm}

\iftoggle{arxiv}
{
\usepackage{subfig}
  \usepackage[breaklinks=true]{hyperref}

  \usepackage{fullpage}
  \usepackage[noend]{algpseudocode}
  \usepackage{algorithm}
}
{
\usepackage{subfig}
  \usepackage[breaklinks=true]{hyperref}
  \usepackage[noend]{algpseudocode}
  \usepackage{algorithm}
}

\usepackage{amsfonts,amssymb,mathrsfs}
\usepackage{mathtools}
\DeclareMathSymbol{\shortminus}{\mathbin}{AMSa}{"39}

\iftoggle{arxiv}{
	\usepackage{natbib}
  }
  {
  }

\usepackage{prettyref}

\usepackage[capitalise,nameinlink]{cleveref}
\Crefname{equation}{Eq.}{Eqs.}
\Crefname{assumption}{Assumption}{Assumptions}
\Crefname{condition}{Condition}{Conditions}
\Crefname{claim}{Claim}{Claims}

\usepackage{breakcites,bm}

\hypersetup{colorlinks,citecolor=blue,linkcolor = blue}
\usepackage{latexsym}
\usepackage{relsize}
\usepackage{thm-restate}
\usepackage{appendix}

\usepackage{xcolor}
\usepackage{dsfont}

\newcommand{\N}{\mathbb{N}}

\newcommand{\R}{\mathbb{R}}

\numberwithin{equation}{section}
\newcommand\numberthis{\addtocounter{equation}{1}\tag{\theequation}}

\newcommand{\eye}{\mathbf{I}}

\newcommand{\nablatwo}{\nabla^{\,2}}

\newcommand{\rmd}{\mathrm{d}}

\newcommand{\bzero}{\ensuremath{\mathbf 0}}

\def\bB{\mathbf{B}}

\def\bu{\mathbf{u}}

\def\by{\mathbf{y}}
\def\bz{\mathbf{z}}

\def\bu{\mathbf{u}}
\def\bv{\mathbf{v}}

\def\bw{\mathbf{w}}
\def\bA{\mathbf{A}}

\def\bI{\mathbf{I}}
\def\bJ{\mathbf{J}}
\def\bP{\mathbf{P}}
\def\bQ{\mathbf{Q}}

\DeclareFontFamily{U}{mathx}{\hyphenchar\font45}
\DeclareFontShape{U}{mathx}{m}{n}{
      <5> <6> <7> <8> <9> <10>
      <10.95> <12> <14.4> <17.28> <20.74> <24.88>
      mathx10
      }{}
\DeclareSymbolFont{mathx}{U}{mathx}{m}{n}
\DeclareFontSubstitution{U}{mathx}{m}{n}
\DeclareMathAccent{\widecheck}{0}{mathx}{"71}
\DeclareMathAccent{\wideparen}{0}{mathx}{"75}

\newcommand{\ignore}[1]{}

\DeclareMathOperator{\BigOm}{\mathcal{O}}

\newcommand{\BigOh}[1]{\BigOm\left({#1}\right)}

\newcommand{\iidsim}{\overset{\mathrm{i.i.d}}{\sim}}

\newcommand{\dimx}{d_{\mathsf{x}}}
\newcommand{\dimu}{d_{\mathsf{u}}}

\newcommand{\op}{\mathrm{op}}

\iftoggle{icml}
{
  \newcommand{\icmlpar}[1]{\textbf{#1}}
}
{
  \newcommand{\icmlpar}[1]{\paragraph{#1}}

}

\newcommand{\algcomment}[1]{\textcolor{blue!70!black}{\footnotesize{\texttt{\textbf{\% #1}}}}}

	\theoremstyle{plain}
	\newtheorem{theorem}{Theorem}
	
	\newtheorem{oracle}{Oracle}[section]

  \newtheorem{desideratum}{Desideratum}

	\newtheorem{lemma}{Lemma}[section]
	\newtheorem{claim}{Claim}[section]
	\newtheorem{corollary}{Corollary}[section]
	\newtheorem{proposition}[lemma]{Proposition}

	\theoremstyle{definition}

	\newtheorem{definition}{Definition}[section]

	\newtheorem{remark}{Remark}[section]

  \newtheorem{assumption}{Assumption}[section]
  \newtheorem{condition}{Condition}[section]

\makeatletter
\newcommand{\neutralize}[1]{\expandafter\let\csname c@#1\endcsname\count@}
\makeatother

\newtheorem*{theorem*}{Theorem}
\newtheorem*{lemma*}{Lemma}
\newtheorem*{corollary*}{Corollary}
\newtheorem*{proposition*}{Proposition}
\newtheorem*{claim*}{Claim}
\newtheorem*{fact*}{Fact}
\newtheorem*{observation*}{Observation}

\newtheorem*{definition*}{Definition}
\newtheorem*{remark*}{Remark}
\newtheorem*{example*}{Example}

\newtheoremstyle{named}{}{}{\itshape}{}{\bfseries}{}{.5em}{\Cref{#3} {\normalfont (informal)} }
{}
\theoremstyle{named}

\theoremstyle{plain}

\DeclareMathAlphabet{\mathbfsf}{\encodingdefault}{\sfdefault}{bx}{n}

\DeclareMathOperator*{\argmin}{arg\,min}

\let\Pr\relax
\DeclareMathOperator{\Pr}{\mathbb{P}}

\newcommand{\floor}[1]{\lfloor #1 \rfloor}

\newcommand{\poly}{\mathrm{poly}}

\newcommand{\half}{\frac{1}{2}}

\renewcommand{\leq}{~\le~}
\renewcommand{\geq}{~\ge~}

\let\oldtfrac\tfrac
\renewcommand{\tfrac}[2]{\smash{\oldtfrac{#1}{#2}}}

\let\nablaold\nabla
\renewcommand{\nabla}{\nablaold\mkern-2.5mu}

\newcommand{\Exp}{\mathbb{E}}

\newcommand{\sigw}{\sigma_{w}}

\newcommand{\Z}{\mathbb{Z}}

\newcommand{\Term}{\mathrm{Term}}
\newcommand{\rank}{\mathrm{rank}}

\definecolor{disccolor}{rgb}{.1, 0, .5}
\definecolor{contcolor}{rgb}{1, 0, 0}

\newcommand{\discfont}[1]{{\color{disccolor} \mathtt{#1}}}

\newcommand{\interior}{\mathrm{int}}

\newcommand{\pione}{\pi^{(1)}}

\newcommand{\bPctk}[1][k]{\discfont{P}^{\mathrm{ct}}_{#1}}
\newcommand{\bPsubk}[1][k]{\discfont{P}^{\mathrm{sub}}_{#1}}

\newcommand{\bXctk}[1][k]{\discfont{X}^{\mathrm{ct}}_{#1}}
\newcommand{\bXsubk}[1][k]{\discfont{X}^{\mathrm{sub}}_{#1}}

\newcommand{\bYctk}[1][k]{\discfont{Y}^{\mathrm{ct}}_{#1}}
\newcommand{\bYsubk}[1][k]{\discfont{Y}^{\mathrm{sub}}_{#1}}

\newcommand{\bKctk}[1][k]{\discfont{K}^{\mathrm{ct}}_{#1}}

\newcommand{\deluk}[1][k]{\updelta {\discfont{u}}_{#1}}

\newcommand{\Mtaypi}{M_{\mathrm{tay},\pi}}
\newcommand{\bztilk}[1][k]{\tilde{\discfont{z}}_{#1}}

\newcommand{\stepric}{\step_{\mathrm{ric}}}
\newcommand{\stepdyn}{\step_{\mathrm{dyn}}}

\newcommand{\stepctrlpi}[1][\pi]{\step_{\mathrm{ctrl},#1}}
\newcommand{\stepricpi}[1][\pi]{\step_{\mathrm{ric},#1}}

\newcommand{\steptaypi}[1][\pi]{\step_{\mathrm{tay},#1}}

\newcommand{\Pialg}{\Pi_{\mathrm{alg}}}

\newcommand{\errpsi}{\mathrm{Err}_{\Psi}}
\newcommand{\errpsipi}[1][\pi]{\mathrm{Err}_{\Psi,#1}}
\newcommand{\errnab}{\mathrm{Err}_{\nabla}}
\newcommand{\errnabpi}[1][\pi]{\mathrm{Err}_{\nabla, #1}}

\newcommand{\dst}{d_{\star}}
\newcommand{\Evest}{\cE_{\mathrm{est}}}
\newcommand{\errx}{\mathrm{Err}_{\hat{x}}}

\newcommand{\Lol}{L_{\mathrm{ol}}}

\newcommand{\LQ}{L_{\mathrm{cost}}}
\newcommand{\MQ}{M_{Q}}
\newcommand{\bhatAkpi}[1][k]{\hat{\discfont{A}}^\pi_{#1}}
\newcommand{\bhatBkpi}[1][k]{\hat{\discfont{B}}^\pi_{#1}}
\newcommand{\gamcont}{\gamma_{\mathrm{ctr}}}
\newcommand{\bPpik}[1][k]{\discfont{P}_{#1}^{\pi}}
\newcommand{\bzk}[1][k]{\discfont{z}_{#1}}
\newcommand{\bPpist}{\bP^{\pi}_{\mathrm{opt}}}
\newcommand{\bPpikst}[1][k]{\discfont{P}_{\mathrm{opt},#1}^{ \pi}}

\newcommand{\bchuk}[1][k]{\check{\discfont{u}}_{#1}}

\newcommand{\btilxpict}{\btilx^{\pi,\mathrm{ct}}}
\newcommand{\btilupict}{\btilu^{\pi,\mathrm{ct}}}

\newcommand{\btilxpijac}{\btilx^{\pi,\mathrm{jac}}}
\newcommand{\btilupijac}{\btilu^{\pi,\mathrm{jac}}}

\newcommand{\Jpi}{\cJ_T^\pi}
\newcommand{\Jpijac}{\cJ_T^{\pi,\mathrm{jac}}}

\newcommand{\LFP}{\mu_{\mathrm{ric}}}

\newcommand{\btilx}{\tilde{\bx}}
\newcommand{\Lpi}{L_{\pi}}

\newcommand{\Jdisc}{\cJ_T^{\mathrm{disc}}}

\newcommand{\Kpiprst}{\mu_{\pi',\star}}
\newcommand{\Mtayjpi}[1][\pi]{M_{\cJ,\mathrm{tay},#1}}
\newcommand{\Btwo}{B_2}
\newcommand{\Binf}{B_{\infty}}
\newcommand{\kname}{\kappa}
\newcommand{\Kpitwo}{\kname_{\pi,2}}
\newcommand{\Kpione}{\kname_{\pi,1}}
\newcommand{\Kpiinf}{\kname_{\pi,\infty}}
\newcommand{\Ltaypiq}{L_{\mathrm{tay},q,\pi}}
\newcommand{\Ltaypiinf}{L_{\mathrm{tay},\infty,\pi}}
\newcommand{\Ltaypitwo}{L_{\mathrm{tay},2,\pi}}

\newcommand{\Lnabpiinf}[1][\pi]{L_{\nabla,#1,\infty}}

\newcommand{\Mtaypitwo}{M_{\mathrm{tay},2,\pi}}
\newcommand{\Mtaypiinf}{M_{\mathrm{tay},\inf,\pi}}
\newcommand{\Mtaypiq}{M_{\mathrm{tay},q,\pi}}
\newcommand{\Btaypitwo}{B_{\mathrm{tay},2,\pi}}
\newcommand{\Btaypiinf}{B_{\mathrm{tay},\inf,\pi}}
\newcommand{\Btaypiq}{B_{\mathrm{tay},q,\pi}}

\newcommand{\bmsf}[1]{\bm{\mathsf{#1}}}

\newcommand{\Sym}{\mathsf{Sym}}
\newcommand{\bLamk}[1][k]{\discfont{\Lambda}_{#1}}

\newcommand{\bOmegak}[1][k]{\discfont{\Omega}_{#1}}

\newcommand{\bZk}[1][k]{\discfont{Z}_{#1}}
\newcommand{\bYk}[1][k]{\discfont{Y}_{#1}}
\newcommand{\bTheta}{\bm{\Theta}}
\newcommand{\KB}{K_{B}}

\newcommand{\KA}{K_{A}}

\newcommand{\subbrack}[1]{#1}
\newcommand{\DelK}{\Delta_K}
\newcommand{\bBk}[1][k]{\discfont{B}_{\subbrack{#1}}}
\newcommand{\Delce}{\Delta_{\mathrm{ce}}}
\newcommand{\bAk}[1][k]{\discfont{A}_{\subbrack{#1}}}
\newcommand{\bAclk}[1][k]{\discfont{A}_{\mathrm{cl},\subbrack{#1}}}
\newcommand{\bPk}[1][k]{\discfont{P}_{\subbrack{#1}}}
\newcommand{\bKk}[1][k]{\discfont{K}_{\subbrack{#1}}}

\newcommand{\bXk}[1][k]{\discfont{X}_{\subbrack{#1}}}
\newcommand{\bXik}[1][k]{\discfont{\Xi}_{\subbrack{#1}}}
\newcommand{\DelA}{\Delta_A}
\newcommand{\DelB}{\Delta_B}

\newcommand{\bhatBk}[1][k]{\hat{\discfont{B}}_{\subbrack{#1}}}
\newcommand{\bhatAk}[1][k]{\hat{\discfont{A}}_{\subbrack{#1}}}
\newcommand{\bhatPk}[1][k]{\hat{\discfont{P}}_{\subbrack{#1}}}
\newcommand{\bhatKk}[1][k]{\hat{\discfont{K}}_{\subbrack{#1}}}

\newcommand{\bxk}[1][k]{\discfont{x}_{\subbrack{#1}}}
\newcommand{\buk}[1][k]{\discfont{u}_{\subbrack{#1}}}

\newcommand{\bstBk}[1][k]{\discfont{B}^\star_{\subbrack{#1}}}
\newcommand{\bstAk}[1][k]{\discfont{A}^\star_{\subbrack{#1}}}
\newcommand{\bstPk}[1][k]{\discfont{P}^\star_{\subbrack{#1}}}
\newcommand{\bstKk}[1][k]{\discfont{K}^\star_{\subbrack{#1}}}
\newcommand{\btilKk}[1][k]{\tilde{\discfont{K}}_{\subbrack{#1}}}
\newcommand{\btilPk}[1][k]{\tilde{\discfont{P}}_{\subbrack{#1}}}

\newcommand{\bXi}{\bm{\Xi}}
\newcommand{\bstTheta}{\bm{\Theta}^\star}
\newcommand{\bhatTheta}{\hat{\bm{\Theta}}}

\newcommand{\maxop}{\max,\mathrm{op}}
\newcommand{\Mtil}{\tilde{M}}
\newcommand{\Pce}{\discfont{P}^{\mathrm{ce}}}

\newcommand{\Poptk}[1][k]{\discfont{P}_{\subbrack{#1}}^{\mathrm{opt}}}
\newcommand{\Pcek}[1][k]{\discfont{P}_{\subbrack{#1}}^{\mathrm{ce}}}
\newcommand{\Pfeedk}[1][k]{\discfont{P}_{\subbrack{#1}}^{\mathrm{feed}}}

\newcommand{\Koptk}[1][k]{\discfont{K}_{\subbrack{#1}}^{\mathrm{opt}}}

    \newcommand{\btilYk}[1][k]{\tilde{\discfont{Y}}_{#1}}

\addauthor{dp}{cyan}
\addauthor{ms}{magenta}
\addauthor{st}{red}
\addauthor{tw}{purple}

\makeatletter
\newcommand{\printfnsymbol}[1]{%
  \textsuperscript{\@fnsymbol{#1}}%
}
\makeatother

\title{The Power of Learned Locally Linear Models for Nonlinear Policy Optimization}
\author[1]{Daniel Pfrommer\thanks{These authors contributed equally to this work.}}
\author[1]{Max Simchowitz\printfnsymbol{1}}
\author[2]{Tyler Westenbroek}
\author[3]{\authorcr Nikolai Matni}
\author[4]{Stephen Tu}
\affil[1]{Massachusetts Institute of Technology}
\affil[2]{University of California, Berkeley}
\affil[3]{University of Pennsylvania}
\affil[4]{Google Brain}

\begin{document}
\maketitle

\begin{abstract}

A common pipeline in learning-based control is to iteratively estimate a model of system dynamics, and apply a trajectory optimization algorithm - e.g.~$\mathtt{iLQR}$ - on the learned model  to minimize a target cost. This paper conducts a rigorous analysis of a simplified variant of this strategy for general nonlinear systems. We analyze an algorithm which iterates between estimating local linear models of nonlinear system dynamics and performing $\mathtt{iLQR}$-like policy updates. We demonstrate that this algorithm attains sample complexity polynomial in relevant problem parameters, and, by synthesizing locally stabilizing gains, overcomes exponential dependence in problem horizon. Experimental results validate the performance of our algorithm, and compare to natural deep-learning baselines.
\end{abstract}

\section{Introduction}\label{sec:intro}
\newcommand{\iLQR}{\mathtt{iLQR}}
\newcommand{\JSP}{\mathtt{JSP}}

Machine learning methods such as model-based reinforcement learning have lead to a number of breakthroughs in key applications across robotics and control \citep{kocijan2004gaussian,tassa2012synthesis,nagabandi2019deep}. A popular technique in these domains is learning-based model-predictive control (MPC) \citep{morari1999model,williams2017information}, wherein a model learned from data is used to repeatedly solve online planning problems to control the real system. It has long been understood that solving MPC \emph{exactly}--both with perfectly accurate dynamics and minimization to globally optimality for each planning problem--enjoys numerous beneficial control-theoretic properties \citep{jadbabaie2001stability}.

Unfortunately, the above situation is not reflective of practice. For one, most systems of practical interest are \emph{nonlinear}, and therefore exact global recovery of system dynamics suffers from a curse of dimensionality. And second, the nonlinear dynamics render any natural trajectory planning problem nonconvex, making global optimality elusive. In this work, we focus on learning-based trajectory optimization, the ``inner-loop'' in MPC. We ask \emph{when can we obtain  rigorous guarantees about the solutions to nonlinear trajectory optimization under unknown dynamics?}

We take as our point of departure the $\iLQR$ algorithm \citep{li2004iterative}. Initially proposed under known dynamics, $\iLQR$ solves a planning objective by solving an iterative linear control problem around a first-order Taylor expansion (the \emph{Jacobian linearization}) of the dynamics, and second-order Taylor expansion of the control costs. In solving this objective, $\iLQR$ synthesizes a sequence of locally-stabilizing feedback gains, and each $\iLQR$-update can be interpreted as a gradient-step through the closed-loop linearized dynamics in feedback with these gains. This has the dual benefit of proposing a locally stabilizing policy (not just an open-loop trajectory), and of stabilizing the gradients to circumvent exponential blow-up in planning horizon.
$\iLQR$, and its variants \citep{todorov2005generalized,williams2017information}, are now ubiquitous in robotics and control applications; and, when dynamics are unknown or uncertain, one can simply substitute the exact dynamics model with an estimate (e.g.~\citet{levine2013guided}). In this case, dynamics are typically estimated with neural networks. Thus, Jacobian linearizations can be computed by automated differentiation (AutoDiff) through the learned model.

\textbf{Contributions.} We propose and analyze an alternative to the aforementioned approach of first learning a deep neural model of dynamics, and then performing AutoDiff to conduct the $\iLQR$ update. We consider a simplified setting with fixed initial starting condition. Our algorithm maintains a \emph{policy}, specified by an open-loop input sequence and a sequence of stabilizing gains, and loops two steps: \textbf{(a)} it learns local linear model of the closed-loop linearized dynamics (in feedback with these gains), which we use to perform a gradient update; \textbf{(b)} it re-estimates a linear model after the gradient step, and synthesizes a new set of set gains from this new model. In contrast to past approaches, our algorithm \emph{only ever estimates linear models of system dynamics.} 

For our analysis, we treat the underlying system dynamics as continuous and policy as discrete; this reflects real physical systems, is representative of discrete-time simulated environments which update on smaller timescales than learned policies, and renders explicit the effect of discretization size on sample complexity. We consider an interaction model where we query an oracle for trajectories corrupted with measurement (but not process) noise. Our approach enjoys the following theoretical properties.
\iftoggle{icml}{}{\begin{itemize}}
\iftoggle{icml}{\icmlpar{1.} }{\item[\textbf{1.}]}
 Using a number of iterations and oracle queries \emph{polynomial} in relevent problem parameters and tolerance $\epsilon$, it computes a policy $\pi$ whose input sequence is an $\epsilon$-first order stationary point for the $\iLQR$ approximation of the planning objective (i.e.,~the gradient through the closed-loop linearized dynamics has norm $\le \epsilon$). Importantly, learning the linearized model at each iteration obviates the need for global dynamics models, allowing for sample complexity polynomial in dimension.

\iftoggle{icml}{\icmlpar{2.} }{\item[\textbf{2.}]}
 We show that contribution $\mathbf{1}$ implies convergence to a local-optimality criterion we call an $\epsilon$-approximate \emph{Jacobian Stationary Point} ($\epsilon$-$\JSP$); this roughly equates to the open-loop trajectory under $\pi$ having cost within $\epsilon$-\emph{globally optimal} for the linearized dynamics about its trajectory. 
\iftoggle{icml}{}{\end{itemize}}

$\JSP$s are purely a property of the open-loop inputs, allowing comparison of the quality of the open-loop plan with differing gains. Moreover, the results of \citet{westenbroek2021stability} show that an approximate $\JSP$s for certain planning objective enjoy favorable \emph{global properties},  despite (as we show) being computable from (local) gradient-based search (see \Cref{app:westenbroek_app} for elaboration).

\icmlpar{Experimental Findings.} We validate our algorithms on standard models of the quadrotor and inverted pendulum, finding an improved performance as iteration number increases, and that the synthesized gains prescribed by $\iLQR$ yield improved performance over vanilla gradient updates.

\icmlpar{Proof Techniques.} Central to our analysis are novel perturbation bounds for controlled nonlinear differential equations. Prior results primarily focus on the open-loop setting \citep[Theorem 5.6.9]{polak2012optimization}, and implicitly hide an exponential dependence on the time horizon for open-loop unstable dynamics. We provide what is to the best of our knowledge the first analysis which demonstrates that local feedback can overcome this pathology. Specifically, we show that if the feedback gains stabilize the Jacobian-linearized dynamics, then (a) the Taylor-remainder of the first-order  approximation of the dynamics does \emph{not} scale exponentially on problem horizon (\Cref{prop:tay_ex_body}), and (b) small perturbations to the nominal input sequence preserve closed-loop stability of the linearized dynamics. These findings are detailed in \Cref{app:taylor_exp_summary}, and enable us to bootstrap the many recent advances in statistical learning for linear systems to our nonlinear setting. 



\subsection{Related Work} 

$\iLQR$ \citep{li2004iterative} is a more computationally expedient variant of differential dynamic programming ($\mathtt{DPP}$) \cite{jacobson1970differential}; numerous variants exist, notably $\mathtt{iLQG}$ \citep{todorov2005generalized} and $\mathtt{MPPI}$ \citep{williams2017information} which better address problem stochasticity. $\iLQR$ is a predominant approach for the ``inner loop'' trajectory optimization step in MPC, with applications  in robotics \citep{tassa2012synthesis}, quadrotors \citep{torrente2021data}, and autonomous racing \citep{kabzan2019learning}.  

A considerable literature has combined $\iLQR$ with learned dynamics models;  here, the Jacobian linearization matrices are typically derived through automated differentiation \citep{levine2013guided,levine2014learning,koller2018learning}, though local kernel least squares regression has also been studied  \citep{rosolia2019learning,papadimitriou2020control}. In these works, the dynamics models are refined/re-estimated as the policy is optimized; thus, these approaches are one instantiation of the broader iterative learning control (ILC) paradigm \citep{arimoto1984bettering}; other instantiations of ILC  include \cite{kocijan2004gaussian, dai2021lyapunov,aswani2013provably,bechtle2020curious}. 

Recent years have seen multiple rigorous guarantees for learning system identification and control \citep{dean2017sample,simchowitz2018learning,oymak2019non,agarwal2019online,simchowitz2020naive}, though a general theory of learning for nonlinear control remains elusive. Recent progress includes nonlinear imitation learning \citep{pfrommer2022tasil}, learning systems with known nonlinearities in the dynamics \citep{sattar2020non,foster2020learning,mania2020active} or perception model \citep{mhammedi2020learning,dean2021certainty}.

Lastly, there has been recent theoretical attention given to the study of first-order trajectory optimization methods. \citet{roulet2019iterative} perform an extension theoretical study of the convergence properties of $\iLQR$, $\mathtt{iLQG}$, and $\mathtt{DPP}$ with \emph{exact} dynamics models, and corroborate their findings experimentally. \citet{westenbroek2021stability} show further that for certain classes of nonlinear systems, all $\epsilon$-first order stationary points of a suitable trajectory optimization objective induce trajectories which converge exponentially to desired system equilbria. In some cases, there may be multiple spurious local minima, each of which is nevertheless exponentially stabilizing. Examining the proof \cite{westenbroek2021stability} shows the result holds more generally for all $\epsilon$-$\JSP$s, and therefore we use their work justify the $\JSP$ criterion proposed in this paper.

\newcommand{\pin}[1][n]{\pi^{(#1)}}
\newcommand{\ltwou}{\cL_{2}(\cU)}

\newcommand{\bsfU}{\bmsf{U}}
\newcommand{\buvec}{\vec{\discfont{u}}}

\newcommand{\xoffpi}{\bx^{\pi}_{\mathrm{orac}}}
\newcommand{\uoffpi}{\bu_{\mathrm{orac}}^\pi}
\newcommand{\uoffpik}[1][k]{\discfont{u}_{\mathrm{orac},#1}^{\pi}}
\newcommand{\xoffpik}[1][k]{\discfont{x}_{\mathrm{orac},#1}^{\pi}}
\newcommand{\uoffpink}[1][k]{\discfont{u}_{\mathrm{orac},#1}^{\pin}}
\newcommand{\xoffpink}[1][k]{\discfont{x}_{\mathrm{orac},#1}^{\pin}}
\newcommand{\sigorac}{\sigma_{\mathrm{orac}}}
\newcommand{\xapxpi}{\bx^{\mathrm{apx},\pi}}
\newcommand{\uapxpi}{\bu^{\mathrm{apx},\pi}}
\newcommand{\uapxpik}[1][k]{\discfont{u}_{#1}^{\mathrm{apx},\pi}}
\newcommand{\xapxpik}[1][k]{\discfont{x}_{#1}^{\mathrm{apx},\pi}}

\newcommand{\oracpix}[1][\pi]{\orac_{\pi,x}}
\newcommand{\oracpiu}[1][\pi]{\orac_{#1,u}}
\newcommand{\orac}{\discfont{orac}}
\newcommand{\tcont}{t_{\mathrm{ctrl}}}
\newcommand{\nucont}{\nu_{\mathrm{ctrl}}}
\newcommand{\kcont}{k_{\mathrm{ctrl}}}

\newcommand{\dt}{\rmd t}
\newcommand{\ds}{\rmd s}
\newcommand{\ddt}{\frac{\rmd}{\rmd t}}
\newcommand{\step}{\uptau}

\newcommand{\dds}{\frac{\rmd}{\rmd s}}
\newcommand{\Utraj}{\cU}
\newcommand{\Utrajstep}{\Utraj_{\step}}
\newcommand{\Projstep}{\mathscr{P}_{\step}}
\newcommand{\istep}{\mathsf{ct}}
\newcommand{\nabstep}{\bar{\nabla}_{\step}}
\newcommand{\nabdisc}{\bm{\nabla}_{\mathrm{disc}}}
\newcommand{\Kdiscgrad}{K_{\mathrm{disc},\nabla}}

\newcommand{\rmD}{\mathrm{D}}

\newcommand{\KF}{\kname_f}

\newcommand{\Rdx}{\R^{\dimx}}
\newcommand{\Rdu}{\R^{\dimu}}
\newcommand{\Radx}{R_x}
\newcommand{\Radu}{R_u}
\newcommand{\fclpi}{f_{\mathrm{cl}}^\pi}
\newcommand{\cctwo}{\cC^2}
\newcommand{\epsfos}[1][\epsilon]{#1\text{-}\texttt{FOS}}

\newcommand{\vpi}{\bv^{\pi}}
\newcommand{\delv}{\updelta \bv}

\newcommand{\fcl}{f_{\mathrm{cl}}}
\newcommand{\bPhicl}{\bPhi_{\mathrm{cl}}}
\newcommand{\bPsicl}{\bPsi_{\mathrm{cl}}}

\newcommand{\nabu}{\nabla_{u}}
\newcommand{\nabx}{\nabla_{x}}

\newcommand{\Qxpi}{Q_{x}^\pi}
\newcommand{\Qupi}{Q_{u}^\pi}

\newcommand{\nabtil}{\tilde{\nabla}}
\newcommand{\tkt}{t_{k(t)}}
\newcommand{\tktpl}{t_{k(t)+1}}
\newcommand{\tkT}{t_{k(T)}}
\newcommand{\tks}{t_{k(s)}}

\newcommand{\bxkpi}[1][k]{\discfont{x}_{#1}^\pi}
\newcommand{\bvkpi}[1][k]{\discfont{v}_{#1}^\pi}\newcommand{\bhatPsi}[2]{\hat{\discfont{\Psi}}_{#1, #2}}

\newcommand{\Api}{\bA_{\mathrm{ol}}^{\pi}}
\newcommand{\Bpi}{\bB_{\mathrm{ol}}^{\pi}}
\newcommand{\Phiclpi}{\bPhi_{\mathrm{cl}}^{\pi}}
\newcommand{\Phitilclpi}{\tilde{\bPhi}_{\mathrm{cl}}^{\pi}}
\newcommand{\Psiclpi}{\bPsi_{\mathrm{cl}}^\pi}
\newcommand{\Phiolpi}{\bPhi_{\mathrm{ol}}^{\pi}}

\newcommand{\bukpi}[1][k]{\discfont{u}_{\subbrack{#1}}^\pi}

\newcommand{\nabdisctwo}{\discfont{\nabla}_{\mathrm{disc}}^{\,2}}
\newcommand{\nabdisck}[1][k]{\discfont{\nabla}_{\mathrm{disc},#1}}
\newcommand{\nabdisctwok}[1][k]{\discfont{\nabla}_{\mathrm{disc},#1}^{\,2}}
\newcommand{\Proj}{\mathrm{Proj}}
\newcommand{\berr}{\discfont{err}}

\newcommand{\bKkpi}[1][k]{\discfont{K}_{\subbrack{#1}}^\pi}
\newcommand{\bBkpi}[1][k]{\discfont{B}^\pi_{\mathrm{ol},\subbrack{#1}}}
\newcommand{\bAkpi}[1][k]{\discfont{A}^\pi_{\mathrm{ol},\subbrack{#1}}}
\newcommand{\bAclkpi}[1][k]{\discfont{A}^\pi_{\mathrm{cl},\subbrack{#1}}}

\newcommand{\btilu}{\tilde{\bu}}

\newcommand{\xvar}{\bx}
\newcommand{\uvar}{\bu}
\newcommand{\utraj}[1][\bu]{\bv^{#1}}
\newcommand{\xtraj}[1][\bu]{\bx^{#1}}
\newcommand{\xtrajol}[1][\bv]{\bx_{\mathrm{ol}}^{#1}}

\newcommand{\partx}{\partial_x}
\newcommand{\partu}{\partial_u}
\newcommand{\partxx}{\partial_{xx}}
\newcommand{\partxu}{\partial_{xu}}
\newcommand{\partuu}{\partial_{uu}}
\newcommand{\parttu}{\partial_{tu}}
\newcommand{\parttx}{\partial_{tx}}

\newcommand{\upi}{\bu^{\pi}}
\newcommand{\xpi}{\bx^{\pi}}

\newcommand{\Phicldisctil}[1]{\discfont{\Phi}^{\tilde{\pi}}_{\mathrm{cl},#1}}

\newcommand{\Phicldisc}[1]{\discfont{\Phi}^\pi_{\mathrm{cl},#1}}
\newcommand{\Psicldisc}[1]{\discfont{\Psi}^\pi_{\mathrm{cl},#1}}
\newcommand{\xipi}{\xi^{\pi}}
\newcommand{\fdyn}{f_{\mathrm{dyn}}}

\newcommand{\delu}{\updelta \bu}

\newcommand{\bnabtil}{\tilde{\discfont{\nabla}}_{\step}}
\newcommand{\Xspace}{\bbX}
\newcommand{\Xtraj}{\cX}
\newcommand{\LF}{L_{f}}
\newcommand{\MF}{M_{f}}
\newcommand{\barLF}{\bar{L}_{F}}
\newcommand{\barMF}{\bar{M}_{F}}
\newcommand{\barKF}{\bar{K}_{F}}
\newcommand{\barLFX}{\bar{L}_{x}}
\newcommand{\barLFU}{\bar{L}_{u}}
\newcommand{\xol}{\bx}
\newcommand{\xoltuxi}{\bx(t\mid \bu,\xi)}

\newcommand{\delx}{\updelta x}
\newcommand{\deldotx}{\updelta \dot{x}}
\newcommand{\bvk}[1][k]{\discfont{v}_{#1}}

\newcommand{\byk}[1][k]{\discfont{y}_{#1}}
\newcommand{\bbaryk}[1][k]{\bar{\discfont{y}}_{#1}}
\newcommand{\btilyk}[1][k]{\tilde{\discfont{y}}_{#1}}
\newcommand{\btilyki}[1][k]{\tilde{\discfont{y}}_{#1}^{(i)}}
\newcommand{\byki}[1][k]{\discfont{y}_{#1}^{(i)}}

\newcommand{\bwk}[1][k]{\discfont{w}_{#1}}
\newcommand{\bwki}[1][k]{\discfont{w}_{#1}^{(i)}}
\newcommand{\bxki}[1][k]{\discfont{x}_{#1}^{(i)}}
\newcommand{\btilxki}[1][k]{\tilde{\discfont{x}}_{#1}^{(i)}}
\newcommand{\bbarxk}[1][k]{\bar{\discfont{x}}_{#1}}
\newcommand{\bhatxk}[1][k]{\hat{\discfont{x}}_{#1}}

\newcommand{\FOS}{\mathtt{FOS}}

\newcommand{\btilxpik}[1][k]{\tilde{\discfont{x}}^\pi_{#1}}
\newcommand{\btilupik}[1][k]{\tilde{\discfont{u}}^\pi_{#1}}

\newcommand{\btilxk}[1][k]{\tilde{\discfont{x}}_{#1}}
\newcommand{\btiluk}[1][k]{\tilde{\discfont{u}}_{#1}}
\newcommand{\sigu}{\sigma_{u}}
\newcommand{\bbaru}{\bar{\bu}}
\newcommand{\trajoracle}{\bmsf{TrajOrac}}
\newcommand{\LFX}{L_{x}}
\newcommand{\LFU}{L_{u}}
\newcommand{\MFX}{M_{x}}
\newcommand{\MFU}{M_{u}}
\newcommand{\cgram}{\cC}
\newcommand{\cgramhat}{\hat{\cgram}}
\newcommand{\bhatPhi}{\hat{\bPhi}}
\newcommand{\opnorm}[1]{\|#1\|_{\op}}
\newcommand{\Lipone}{L_{[1]}}
\newcommand{\Liptwox}{L_{[2]}^x}
\newcommand{\Liptwou}{L_{[2]}^u}
\newcommand{\Liptwo}{L_{[2]}}
\newcommand{\Cld}{\mathsf{C}_{\textsc{ld}}}
\newcommand{\algfont}[1]{\textsc{#1}}
\newcommand{\estpsi}{\algfont{EstMarkov}}

\newcommand{\I}{\mathbb{I}}
\newcommand{\bPsi}{\boldsymbol{\Psi}}
\newcommand{\bnabla}{\discfont{\nabla}}
\newcommand{\bhatnabla}{\hat{\discfont{\nabla}}}
\newcommand{\bnabhatk}[1][k]{\hat{\discfont{\nabla}}_{#1}}
\newcommand{\discbrak}[1]{(#1)}

\newcommand{\btilxpi}{\tilde{\bx}^\pi}
\newcommand{\btilupi}{\tilde{\bu}^\pi}

\newcommand{\bnabhatkn}[1][k]{\hat{\discfont{\nabla}}_{#1}^{\,(n)}}
\newcommand{\bnabhatknout}[1][k]{\hat{\discfont{\nabla}}_{#1}^{\,(n_{\mathrm{out}})}}
\newcommand{\bukn}[1][k]{{\discfont{u}}_{#1}^{(n)}}
\newcommand{\bunpl}{{\mathbf{u}}^{(n+1)}}
\newcommand{\butilknpl}[1][k]{\tilde{\discfont{u}}_{#1}^{(n+1)}}
\newcommand{\butilnpl}{\tilde{\mathbf{u}}^{(n+1)}}
\newcommand{\bKkn}[1][k]{\discfont{K}_{\subbrack{#1}}^{(n)}}
\newcommand{\bKknpl}[1][n]{\discfont{K}_{\subbrack{#1}}^{(n+1)}}

\newcommand{\Jjac}{\cJ^{\mathrm{jac}}}
\newcommand{\xjac}{\bx^{\mathrm{jac}}}

\newcommand{\policy}{\mathsf{policy}}
\newcommand{\Uspace}{\mathscr{U}}
\newcommand{\delw}{\updelta w}
\newcommand{\bPhi}{\boldsymbol{\Phi}}
\newcommand{\bDel}{\boldsymbol{\Delta}}
\newcommand{\hvsd}{\mathscr{H}}
\newcommand{\estgrad}{\algfont{EstimateGrad}}
\newcommand{\estgains}{\algfont{EstGains}}

\newcommand{\estgradgains}{\algfont{EstimateGradGains}}

\newcommand{\buknpl}[1][k]{\discfont{u}^{(n+1)}_{#1}}

\newcommand{\cgramin}[1][k]{\cgram_{#1, \mathrm{in}}}
\newcommand{\cgramout}[1][k]{\cgram_{#1, \mathrm{out}}}
\newcommand{\cgramhatin}[1][k]{\cgramhat_{#1, \mathrm{in}}}
\newcommand{\cgramhatout}[1][k]{\cgramhat_{#1, \mathrm{out}}}

\newcommand{\nfin}{n_{\mathrm{iter}}}
\newcommand{\Pistep}{\Pi_{\step}}
\newcommand{\xiinit}{\xi_{\mathrm{init}}}
\newcommand{\JL}{$\mathtt{JL}${}}

\section{Setting}\label{sec:setting}
We consider a continuous-time nonlinear control system with state $\bx(t) \in \Rdx$, input $\bu(t) \in \Rdu$ with finite horizon $T > 0$, and fixed initial condition $\xiinit \in \R^{\dimx}$. We denote the space of bounded input signals $\cU := \{\bu(\cdot):[0,T] \to \Rdu: \sup_{t \in [0,T]}\|\bu(t)\| < \infty\}$. We endow $\cU$ with an inner product $\langle \bu(\cdot),\bu'(\cdot)\rangle_{\ltwou} := \int_{0}^T \langle \bu(s),\bu'(s)\rangle\rmd s$, where $\langle\cdot,\cdot \rangle$ is the standard Euclidean inner product, which induces a norm $\|\bu(\cdot)\|_{\ltwou}^2 := \langle \bu(\cdot),\bu(\cdot)\rangle_{\ltwou}$. For $\bu \in \cU$, the  open-loop dynamics are governed by the ordinary differential equation (ODE)
\begin{align*}
\tfrac{\rmd}{\rmd t} \bx(t \mid \bu) = \fdyn( \bx(t \mid \bu), \bu(t)), ~~ \bx(0 \mid \bu) = \xiinit, 
\end{align*}
where $\fdyn: \Rdx \times \Rdu \to \Rdx$ a $\cctwo$ map. Given a terminal cost $V(\cdot):\Rdx \to \R$ and running $Q(\cdot,\cdot,\cdot):\Rdx \times \Rdu \times [0,T] \to \R$, we optimize the control objective
\begin{align*}
\cJ_T(\bu) := \textstyle V(\xol(T \mid \bu)) + \int_{t = 0}^T Q(\bx(t\mid \bu),\bu(t),t)\rmd t.  
\end{align*}
We make the common assumption that the costs are strongly $\cctwo$, and that $\cQ$ is strongly convex:
\begin{assumption}\label{asm:convexity} For all $t \in [0,T]$, $V(\cdot)$ and $Q(\cdot,\cdot,t)$ are twice-continuously differentiable ($\cctwo$), and $x \mapsto V(x)$ and  $(x,u) \mapsto Q(x,u,t) - \frac{\alpha}{2} (\|x\|^2 + \|u\|^2)$ are convex.
\end{assumption}
Given a continuously differentiable function $\cF:\cU \to \R^n$ and perturbation $\delu \in \cU$, we define its \emph{directional derivative} $\rmD \cF(\bu)[\delu]:= \lim_{\eta \to 0} \eta^{-1}(\cF(\bu +\eta\delu ) - \cF(\bu))$. The \emph{gradient} $\nabla \cF(\bu ) \in \cU$ is the (almost-everywhere) unique element of $\cU$ such that
\iftoggle{arxiv}
{
\begin{align}
\textstyle
\forall \delu \in \cU, \int_{0}^T\nabla \cF(\bu)(t) \delu(t)\rmd t  = \rmD\cF(\bu )[\delu]. 
\end{align}
In particular, we
}
{
   $\forall \delu \in \cU$, $\int_{0}^T\nabla \cF(\bu)(t) \delu(t)\rmd t  = \rmD\cF(\bu )[\delu]$. We
}
denote the gradients of $\bu \mapsto \bx(t\mid \bu)$ as $\nabla_{\bu}\, \bx(t\mid \bu)$, and of $\bu \mapsto \cJ_T(\bu)$ as $\nabla_{\bu} \cJ_T(\bu)$.


\paragraph{Discretization and Feedback Policies. }   
Because digital controllers cannot represent continuous open-loop inputs, we  compute $\epsilon$-$\JSP$s  $\bu \in \cU$ which are the zero-order holds of discrete-time control sequences. We let $\step \in (0,T]$ be a discretization size, and set $K = \floor{T/\step}$. Going forward, we denote discrete-time quantities in {\color{disccolor}  \texttt{colored, bold-seraf font}}.

For $k \ge 1$, define $t_{k} = (k-1)\step$, and define the intervals $\cI_{k} = [t_k,t_{k+1})$. For $t \in [0,T]$, let $k(t) := \sup\{k: t_k \le t\}$. We let $\bsfU := (\R^{\dimu})^{K}$, whose elements are denoted $\buvec = \buk[1:K]$, and let $\istep:\bsfU \to \cU$ denote the natural inclusion $\istep(\buvec)(t) := \buk[k(t)]$.


\newcommand{\icmlinlinest}[1]{\iftoggle{icml}{$#1$}{\begin{align*}#1\end{align*}}}
\newcommand{\icmlinline}[1]{\iftoggle{icml}{$#1$}{\begin{align*}#1\end{align*}}}
Next, to mitigate the curse of horizon, we  study \emph{policies} which (a) have discrete-time open-loop inputs and (b) have discrete-time feedback gains to stabilize around the trajectories induced by the nominal inputs. In this work, $\Pistep$ denotes the set of all policies $\pi = (\bukpi[1:K],\bKkpi[1:K])$  defined by a discrete-time open-loop policy $\bukpi[1:K] \in \bsfU$, and a sequence of feedback gains $(\bKkpi)_{k \in [K]} \in (\R^{\dimx\dimu})^K$. A policy $\pi$ induces nomimal dynamics $\upi(\cdot) = \istep(\bukpi[1:K])$ and  $\xpi(t ) = \xol(t\mid \upi)$; we set $\bxkpi = \xpi(t_k)$. It also induces the following dynamics by stabilizing around the policy.
\begin{restatable}{definition}{defnstabdyn} \label{defn:stab_dyn}
Given  a continuous-time input $\bbaru \in \cU$, we define the \emph{stabilized trajectory} $\btilxpict(t \mid \bbaru) := \xol(t \mid \btilupict)$, where 
\icmlinlinest{\btilupict\,(t \mid \bbaru) := \bbaru(t) + \bKkpi[k(t)] (\btilxpict(\tkt \mid \bbaru)) - \bxkpi[k(t)]).}
This induces a stabilized objective: 
\icmlinlinest{\Jpi(\bbaru)  := V(\btilxpict(t \mid \bbaru)) + \int_{0}^T Q(\btilxpict(t \mid \bbaru),\btilupict\,(t \mid \bu),t)\rmd t.} We define the shorthand $\nabla \cJ_T(\pi) := \nabla_{\bbaru} \Jpi(\bbaru)\big{|}_{\bbaru = \upi}$ 
\end{restatable}
Notice that, while $\pi$ is specified by \emph{discrete-time} inputs, $\btilxpict(\cdot),\btilupict(\cdot)$ are \emph{continuous-time} inputs and trajectories stabilized by $\pi$ and the gradient $\nabla \Jpi(\cdot)$ is defined over \emph{continuous-time perturbations}.

\paragraph{Optimization Criteria.} 
Due to nonlinear dynamics, the objectives $\cJ_T, \cJ_T^{\pi}$ are nonconvex, so we can only aim for local optimality. Approximate first-order stationary points ($\FOS$)
are a natural candidate  \citep{roulet2019iterative}.
\begin{restatable}{definition}{defnfos}\label{defn:fos} We say $\bu$ is an $\epsilon$-$\FOS$ of a function $\cF:\cU \to \R$ if $\|\nabla_{\bu} \cF(\bu)\|_{\ltwou} \le \epsilon$. We say $\pi$ is \emph{$\epsilon$-stationary} if \icmlinlinest{\|\nabla J_T(\pi)\|_{\ltwou} := \| \nabla_{\bbaru} \Jpi(\bbaru)\big{|}_{\bbaru = \upi}\|_{\ltwou} \le \epsilon.}
\end{restatable}
Our primary criterion is to compute $\epsilon$-stationary policies $\pi$. However, this depends both on the policy inputs $\upi$ (and induced trajectory $\xpi$), \emph{as well as} the gains. We therefore propose a secondary optimization criterion which depends only on the policies inputs/trajectory. It might be tempting to hope that $\upi$ is an $\epsilon$-$\FOS$ of the original objective $\cJ_T(\bu)$. However, when the Jacobian linearized trajectory (\Cref{defn:JL_traj} below) of the dynamics around $(\xpi,\upi)$ are unstable, the open-loop gradient $\|\nabla \cJ_T(\upi)\|_{\ltwou}$ can be a factor of $e^{T}$ larger than the stabilized gradient $\|\nabla \cJ_T(\pi)\|_{\ltwou}$ despite the fact that, definitionally, $\cJ_T(\upi) = \cJ_T^\pi(\upi)$  (see \Cref{sec:JSP_justification_exp_gap}). We therefore propose an alternative definition in terms of Jacobian-linearized trajectory.
\begin{definition}\label{defn:JL_traj} Given $\bu, \bbaru \in \cU$, define the Jacobian-linearized (\JL) \emph{trajectory}   $\xjac(t \mid \bbaru; \bu) = \bx(t\mid \bu) +  \langle \nabla_{\bu} \,\bx(t \mid \bu), \bbaru - \bu \rangle$ , and \emph{cost} 
\icmlinlinest{\Jjac_T(\bbaru;\bu) := V(\xjac(T \mid \bbaru;\bu)) + \int_{t = 0}^T Q(\xjac(t\mid \bbaru;\bu),\bbaru(t),t)\rmd t.} 
\end{definition}
In words, the \JL{} trajectory is just the first-order Taylor expansion of the dynamics around an input $\bu \in \cU$, and the cost is the cost functional applied to those \JL{} dynamics. We propose an optimization criterion which requires that $\bu$ is near-\emph{globally} optimal for the \JL{} dynamics around $\bu$:
\begin{restatable}{definition}{defJSP}\label{defn:JSP} We say $\bu \in \cU$ is an $\epsilon$-Jacobian Stationary Point ($\JSP$) if  
\icmlinlinest{\cJ_T(\bu)  \le \inf_{\bbaru \in \cU} \Jjac_T(\bbaru;\bu) + \epsilon.}
\end{restatable}
The consideration of $\JSP$s has three advantages: (1) as noted above, $\JSP$s depend only on a trajectory and not on feedback gains; (2) a $\JSP$ is sufficient to ensure that the exponential-stability guarantees derived in \citet{westenbroek2021stability} (and mentioned in the introduction above) hold for certain systems; this provides a link between the local optimality derived in this work and  \emph{global} trajectory behavior (see \Cref{app:westenbroek_app} for further discussion); (3) despite the potentially exponential-in-horizon gap between gradients of $\cJ_T$ and $\cJ_T^\pi$, the following result enables us to compare stationary points of the two objectives in a manner that is \emph{independent} of the horizon $T$. 
\newtheorem*{thm:jsp_informal}{Propostion \ref*{prop:Jpijac} (informal)}
\begin{thm:jsp_informal}
   Suppose $\pi$ is $\epsilon$-stationary,  and $\step$ is sufficiently small. Then, $\upi$ is an $\epsilon'$-$\JSP$ of $\cJ_T$, where $\epsilon' = \BigOh{\epsilon^2/({2\alpha(1+\max_{k}\|\bKkpi\|^2)}} $. 
\end{thm:jsp_informal}



\paragraph{Oracle Model and Problem Desideratum.} In light of the above discussion, we aim to compute a approximately stationary policy, whose open-loop is therefore an approximate $\JSP$ for the original objective. To do so, we assume access to an oracle which can perform feedback with respect to gains $\bKkpi$.
\begin{definition}[Oracle Dynamics]\label{defn:orac_dyn} Given $\buvec = \buk[1:K] \in \bsfU$, we define the \emph{oracle dynamics} 
\iftoggle{icml}
{$\xoffpi(t \mid \buvec) := \xol(t \mid \istep(\uoffpik[1:K]\discbrak{\buvec}))$, where we define $\uoffpik\,\discbrak{\buvec} := \buk + \bKkpi    \xoffpi(t_k \mid \buvec)$, and define  $\xoffpik\,\discbrak{\buvec} :=   \xoffpi(t_k \mid \buvec)$.}
{
   \begin{align*}
   \xoffpi(t \mid \buvec) := \xol(t \mid \istep(\uoffpik[1:K]\discbrak{\buvec}), \quad \text{ where } \uoffpik\,\discbrak{\buvec} := \buk + \bKkpi    \xoffpi(t_k \mid \buvec),
   \end{align*}
   and further define  $\xoffpik\,\discbrak{\buvec} :=   \xoffpi(t_k \mid \buvec)$.
}
\end{definition}

\begin{oracle}\label{orac:our_orac} We assume access to an oracle $\orac$ with variance $\sigorac^2 > 0$, which given any $\pi\in\Pistep$ and $\buvec = \buk[1:K]$, returns, 
\iftoggle{icml}{$\oracpix(\buvec) \sim \cN(\xoffpik[1:K+1]\,\discbrak{\buvec}, \,\eye_{(K+1)\dimx}\sigorac^2)$ and $\oracpiu(\buvec)  =  \uoffpik[1:K]\,\discbrak{\buvec}$}
{
   \begin{align*}
   \oracpix(\buvec) \sim \cN(\xoffpik[1:K+1]\,\discbrak{\buvec}, \,\eye_{(K+1)\dimx}\sigorac^2), \quad \text{ and } \oracpiu(\buvec)  =  \uoffpik[1:K]\,\discbrak{\buvec}.
   \end{align*}
}
\end{oracle}
In words, \Cref{orac:our_orac} returns entire trajectories by applying feedback along the gains $\bKkpi$. The addition of measurement noise is to introduce statistical tradeoffs that prevent near-exact zero-order differentiation; we discuss extensions to process noise in \Cref{sec:process_noise}. Because of this, the oracle trajectory in \Cref{defn:orac_dyn} differs from the trajectory dynamics in \Cref{defn:stab_dyn} in that the feedback does not subtract off the normal $\bxkpi$; thus, the oracle can be implemented without noiseless access to the nominal trajectory. Still, we assume that the feedback applied by the oracle is exact. Having defined our oracle, we specify the following problem desideratum (note below that $M$ is scaled by $1/\step$ to capture the computational burden of finer discretization).

\begin{desideratum}\label{desider} Given $\epsilon,\epsilon'$ and unknown dynamical system $\fdyn(\cdot,\cdot)$, compute a policy $\pi$ for which (a) $\pi$ is $\epsilon$-stationary, and (b) $\upi$ is an $\epsilon'$-$\JSP$ of $\cJ_T$, using $M$ calls to \Cref{orac:our_orac}, where $M/\step$ is polynomial in $1/\epsilon$, $1/\epsilon'$, and relevant problem parameters.
\end{desideratum}

\paragraph{Notation.} We let $[j:k] := \{j,j+1,\dots,k\}$, and $[k] = [1:k]$.
We use standard-bold for continuous-time quantities ($\bx,\bu$), and bold-serif for discrete (e.g. $\bukpi$). We let $\bukpi[j:k] = (\buk[j],\buk[j+1],\dots,\buk)$. Given vector $\discfont{v}$ and matrices $\discfont{X}$, let $\|\discfont{v}\|$ and $\|\discfont{X}\|$ Euclidean and operator norm, respectively; for clarity, we write $\|\bukpi[j:k]\|_{\ell_2}^2 = \sum_{i=j}^k\|\bukpi[i]\|^2$. As denoted above, $\langle \cdot,\cdot\rangle_{\ltwou}$ and $\|\cdot\|_{\ltwou}$ denote inner products and norms in $\cL_2(\cU)$. We let $x \vee y := \max\{x,y\}$, and $x \wedge y := \min\{x,y\}$.

\newcommand{\nout}{n_{\mathrm{out}}}
\newcommand{\clip}{\mathsf{clip}}
\section{Algorithm}\label{sec:alg}
Our iterative approach is summarized in \Cref{alg:learn_mpc_feedback} and takes in a time step $\step > 0$, horizon $T$, a per iteration sample size $N$,  iteration number $\nfin$, a noise variance $\sigw$, a gradient step size $\eta>0$ and a controllability parameter $k_0$. The algorithm produces  a sequence of polices $\pin[n] = (\bukn[1:K],\bKkn[1:K])$, where $K = \floor{T/\step}$ is the number of time steps per roll-out.  Our algorithm uses the primitive $\estpsi(\pi;N,\sigw)$ (\Cref{alg:est_markov}), which makes $N$ calls to the oracle to produce estimates $\bhatxk[1:K+1]$ of the nominal state trajectory, and another $N$ calls with randomly-perturbed inputs of perturbation-variance $\sigw$ to produce estimates $(\bhatPsi{j}{k})_{k<j}$ of  the closed-loop Markov parameters  associated to the current policy $\Psicldisc{j,k}$, defined in \Cref{defn:cl_linearizations}. We use a method-of-moments estimator for simplicity. At each iteration $n$, \Cref{alg:learn_mpc_feedback} calls calls $\estpsi(\pi;N,\sigw)$ first to produce an estimate of the gradient of the closed-loop objective with respect to the current discrete-time nominal inputs. The gradient with respect to the $k$-th input $\bukn[k]$ is given by:
\begin{align}\label{eq:bnabhatkn}
&\textstyle\bnabhatk^{\,(n)} =   \bhatPsi{K+1}{k}^\top V_x(\bhatxk[K+1]) + Q_u(\bhatxk[k], \buk[k]^\pi, t_k)\\
    &~\textstyle+ \step  \sum_{j=k+1}^K \bhatPsi{j}{k}^\top \left( Q_x(\bhatxk[j],\buk[j]^\pi,t_j) +  (\bKkpi[j])^\top   Q_u(\bhatxk[j],\buk[j]^\pi,t_j)\right) \nonumber
\end{align}
The form of this estimate corresponds to a natural plug-in estimate of the gradient of the discrete-time objective defined in \Cref{defn:dt_things}. We use this gradient in \Cref{eq:bnabhatkn} to update the current input;  this update is rolled-out in feedback with the current feedback controller to produce the nominal input $\buknpl[1:K]$ for the next iteration (\Cref{alg:learn_mpc_feedback}, \Cref{line:grad_update}). Finally, we call $\estgains$ (\Cref{alg:gain_est}), which synthesizes gains for the new policy using a Ricatti-type recursion along a second estimate of the linearized dynamics, produced by unrolling the system with the new nominal input and old gains described above. The algorithm then terminates at $\nfin$ iterations and chooses the policy with the smallest estimated gradient that was observed. 

\begin{algorithm}[!t]
    \begin{algorithmic}[1]
    \State{}\textbf{Initialize } time step $\step > 0$, horizon $T \ge \step$, $K \gets \floor{T/\step}$, initial policy $\pi^{(1)}$, sample size $N$,  noise variance $\sigw$, gradient step size $\eta$, controllability parameter $k_0$, iteration number $\nfin$.
    \For{iterations $n = 1, 2, \dots, \nfin$}
    \State{}$(\bhatPsi{j}{k})_{k <j},\bhatxk[1:K+1] = \estpsi(\pi;N,\sigw)$.
    \State{}Compute $\bnabhatk^{\,(n)}$ in \Cref{eq:bnabhatkn} \label{line:grad_est}
    \State{}Gradient update $\buknpl[1:K] \gets \oracpiu[\pin](\btiluk[1:K]^{(n)})$, where  $\btiluk^{(n)} := \bukn - \frac{\eta}{\step} \bnabhatkn  - \bKk^{\pin}\bhatxk$.\label{line:grad_update}
    \State{}Estimate $\bKknpl[1:K] = \estgains(  \tilde{\pi}^{(n)} ;\sigw, N,k_0)$, where $\tilde{\pi}^{(n)} = (\bunpl(\cdot), \bKkn[1:K])$
    \State{{Update} policy} $\pin[n+1] = (\buknpl[1:K],\bKknpl[1:K])$
    \EndFor 
    \Return $\pi^{(\nout)}$, $\nout \in \argmin_{n \in[\nfin]}\|\bnabhatk^{\,(n)}\|$. \label{line:nout}
    \end{algorithmic}
      \caption{Trajectory Optimization}
      \label{alg:learn_mpc_feedback}
\end{algorithm}

\begin{algorithm}[!t]
    \begin{algorithmic}[1]
    \Statex{}\algcomment{estimate nominal trajectory}
    \For{samples $i = 1,2,\dots,N$} 
    \State{}\textbf{Collect}  trajectory $\bxk[1:K+1]^{(i)} \sim \trajoracle_\pi(\bukpi[1:K])$.
    \EndFor
    \State{Average} $\bhatxk[1:K+1] = \frac{1}{N}\sum_{i=1}^N \bxk[1:K+1]^{(i)}$
     \Statex{}\algcomment{estimate perturbed trajectory}
   	\For{samples $i = 1,2,\dots,N$} 
    \State\text{Draw} $\bwki[1:K]$ uniformly from $\sigw \cdot (\{-1,1\}^{\dimu})^K$.
    \State\textbf{Let} $\buk^{(i)} = \bukpi + \bwki - \bKkpi\bhatxk$, for $k \in [K]$ 
    \State{}\textbf{Collect}  trajectory $\byk[1:K+1]^{(i)} \sim \oracpix(\buk[1:K]^{(i)})$.\label{line:byk}
    \EndFor
    \State \textbf{Estimate} transition operators $\bhatPsi{j}{k} := \frac{1}{N \sigw^2}\sum_{i=1}^N(\byk[j]^{(i)} - \bhatxk[j])(\bwki)^\top$, $k < j$\\
    \Return{} $(\bhatPsi{j}{k})_{k<j}, \bhatxk[1:K+1]$
        \end{algorithmic}
      \caption{$\estpsi(\pi;N,\sigw)$}
      \label{alg:est_markov}
\end{algorithm}
\newcommand{\knot}{k_0}
\begin{algorithm}[!t]
    \begin{algorithmic}[1]
    \State{}\textbf{Initialize }  number of samples $N$, noise variance $\sigw$, (discrete) controllability window $k_0 \in \N$
    \State{}\textbf{Estimate} Markov Parameters $(\bhatPsi{j}{k})_{k<j} = \estpsi(\pi;N,\sigw)$
    \Statex{}\algcomment{Define $\cgramhat_{k \mid j_2,j_1} := [\bhatPsi{k+1}{j_2} \mid \bhatPsi{k+1}{j_2 - 1} \mid \dots \bhatPsi{k+1}{j_1}]$}
    \For{$k=\knot,\knot+1,\dots,K$}
    \State{}Define $\bhatB_k = \bhatPsi{k+1}{k}$
    \Statex{}\algcomment{Define $ \cgramhatin := \cgramhat_{k-1 \mid k-1,k-\knot+1}, \quad \cgramhatout := \cgramhat_{k \mid k-1,k-\knot+1}$}
    \State{}Define $\bhatAk := \cgramhatout\cgramhatin^{\dagger} - \bhatB_k\bKk^\pi$
    \EndFor
    \State{} Set $\bhatPk[K+1] = \eye_{\dimx}$.
    \For{$k = K,K-1,\dots,\knot$}\label{line:control_synth_start}
    \State{} $\bhatKk := ( \eye_{\dimu} + \bhatBk^\top \bhatPk[k+1]\bhatBk)^{-1} \left(\bhatBk^\top\bhatPk[k+1]\bhatAk\right)$.
    \State{} $\bhatPk =  (\bhatAk+\bhatBk\bhatKk)^\top \bhatPk[k+1](\bhatAk+\bhatBk\bhatKk) + \step (\eye_{\dimx}+\bhatKk^\top\bhatKk)$.
    \EndFor
    \State{} Set $\bhatKk = 0$ for $k \le \knot$.\label{line:control_synth_end}
    \State{}\textbf{Return} $\bhatKk[1:K]$.
        \end{algorithmic}
      \caption{$\estgains(\pi;N,\sigw,k_0)$}
      \label{alg:gain_est}
\end{algorithm}

\newcommand{\Rfeas}{R_{\mathrm{feas}}}
\newcommand{\Jpidisc}[1][\pi]{\cJ_{T}^{#1,\mathrm{disc}}}
\newcommand{\tnot}{t_0}
\newcommand{\Kpi}{K_{\mathrm{\pi}}}
\newcommand{\gampi}{\gamma_{\pi}}
\newcommand{\Bigohpi}[1][\cdot]{\cO_{\pi}(#1)}
\newcommand{\Omegapi}[1][\cdot]{\Omega_{\pi}(#1)}
\newcommand{\Bigohst}[1][\cdot]{\cO_{\star}(#1)}
\newcommand{\Lampik}[1][k]{\discfont{\Lambda}^{\pi}_{#1}}
\newcommand{\Kpist}{\mu_{\pi,\star}}
\newcommand{\ltwo}{\ell_2}

\newcommand{\stepgrad}{\step_{\mathrm{grad}}}
\newcommand{\stepjsp}{\step_{\mathrm{jsp}}}
\newcommand{\cgrad}{c_{\mathrm{grad}}}
\newcommand{\cjsp}{c_{\mathrm{jsp}}}

\section{Algorithm Analysis}

For simplicity, we assume $K = \floor{T/\step} \in \N$ is integral. In order to state uniform regularity conditions on the dynamics and costs, we fix an \emph{feasible radius} $\Rfeas > 0$ and restrict to states and inputs bounded thereby.
\begin{restatable}{definition}{defnfeas}\label{defn:feas} We say $(x,u) \in \R^{\dimx \times \dimu}$ are \emph{feasible} if $\|x\| \vee \|u\| \le \Rfeas$. We say a policy $\pi$ is feasible if $(2\xpi(t),2\upi(t))$ are feasible for all $t \in [0,T]$. 
\end{restatable}
We adopt the following boundedness condition.
\begin{condition}\label{cond:feasbility}For all $n$, the policies $\pin$ and $\tilde{\pi}^{(n)}$ produced by \Cref{alg:learn_mpc_feedback} are feasible. 
\end{condition}
If $\pi$ and $\tilde{\pi}^{(n)}$ produce bounded inputs, and the resulting state trajectories also remain bounded, then \Cref{cond:feasbility} will hold for $\Rfeas>0$ sufficiently large. This is a common assumption in the control literature (see e.g.~\citet{jadbabaie2001stability}), as physical systems, such as those with Lagrangian dynamics, will remain bounded under bounded inputs (see \Cref{sec:feasibility} for discussion). 

\begin{assumption}[Dynamics regularity]\label{asm:max_dyn_final} $\fdyn$ is $\cctwo$, and  for all feasible $(x,u)$, the following hold $\|\fdyn(x,u)\| \le \KF$, $\|\partx \fdyn(x,u)\|\vee\|\partu \fdyn(x,u)\| \le \LF$, $\|\nabla^{\,2} \fdyn(x,u)\| \le \MF$.
\end{assumption}

 \newcommand{\bPpi}{\bP^{\pi}}
\newcommand{\KQ}{\kname_{\mathrm{cost}}}

\newcommand{\Kcost}{\kname_{\mathrm{cost}}}
\renewcommand{\MQ}{M_{\mathrm{cost}}}
\begin{assumption}[Cost regularity]\label{asm:cost_asm} For all feasible $(x,u)$, the following hold $0 \le V(x) \vee Q(x,u,t) \le \Kcost$, $\|\partx V(x)\| \vee \|\partx Q(x,u,t)\| \vee \|\partu Q(x,u,t)\|  \le \LQ$, $\|\nablatwo V(x)\| \vee \|\nablatwo Q(x,u,t)\| \le \MQ$. 
\end{assumption} 

To take advantage of stabilizing gains, we require two additional assumptions, which are defined in terms of the \JL{} dynamics. 
\begin{restatable}[Open-Loop Linearized Dynamics]{definition}{defoldynamics}\label{defn:ol_lin_dyn} We define the (open-loop) \JL{} dynamic matrices about $\pi$ as 
\iftoggle{icml}
{$\Api(t) = \partx \fdyn(\xpi(t),\upi(t))$ and $\Bpi(t) = \partu \fdyn(\xpi(t),\upi(t))$.}
{\begin{align*}
\Api(t) = \partx \fdyn(\xpi(t),\upi(t)), \quad \text{and} \quad \Bpi(t) = \partu \fdyn(\xpi(t),\upi(t))
\end{align*}} We define the \emph{open-loop} \JL{} transition function $\Phiolpi(s,t)$, defined for $t \ge s$ as the solution to $\dds \Phiolpi(s,t) = \Api(s)\Phiolpi(s,t)$, with initial condition $ \Phiolpi(t,t) = \eye$.
\end{restatable}
We first require that stabilizing gains can be synthesized; this is formulated in terms of an upper bound on the cost-to-go for the LQR control problem (\citet[Section 2]{anderson2007optimal}) induced by the \JL{} dynamics. 
\begin{assumption}[Stabilizability]\label{asm:LFP} 
Given a \emph{policy} $\pi$, and a sequence of controls $\btilu(\cdot) \in \cU$,  let \icmlinlinest{V^\pi(t \mid \btilu, \xi) = \int_{s=t}^T (\|\btilx(s)\|^2 +\|\btilu(s)\|^2)\rmd s + \|\btilx(T)\|^2,} under the linearized dynamics  $\dds \btilx(s) = \Api(s)\btilx(s) + \Bpi(s)\btilu(s), \quad \btilx(t) = \xi$. We assume that, for all feasible policies,  $\sup_{t \in [0,T]}V^\pi(t \mid \btilu, \xi)  \le \LFP\|\xi\|^2$. Moreover, we assume (for simplicity) that the initial policy has (a) no gains: $\bKk^{\pi^{(1)}} = 0$ for all $k \in [K]$, and (b) satisfies $V^{\pi^{(1)}}(t \mid 0, \xi) \le \LFP\|\xi\|^2$.
\end{assumption}
The assumption on $\pi^{(1)}$ can easily be generalized to accomodate initial policies with stabilizing gains.  Our final assumption is controllability (see e.g. \citet[Appendix B]{anderson2007optimal}), which is necesssary for identification of system parameters to synthesize stabilizing gains. 

\begin{assumption}[Controllability]\label{asm:ctr} There exists constants $\tcont,\nucont > 0$ such that, for all feasible $\pi$ and $t \in [\tcont,T]$,   \iftoggle{icml}{$\int_{s=t-\tcont}^{t}$ {\small {$\Phiolpi(t,s)\Bpi(s)\Bpi(s)^\top\Phiolpi(t,s)^\top \rmd s$}} $\succeq \nucont \eye_{\dimx}$}{
    \begin{align}
    \int_{s=t-\tcont}^{t} \Phiolpi(t,s)\Bpi(s)\Bpi(s)^\top\Phiolpi(t,s)^\top \rmd s \succeq \nucont \eye_{\dimx}.
    \end{align}
}
\end{assumption}
For simplicity, we assume $\kcont:= \tcont/\step$ is integral. Finally, to state our theorem, we adopt an asymptotic notation which suppresses all parameters except $\{T,\step,\alpha\}$.
\begin{definition}[Asymptotic Notation] We let $\Bigohst{\cdot}$ term a term which hides polynomial dependences on $\dimx,\dimu,\Rfeas,\KF,\MF,\LF,\KQ,\LQ,\MQ,\LFP,\nucont,\tcont$, and on $\exp(t_0 \LF)$, where $t_0 = \step k_0 \ge \tcont$. 
\end{definition}
Notice that we suppress an \emph{exponential} dependence on our proxy $t_0$ for the controllability horizon $\tcont$; this is because the system cannot be stabilized until the dynamics can be accurately estimated, which requires waiting as long as the controllability window \citep{chen2021black,tsiamis2022learning}. We discuss this dependence further in \Cref{sec:discussion_exp_LF}. Finally, we state a logarithmic term which addresses high-probability confidence:
\begin{align}
\iftoggle{icml}{\textstyle}{} \iota(\delta) := \log \frac{24 T^2\nfin\max\{\dimx,\dimu\}}{\step^2\delta}. \label{eq:iota_def}
\end{align}
We can now state our main theorem, which establishes that, with high probability, for a small enough step size $\step$, and large enough sample size $N$ and iteration number $\nfin$, we obtain an $\epsilon$-stationary policy and $\epsilon'$-$\JSP$, where $\epsilon^2, \epsilon'$ scale as $\mathrm{poly}(T)(\step^{2} + \frac{1}{\step^2\sqrt{N}})$:
\begin{theorem}\label{thm:main_body} Fix $\delta \in (0,1)$, and suppose for the sake of simplicity that $\step \le 1 \le T$.  Then, there are constants $c_1,\dots,c_5 = \Bigohst[1]$ such that if we tune $\eta = 1/c_1\sqrt{T}$, $\sigw = (\sigorac^2 \iota(\delta)/N)^{\frac{1}{4}}$ and $k_0 \ge \kcont +2$, then as long as
\begin{align*}
\textstyle
\step \le \frac{1}{c_2}, \quad N \ge   c_3\iota(\delta) \max\left\{\frac{T^2}{\step^2}, \frac{1}{\step^4}, \sigorac^2\frac{T^4}{\step^2}, \frac{\sigorac^2}{\step^8}\right\}.
\end{align*}
Then, with probability $1-\delta$, if \Cref{cond:feasbility} and all aforementioned Assumptions hold, 
\begin{itemize}
    \item[(a)] For all $n \in [\nfin]$, and $\pi' \in \{\pi^{(n)},\tilde{\pi}^{(n)}\}$, $\mu_{\pi',\star} \le 8\LFP$ and $L_{\pi'} \le 6\max\{1,\LF\}\LFP$. 
    \item[(b)] $\pi  = \pi^{(\nout)}$ is $\epsilon$-stationary, where \iftoggle{icml}{$\epsilon^2 = c_4 (T\step^2 + \frac{T^{\frac{3}{2}}}{\nfin}) $ $+c_4(\frac{T^{\frac{7}{2}}}{\step^2}(\frac{\iota(\delta)^2}{N} + \sigorac \sqrt{\frac{\iota(\delta)}{N}}) + \sigorac^2 \frac{T^{\frac{3}{2}}\iota(\delta)^2}{N})$.}{
        \begin{align*}
        \epsilon^2 = c_4 (T\step^2 + \frac{T^{\frac{3}{2}}}{\nfin})  +c_4(\frac{T^{\frac{7}{2}}}{\step^2}(\frac{\iota(\delta)^2}{N} + \sigorac \sqrt{\frac{\iota(\delta)}{N}}) + \sigorac^2 \frac{T^{\frac{3}{2}}\iota(\delta)^2}{N}).
        \end{align*}
    }
    \item[(c)] For $\pi  = \pi^{(\nout)}$, $\upi$ is an $\epsilon'$-$\JSP$, where $\epsilon' = c_5\frac{\epsilon^2}{\alpha}$. 
\end{itemize}
\end{theorem}
As a corollary, we achieve \Cref{desider}.
\begin{corollary} For any $\epsilon,\epsilon' > 0$ and $\delta \in (0,1)$, there exists an appropriate choices of $\{\step,N,\eta,\sigw\}$ such that \Cref{alg:learn_mpc_feedback} finds, with probability $\ge 1-\delta$, an $\epsilon$-stationary policy $\pi$ with $\upi$ being an $\epsilon'$-$\JSP$ using at most $M$ oracle calls, where $M/\step = \Bigohst[\mathrm{poly}(T,1/\epsilon,1/\epsilon',\log(1/\delta))]$. 
\end{corollary}

\subsection{Analysis Overview}
In this section, we provide a high-level sketch of the analysis. \Cref{app:formal_analysis} provides the formal proof, and carefully outlines the organization of the subsequent appendices which establish the subordinate results. 

As our policies consists of zero-order hold discrete-time inputs, our analysis is mostly performed in discrete-time. 
\begin{restatable}[Stabilized trajectories, discrete-time inputs]{definition}{defnstabtrajs}\label{defn:dt_things} Let $\buvec \in \sfU$, and recall the continuous-input trajectories  $\btilxpict,\btilupict$ in \Cref{defn:stab_dyn}. We define $\btilxpi(t \mid \buvec) := \btilxpict(t \mid \istep(\buvec))$ and $\btilupi(t \mid \buvec) := \btilupict(t \mid \istep(\buvec))$, and their discrete samplings  $\btilxpik\discbrak{\buvec} := \btilxpi(t_k \mid \buvec) $ and  $\btilupik\discbrak{\buvec} := \btilupi(t_k \mid \buvec)$. We define a discretized objective \icmlinlinest{\Jpidisc\discbrak{\buvec} :=  V(\btilxpik[K+1]\,\discbrak{\buvec}) + \step \sum_{k=1}^K Q( \btilxpik\,\discbrak{\buvec}, \btilupik\,\discbrak{\buvec}, t_k)}, and the shorthand $ \Jdisc\discbrak{\pi} = \Jpidisc\discbrak{\bukpi[1:K]}$ and $ \nabla \Jpidisc\discbrak{\pi} := \nabla_{\buvec}\Jpidisc\discbrak{\buvec} \big{|}_{\buvec= \bukpi[1:K]}$.
\end{restatable}


What we shall show is that our algorithm \textbf{(a)} finds a policy $\pi$ such that $\|\nabla \Jdisc\discbrak{\pi}\|_{\ell_2} \le \epsilon$ is small, \textbf{(b)} by discretization, $\|\nabla\Jpi\discbrak{\upi}\|_{\ltwou} \le \epsilon + \BigOh{\step}$ is small (i.e. $\pi$ is approximately stationary), and that \textbf{(c)} this implies that $\upi$ is an approximate-$\JSP$ of $\cJ_T(\upi)$. Part \textbf{(a)} requires the most effort, part \textbf{(b)} is a tedious discretization, and part \textbf{(c)} is by \Cref{prop:Jpijac} stated below. Key in these steps are certain regularity conditions on the policy $\pi$. The first is the magnitude of the gains:
\begin{definition}\label{defn:pigains} We define an upper bound on the gains of policy $\pi$ as $\Lpi := \max\{1,\max_{k \in [K]}\|\bKkpi\|\}$. 
\end{definition}
This term suffices to translate stationary policies to $\JSP$s:
\begin{proposition}\label{prop:Jpijac} Suppose \Cref{asm:convexity,asm:max_dyn_final,asm:cost_asm}, $\pi$ is feasible,  $\step \le \frac{1}{16\Lpi\LF}$. Then, if $\|\nabla\cJ_T(\pi)\|_{\ltwou}  \le \epsilon$,  $\upi(t)$ is an  $\epsilon'$-$\JSP$ of $\cJ_T$ for $\epsilon' = 64\epsilon^2\Lpi^2/\alpha$.
\end{proposition}
\begin{proof}[Proof Sketch] We construct a Jacobian linearization $\Jpijac$ of $\Jpi$ by analogy to $\Jjac$, and define $\epsilon$-$\JSP$s of $\Jpi$ analogously. We show by inverting the gains that an $\epsilon$-$\JSP$ of $\Jpi$ is precisely an $\epsilon$-$\JSP$ of $\cJ_T$. We then establish strong convexity of $\Jpijac$ (non-trivial due to the gains), and use the PL inequality for strongly convex functions to conclude. The formal proof is given in \Cref{sec:proof:Jpijac}.
\end{proof}

To establish parts \textbf{(a)} and \textbf{(b)}, we need to measure the stability of the policies. To this end, we first introduce \emph{closed-loop} (discrete-time) linearizations of the dynamics, in terms of which we define a Lyapunov stability modulus.
\begin{restatable}[Closed-Loop Linearizations]{definition}{defncllinearizations}\label{defn:cl_linearizations} We {discretize} the open-loop linearizations in \Cref{defn:ol_lin_dyn}  defining  $\bAkpi = \Phiolpi(t_{k+1},t_k)$ and $\bBkpi := \int_{s=t_k}^{t_{k+1}}\Phiolpi(t_{k+1},s)\Bpi(s)\rmd s$. We define an \emph{discrete-time closed-loop} linearization $\bAclkpi := \bAkpi + \bBkpi\bKkpi$, and a discrete closed-loop \emph{transition operator} is defined, for $1 \le k_1 \le k_2 \le K+1$, $\Phicldisc{k_2,k_1} = \bAclkpi[k_2-1]\cdot \bAclkpi[k_2-2]\dots \cdot \bAclkpi[k_1]$, with the convention $\Phicldisc{k_1,k_1} = \eye$. For $1 \le k_1 < k_2 \le K+1$, we define the closed-loop \emph{Markov operator} $\Psicldisc{k_2,k_1} := \Phicldisc{k_2,k_1+1}\bBkpi[k_1]$. 
\end{restatable}

\begin{restatable}[Lyapunov Stability Modulus]{definition}{picond}\label{cond:pi_cond} Given a policy $\pi$, define $\Lampik[K+1] = \eye$, and $\Lampik = (\bAclkpi)^\top\Lampik[k+1]\bAclkpi + \step \eye$. We define $\Kpist := \max_{k \in \{k_0,\dots,K+1\}}\|\Lampik\|$.
\end{restatable}
\newcommand{\lohpi}[1][\cdot]{o_{\pi}(#1)}
Notice that the stability modulus is taken after step $k_0$, which is where we terminate the Riccati recrusion in \Cref{alg:gain_est}. We shall show that, with high probability, \Cref{alg:learn_mpc_feedback} synthesizes policies $\pi$ which satisfy 
\begin{align}
\Lpi  \le 6\max\{1,\LF\}\LFP, \quad \Kpist \le 8\LFP, \label{eq:pi_cond_inv}
\end{align}
so that $\Lpi,\Kpist = \Bigohst[1]$. Going forward, we let $\Bigohpi{\cdot}$ denote a term suppressing polynomials in $\Lpi$, $\Kpist$ and terms $\Bigohst[1]$; when $\pi$ satisfies \Cref{eq:pi_cond_inv}, then $\Bigohpi[\cdot] = \Bigohst[\cdot]$. We say $x \le 1/\Bigohpi[y]$, if $x \le 1/y'$, where $y' = \Bigohpi[y]$. In \Cref{sec:prop:grad_disc}, we translate discrete-time stationary points to continuous-time ones, establishing part \textbf{(b)} of the argument. 
\begin{proposition}\label{prop:confirm_eps_stat} For $\pi$ feasible, $\|\nabla \cJ_T( \pi)\|_{\ltwou}  \le \frac{1}{\sqrt{\step}}\|\nabla \Jdisc( \pi)\|_{\ltwo} + \Bigohpi[\step\sqrt{T}]$.
\end{proposition}
A more precise statement and explanation of the proof are given in \Cref{sec:grad_disc_main}. The rest of the analysis boils down to \textbf{(a)}: finding an approximate stationary point of the time-discretized objective. 
\subsection{Finding a stationary point of $\Jpidisc$}

\icmlpar{Taylor expansion of the dynamics.} To begin, we derive perturbation bounds for solutions to the stabilized ordinary differential equations. Specifically, we provide bounds for when $\bukpi[1:K]$ is perturbed by a sufficiently small input $\updelta \buk[1:K]$. Our formal proposition, \Cref{app:taylor_exp_summary} states perturbations in both the $\ell_{\infty}$ and normalized $\ell_{2}$-norms; for simplicity, state the special case for $\ell_{\infty}$-perturbation. 
\begin{proposition}\label{prop:tay_ex_body} Let $\buk[1:K] = \bukpi + \deluk[1:K]$, and suppose $\max_{k}\|\updelta \buk\| \le B_{\infty} \le 1/\Bigohpi[1]$. Then,  for all $k \in [K+1]$, \icmlinlinest{\| \btilxpik\,[\buk[1:K]] - \bxkpi - \sum_{j=1}^{k-1}\Psicldisc{k,j}\deluk[j] \| \le \Bigohpi[B_{\infty}^2].}
\end{proposition}
We also show, that if $B_{\infty} = 1/\Bigohpi[T]$, then the policy with $\pi'$ with the same gains $\bKk^{\pi} = \bKk^{\pi}$ as $\pi$, but the perturbed inputs $\buk^{\pi'} = \buk$ at most double its Lyapunov stability modulus $\mu_{\pi',\star} \le 2\mu_{\pi,\star}$. This allows small gradient steps to preserve stability.

\icmlpar{Estimation of linearizations and gradients.} We then argue that by making $\sigw$ small, then to first order, the estimation procedure in \Cref{alg:est_markov} recovers the \emph{linearization} of the dynamics. The proof combines standard method-of-moments analysis based on matrix Chernoff concentration \citep{tropp2012user} and \Cref{prop:tay_ex_body} to argue the dynamics can be approximated by their linearization. Specifically, \Cref{app:est_errs} argues that, for all rounds $n \in [\nfin]$ and $1 \le j < k \le K+1$, it holds that $\textstyle \|\Psicldisc{k,j}  - \bhatPsi{k}{j}\|  \le \errpsi(\delta)$ where
$\errpsi(\delta) =  \Bigohpi[\sqrt{\frac{\iota(\delta)}{N}}(1+\frac{\sigorac}{\sigw} + \sigw]$,
which can be made to scales as $N^{-\frac{1}{4}}$ by tuning $\sigw = (\sigorac^2 \iota(\delta)/N)^{\frac{1}{4}}$. From the Markov-recovery error, as well as a simpler bound for recovering $\bxk[1:K]^{\pi}$ in \Cref{alg:est_markov} (Lines 1-3), we show accurate recovery of the gradients: $\max_k\|\bnabhatk^{\,(n)} - (\nabla \Jdisc(\pin))_k\|  \le T\Bigohpi[\errpsi(\delta)]$. 

The last step here is to argue that we also approximately recover $\bAkpi,\bBkpi$ in \Cref{alg:gain_est} for synthesizing the gains: for all $k \ge k_0$,
\begin{align*}
\textstyle\|\bhatBk - \bBkpi\|  \vee \|\bhatAk - \bAkpi\| \le \Bigohpi[\frac{\errpsipi(\delta)}{\step}].
\end{align*}
This consists of two steps: (1) using controllability to show the matrices $\cgramhatin$ in \Cref{alg:gain_est} are well-conditioned and (2) using closeness of the Markov operators to show that $\cgramhatin$ and $\cgramhatout$ concentrate around their idealized values. Crucially, we only estimate system matrices for $k \ge k_0$ to ensure $\cgramhatin$ is well-defined, and we use window $k_0 \ge \kcont + 2$ to ensure $\cgramhatout$  is sufficiently well-conditioned.

\icmlpar{Concluding the proof.} \Cref{app:desc_stab,app:formal_conclude} conclude the proof with two steps: first, we show that cost-function decreases during the gradient step \Cref{alg:learn_mpc_feedback} (\Cref{line:grad_update}) at round $n \in [\nfin]$ in proportion to $-\|\bnabhatk^{\,(n)}\|^2$ (a consequence ofthe standard smooth descent argument). Here, we also apply the aforementioned result that small gradient steps preserve stability: $\mu_{\tilde{\pi}^{(n)},\star} \le 2\mu_{\pin,\star}$. Second, we argue that the gains synthesized by \Cref{alg:gain_est} ensure that the Lyapunov stability modulus of $\pi^{(n+1)}$ and the magnitude of its gains stay bounded by an algorithm-independent constant: $\mu_{{\pi}^{(n+1)},\star} \le 4\LFP = \Bigohst[1]$ and $L_{\pi^{(n+1)}} \le \Bigohst[1]$; we use a novel certainty-equivalence analysis for discretized, time-varying linear systems which may be of independent interest (\Cref{append:certainty_equiv}). By combining these two results, we inductively show that all policies constructed satisfy \eqref{eq:pi_cond_inv}, namely they have $\Kpist$ and $L_{\pi}$ at most $\Bigohst[1]$. We then combine this with the typical analysis of nonconvex smooth gradient descent to argue that the policy $\pi^{(\nout)}$ has small discretized gradient, as needed.

\newcommand{\icmlspace}[1]{\iftoggle{icml}{\vspace{#1}}{}}
\section{Experiments}\label{sec:experiments}


Our experiments evaluate the performance of our proposed trajectory optimization
algorithm (\Cref{alg:learn_mpc_feedback}) and compare it with the well-established model-based
baseline of trajectory optimization ($\iLQR$) on top of learned dynamics (e.g.~ \citet{levine2013guided}). Though our analysis considers a fixed horizon, we perform experiments in a receeding horizon control (RHC) fashion.
We consider two control tasks: \textbf{(a)} a pendulum swing up task, and \textbf{(b)}
a 2D quadrotor stabilization task.
We implement our experiments using the \texttt{jax}~\cite{jax2018github} ecosystem.
%
More details regarding the environments, tasks, and experimental setup details
are found in \Cref{append:experimentdetails}. Though our analysis considers the noisy oracle model, all experiments assume \emph{noiseless} observations.

\icmlpar{Least-squares vs.~Method-of-Moments.} \Cref{alg:learn_mpc_feedback} prescribes the method-of-moments estimator to simplify the analysis; in our implementation, we find that estimating
the transition operators using regularized least-squares instead yields to more sample efficient
gradient estimation. This choice can also be analyzed with minor modifications (see e.g. \citet{oymak2019non,simchowitz2019learning}). 

\icmlpar{$\iLQR$ baseline.} We first collect a training dataset according to a prescribed exploration strategy, then train a neural network dynamics model on these dynamics, and finally optimize our policy by applying the $\iLQR$ algorithm directly on the learned model. We consider several variants of our $\iLQR$ baseline which use different exploration strategies
and different supervision signals for model learning.
\iftoggle{icml}{}{\begin{itemize}}

\iftoggle{icml}{\textbf{(1) Sampling strategies:} }{\item[\textbf{(1)}]\textbf{ Sampling strategies:}}
We consider two sampling strategies;
    (a) \texttt{Opt} runs $\iLQR$ with
    the ground truth cost and dynamics in a receeding horizon fashion, performing noiseless rollouts and perturbing the resulting trajectories with noise to encourage exploration,
    and (b) \texttt{Rand} executes rollouts with random inputs starting from random initial conditions. The rationale is that the \texttt{Opt} strategy provides better data coverage for the desired task than \texttt{Rand}.

\iftoggle{icml}{\textbf{(2) Loss supervision:} }{\item[\textbf{(2)}] \textbf{Loss supervision:}}
    The standard loss supervision for learning dynamics is to regress against the next state transition.
    Inspired by our analysis, we also consider an idealized oracle that augments the supervision to also include {noiseless}
    the Jacobians of the ground truth model with respect to both the state and control input;
    we refer to this augmentation as \texttt{JacReg}.

\iftoggle{icml}{\textbf{(2) Model architecture:} }{\item[\textbf{(3)}] \textbf{Model architecture:}}
We use a fully connected three layer MLP network to 
    for fitting the dynamics of the environment. Specifically, our model
    takes in input $(\bxk, \buk)$ and predicts
    the state difference $\bxk[k+1] - \bxk$.

\iftoggle{icml}{}{\end{itemize}}

\Cref{fig:sweeps} shows the results of \Cref{alg:learn_mpc_feedback}
compared with several $\iLQR$ baselines on the pendulum and quadrotor tasks, respectively.
In these figures, the x-axis plots the number of trajectories available to each algorithm,
and the y-axis plots the cost suboptimality $(\cJ_T^{\mathrm{alg}} - \cJ_T^\star)/\cJ_T^\star$ incurred
by each algorithm; where $\cJ_T^{\mathrm{alg}}$ is algorithmic cost and $\cJ_T^\star$ is the cost obtained via $\iLQR$ with the ground truth dynamics. The error bars in the plot are 95\% confidence intervals computed over 10 different evaluation seeds.

\icmlpar{Discussion.} We observe that \Cref{alg:learn_mpc_feedback} with feedback-gains consistently outperforms
\Cref{alg:learn_mpc_feedback} without gains, validating the important of locally-stabilized dynamics. Second, we see that the performance of the $\iLQR$ baselines does not significantly improve as more
trajectory data is collected. We find that our learned models achieve very low train and test
error, over the sampling distribution (i.e., \texttt{Opt} or \texttt{Rand}) used for learning. For $\texttt{Rand}$, we postulate that the distribution shift incurred by performing RHC
via trajectory optimization on the learned model limits the closed-loop performance of our baseline. However, we note that \texttt{Opt}+\texttt{JacReg} achieves stellar performance early on, suggesting that (a) the \texttt{Opt} data collection method suffices for strong closed-loop performance (notice that \texttt{Rand}+\texttt{JacReg} fares far worse), and (b) that a second limiting factor is that estimating \emph{dynamics} and performing automated differentiation is less favorable than directly estimating \emph{Jacobians}, which are the fundamental quantities used by the $\iLQR$ algorithm. This gap between estimation of dynamics and derivatives has been observed in prior work \cite{pfrommer2022tasil}.
\icmlspace{-.5em}
\begin{figure}[]
    \centering
    \begin{minipage}[c]{\columnwidth}
    \centering
    \includegraphics[width=0.78\columnwidth]{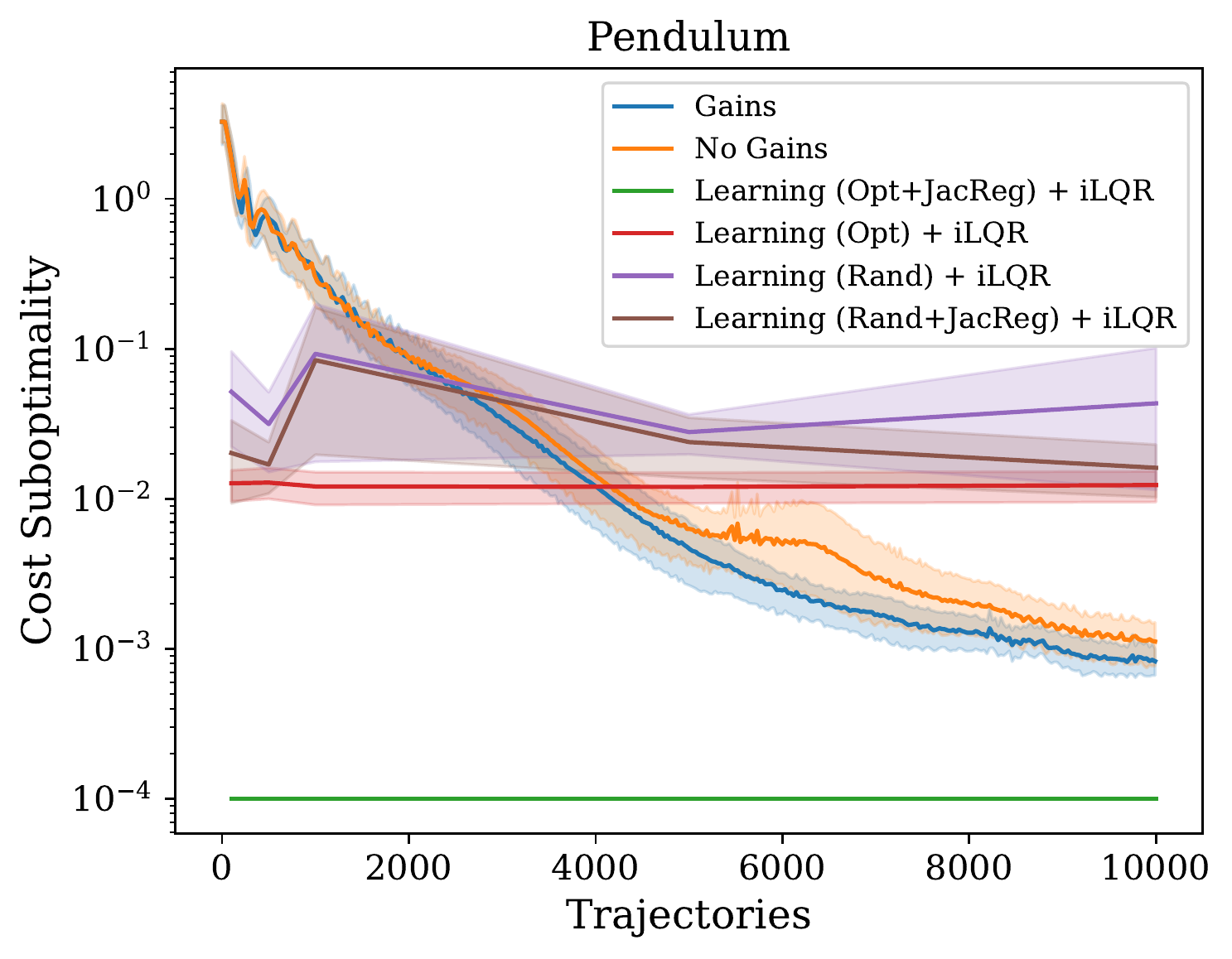}
    \end{minipage}
    \begin{minipage}[c]{\columnwidth}
    \centering
    \includegraphics[width=0.78\columnwidth]{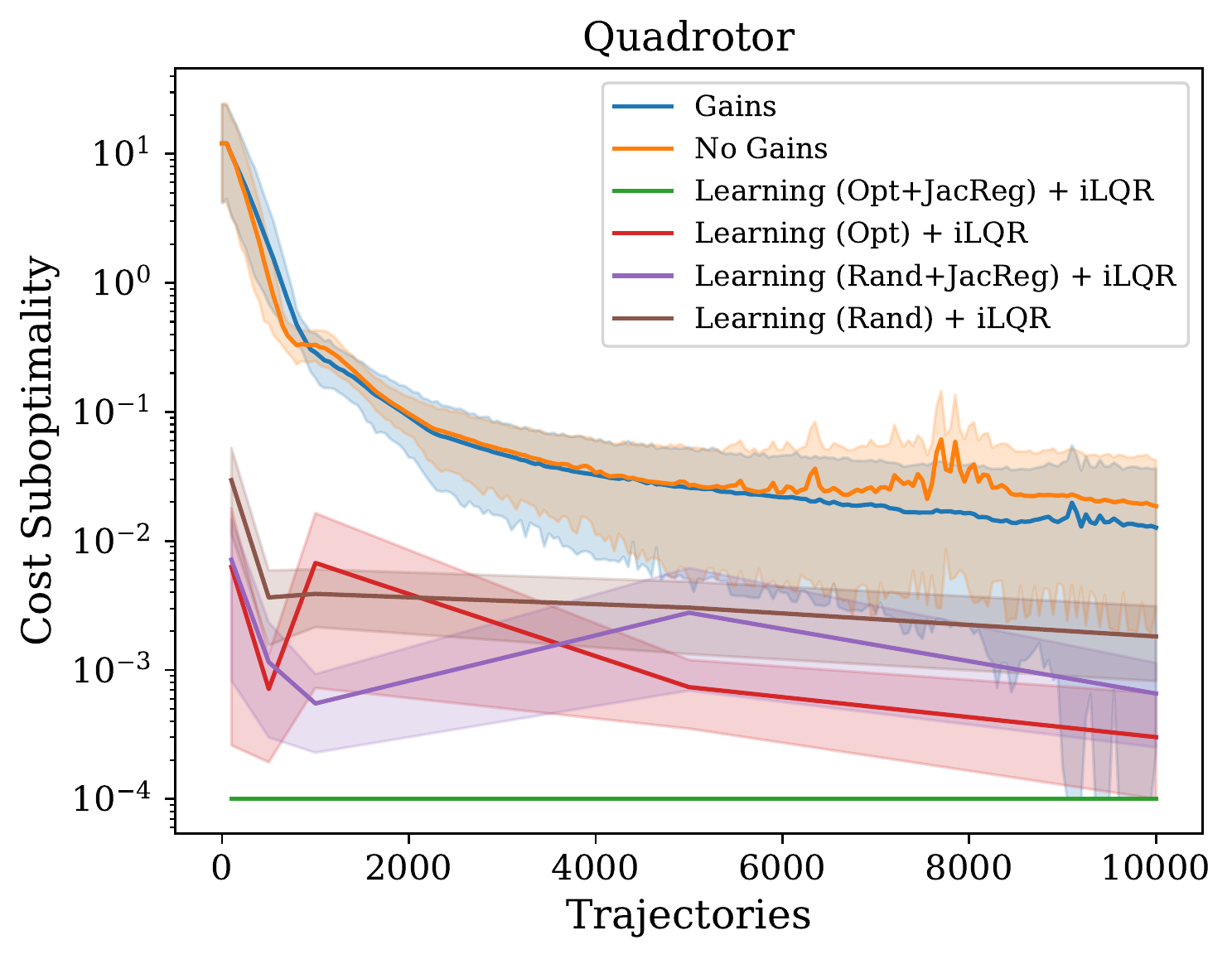}
    \end{minipage}
    \caption{\small Cost suboptimality $(\cJ_T^{\mathrm{alg}} - \cJ_T^\star)/\cJ_T^\star$ versus number of trajectories
    available to both \Cref{alg:learn_mpc_feedback} and $\iLQR$ baselines.
    For visualization, the suboptimality is clipped to $(10^{-4}, \infty)$.} 
    \label{fig:sweeps}
\end{figure}
\icmlspace{-.5em}
Though we find that our method outperforms deep-learning baselines (excluding $\texttt{OPT}$+$\texttt{JacReg}$) on the simpler inverted pendulum environment, the learning+$\iLQR$ approaches fare better on the quadrotor. We suspect that this is attributable to data-reuse, as \Cref{alg:learn_mpc_feedback} estimates an entirely new model of system dynamics at each iteration. We believe that finding a way to combine the advantages of directly estimating linearized dynamics (observed in \Cref{alg:learn_mpc_feedback}, as well as $\texttt{OPT}$+$\texttt{JacReg}$) with the advantages of data-reuse.

\newpage
\bibliographystyle{icml2023}
\bibliography{refs}
\newpage 
\appendix
\onecolumn
\tableofcontents

\part{Analysis}\label{part:analysis}


\newcommand{\errdecbar}{\overline{\mathrm{Err}}_{\mathrm{dec}}}
\newcommand{\Kpiq}{{\kname}_{q}}
\newcommand{\errnot}{\mathrm{Err}_0}
\newcommand{\kgrad}{\kname_{\nabla}}

\newcommand{\Bigohsttil}[1][\cdot]{\widetilde{\cO}_{\star}\left(#1\right)}
\newcommand{\Bigohpilr}[1][\cdot]{\cO_{\pi}\left(#1\right)}
\newcommand{\Bigohstlr}[1][\cdot]{\cO_{\star}\left(#1\right)}
\newcommand{\Kinfbar}{\bar{\kname}_{\infty}}
\newcommand{\Ktwobar}{\bar{\kname}_{2}}
\newcommand{\Kone}{\kname_{1}}
\newcommand{\Kinf}{\kname_{\infty}}
\newcommand{\Ktwo}{\kname_{2}}
\newcommand{\Konebar}{\bar{\kname}_{1}}
\newcommand{\Kqbar}[1][q]{\bar{\kname}_{#1}}
\newcommand{\Btayqbar}{\bar{B}_{\mathrm{tay},q}}
\newcommand{\Btaytwobar}{\bar{B}_{\mathrm{tay},2}}
\newcommand{\Btayinfbar}{\bar{B}_{\mathrm{tay},\infty}}
\newcommand{\Mtayqbar}{\bar{M}_{\mathrm{tay},q}}
\newcommand{\Mtaytwobar}{\bar{M}_{\mathrm{tay},q}}
\newcommand{\Ltayqbar}{\bar{L}_{\mathrm{tay},q}}
\newcommand{\Mtayjbar}{\bar{M}_{\cJ,\mathrm{tay}}}
\newcommand{\Lnabinfbar}{\bar{L}_{\nabla,\infty}}
\newcommand{\Lbar}{\bar{L}}
\newcommand{\mubar}{\bar{\mu}}

\section{Formal Analysis}\label{app:formal_analysis}

\subsection{Organization of the Appendix}
First, we begin with an outline of \Cref{app:formal_analysis}:
\begin{itemize}
	\item \Cref{sec:not_review_app_A} reviews essential notation. 
	\item \Cref{sec:main_result} gives a restatement of our main result, \Cref{thm:main_body}, as \Cref{thm:main_asymptotic_thm}.
\end{itemize}
The rest of \Cref{app:formal_analysis} carries out the proof of \Cref{thm:main_asymptotic_thm}. Speficially,
\begin{itemize}
	\item \Cref{sec:app_a_prob_params} defines numerous problem parameters on which our arguments depend. 
	\item \Cref{sec:grad_disc_main} proves \Cref{cor:grad_test_extra}, a precise statement of \Cref{prop:confirm_eps_stat} in the main text. It does so  via an intermediate result, \Cref{prop:grad_disc}, which bounds the $\cL_{\infty}$ difference between the continuous-time gradient, and the imagine of the discrete-time gradient under the continuous-time inclusion map $\istep(\cdot)$.
	\item \Cref{app:taylor_exp_summary} states key results based on Taylor expansions of dynamics around their linearizations, and norms of various derivative-like quantities.
	\item \Cref{app:est_errs} contains the main statements of the various estimation guarantees, notably, the recovery of nominal trajectories, Markov operators, discretized gradients, and linearized transition matrices $(\bAkpi,\bBkpi)$. 
	\item \Cref{app:desc_stab} leverages the previous section to demonstrate (a) a certain descent condition holds for each gradient step and (b) that sufficiently accurate estimates of transition matrices lead to the synthesis of gains for which the corresponding policies have bounded stability moduli. 
	\item Finally, \Cref{app:formal_conclude} concludes the proof, as well as states a more granular guarantee in terms of specific problem parameters and not general $\Bigohst$ notation. 
\end{itemize}
The rest of \Cref{part:analysis} of the Appendix provides the proofs of constituent results. Specifically,
\begin{itemize}
	\item \Cref{sec:disucssion_app} presents various discussion of main results, as well as gesturing to extensions. Specifcally, \Cref{sec:JSP_justification_exp_gap} describes the exponential gap between $\FOS$s of $\cJ_T$ and $\JSP$s, and \Cref{app:westenbroek_app} explains the consequences of combining our result with \cite{westenbroek2021stability}. We discuss how to implement a projection step to ensure \Cref{cond:pi_cond} in \Cref{sec:feasibility}. Finally, we discuss extensions to an oracle with process noise in \Cref{sec:process_noise}. 
	\item \Cref{app:jac_lins} presents various computations of Jacobian linearizations, establishing that they do accurately capture first-order expansions.
	\item \Cref{app:taylor_exp} proves all the Taylor-expansion like results listed in \Cref{app:taylor_exp_summary}. 
	\item \Cref{app:estimation} proves all the estimation-error bounds in \Cref{app:est_errs}.
	\item \Cref{append:certainty_equiv} provides a general certainty-equivalence and Lyapunov stability perturbation results for time-varying, discrete-time linear systems, in the regime that naturally arises when the state matrices are derived from discretizations of continuous-time dynamics.
	\item \Cref{app:ce} instantiates the bounds in \Cref{append:certainty_equiv} to show that the gains synthesized by \Cref{alg:est_markov} do indeed lead to policies with bounded stability modulus. 
	\item \Cref{app:opt_proofs} contains the proofs of optimization-related results: the proof of the descent lemma (\Cref{lem:descent_lem} (in \Cref{sec:lem:descent_lem}) and the proof of the conversion between stationary points and $\JSP$s,  \Cref{prop:Jpijac} (in \Cref{sec:proof:Jpijac})
	\item Finally, \Cref{app:dt_args} contains various time-discretization arguments, and in particular establishes the aforementiond \Cref{prop:grad_disc} from \Cref{sec:grad_disc_main}. 
\end{itemize}

\subsection{Notation Review}\label{sec:not_review_app_A}

In this section, we review our basic notation. 

\paragraph{Dynamics.} Recall the nominal system dynamics are given by
	\begin{align*}
	\tfrac{\rmd}{\rmd t} \bx(t \mid \bu) = \fdyn( \bx(t \mid \bu), \bu(t)), ~~ \bx(0 \mid \bu) = \xiinit.
	\end{align*}
	We recall the definition of various stabilized dynamics.

	\defnstabdyn*

	\defnstabtrajs*

\paragraph{Linearizations.} Next, we recall the definition of the various linearizations.
	\defoldynamics*
	\defncllinearizations*
We also recall the definition of stationary policies and $\JSP$s.
\defnfos*
	\defJSP*

\paragraph{Problem Constants.} We recall the dynamics-constants $\KF,\LF,\MF$ defined in \Cref{asm:max_dyn_final}, $\KQ,\LQ,\MQ$ in \Cref{asm:cost_asm}, the strong-convexity parameter $\alpha$ in \Cref{asm:convexity}, the controllability parameters $\tcont,\nucont$ from \Cref{asm:ctr}, with $\kcont := \tcont/\step$, and the Riccati parameter $\LFP$ from \Cref{asm:LFP}. Finally, we recall the feasibility radius from \Cref{cond:feasbility}. We also recall
\defnfeas*

\paragraph{Gradient and Cost Shorthands.} Notably, we bound out the following shorthand for gradients and costs:
\begin{align}
\nabla \cJ_T(\pi) := \nabla_{\bu} \cJ_T^\pi(\bu)\big{|}_{\bu = \upi}, \quad \nabla \Jdisc(\pi):= \nabla_{\buvec} \Jpidisc(\buvec)\big{|}_{\buvec = \bukpi[1:K]}, \quad \Jdisc(\pi):=  \Jpidisc(\bukpi[1:K]).
\end{align}

\subsection{Full Statement of Main Result}\label{sec:main_result}
	The following is a slightly more general statement of \Cref{thm:main_body}, which implies \Cref{thm:main_body} for appropriate choice of $\eta  \gets \frac{1}{c_1}\min\left\{\frac{1}{\sqrt{T}},1\right\}$, $\sigw \gets (\sigorac^2 \iota(\delta)/N)^{\frac{1}{4}}$, and with the simplifications $\step \le 1 \le T$.
	\begin{theorem}\label{thm:main_asymptotic_thm} Fix $\delta \in (0,1)$,  define $\iota(\delta) := \log \frac{24 T^2\nfin\max\{\dimx,\dimu\}}{\step^2\delta}$ and $\errnot(\delta) := \sqrt{\iota(\delta)/N}$, where $N$ is the sample size, and suppose we select $\sigw = c(\sigorac^2 \iota(\delta)/N)^{\frac{1}{4}}$ for any $c \in [\frac{1}{\Bigohst[1]},\Bigohst[1]]$. Then, there exists constants $c_1,c_2,\dots,c_5 = \Bigohst[1]$ depending on $c$ such that the following holds. Suppose that 
	\begin{equation}\label{eq:all_asymptotic_conditions}
	\begin{aligned}
	\eta \le \frac{1}{c_1}\min\left\{\frac{1}{\sqrt{T}},1\right\}, \quad \step \le \frac{1}{c_2} \quad N \ge  c_3\iota(\delta) \max\left\{1,\frac{T^2}{\step^2}, \frac{1}{\step^4}, T, \sigorac^2\frac{T^4}{\step^2}, \frac{\sigorac^2}{\step^8}\right\}.
	\end{aligned}
	\end{equation}
	Then, with probability $1-\delta$, if \Cref{cond:feasbility} and all listed Assumptions hold,
	\begin{itemize}
		\item[(a)] For all $n \in [\nfin]$, and $\pi' \in \{\pi^{(n)},\tilde{\pi}^{(n)}\}$, $\mu_{\pi',\star} \le 8\LFP$ and $L_{\pi'} \le 6\max\{1,\LF\}\LFP$.
		\item[(a)] For $\pi = \pi^{(\nout)}$, $\pi$ is $\epsilon$-stationary where 
		\begin{align*}
		\epsilon^2 = c_4 T\left(\step^2 +  \frac{1}{\eta} \left(\frac{1}{\nfin} +\left(\eta T^2+  \frac{T^2}{\step^2}\right)\left(\frac{\iota(\delta)^2}{N} + \sigorac \sqrt{\frac{\iota(\delta)}{N}}\right) + \sigorac^2 \frac{\iota(\delta)^2}{N}\right)\right).
		\end{align*}
		\item[(c)] For $\pi = \pi^{(\nout)}$, $\upi$ is an $\epsilon'$-$\JSP$, where $\epsilon' = c_5\frac{\epsilon^2}{\alpha}$. 
	\end{itemize}
	\end{theorem}

\subsection{Problem Parameters}\label{sec:app_a_prob_params}
	In this section, we provide all definitions of various problem paramaters. The notation is extensive, but we maintain the following conventions:
	\begin{enumerate}
		\item $\mu_{(\cdot)}$ refers to upper bounds on Lyapunov operators, $\kappa_{(\cdot)}$ to upper bounds on zero-order terms (e.g. $\|\fdyn(x,u)\|$) or magnitudes of transition operators,
		 $M_{(\cdot)}$ to bounds on second-order derivatives, $L_{(\cdot)}$ to bounds on first-order derivatives, $B_{(\cdot)}$ to upper bounds on radii,  $\step_{(\cdot)}$ to step sizes, $\mathrm{Err}_{(\cdot)}$ to error terms. 
		 \item $q \in \{1,2,\infty\}$ corresponds to $\ell_{q}$ norms
		 \item Subscripts $\mathrm{tay}$ denote relevance to Taylor expansions of the dynamics.
		\item Terms with have a subscript $\pi$ hide dependence on $\Lpi$, $\Kpist$ and $\Kpiq$ for $q \in \{1,2,\infty\}$
	\end{enumerate}

	\begin{remark}[Reminder on Asymptotic Notation] We let $\Bigohst[x]$ denote a term which suppresses polynomial dependence on all the constants in \Cref{asm:max_dyn_final,asm:cost_asm}, as well as $\LFP$ in \Cref{asm:LFP}, and $\nucont$, $t_0 \ge  \tcont$ and $e^{\LF t_0} \ge e^{\LF \tcont}$, where $t_0 = \step k_0$,  and $\nucont,\tcont$ are given in \Cref{asm:ctr}. We let $\Bigohpi[x]$ suppress all of these constants, as well as polynomials in $\Lpi$ and $\Kpist$.
	\end{remark}

	\subsubsection{Stability Constants}
	We begin by recalling the primary constants controlling the stability of a policy $\pi$.
	\picond*
	It is more convenient to prove bounds in terms of the following three quantity, which are defined in terms of the magnitudes of the closed-loop transition operators. 
	\begin{restatable}[Norms of $\pi$]{definition}{pinorms}\label{defn:pinorms} 
	We define  the constants $\Kpiinf := \max_{1 \le j \le k \le K+1}\|\Phicldisc{k,j}\|$,  and 
	\begin{align*}
	\Kpione &:=  \max_{k \in [K+1]}\step \left(\sum_{j=1}^{k}\|\Phicldisc{k,j}\| \vee \sum_{j=k}^{K+1}\|\Phicldisc{j,k}\|\right)\\
	\Kpitwo^2 &:= \max_{k\in [K+1]}\step \left(\sum_{j=1}^{k}\|\Phicldisc{k,j}\|^2 \vee  \sum_{j=k}^{K+1}\|\Phicldisc{j,k}\|^2\right)
	\end{align*}
	We also define the following upper bounds on these quantities:
	\begin{align*}
	\Kinf(\mu,L) &:= \sqrt{\max\{1,6\LF L\}\mu}\exp(t_0 \LF)\\
	\Ktwo(\mu,L) &:= \max\{1,6\LF L\}\mu\left(t_0\exp(2t_0\LF) + \mu\right)\\
	\Kone(\mu,L) &:= \sqrt{ \max\{1,6\LF L\}\mu}\left(t_0\exp(t_0\LF) + 2\mu\right)
	\end{align*}
	\end{restatable}

	The following lemma is proven in \Cref{sec:proof:lem:Kpi_bounds}, and shows that shows that each of the above terms is $\Bigohpi[1]$.
	\begin{lemma}\label{lem:Kpi_bounds} Let $\pi$ be any policy. Recall $t_0 = \step k_0$ Then, as long as $\step \le 1/6\LF\Lpi$, 
	\begin{align*}
	\mu_{\pi,q} \le \mu_{q}(\Kpist,L_{\pi}) = \Bigohpi[1].
	\end{align*} 
	\end{lemma}


	\subsubsection{Discretization Step Magnitudes}
	Next, we introduce various maximal discretization step sizes for which our discrete-time dynamics are sufficiently faithful to the continuous ones. The first is a general condition for the dynamics to be ``close'', the second is useful for closeness of solutions of Ricatti equations,  the third for the discrete-time dynamics to admit useful Taylor expansions, and  the fourth for discrete-time controllability. We note that the first two do not depend on $\pi$, while the second two do.
	\begin{definition}[Discretization Sizes]\label{defn:step_sizes} We define
	\begin{align*}
	\stepdyn &:= \frac{1}{4\LF}\\
	\stepric  &:= \frac{1}{4\LFP^2\left(3\MF \KF\LFP\LF +  13\LF^2(1+\LF\LFP)^2 \right)}\\
	\steptaypi &:= \min\left\{\frac{1}{16\LF \Lpi},\frac{1}{8\KF}\right\} \le \stepdyn.\\
	\stepctrlpi&:= \frac{\nucont}{8\Lpi^2\Kpi^2 \gamcont^3\exp(2\gamcont)\left( \KF\MF + 2\LF^2\right)}, \quad \gamcont:= \max\{1,\LF \tcont\}
	\end{align*} 
	We note that $\stepdyn,\stepric = 1/\Bigohst[1]$ and $\steptaypi,\stepctrlpi = 1/\Bigohpi[1]$.
	\end{definition}

	\newcommand{\Bstabpi}{B_{\mathrm{stab},\pi}}

	\subsubsection{Taylor Expansion Constants.}
	We now define the relevant constants in terms of which we bound our taylor expansions. 
	\begin{definition}[Taylor Expansion Constants, Policy Dependent]\label{defn:taylor_expansion_constants} We define $\Ltaypiinf  = 2\LF \Kpione$, $\Ltaypitwo := 2\LF\Kpitwo$, and 
	\begin{align*}
	\Mtaypitwo &:= 8\MF( \Kpiinf + 10\Lpi^2\LF^2\Kpitwo^2\Kpione)\\
	\Mtaypiinf &:= 8\MF(\Kpione + 10\Lpi^2\LF^2\Kpione^3)\\
	\Btaypitwo &= \min\left\{\frac{1}{\sqrt{40\MF\Lpi^2\Kpione\Mtaypitwo}} ,\frac{\LF\Kpitwo}{2\Mtaypitwo},\frac{\Rfeas}{16\Lpi\LF\Kpitwo}\right\} \\
	\Btaypiinf &= \min\left\{\frac{1}{40\Lpi^2\Kpione\Mtaypiinf}, \frac{\LF\Kpione}{2\Mtaypiinf}, \frac{\Rfeas}{16\Lpi\LF\Kpione}\right\} 
	\end{align*}
	We also define
	\begin{align*}
	\Mtayjpi &:= 2\MQ\LF^2 \Kpitwo^2(1+3\Lpi^2 T) \Mtaypitwo + \LQ(1+2\Lpi T)\Mtaypitwo + 2\Lpi \LQ,\\
	\Bstabpi &:=  \left(\max\{6,36\LF\Lpi\}\Kpist \cdot 12T\MF\Lpi  (1+\LF\Kpi) B_{\infty}\right)^{-1}\\
	\Lnabpiinf&:=  \LQ(1 + \frac{3\LF}{2}\Kpiinf + 3\Lpi \Kpione)
	\end{align*}
	\end{definition}
	The following is a consequence of \Cref{lem:Kpi_bounds}.
	\begin{lemma}\label{lem:Kpi_bound_taylor} By \Cref{lem:Kpi_bounds}, $\Mtaypitwo,\Mtaypiinf,\Lnabpiinf,\Ltaypiq = \Bigohpi[1]$, $\Btaypitwo,\Btaypiinf = 1/\Bigohpi[1]$, $\Mtayjpi = T\cdot\Bigohpi[1]$, and $\Bstabpi = \frac{1}{T} \cdot \Bigohpi[1]$. 
	\end{lemma}

	The first group of four constants arises in Taylor expansions of the dynamics, the fith in a Taylor expansion of the cost functional, and the sixth in controlling the stability of policies under changes to the input, and the last upper bounds the norm of the gradient. 

	\subsubsection{Estimation Error Terms.}
	Finally, we define the following error terms which arise in the errors of the extimated nominal trajectories, Markov operators, and gradients. Note that the first term has no dependence on $\pi$, while the latter two do.
	\begin{definition}[Error Terms]\label{defn:err_terms} Define $\iota(\delta) := \log \frac{24 T^2\nfin\max\{\dimx,\dimu\}}{\step^2\delta} = \log \frac{24 K^2\nfin\dst}{\delta}$, where $d_{\star} := \max\{\dimx,\dimu\}$. Further, define
	\begin{align*}
	 \errx(\delta) &:= \sigorac\sqrt{2\frac{\dst  \iota(\delta)}{N}}\\
	\errpsipi(\delta) &:=  \sqrt{\frac{\iota(\delta)}{N}}\left(\frac{2\sigorac}{\sigw}\dst^{3/2} + 8\Ltaypiinf \dst\right)+ 4\sigw\Mtaypitwo \dst^{3/2} = \Bigohpi[\sqrt{\frac{\iota(\delta)}{N}}(1+\frac{\sigorac}{\sigw} + \sigw]\\
	 \errnabpi(\delta) &:= \left(\LQ \errpsipi(\delta) + (1+\Kpiinf)\MQ \errx(\delta)\right)(1+2T\Lpi).\\
	\end{align*}
	We note that, in view of \Cref{lem:Kpi_bounds}, 

	\end{definition}

	By \Cref{lem:Kpi_bounds,lem:Kpi_bound_taylor}, we have
	\begin{lemma}\label{lem:err_term_asymptotics} Define $\errnot(\delta) = \sqrt{\frac{\iota(\delta)}{N}}$. Then,
	\begin{equation}\label{eq:err_term_sizes_untuned}
	\begin{aligned}
	\errx(\delta)&= \sigorac\sqrt{2\dst}\errnot(\delta) \le\Bigohstlr[\errnot(\delta)]\\
	\errpsi(\delta) &\le \Bigohpilr[\errnot(\delta)(1+\frac{\sigorac}{\sigw}) + \sigw]\\
	\errnabpi(\delta) &\le \Bigohpilr[T\left(\errnot(\delta)(1+\sigorac +\frac{\sigorac}{\sigw}) + \sigw\right)]. 
	\end{aligned}
	\end{equation}
	If we further tune $\sigw = c\sqrt{\sigorac \errnot(\delta)}$ for any $c \in [1/\Bigohst[1],\Bigohst[1]]$, then 
	\begin{equation}
	\begin{aligned}\label{eq:err_term_sizes_tuned}
	\errx(\delta)&\le \Bigohstlr[\sigorac\errnot(\delta)]\\
	\errpsi(\delta) &\le \Bigohpilr[\errnot(\delta) + \sqrt{\sigorac\errnot(\delta)} ]\\
	\errnabpi(\delta) &\le \Bigohpilr[T\left(\errnot(\delta) + \sqrt{\sigorac\errnot(\delta)}\right)]. 
	\end{aligned}
	\end{equation}
	\end{lemma}

	\subsection{Gradient Discretization}\label{sec:grad_disc_main}
	We begin by stating with the precise statement of \Cref{prop:confirm_eps_stat}, which relates norms of gradients of the discretized objective to that of the continuous-time one. We begin with the following proposition which bounds the difference between the continuous-time gradient, and a (normalized) embedding of the discrete-time gradient into continuous-time. We define the constant 
	\begin{align}
	\kgrad := \left((1+\LF) \MQ(1+\KF)  + \LQ (3\KF\MF + 8\LF^2 +\LF) \right) = \Bigohst{1} \label{eq:kgrad}
	\end{align}
	\begin{proposition}[Discretization of the Gradient]\label{prop:grad_disc} Let $\pi$ be feasible, and let $\nabtil\cJ_T(\pi) = \frac{1}{\step}\step(\nabla \Jdisc(\pi))$ is the continuous-time inclusion of the discrete-time gradient, normalized by $\step^{-1}$. Then, 
	\begin{align*}
	\sup_{t \in [0,T)} \|\nabla \cJ_T( \pi)(t)  - \tilde \nabla \cJ_T(\pi)(t)\| \le  \step e^{\step \LF}\max\{\Kpiinf,\Kpione,1\}\Lpi \kgrad,
	\end{align*}
	\end{proposition}
	The above result is proven in \Cref{sec:prop:grad_disc}. By integrating, we see that $\|\nabla \cJ_T( \pi)  - \tilde \nabla \cJ_T( \pi)\|_{\ltwou} \le \sqrt{T} \step e^{\step \LF}\max\{\Kpiinf,\Kpione,1\}\Lpi \kgrad$, and thus the triangle inequality gives $\|\nabla \cJ_T( \pi)\|_{\ltwou} \le \|\nabtil \cJ_T( \pi)\|_{\ltwou} + \sqrt{T} \step e^{\step \LF}\max\{\Kpiinf,\Kpione,1\}\Lpi \kgrad$. We can see that for any $\buvec = \buk[1:K] \in \sfU$, $\|\buvec\|_{\ell_2}^2 = \sum_{k=1}^K\|\buk\|^2 = \frac{1}{\step}\int_{0}^T\|\buk[k(t)]\|^2 = \frac{1}{\step}\|\istep(\buvec)\|_{\ltwou}^2$. Hence, in particular, $\|\nabla \cJ_T( \pi)\|_{\ltwou} \le \frac{1}{\sqrt{\step}}\|\nabla \Jdisc( \pi)\|_{\ltwo} + \sqrt{T} \step e^{\step \LF}\max\{\Kpiinf,\Kpione,1\}\Lpi \kgrad$. From this, and from using \Cref{lem:Kpi_bounds} to bound $\Kpiinf, \Kpione = \Bigohpi[1]$, we obtain the following corollary, which is a precise statement of \Cref{prop:confirm_eps_stat}.
	\begin{corollary}\label{cor:grad_test_extra} Suppose $\pi$ is feasible. Then, recalling $\kgrad$ from \Cref{eq:kgrad}, 
	\begin{align*}
	\|\nabla \cJ_T( \pi)\|_{\ltwou} \le \frac{1}{\sqrt{\step}}\|\nabla \Jdisc( \pi)\|_{\ltwo} + \sqrt{T} \step e^{\step \LF}\max\{\Kpiinf,\Kpione,1\}\Lpi \kgrad.
	\end{align*}
	In particular, for $\step \le 1/4\LF$,
	\begin{align*}
	\|\nabla \cJ_T( \pi)\|_{\ltwou}  \le \frac{1}{\sqrt{\step}}\|\nabla \Jdisc( \pi)\|_{\ltwo} + \sqrt{T}\step\cdot\underbrace{2 \max\{\Kpiinf,\Kpione,1\}\Lpi \kgrad}_{=\Bigohpi[1]}.
	\end{align*}
	\end{corollary}

\subsection{Main Taylor Expansion Results}\label{app:taylor_exp_summary}
	
	We now state various bounds on Taylor-expansion like terms. All the following results are proven in \Cref{app:taylor_exp}. The first is a Taylor expansion of the dynamics (proof in \Cref{sec:prop:taylor_exp_dyn}).
	\begin{proposition}\label{prop:taylor_exp_dyn} Let $\pi$ be feasible, $\step \le \steptaypi$ Fix a $\buk[1:K] \in \bsfU$, and define the perturbation $\deluk[1:K]:= \buk[1:K] - \bukpi[1:K]$, and define
	\begin{align*}
	\Btwo := \sqrt{\step}\|\deluk[1:K]\|_{\ell_2}, \Binf := \max_{k}\|\deluk\|.
	\end{align*}
	Then, if $B_\infty \le \Rfeas/8$, and if for either $q \in \{2,\infty\}$, it holds that $B_q \le \Btaypiq$, then
	\begin{itemize}
		\item[(a)] The following bounds hold for all $k \in [K+1]$
	\begin{align*}
	\| \btilxpik\,(\buk[1:K]]) - \bxkpi - \sum_{j=1}^{k-1}\Psicldisc{k,j}\deluk[j] \| \le \Mtaypiq B_q^2, \quad \| \btilxpik\,(\buk[1:K]) - \bxkpi\|  \le \Ltaypiq B_q,
	\end{align*}
	\item[(b)] Moreover, for all $k \in [K+1]$ and $t \in [0,T]$, 
	\begin{align*}
	\max\{\|\btilxpik\,(\buk[1:K])\|, \|\btilupik\,(\buk[1:K]])\} \le \frac{3\Rfeas}{4}, and 
	\|\btilxpi(t \mid \buk[1:K])\| \le \Rfeas.
	\end{align*} 
	\end{itemize}
	\end{proposition}
	Next, we provide a Taylor expansion of the discrete-time cost functional (proof in \Cref{sec:lem:taylor_expansion_of_cost}).
	\begin{lemma}\label{lem:taylor_expansion_of_cost} Consider the setting of \Cref{prop:taylor_exp_dyn}, and suppose $B_\infty \le \Rfeas/8$ and $B_2 \le \Btaypitwo$. Then, 
	\begin{align*}
	\|\Jpidisc\discbrak{\deluk[1:K] + \bukpi[1:K]} -  \Jpidisc\discbrak{\bukpi[1:K]} - \langle \deluk ,  \nabla \Jpidisc(\bukpi[1:K]) \rangle \| \le  \Mtayjpi B_2^2.
	\end{align*}
	\end{lemma}
	Next, we show sufficiently small perturbations of the nomimal input preserve stability of the dynamics (proof in \Cref{sec:prop_Kpi_bounds_stab}).
	\begin{lemma}\label{prop:Kpi_bounds_stab}  Again consider the setting of \Cref{prop:taylor_exp_dyn}, and suppose $B_\infty \le\min\{\Rfeas/8,\Btaypiinf,\Bstabpi\}$. Then, 
	\begin{align*}
	\mu_{\pi',\star} \le(1 +  B_\infty/\Bstabpi )\Kpist \le 2\Kpist, \quad L_{\pi'} = \Lpi.
	\end{align*}
	\end{lemma}

	Lastly, we bound the norm of the discretized gradient (\Cref{sec:lem:grad_bound}). 
	\begin{lemma}\label{lem:grad_bound} Let $\pi$ be feasible, and let $\step \le \stepdyn$. Then 
	\begin{align*}
	\max_{k \in [K]}\|(\nabla \Jdisc(\pi))_k\|  \le \step\Lnabpiinf
	\end{align*}
	\end{lemma}

\subsection{Estimation Errors}\label{app:est_errs}
	In this section, we bound the various estimation errors. All the proofs are given in \Cref{app:estimation}. We begin with a simple condition we need for estimation of Markov parameters to go through.
	\begin{definition}\label{cond:taylor_cond}  We say $\pi$ is \emph{estimation-friendly if} $\pi$ is feasible, and if 
	\begin{align*}
	\sigorac\sqrt{\frac{\iota(\delta)}{2N\Lpi}} \le \sigw \le \frac{\Btaypiinf}{2\sqrt{\dst}}, \quad \step \le \steptaypi
	\end{align*}
	\end{definition}

	Our first result is recovery of the nominal trajectory and Markov operators. Recovery of the nominal trajectory follows from Gaussian concentration, and recovery of the Markov operator for the Matrix Hoeffding inequality (\citet[Theorem 1.4]{tropp2012user}) combined with the Taylor expansion of the dynamics due to \Cref{prop:taylor_exp_dyn}. The following is proven in \Cref{sec:lem:prop:markov_est}. To state the bound, we recall the estimation error terms in \Cref{defn:err_terms}.
	\begin{proposition}\label{prop:markov_est}  Fix $\delta \in (0,1)$ and suppose that $N$ is large enough that $\pi$ is estimation friendly. Then, for any estimation-friendly $\estpsi(\pi;N,\sigw)$ (\Cref{alg:est_markov}) returns estimates with such that, with probability $1 - \delta/2\nfin$.
	\begin{align}
	&\max_{1 \le j < k \le K+1} \left\|\Psicldisc{k,j}  - \bhatPsi{k}{j}  \right\|_{\op} \le \errpsipi(\delta) \quad \max_{k \in [K+1]}\|\bhatxk - \bxkpi\| \le \errx(\delta) \label{eq:good_event}
	\end{align}
	\end{proposition}
	Let $\Pialg := \{\pi^{(n)},\tilde{\pi}^{(n)} : n \in [\nfin]\}$ denote the set of policies constructed by the algorithm, and note that $\estpsi$ is called once for each policy in $\Pialg$. We define the good estimation event as
	\begin{align}
	\Evest(\delta) &:= \bigcap_{n=1}^{\infty} (\cE_n(\delta) \cap\tilde \cE_{n}(\delta)), \quad\\
	\cE_n(\delta) &:= \{\text{\Cref{eq:good_event}  holds for $\pi = \pin$ if $\pin$ is estimation friendly}\}\\
	\tilde\cE_n(\delta) &:= \{\text{\Cref{eq:good_event}holds for $\tilde\pi = \pin$ if $\tilde{\pi}^{(n)}$ is estimation-friendly}\}
	\end{align}
	By \Cref{prop:markov_est}  and a union bound implies
	\begin{align*}
	\Pr[\Evest(\delta)] \ge 1-\delta.
	\end{align*}
	We now show that on the good estimation event, the error of the gradient is bounded. The proof is \Cref{sec:lem:grad_err}.
	\begin{lemma}[Gradient Error]\label{lem:grad_err} On the event $\Evest(\delta)$, it holds that that if $\pin$ is estimation-friendly, then \Cref{alg:learn_mpc_feedback}(\Cref{line:grad_est}) produces
	\begin{align*}
	\max_k\|\bnabhatk^{\,(n)} - (\nabla \Jdisc(\pin))_k\|  \le \errnabpi[\pin](\delta).
	\end{align*}
	\end{lemma}
	We also bound the error in the recovery of the system paramters used for synthesizing the stabilizing gains. Recovery of said parameters requires first establishing controllability of the discrete-time Markov operator. We prove the following in {\Cref{sec:prop:ctr_disc}}:
	\begin{proposition}\label{prop:control_disc} Define $\gamcont := \max\{1,\LF \tcont\}$, and suppose that $\step \le \min\{\stepctrlpi,\stepdyn\}$. Then, for $k \ge \kcont + 1$, it holds that
	\begin{align*}
	\lambda_{\min}\left(\sum_{j=k-\kcont}^{k-1} \Psicldisc{k,j}(\Psicldisc{k,j})^\top\right)  \succeq \step \cdot \frac{\nucont}{8\Lpi^2\gamcont^2\exp(2\gamcont)}
	\end{align*}
	\end{proposition}
	With this result, {\Cref{sec:prop:A_est}} upper bounds the estimation error for the discrete-time system matrices. 
	\begin{proposition}\label{prop:A_est} Suppose $\Evest(\delta)$ holds, fix $n \in \nfin$, and let $\pi = \tilde{\pi}^{(n)}$. Then, suppose that $\step \le \min\{\stepctrlpi,\stepdyn\}$, $k_0 \ge \kcont + 2$, and 
	\begin{align}
	\errpsi(\delta) \le \step \frac{\sqrt{\nucont/\tcont}}{2\sqrt{2}\Lpi \gamcont \exp(\gamcont) },\label{eq:errpspi_stab_intermed}
	\end{align}
	Then, on $\Evest(\delta)$, if $\pi$ is estimation-friendly, the estimates from the call of $\estgains(\pi;N,\sigma)$ satisfy  
	\begin{align*}
	\|\bhatBk - \bBkpi\|  \vee \|\bhatAk - \bAkpi\| \le \frac{\errpsipi(\delta)}{\step} \cdot \tnot\Kpiinf\Lpi^2 \frac{192\gamcont^3\exp(2\gamcont)}{\nucont}. 
	\end{align*}
	\end{proposition}

\subsection{Descent and Stabilization}\label{app:desc_stab}
	In this section, we leverage the estimation results in the previous section to demonstrate the two key features of the algorithm: descent on the discrete-time objective, and stability after the synthesized gains.  We begin with a standard first-order descent lemma, whose proof is given in \Cref{sec:lem:descent_lem}. This lemma also ensures, by invoking \Cref{lem:Kpi_bounds}, that the step size is sufficently small to control the stability of $\tilde{\pi}^{(n)}$, which uses the same gains as $\pin$ but has a slightly perturbed control input.
	\begin{lemma}[Descent Lemma]\label{lem:descent_lem} Suppose $\pi = \pin$ is estimation friendly, let $M \ge \Mtayjpi$, and suppose 
	\begin{align*}
	\eta &\le \frac{1}{4M}, \quad (\eta( \Lnabpiinf + \frac{1}{\step}\errnabpi[\pin](\delta)  ) +\errx(\delta)) \le \min\left\{\frac{\Rfeas}{8},\Bstabpi,\Btaypiinf,\frac{\Btaypitwo}{\sqrt{T}}\right\}.
	\end{align*}
	Then, on event $\Evest(\delta)$, it holds (again setting $\pi \gets \pin$ on the right-hand side)
	\begin{align*}
	\Jdisc(\tilde\pi^{(n)}) - \Jdisc(\pin) 
	&\le -\frac{\eta}{2\step}\|\bnabhatkn[1:K]\|_{\ell_2}^2 +  T\left(\frac{\errnabpi(\delta)^2}{4\step^2 M} + \errx(\delta) \Lnabpiinf + M \errx(\delta)^2\right).
	\end{align*}
	and that 
	\begin{align*}L_{\tilde{\pi}^{(n)}} = L_{\pi^{(n)}}, \quad \mu_{\tilde{\pi}^{(n)},\star} \le 2 \mu_{\pin,\star}.
	\end{align*}
	\end{lemma}
	The next step is to establish a stability guarantee for the certainty-equivalent gains synthesized. We begin with a generic guarantee, whose proof is given in \Cref{app:ce}. 
	\begin{proposition}[Certainty Equivalence Bound]\label{prop:ce_bound} Let  $\bhatAkpi$ and $\bhatBkpi$ be estimates of $\bAkpi$ and $\bBkpi$, and let $\bhatKk$ denote the corresponding certainty equivalence controller synthesized by 
	\Cref{alg:gain_est}(\Cref{line:control_synth_start,line:control_synth_end}). Suppose that  $\step \le \min\{\stepric,\stepdyn\}$ and
	\begin{align*}\max_{k \in [k_0:K]}\|\bhatAkpi - \bAkpi\|_{\op} \vee \|\bhatBkpi - \bBkpi\|_{\op} \le \step(2^{17}\LFP^4\max\{1,\LF^3\})^{-1}
	\end{align*}
	Then, if $\pi' = (\bukpi[1:K],\bhatKk[1:K])$, we have
	\begin{align*}
	\Kpiprst \le 4\LFP, \quad L_{\pi'} \le 6\max\{1,\LF\}\LFP. 
	\end{align*}
	\end{proposition}
	As a direct corollary of the above proposition and \Cref{prop:A_est}, we obtain the following:
	\begin{lemma}\label{lem:stabilization_together} Suppose $\Evest(\delta)$ holds, fix $n \in \nfin$, and let $\pi = \tilde{\pi}^{(n)}$. Then, suppose that $\step \le \min\{\stepctrlpi,\stepdyn\}$, $\pi$ is estimation-friendly, $k_0 \ge \kcont + 2$, and
	\begin{align}
	\errpsipi(\delta) \le \step^{2}\left(2^{25}\LFP^4\max\{1,\LF^3\} \cdot \tnot\Kpiinf\Lpi^2 \frac{\gamcont^3\exp(2\gamcont)}{\nucont}\right)^{-1} \label{eq:errpspi_stab_final}
	\end{align}
	Then, 
	\begin{align*}
	\mu_{\pi^{(n+1)},\star} \le 4\LFP, \quad L_{\pi^{(n+1)}} \le 6\max\{1,\LF\}\LFP. 
	\end{align*}
	\end{lemma}
	\begin{proof} One can check that, as $\LFP,L_{\pi} \ge 1$, \Cref{eq:errpspi_stab_final} implies \Cref{eq:errpspi_stab_intermed}. Thus, the lemma follows directly from \Cref{prop:A_est,prop:ce_bound}, as well as noting $192\cdot 2^{17} \le 2^{25}$
	\end{proof}

\subsection{Concluding the proof.}\label{app:formal_conclude}
\newcommand{\stepctrlbar}{\bar{\step}_{\mathrm{ctrl}}}
\newcommand{\steptaybar}{\bar{\step}_{\mathrm{tay}}}
\newcommand{\errpsibar}{\overline{\mathrm{Err}}_{\Psi}}
\newcommand{\errnabbar}{\overline{\mathrm{Err}}_{{\nabla}}}

\newcommand{\Bstabbar}{\bar{B}_{\mathrm{stab}}}

In this section, we conclude the proof. First, we define uniform upper bounds on all $\pi$-dependent parameters. 

\paragraph{Uniform upper bounds on parameters.} To begin, define
\begin{align}
\mubar = 8\LFP, \quad \Lbar = 6\max\{\LF,1\}\LFP. \label{eq:mubar_Lbar}
\end{align}
Next, for $q \in \{1,2,\infty\}$, define $\bar{\kname}_{q} := \kname_q(\mubar,\Lbar)$ defined in \Cref{defn:pinorms}. We define $\steptaybar,\stepctrlbar$ alogously to $\steptaypi,\stepctrlpi$ in \Cref{defn:step_sizes}  with $\Kpiinf$ replaced by $\Kinfbar$ and $\Lpi$ with $\Lbar$. For $q \in \{2,\infty\}$, we define $\Mtayqbar,\Mtayjbar,\Ltayqbar,\Lnabinfbar,\Btayqbar, \Bstabbar$ analogously to $\Mtaypiq,\Mtayjpi,\Ltaypiq,\Lnabpiinf,\Btaypiq \Bstabpi$ in \Cref{defn:taylor_expansion_constants}, with all occurences of $\Kpiinf,\Kpione,\Kpitwo$ replaced by $\Kinfbar,\Konebar,\Ktwobar$ and all occurences of $\Lpi$ replaced by $L$. Finally, we define $\errpsibar,\errnabbar$ to be analogous to $\errpsipi,\errnabpi$ but with the same above substitutions. From \Cref{lem:Kpi_bounds,lem:Kpi_bound_taylor}, we have
\begin{align*}
&\bar{\kname}_q, \Mtayqbar, \Ltayqbar, \Lnabinfbar,\Btayqbar = \Bigohst[1]\\
&\stepdyn,\stepric,\steptaybar,\stepctrlbar = 1/\Bigohst[1]\\
&\Mtayjbar = T \cdot \Bigohst[1].\\
&\Bstabbar = \frac{1}{T}\Bigohst[1]
\end{align*} 
Moreover, recalling $\errnot(\delta) := \sqrt{\iota(\delta)/N}$, and setting $\sigw = c\sqrt{\errnot \sigorac}$ for any $c \in [1/\Bigohst[1],\Bigohst[1]]$, \Cref{lem:err_term_asymptotics} gives
\begin{equation}\label{eq:err_terms:asym}
\begin{aligned}
\errx(\delta) &= \Bigohst[\sigorac\errnot(\delta)]\\
\errpsipi(\delta) &= \Bigohst[\errnot(\delta) + \sqrt{\sigorac \errnot(\delta)}]\\ \errnabpi(\delta) &= \Bigohst[T(\errnot(\delta) + \sqrt{\sigorac \errnot(\delta)})]. 
\end{aligned}
\end{equation}

\paragraph{Statement of Main Guarantee, Explicit Constants.} We begin by stating our main guarantee, first with explicit constants. We then translate into a $\Bigohst[1]$ notation. To begin, define the following descent error term:
\begin{align}
\errdecbar(\delta) &:=  T\left(\frac{\errnabbar(\delta)^2}{4\step^2 \Mtayjbar} + \errx(\delta) \Lnabinfbar + \Mtayjbar \errx(\delta)^2\right) \label{eq:errdecbar}
\end{align}
And note that for $\sigw = c\sqrt{\errnot \sigorac}$ for $c \in[1/\Bigohst[1],\Bigohst[1]]$ (using numerous simplifications, such as $T/\step \ge 1$)
\begin{align*}
\errdecbar(\delta) &:=  \Bigohst[1]\cdot \left(\frac{T^3}{\step^2}\left(\errnot(\delta)^2 + \sigorac \errnot(\delta)\right) + T\sigorac^2 \errnot(\delta)^2\right).
\end{align*}

\begin{theorem}\label{thm:main_explicit_constants} Fix $\delta \in (0,1)$,  and suppose that $\eta \le \frac{1}{4\Mtayjbar}$, $k_0 \ge \kcont + 2$, and suppose
 \begin{subequations}
\begin{align}
&\sigorac\sqrt{\frac{\iota(\delta)}{2N\Lbar}} \le \sigw \le \frac{\Btayinfbar}{2\sqrt{\dst}}, \label{eq:sigw_cond_bar} \quad \\
&(\eta( \Lnabinfbar + \frac{1}{\step}\errnabbar(\delta)  ) +\errx(\delta)) \le \min\left\{\frac{\Rfeas}{8},\Bstabbar,\Btayinfbar,\frac{\Btaytwobar}{\sqrt{T}}\right\} \label{eq:err_dec_cond_bar}\\
&\errpsibar(\delta) \le \step^{2}\left(2^{25}\LFP^4\max\{1,\LF^3\} \cdot \tcont\Kpiinf\Lbar^2 \frac{\gamcont^3\exp(2\gamcont)}{\nucont}\right)^{-1} = \frac{\step^{2}}{\Bigohst[1]} \label{eq:errpsi_cond_bar} \\
&\step \le \min\{\steptaybar,\stepctrlbar,\stepric\} = \frac{1}{\Bigohst[1]}\label{eq:step_cond_bar}
\end{align}
\end{subequations}
Then, for $\pi = \pi^{(\nout)}$ returned by \Cref{alg:learn_mpc_feedback} satisfies all four properties with probability $1-\delta$:
\begin{enumerate} 
	\item[(a)] $\Kpist \le 8\LFP$ and $\Lpi \le 6\max\{1,\KF\}\LFP = \Lbar$. In fact, for all $n \in [\nfin]$, and $\pi' \in \{\pi^{(n)},\tilde{\pi}^{(n)}\}$, $\mu_{\pi',\star} \le \mubar = 8\LFP$ and $L_{\pi'} \le \Lbar = 6\max\{1,\LF\}\LFP$. 
	\item[(b)] The discrete-time stabilized gradient is bounded by
\begin{align*}
\step\|\Jdisc(\pi)\|_{\ell_2}^2 &\le 2T\errnabbar(\delta)^2 + \frac{2}{\eta}\left(\frac{2(1+T)\KQ}{\nfin} + \errdecbar(\delta)\right)\\
&=  \frac{1}{\eta}\Bigohst[1]\cdot \left( \frac{T}{\nfin} + (\eta T^3+  \frac{T^3}{\step^2})\left(\errnot(\delta)^2 + \sigorac \errnot(\delta)\right) + T\sigorac^2 \errnot(\delta)^2\right),
\end{align*}
where the last line holds when $\sigw = c\sqrt{\sigorac \errnot(\delta)}$ for some $c \in [\frac{1}{\Bigohst[1]},\Bigohst[1]]$. 
\item[(c)] Recall $\kgrad := \left((1+\LF) \MQ(1+\KF)  + \LQ (3\KF\MF + 8\LF^2 +\LF) \right) = \Bigohst[1]$ from \Cref{eq:kgrad}. Then $\pi$ is $\epsilon$-stationary for
\begin{align*}
\epsilon^2 &= 4T\errnabbar(\delta)^2 + \frac{4}{\eta}\left(\frac{2(1+T)\KQ}{\nfin} + \errdecbar(\delta)\right) + 4T\step^2\left(\max\{\bar{\kname}_{\infty},\bar{1}_{\infty},1\}\Lbar \kgrad\right)^2\\
&= \Bigohst[1]\cdot T\left( \step^2  + \frac{1}{\eta} \left(\frac{1}{\nfin} + (\eta T^2+  \frac{T^2}{\step^2})\left(\errnot(\delta)^2 + \sigorac \errnot(\delta)\right) + \sigorac^2 \errnot(\delta)^2\right)\right).
\end{align*}
where the last line holds when $\sigw = c\sqrt{\sigorac \errnot(\delta)}$ for some $c \in [\frac{1}{\Bigohst[1]},\Bigohst[1]]$. 
\item[(d)] $\upi$ is an $\epsilon'$-$\JSP$, where $\epsilon' = 64\epsilon^2\Lbar^2/\alpha = \Bigohst[1]\cdot \frac{\epsilon^2}{\alpha}$.
\end{enumerate}
\end{theorem}
We prove \Cref{thm:main_explicit_constants} from the above results in \Cref{sec:thm:main_explicit_constants} just below. \Cref{sec:thm_translate} below translates the above theorem into \Cref{thm:main_asymptotic_thm} which uses $\Bigohst$ notation.

\subsubsection{Translating \Cref{thm:main_explicit_constants} into \Cref{thm:main_asymptotic_thm}}\label{sec:thm_translate}

\begin{proof} It suffices to translate the conditions \Cref{eq:sigw_cond_bar,eq:err_dec_cond_bar,eq:errpsi_cond_bar,eq:step_cond_bar} into $\Bigohst$ notation. Again, recall $\errnot(\delta) = \sqrt{\iota(\delta)/N}$, and take $\sigw = c\sqrt{\errnot(\delta) \sigorac}$ for $c \in [1/\Bigohst[1],\Bigohst[1]]$. Then, \Cref{eq:sigw_cond_bar} holds for $\errnot(\delta) \le 1/c_1$, where $c_1 = \Bigohst[1]$. Next, to make \Cref{eq:err_dec_cond_bar} hold, it suffices that
\begin{align*}
\max\left\{(\eta \Lnabinfbar, \frac{\eta}{\step}\errnabbar(\delta), \errx(\delta)\right\} \le \frac{1}{3}\min\left\{\frac{\Rfeas}{8},\Bstabbar,\Btayinfbar,\frac{\Btaytwobar}{\sqrt{T}}\right\},
\end{align*}
The term $\eta \Lnabinfbar$ is sufficiently bounded where $\eta \le \frac{1}{c_2\sqrt{T}}$ for $c_2 = \Bigohst[1]$. Recalling $\errnabpi(\delta) \le \Bigohst[T(\errnot(\delta) + \sqrt{\sigorac \errnot(\delta)})]$ from \Cref{eq:err_terms:asym}, and that $\eta \le \frac{1}{c_2\sqrt{T}}$, it is enough that $(\errnot(\delta) + \sqrt{\sigorac \errnot(\delta)}) \le \frac{c_2\step}{c_3 T}$ for $c_3 = \Bigohst[1]$. Finally $\errx(\delta)$ is bounded for $\errx(\delta)   = \errnot(\delta) \le 1/c_4\sqrt{T}$, where $c_4 = \Bigohst[1]$. Collecting these conditions, we have that for $c_1, c_{2},c_3,c_4 = \Bigohst[1]$, \Cref{eq:sigw_cond_bar,eq:err_dec_cond_bar} hold for
\begin{align*}
\eta \le \frac{1}{c_2\sqrt{T}}, \quad \errnot(\delta) + \sqrt{\sigorac \errnot(\delta)} \le \frac{c_2\step}{c_3 T}, \quad 
\end{align*}
Next,as $\errpsipi(\delta) = \Bigohst[\errnot(\delta) + \sqrt{\sigorac \errnot(\delta)}]$ from \Cref{eq:err_terms:asym},$\Cref{eq:errpsi_cond_bar}$ holds as long as $\errnot(\delta) + \sqrt{\sigorac \errnot(\delta)}  \le \step^2/c_5$ for a $c_5 = \Bigohst[1]$. Combining, 
\begin{align*}
\eta \le \frac{1}{c_2\sqrt{T}}, \quad \errnot(\delta) + \sqrt{\sigorac \errnot(\delta)} \le \min\left\{\frac{c_2\step}{c_3 T}, \frac{\step^2}{c_5}\right\} \quad \errnot(\delta) \le \min\left\{\frac{1}{c_1},\frac{1}{c_4\sqrt{T}}\right\}
\end{align*}
Finally, \Cref{eq:step_cond_bar} requirs $\step \le 1/c_6$, for $c_6 = \Bigohst[1]$, and that $\eta \le 1/c_7$ where $c_7 = 4\Mtayjbar = \Bigohpi[1]$. By shrinking constants if necessary, this can be simplified into
\begin{align*}
\eta \le \min\{\frac{1}{c_7},\}\frac{1}{c_2\sqrt{T}}\}, \quad \errnot(\delta) \le \min\left\{\min\left\{\frac{c_2\step}{c_3 T}, \frac{\step^2}{c_5},\frac{1}{c_1},\frac{1}{c_4\sqrt{T}}\right\}, \frac{1}{\sigorac}\min\left\{\frac{c_2\step}{c_3 T}, \frac{\step^2}{c_5}\right\}^2 \right\}.
\end{align*}
And recall $\errnot(\delta) = \sqrt{\iota(\delta)/N}$, this becomes
\begin{align*}
\step \le \frac{1}{c_6},\quad \eta \le \frac{1}{c_2}\min\left\{1,\frac{1}{\sqrt{T}}\right\}, \quad N &\ge \iota(\delta) \min\left\{\min\left\{\frac{1}{c_1},\frac{c_2\step}{c_3 T}, \frac{\step^2}{c_5},\frac{1}{c_4\sqrt{T}}\right\}, \frac{1}{\sigorac}\min\left\{\frac{c_2\step}{c_3 T}, \frac{\step^2}{c_5}\right\}^2 \right\}^{-2}.
\end{align*}
By consolidating constants and relabeling $c_1,c_2 = \Bigohst[1]$ as needed, it suffices that
\begin{align*}
\eta \le c_1 \min\left\{\frac{1}{\sqrt{T}},1\right\}, \quad \step \le \frac{1}{c_2}, \quad N &\ge c_3\iota(\delta) \min\left\{\min\left\{1,\frac{\step}{ T}, \step^2,\frac{1}{\sqrt{T}}\right\}, \frac{1}{\sigorac}\min\left\{\frac{\step}{ T}, \step^2\right\}^2 \right\}^{-2}\\
&= c_3\iota(\delta) \max\left\{1,\frac{T^2}{\step^2}, \frac{1}{\step^4}, T, \sigorac^2\frac{T^4}{\step^2}, \frac{\sigorac^2}{\step^8}\right\}.
\end{align*}
Having shown that the above conditions suffice to ensure \Cref{thm:main_explicit_constants} holds, the bound follows (again replacing $\errnot(\delta)$ with $\sqrt{\iota(\delta)/N}$).

\end{proof}

\subsubsection{Proof of \Cref{thm:main_explicit_constants}}\label{sec:thm:main_explicit_constants}
We shall show the following invariant. At each step $n$,
\begin{align}
\mu_{\pin,\star} \le \mubar/2, \quad L_{\pin} \le \Lbar. \label{eq:invariant}.
\end{align}
\Cref{lem:init_policy} shows that \Cref{eq:invariant} holds for $n = 1$. Next, for $n \ge 1$, directly combining \Cref{lem:descent_lem,lem:stabilization_together} imply the following per-round guarantee.
 \begin{lemma}[Per-Round Lemma]\label{lem:per_round} Suppose that  $\eta \le \frac{1}{4\Mtayjbar}$, $k_0 \ge \kcont + 2$, Then if $\pin$ satisfies \Cref{eq:invariant} and \Cref{eq:sigw_cond_bar,eq:err_dec_cond_bar,eq:errpsi_cond_bar,eq:step_cond_bar}. 
Then, on $\Evest(\delta)$,
\begin{itemize}
	\item[(a)] $\max_k\|\bnabhatk^{\,(n)} - (\nabla \Jdisc(\pin))_k\|  \le \errnabbar(\delta)$; thus $\step \|\bnabhatk[1:K]^{\,(n)} - (\nabla \Jdisc(\pin)\|_{\ell_2}^2  \le T\errnabbar(\delta)^2$.
	\item[(b)] The following descent guarantee holds
\begin{align*}
\Jdisc(\tilde\pi^{(n)}) - \Jdisc(\pin) 
&\le -\frac{\eta}{2\step}\|\bnabhatkn[1:K]\|_{\ell_2}^2 + \errdecbar(\delta)
\end{align*}
\item[(c)]  $L_{\tilde{\pi}^{(n)}} = L_{\pi^{(n)}} \le \Lbar$ and $ \mu_{\tilde{\pi}^{(n)},\star} \le 2 \mu_{\pin,\star} \le \mubar$.
\item[(d)] $\mu_{\pi^{(n+1)},\star} \le 4\LFP = \Lbar/2, \quad L_{\pi^{(n+1)}} \le 6\max\{1,\LF\}\LFP = \Lbar. 
$; that is $\pi^{(n+1)}$ satisfies \Cref{eq:feasability_invariant}
\end{itemize}
\end{lemma}
\begin{proof} Part (a) follows from \Cref{lem:grad_err}, and parts (b) and (c) follow from \Cref{lem:descent_lem}, with the necessary replacement of $\pi$-dependent terms wither $\overline{(\cdot)}$ terms. Part (b) allows us to make the same substiutions in \Cref{lem:stabilization_together}, which gives part (c).
\end{proof}
\begin{proof}[Proof of \Cref{thm:main_explicit_constants}] Under the conditions of this lemma, \Cref{lem:per_round} holds. As  \Cref{lem:init_policy} shows that \Cref{eq:invariant} holds for $n = 1$, induction implies  \Cref{lem:init_policy} holds for all $n \in [\nfin]$ on $\Evest(\delta)$, an event which occurs with probability $1-\delta$. We now prove each part of the present theorem in sequence.

\paragraph{Part (a).} Directly from \Cref{lem:per_round}(d)

\paragraph{Part (b).} Notice that, since $\tilde\pi^{(n)}$ and $\pi^{(n+1)}$ differ only in their gains, $\Jdisc(\pi^{(n+2)}) = \Jdisc(\tilde\pi^{(n)})$. Therefore, summing up the descent guarantee in \Cref{lem:per_round}(b), we have
\begin{align*}
\Jdisc(\pi^{(\nfin+1)}) - \Jdisc(\pi^{(1)}) 
&\le -\frac{\eta}{2\step}\sum_{n=1}^{\nfin}\|\bnabhatkn[1:K]\|_{\ell_2}^2 +  \nfin\errdecbar(\delta)\\
&\le -\frac{\eta}{2\step}\nfin\min_{n \in [\nfin]}\|\bnabhatkn[1:K]\|_{\ell_2}^2 +  \nfin\errdecbar(\delta)\\
&= -\frac{\eta}{2\step}\|\bnabhatknout[1:K]\|_{\ell_2}^2 +  \nfin\errdecbar(\delta)
\end{align*}
where we recall that our algorithm selects to output $\pi^{(\nout)}$, where $\nout$ minimizes $\|\bnabhatkn[1:K]\|_{\ell_2}^2$.
Recall that $ \Jpidisc\discbrak{\pi} = \Jpidisc\discbrak{\bukpi[1:K]}$ where $\Jpidisc\discbrak{\buvec} :=  V(\btilxpik[K+1]\,\discbrak{\buvec}) + \step \sum_{k=1}^K Q( \btilxpik\,\discbrak{\buvec}, \btilupik\,\discbrak{\buvec}, t_k)$. Hence, for all feasible $\pi$, \Cref{asm:cost_asm} implies $0 \le \Jpidisc\discbrak{\pi} \le \KQ(1+\step K) = (1+T)\KQ$. By \Cref{cond:feasbility}, $\pi^{(n+1)}$ and $\pi^{(1)}$ are by feasible, and thus $\Jdisc(\pi^{(n+1)}) - \Jdisc(\pi^{(1)}) \ge -(1+T)\KQ$. Therefore, by rearranging the previous display,
\begin{align*}
\step\|\bnabhatknout[1:K]\|_{\ell_2}^2 &\le \frac{1}{\eta}\left(\frac{2(1+T)\KQ}{\nfin} + \errdecbar(\delta)\right).
\end{align*}
By  \Cref{lem:per_round}(a), and AM-GM imply then
\begin{align*}
\step\|\Jdisc(\pi^{(\nout)})\|_{\ell_2}^2 &\le 2T\errnabbar(\delta)^2 + \frac{2}{\eta}\left(\frac{2(1+T)\KQ}{\nfin} + \errdecbar(\delta)\right).
\end{align*}
\paragraph{Part (c).} Note that $\step \le \steptaybar$ implies $\step \le 1/4\LF$. From \Cref{cor:grad_test_extra}, and for $\kgrad = \Bigohst[1]$ as in \Cref{eq:kgrad}, the following holds for any feasible $\pi$:
\begin{align*}
\|\nabla \cJ_T( \pi)\|_{\ltwou}^2  &\le \left(\frac{1}{\sqrt{\step}}\|\nabla \Jdisc( \pi)\|_{\ltwo} + \sqrt{T}\step\cdot 2 \max\{\Kpiinf,\Kpione,1\}\Lpi \kgrad\right)^2\\
&\le \frac{2}{\step}\|\nabla \Jdisc( \pi)\|_{\ltwo}^2 + 4T\step^2\left(\max\{\Kpiinf,\Kpione,1\}\Lpi \kgrad\right)^2.
\end{align*}
Apply the above with $\pi = \pi^{(\nout)}$ gives part $c$, and upper bound $\Kpiinf,\Kpione,\Lpi$ by $\bar{\kname}_{\infty},\bar{\kname}_{1},\Lbar$ concludes. 
 
 \paragraph{Part (d).} This follows directly from \Cref{prop:Jpijac}, noting that $L_{\pi} \le \Lbar$ for $\pi  = {\pi}^{(\nout)}$,  and that $\steptaybar = \frac{1}{16 \Lbar \LF}$, so that the step-size condition of \Cref{prop:Jpijac} is met. 
\end{proof}

\section{Discussion and Extensions}\label{sec:disucssion_app}

\subsection{Separation between and Open-Loop and Closed-Loop Gradients}\label{sec:JSP_justification_exp_gap}
In this section, we provided an illustrative example as to why a approximation $\JSP$ is more natural than canonical stationary points. Fix an $\epsilon \in (0,1]$, and consider the system with dynamic map
\begin{align*}
f_{\epsilon}(x,u) = 2x + u - \epsilon.
\end{align*}
Let $\bx_{\epsilon}(t \mid \bu)$ denote the scalar trajectory with 
\begin{align*}
\ddt \bx_{\epsilon}(t \mid \bu) = f_{\epsilon}(\bx_{\epsilon}(t \mid \bu),\bu(t)), \quad \bx_{\epsilon}(0 \mid \bu) = \epsilon.
\end{align*}
Then, $\bx_{\epsilon}(t \mid \bzero) = \epsilon$ for all $t$. We can now consider the following planning objective
\begin{align}
\cJ_{T,\epsilon}(\bu) = \frac{1}{2}\int_{0}^T \left(\bx_{\epsilon}(t \mid \bu)^2 + \bu(t)^2\right) \rmd t.
\end{align}
Since the dynamics $f_{\epsilon}$ are affine, we find that 
\begin{align*}
\bu \text{ is an  $\epsilon'$-$\JSP$ of $\cJ_{T,\epsilon}$} \quad \iff \quad \bu \le \inf_{\bu'}\cJ_{T,\epsilon}(\bu') + \epsilon'.
\end{align*}
In particular, as $\cJ_{T,\epsilon}(\bzero) = \frac{1}{2}T \epsilon^2$, and as $\cJ_{T,\epsilon} \ge 0$,
\begin{align}
\bu = 0 \text{ is an $\frac{T\epsilon^2}{2}$-$\JSP$ of }\cJ_{T,\epsilon}.
\end{align}
However, we show that the magnitude of the gradient at $\bu = 0$ is much larger. We compute the following shortly below.
\begin{lemma}\label{lem:grad_bad_example} For $T \ge 1$, we have $\|\nabla \cJ_{T,\epsilon}(\bzero)(t)\|_{\ltwou}  \ge \sqrt{T}\epsilon e^T/4\sqrt{2}$.
\end{lemma}
Thus, the magnitude of the gradient (through open-loop dynamics) is exponentially larger than the suboptimality of the cost. This suggests that gradients through open-loop dynamics are poor proxy for global optimality, motivating instead the $\JSP$. Moreover, one can easily compute that if $\pi$ has inputs $\bukpi = 0$ and stabilizing gains $\bKk = -3$, then for sufficiently small step sizes, the gradients of $\bJ_T^{\pi}(\bu)\big{|}_{\bu = 0}$ scale only as $c\epsilon\sqrt{T}$ for a universal $c > 0$, and do not depend exponentially on the horizon.
\begin{proof}[Proof of \Cref{lem:grad_bad_example}] We have from \Cref{lem:grad_compute_ct_ol} that
\begin{align*}
\nabla \cJ_{T,\epsilon}(\bzero)(t) &= \int_{s=t}^T\underbrace{\bx_{\epsilon}(s \mid \bzero)}_{=\epsilon} \cdot \Phi(s,t)\bB(t), 
\end{align*}
where $\Phi(s,t)$ solves the ODE $\Phi(t,t) = 1$ and $\dds \Phi(s,t) = 2\Phi(s,t)$. Thus, $\Phi(s,t) = \exp(2(t-s))$. Moreover, $\bB(t) = 1$. Hence, 
\begin{align*}
\nabla \cJ_{T,\epsilon}(\bzero)(t) &= \epsilon\int_{s=t}^T\exp(2(t-s))\\
&= \epsilon \frac{1}{2}(e^{2(T-t)}) - 1)
\end{align*}
Hence, for  $t \le T/2$ and $T \ge 1$,
\begin{align*}
|\nabla \cJ_{T,\epsilon}(\bzero)(t)| \ge \epsilon \frac{1}{2}(e^{T} - 1) \ge \epsilon \frac{1}{2}(e^{T} - 1) \ge \frac{\epsilon e^T}{4}.
\end{align*}
Hence,
\begin{align*}
\|\nabla \cJ_{T,\epsilon}(\bu)(t)\|_{\ltwou}^2 &\ge \frac{T}{2}\left(\frac{\epsilon e^T}{4}\right)^2,
\end{align*}
so $\|\nabla \cJ_{T,\epsilon}(\bu)(t)\|_{\ltwou}  \ge \sqrt{T}\epsilon e^T/4\sqrt{2}$.
\end{proof}

\subsection{Global Stability Guarantees of $\JSP$s and Consequences of \cite{westenbroek2021stability}}\label{app:westenbroek_app}

\cite{westenbroek2021stability} demonstrate that, for a certain class of nonlinear systems whose Jacobian Linearizations satisfy various favorable properties, an $\epsilon$-$\FOS$ point $\bu$ of the objective $\cJ_T$ corresponds to a trajectory which converges exponentially to a desired equilibirum. Examining their proof, the first step follows from \citet[Lemma 2]{westenbroek2021stability}, which establishes that $\bu$ is an $\epsilon' = \epsilon^2/2\alpha$-$\JSP$, and it is this property (rather than the $\epsilon$-$\FOS$) that is used throughout the rest of the proof. Hence, their result extends from $\FOS$s to $\JSP$s. Hence, the \emph{local} optimization guarantees established in this work imply, via \citet[Theorem 1]{westenbroek2021stability}, exponentially stabilizing \emph{global} behavior.

\subsection{Projections to ensure boundedness.}\label{sec:feasibility}
\newcommand{\phisys}{\phi_{\mathrm{sys}}} Let us describe one way to ensure the feasibility condition, \Cref{cond:pi_cond}. Suppose that $\fdyn$ has the following stability property, which can be thought of as the state-output anologue of BIBO stability, and is common in the control literature \cite{jadbabaie2001stability}. For example, we may consider the following assumption.
\begin{assumption} There exists some function $\phisys: \R_{\ge 0} \to \R_{\ge 0}$ such that that, if $\|\bu(t)\| \le R$ for all $t \in [0,T]$, then $\|\xol(t \mid \bu)\| \le \phisys(R,T)$ for all $t \in [0,T]$. 
\end{assumption}
Next, fix a bound $R_u > 0$, and set
\begin{align*}
\Rfeas := 2\max\{R_u,\phisys(R_u,T)\}
\end{align*}
Then, it follows that for any policy for  which
\begin{align}\label{eq:bukpi_thing}
\bukpi \le R_u
\end{align} 
for all $k \in [K]$ is feasible in the sense of \Cref{defn:feas}. We therefore modify \Cref{alg:learn_mpc_feedback},\Cref{line:grad_update} to the \emph{projected gradient step}
\begin{align*}
\buknpl[1:K] \gets \mathsf{Proj}_{(\cB_{\dimu}(R_u))^K}\left[\oracpiu[\pin](\btiluk[1:K]^{(n)})\right], \quad \text{ where again } \btiluk^{(n)} := \bukn - \frac{\eta}{\step} \bnabhatkn  - \bKk^{\pin}\bhatxk,
\end{align*}
where we let $\mathsf{Proj}_{(\cB_{\dimu}(R_u))^K}$ denote the orthogonal-projection on the $K$-fold project of $\dimu$-dimensional balls of Euclidean radius $R_u$, $\cB_{\dimu}(R_u)$. This projection is explicitly given by
\begin{align*}
\left(\mathsf{Proj}_{(\cB_{\dimu}(R_u))^K}\left[\buk[1:K])\right]\right)_{k} = \buk \cdot \min\left\{1, \frac{R_u}{\|\buk\|}\right\},
\end{align*}
here using the convention that when $\buk = 0$, the above evaluates to $0$. In this case, our algorithm converges (up to gradient estimation error) to a stationary-point of the projected gradient descent algorithm (see, e.g. the note \url{https://damek.github.io/teaching/orie6300/lec22.pdf} for details). We leave the control-theoretic interpretation of such stationary points to future work. 

\subsection{Extensions to include Process Noise}\label{sec:process_noise}
As explained in \Cref{sec:setting}, \Cref{orac:our_orac} only adds observation noise but not process noice. Process noise somewhat complicates the analysis, because then our method will only learn the Jacobians dynamics up to a noise floor determined by the process noise. However, by generalization our Taylor expansion of the dynamics (e.g. \Cref{prop:taylor_exp_dyn}), we can show that  as the process noise magnitude decreases, we would achieve better and better accuray, recovering the noiseless case in the limit. In addition, process noise may warrant greater algorithmic modifications: for example we may want to incorporate higher-order Taylor expansions of the dynamics (not just the Jacobian linearization), or more sophisticated gradient updates (i.e. $\mathtt{iLQG}$ (\cite{todorov2005generalized})) better tuned to handle process noise. 

\subsection{Discussion of the $\exp(\LF t_0)$ dependence.}\label{sec:discussion_exp_LF} 

There are two sources of the exponential dependence on $t_0 = \step k_0$ that arises in our analysis. First, we translate open-loop controllability (\Cref{asm:ctr}) to closed-loop controllbility needed for recovery of system matrices, in an argument based on \cite{chen2021black}, and which incurs dependent on $\exp(\LF \tcont) \le \exp(\LF t_0)$. Second, we only consider a stability modulus (\Cref{cond:pi_cond}) for a Lyapunov equation terminating at $k = k_0$, because we do not estimate $\bAkpi,\bBkpi$, and therefore cannot synthesize the system gains, for $k \le k_0$. This means that (see \Cref{lem:Kpi_bounds}) that many natural bounds on the discretized transition operators $\|\Phicldisc{k,j}\|$ scale as $\poly(\Kpist,\exp(t_0 \LF)$, yielding exponential dependence on $t_0 \LF$.
\section{Jacobian Linearizations}\label{app:jac_lins}

\subsection{Preliminaries}\label{defn:Jlin_prelim}

Recall $\cU$ denotes the space of continuous-time inputs $\bu:[0,T] \to \R^{\dimu}$, and $\sfU$ continuous-time inputs $\buvec \in (\R^{\dimu})^K$

\subsubsection{Exact Trajectories}
We recall definitions of various trajectories.
\begin{definition}[Open-Loop Trajectories and Nomimal Trajectories]\label{defn:ol_traj} For a $\bu \in \cU$, we define $\xol(t \mid \bu)$ as the curve given by 
\begin{align*}
\ddt \xol(t \mid \bu) = \fdyn(\xol(t \mid \bu),\bu(t)), \quad \xol(0 \mid \bu) = \xiinit.
\end{align*}
For a policy $\pi = (\bukpi[1:K],\bKkpi[1:K])$, we define $\upi = \istep(\bukpi[1:K])$, $\xpi(t) = \xol(t \mid \upi)$, and $\bxkpi = \xpi(t_k)$. 
\end{definition}
Similarly, we present a summary of the definition of various stabilized trajectories, consistent with \Cref{defn:stab_dyn,defn:dt_things}.
\begin{definition}[Stabilized Trajectories]
For $\bbaru \in \cU$ and a policy $\pi$, we define continuous-time perturbations of the dynamics with feedback
\begin{align*}
\btilxpict(t \mid \bbaru) &:= \xol(t \mid \btilupict), \quad \btilupict\,(t \mid \bbaru) := \bbaru(t) + \bKkpi[k(t)] (\btilxpict(\tkt \mid \bbaru)) - \bxkpi[k(t)]),
\end{align*}
there specialization to discrete-time inputs $\buvec \in \sfU $
\begin{align*}
\btilxpi(t \mid \buvec) &:= \btilxpict(t \mid \istep(\bbaru)), \quad \btilupi(t \mid \buvec) := \btilupict(t \mid \istep(\bbaru)),
\end{align*} and their discrete samplings  
\begin{align*}
\btilxpik\discbrak{\buvec} := \btilxpi(t_k \mid \buvec), \quad  \btilupik\discbrak{\buvec} := \btilupi(t_k \mid \buvec)
\end{align*}

\end{definition}

\subsubsection{Trajectory Linearizations}

\begin{definition}[Open-Loop Jacobian Linearizations of Trajectories]\label{defn:ol_jacobians} We define the continuous-time Jacobian linearizations
\begin{equation}
\begin{aligned}
  \xjac(t \mid \bbaru) &:= \xol(t \mid \bu) + \langle \nabla_{\bu}\xol(t \mid \bu)\big{|}_{\bu = \upi}, \bbaru - \upi \rangle_{\ltwou}\\
 &\updelta \xjac(t \mid \bu;\bbaru) := \xjac(t\mid \bu;\bbaru)- \xol(t \mid \bu)
 \label{eq:utraj_eq_a}
 \end{aligned}
\end{equation}
\end{definition}
\begin{definition}[Closed-Loop Jacobian Linearizations of Trajectories, Discrete-Time]\label{defn:Jac_lin_ct}
\begin{equation}
\begin{aligned}
  \btilxpijac(t \mid \bbaru) &:= \btilxpict(t \mid \bu) + \langle \nabla_{\bu}\btilxpict(t \mid \bu)\big{|}_{\bu = \upi}, \bbaru - \upi \rangle_{\ltwou}\\
 \btilupijac(t \mid \bbaru) &:= \btilupict(t \mid \bu) + \langle \nabla_{\bu}\btilupict(t \mid \bu)\big{|}_{\bu = \upi}, \bbaru - \upi \rangle_{\ltwou}
 \label{eq:utraj_eq_a}
 \end{aligned}
\end{equation}
We further define the linearized differences
\begin{equation}
\begin{aligned}
&\updelta \btilxpijac(t \mid \bbaru) := \btilxpijac(t\mid \bbaru)- \xpi(t), \quad \updelta \btilupijac(t \mid \bbaru) := \btilupijac(t\mid \bbaru)- \upi(t)\\
\end{aligned}
\end{equation}
\end{definition}

\newcommand{\btilxpijack}[1][k]{\tilde{\discfont{x}}_{#1}^{\pi,\mathrm{jac}}}
\newcommand{\btilupijack}[1][k]{\tilde{\discfont{u}}_{#1}^{\pi,\mathrm{jac}}}

\begin{definition}[Jacobian Linearization, with gains, dicrete Time]\label{defn:Jac_lin_ct} Given $\buvec \in \sfU$, we define
\begin{equation}
\begin{aligned}
  \btilxpijack(\buvec) &:= \btilxpik(\buvec) + \langle \nabla_{\bu}\btilxpik(t \mid \buvec)\big{|}_{\buvec = \bukpi[1:K]}, \bbaru - \bukpi[1:K] \rangle\\
 \btilupijack(\buvec) &:= \btilupik(\buvec) + \langle \nabla_{\buvec}\btilupik(t \mid \bu)\big{|}_{\buvec = \bukpi[1:K]}, \buvec - \bukpi[1:K] \rangle_{\ltwou}
 \label{eq:utraj_eq_a}
 \end{aligned}
\end{equation}
We further define the linearized differences
\begin{equation}
\begin{aligned}
&\updelta \btilxpijack(\buvec):= \btilxpijack(\buvec) - \bxkpi, \quad \updelta \btilupijack(\buvec):= \btilupijack(\buvec) - \bukpi\\
\end{aligned}
\end{equation}
\end{definition}

\subsubsection{Jacobian Linearizated Dynamics} 

We now recall the definitions of various linearizations, consistent with \Cref{defn:cl_linearizations}.
\begin{definition}[Open-Loop, On-Policy Linearized Dynamics]
We define the open-loop, on-policy linearization around a policy $\pi$ via
\begin{align*}
\Api(t) = \partx \fdyn(\xpi(t),\upi(t)), \quad \Bpi(t) = \partu \fdyn(\xpi(t),\upi(t)).
\end{align*}
\end{definition}

\begin{definition}[Open-Loop, On-Policy Linearized Transition, Markov Operators, and Discrete-Dynamics]\label{defn:open_loop_linearized_matrices} 
We define the linearized transition function $\Phiolpi(s,t)$ defined for $s > t$ as the solution to $\dds \Phiolpi(s,t) = \Api(s)\Phiolpi(s,t)$, with initial condition $ \Phiolpi(t,t) = \eye$.  We \emph{discretize} the open-loop transition function by define  
\begin{align*}
\bAkpi = \Phiolpi(t_{k+1},t_k), \quad \bBkpi := \int_{s=t_k}^{t_{k+1}}\Phiolpi(t_{k+1},s)\Bpi(s)\rmd s.
\end{align*}
\end{definition}

\begin{definition}[Closed-Loop Jacobian Linearization, Discrete-Time]\label{defn:disc_Jac_linz_dyns} We define a \emph{discrete-time closed-loop} linearization 
\begin{align*}
\bAclkpi := \bAkpi + \bBkpi\bKkpi = \Phiolpi(t_{k+1},t) + \int_{s=t}^{t_{k+1}} \Phiolpi(t_{k+1},s)\Bpi(s)\bKkpi,
\end{align*} 
and a discrete closed-loop \emph{transition operator} is defined, for $1 \le k_1 \le k_2 \le K+1$, $\Phicldisc{k_2,k_1} = \bAclkpi[k_2-1]\cdot \bAclkpi[k_2-2]\dots \cdot \bAclkpi[k_1]$, with the convention $\Phicldisc{k_1,k_1} = \eye$. Finally, we define the closed-loop \emph{markov operator} via $\Psicldisc{k_2,k_1} := \Phicldisc{k_2,k_1+1}\bBkpi[k_1]$ for $1 \le k_1 < k_2 \le K+1$.
\end{definition}

\begin{definition}[Closed-Loop Jacobian Linearizations, Continuous-Time]\label{defn:cl_jac_ct} We define
\begin{align*}
\Phiclpi(s,t) := \begin{cases} \Phiolpi(s,t) &  s,t \in \cI_k\\
\Phitilclpi(s,t_{k_2})\cdot \Phicldisc{k_2,k_1} \cdot \Phiolpi(t_{k_1+1},t) & t \in \cI_{k_1}, s \in \cI_{k_2}, k_2 > k_1,
\end{cases}
\end{align*}
where above, we define
\begin{align*}
\Phitilclpi(s,t_{k}) = \Phiolpi(s,t_{k}) + (\int_{s'=t_k}^{s}\Phiolpi(s,s')\Bpi(s')\rmd s)\bKk.
\end{align*}
Lastly, we define
\begin{align*}
\Psiclpi(s,t) = \Phiclpi(s,t)\bB(t).
\end{align*}
\end{definition}


\subsection{Characterizations of the Jacobian Linearizations}\label{sec:jac_chars}

In this section we provide characterizations of the Jacobian Linearizations of the open-loop and closed-loop trajectories.

\begin{lemma}[Implicit Characterization of the linearizations in open-loop]\label{lem:jac_ol_implicit} 
Given $\bu,\bbaru \in \cU$, define $\updelta \bbaru(t) = \bbaru(t) - \bu(t)$. Then,
\begin{align*}
\ddt \updelta\xjac(t \mid \bbaru;\bu) = \bA(t\mid \bu)\updelta\xjac(t \mid \bbaru;\bu) + \bB(t\mid \bu) \updelta \bbaru(t)
\end{align*} 
with initial condition $\updelta\xjac(0 \mid \bbaru) = 0$, where 
\begin{align*}
\bA(t \mid \bu) = \partx \fdyn(x,u) \big{|}_{x = \bx(t \mid \bu), u = \bu(t)} \quad \text{and} \quad  \bB(t \mid \bu) = \partu \fdyn(x,u) \big{|}_{x = \bx(t \mid \bu), u = \bu(t)}.
\end{align*}
\begin{proof}
The result follows directly from \Cref{lem:state_pert} and the definition of $\updelta\xjac(t \mid \bbaru;\bu)$.
\end{proof}

\end{lemma}
\begin{lemma}[Implicit Characterization of the linearizations in closed-loop]\label{lem:jac_cl_implicit} Given a policy $\pi$ and $\bbaru \in \cU$, set $\updelta \bbaru^\pi(t) = \upi(t) -\bu(t)$ . Then, recalling $\updelta\btilxpijac(t \mid \bbaru)  = \btilxpijac(t \mid \bbaru) - \xpi(t) $,
\begin{align*}
\ddt \updelta\btilxpijac(t \mid \bbaru) &= \Api(t)\updelta\btilxpict(t \mid \bbaru) + \Bpi(t)\updelta \btilupijac(t)\\
\updelta \btilupijac &:=  \updelta \bbaru^\pi(t)+ \bKkpi[k(t)]\updelta\btilxpijac(\tkt \mid \bbaru),
\end{align*} 
with initial condition $\updelta\btilxpijac(0 \mid \bbaru) = 0$.
\end{lemma}
\begin{proof}
The result follows directly from \Cref{lem:state_pert} and the definitions of $\updelta\btilxpijac(t \mid \bbaru)$ and the construction of the perturbed input $\updelta \btilupijac$.
\end{proof}

\begin{lemma}[Explicit Characterizations of Linearizations, Continuous-Time]\label{lem:linearizations} For a policy $\pi$, we have:
\begin{align*}
\updelta\xjac(t \mid \upi + \updelta \bu;\upi) = \int_{s=0}^t \Phiolpi(t,s)\Bpi(t)\updelta \bu(s)\rmd s.\\
\updelta\btilxpijac(t \mid \upi + \updelta \bu) = \int_{s=0}^t \Phiclpi(t,s)\Bpi(s)  \updelta\bu(s)\rmd s.
\end{align*}
\end{lemma}
\begin{proof}
The first condition follows directly from the characterization of the evolution of $\updelta\xjac(t \mid \bbaru;\bu)$ and \Cref{lem:state_pert}. For the second condition, we will directly argue that the proposed formula satisfied the differential equation in \Cref{lem:jac_cl_implicit}. By the Leibniz integral rule we have:

\begin{align*}
      \tfrac{\rmd}{\rmd t} \bigg(\int_{s=0}^t \Phiclpi(t,s)\Bpi(s)  \updelta\bu(s)\rmd s \bigg) &= \int_{s=0}^{t}  \tfrac{\rmd}{\rmd t}\bigg(\Phiclpi(t,s) \Bpi(s)  \delu(s)\bigg) \rmd s + \Phiclpi(t,t)\Bpi(t)\delu(t)\\
    & =  \int_{s=0}^{t} \tfrac{\rmd}{\rmd t} \bigg(\Phiclpi(t,s) \Bpi(s)  \delu(s)\bigg) \rmd s + \Bpi(t)\delu(t).
\end{align*}
Using the expression for $\Phiclpi(t,s)$ in \Cref{defn:cl_jac_ct}, we can similarly calculate:

\begin{equation*}
   \tfrac{\rmd}{\rmd t}\Phiclpi(t,s) =
    \Api(t)\Phiclpi(t,s) + \Bpi(t)\bKk \Phiclpi(t_{k(t)},s).
\end{equation*}
Together the precding quatinities demonstrate that:
\begin{align*}
    \tfrac{\rmd}{\rmd t} \bigg(\int_{s=0}^t \Phiclpi(t,s)\Bpi(s)  \updelta\bu(s)\rmd s \bigg) &= \Api(t) \cdot \bigg(\int_{s=0}^t \Phiclpi(t,s)\Bpi(s)  \updelta\bu(s)\rmd s \bigg) + \Bpi(t)\delu(t) \\
    &+\bKkpi[k(t)] \cdot \bigg(\int_{s=0}^{\tkt} \Phiclpi(t,s)\Bpi(s)  \updelta\bu(s)\rmd s \bigg),
\end{align*}
which demonstrates the proposed solutions satisfies the desired differential equation.  
\end{proof}
\begin{lemma}[Explicit Characterizations of Linearizations, Discrete-Time] For a policy $\pi$, and perturbation $\updelta \buk[1:K] \in \bsfU$, 
\begin{align*}
\updelta\btilxpijack(\bukpi[1:K]+\updelta \buk[1:K]) = \sum_{j=1}^{k-1} \Psicldisc{k,j}\updelta \buk[j] = \sum_{j=1}^{k-1} \Phicldisc{k,j+1}\bBkpi[j]\updelta \buk[j]. 
\end{align*}
\end{lemma}
\begin{proof}
The proof follows directly from \Cref{defn:cl_jac_ct}, \Cref{defn:disc_Jac_linz_dyns} and \Cref{lem:linearizations}.
\end{proof}

\subsection{Gradient Computations}

\begin{restatable}[ Computation of Continuous-Time Gradient, Open-Loop]{lemma}{ctgradcomp}\label{lem:grad_compute_ct_ol} Fix $\xi = \xipi$. Define $Q_{(\cdot)}^{\pi}(t) := \partial_{(\cdot)} Q(\xpi(t),\upi(t),t)$. Then, 
\begin{align*}
\nabla \cJ_T( \bu)(t)\big{|}_{\bu = \upi} &=  \Qupi(t) +\int_{s=t}^T \Phiolpi(s,t)\Bpi(t)\rmd s.
\end{align*}
and
\begin{align*}
\langle \nabla \cJ_T( \bu)(t)\big{|}_{\bu = \upi}, \delu \rangle &=  \int_{0}^T (\langle\Qupi(t),\delu(t)) + \langle \Qxpi(t), \updelta \xjac(t \mid \bu + \delu)\rangle)\rmd t.
\end{align*}
\begin{proof}
For a given perturbation $\delu$, by the chain rule we have:
\begin{align*}
\rmD \cJ_T( \bu)[\delu] = \int_{0}^T (\langle\Qupi(t),\delu(t)) + \langle \Qxpi(t), \updelta \xjac(t \mid \bu + \delu)\rangle)\rmd t + \langle \partx V(\xpi(T)), \updelta \xjac(T \mid \bu + \delu)\rangle
\end{align*}
Because $\delu$ is arbitrary, an application of \Cref{lem:linearizations} demonstrates the desired results. 
\end{proof}
\end{restatable}

\begin{restatable}[Computation of Continuous-Time Gradient, Closed-Loop]{lemma}{ctgradcomp}\label{lem:grad_compute_ct} Fix $\xi = \xipi$. Define $Q_{(\cdot)}^{\pi}(t) := \partial_{(\cdot)} Q(\xpi(t),\upi(t),t)$. Then, 
\begin{align*}
\nabla \cJ_T( \pi)(t) :=  \nabla_{\bbaru} \Jpi(\bbaru)\big{|}_{\bbaru = \upi}(t) &=  \Qupi(t) + \Psiclpi(T,t)^\top\left(\partx V(\xpi(T))\right)\\
&\quad+\int_{s=t}^T \Psiclpi(s,t)^\top  \Qxpi(s)\rmd s  + \int_{s =t_{k(t)+1}}^T\Psiclpi(t_{k(s)},t)^\top \bKk[k(s)]\top \Qupi(s) \rmd s.
\end{align*}
\end{restatable}
\begin{proof} 
The proof follows the steps of \Cref{lem:grad_compute_ct_ol}, but replaces the open-loop state and input perturbations with the appropriate closed-loop perturbations, as defined in \Cref{lem:linearizations} and calculated in \Cref{lem:jac_cl_implicit}.
\end{proof}

\newcommand{\delxtilkjac}[1][k]{\updelta \btilxk[#1]^{\mathrm{jac}}}

Similarly, we can compute the gradient of the discrete-time objective. Its proof is analogous to the previous two.
\begin{restatable}[Computation of Discrete-Time Gradient]{lemma}{dtgradcomp}\label{lem:grad_compute_dt}
\begin{align*}
(\nabla \Jdisc(\pi))_k &=  \step Q_u( \bxkpi, \bukpi, t_k)  + (\Psicldisc{K+1,k})^\top V_x(\bxkpi[K+1])  \\
&+ \step\sum_{j=k+1}^K (\Psicldisc{j,k})^\top(Q_x( \bxkpi[j], \bukpi[j], t_j) + (\bKkpi[j])^\top Q_u( \bxkpi[j], \bukpi[j], t_j))
\end{align*}
Moreorever, defining the shorthand $ \delxtilkjac = \updelta \btilxpijack(\bukpi[1:K] + \deluk[1:K])$, 
\begin{align*}
&\langle \deluk[1:K] ,  \nabla \Jpidisc(\pi) \rangle \\
&= \langle \partx V (\bxkpi[K+1]), \delxtilkjac \rangle + \step \sum_{k=1}^K\langle\partx Q( \bxkpi, \bukpi, t_k), \delxtilkjac\rangle +\langle \partu Q( \bxkpi, \bukpi, t_k), \deluk + \bKk \delxtilkjac \rangle.
\end{align*}
\end{restatable}

\subsection{Technical Tools}
 The first supportive lemma is a standard result from variational calculus, and characterizes how the solution to the controlled differential equation changes under perturbations to the input. Note that this result does not depend on how the input is generated, namely, whether the perturbation is generated in open-loop or closed-loop. Concretely, the statement of the following result is equivalent to Theorem 5.6.9 from \cite{polak2012optimization}. 
 
\begin{lemma}[State Variation of Controlled CT Systems]\label{lem:state_pert} For each nominal input $\bu \in \cU$ and perturbation $\delu \in \cU$ we have:
\begin{equation}
    \langle \nabla_{\bu}\bx(t \mid \bu) ,  \delu \rangle  = \delx(t),
\end{equation}
where the curve $\delx(\cdot)$ satisfies:
\begin{equation}
\frac{d}{dt} \delx(t) = \bA(t \mid \bu)\delx(t) + \bB(t \mid \bu)\delu(t),
\end{equation}
where we recall that:
\begin{align*}
\bA(t \mid \bu) = \partx \fdyn(x,u) \big{|}_{x = \bx(t \mid \bu), u = \bu(t)} \quad \text{and} \quad  \bB(t \mid \bu) = \partu \fdyn(x,u) \big{|}_{x = \bx(t \mid \bu), u = \bu(t)}.
\end{align*}
Moreover, we have 
\begin{equation*}
    \delx(t) = \int_{s=0}^t \bPhi(t,s)\bB(s \mid \bu)\delu(s),
\end{equation*}
where the transition operator satisfies $\frac{d}{dt} \bPhi(t,s) =  \bA(t \mid \bu)  \bPhi$ and $\bPhi(s,s) = \eye$.
\end{lemma}

The following result is equivalent to Lemma 5.6.2 from \cite{polak2012optimization}.

\begin{lemma}[Picard Lemma]\label{lem:picard} 
Consider two dynamical $\ddt y(t) = \phi(y(t),t)$, $i \in \{1,2\}$, and suppose that $y\mapsto \phi(y,s)$ is  $L$-Lipschitz fo each $s$ fixed. Let $z(t)$ be any other absolutely continuous curve. Then, 
\begin{align*}
\|y(t) - z(t)\| \le \exp(t L)\cdot\left(\|y(0) - z(0)\| + \int_{s=0}^t\|\dds z(s) - f(z(s),s)\|\right).
\end{align*}
\end{lemma}

\begin{lemma}[Solution to Afine ODEs]\label{lem:solve_affine_ode} Consider an affine ODE given by $\by(0) = y_0$,  $\ddt\by(t) = \bA(t)\by(t) + \bB(t)u$. Then, 
\begin{align*}
y(t) = \Phi(t,0)y_0 + \left(\int \Phi(t,s)\bB(s)\rmd s\right)u,
\end{align*}
where $\Phi(t,s)$ solves the ODE $\Phi(s,s) = \eye$ and $\ddt \Phi(t,s) = \bA(t)\Phi(t,s)$.
\end{lemma}

\newcommand{\bbaruk}[1][k]{\bar{\discfont{u}}_{#1}}

\newcommand{\ejac}{\be^{\mathtt{jac}}}
\newcommand{\ejack}[1][k]{\discfont{e}_{#1}^{\mathtt{jac}}}
\newcommand{\bchyjack}{\check{\by}^{\mathtt{jac}}_k}
\newcommand{\byjac}{\by^{\mathtt{jac}}}
\newcommand{\bynom}{\by^{\mathtt{nom}}}
\newcommand{\byjack}[1][k]{\discfont{y}^{\mathtt{jac}}_{#1}}

\newcommand{\bAclkpipr}[1][k]{\discfont{A}^{\pi'}_{\mathrm{cl},#1}}

\newpage
\section{Taylor Expansions of the Dynamics}\label{app:taylor_exp}

\subsection{Proof of \Cref{prop:taylor_exp_dyn}}\label{sec:prop:taylor_exp_dyn}

Recall $\deluk = \buk - \bukpi$, and define $\bu = \istep(\buk[1:K])$ and $\delu := \istep(\bukpi[1:K])$.  We define shorthand for relevant continuous curves and their discretizations:
\begin{equation}
\begin{aligned}
\by(t) &= \btilxpi(t \mid \upi + \delu) = \btilxpi(t \mid \upi) \\
\byk &= \by(t_k) = \btilxpik\,(\bukpi[1:K] + \deluk[1:K] ) = \btilxpik\,(\buk[1:K])\\
\byjac(t) &= \btilxpijac(t \mid \upi + \delu), \quad \byjack = \byjac(t_k)
\end{aligned}
\end{equation}
We also define their differences from the nominal as 
\begin{align*}
\updelta \by(t) = \by(t) - \xpi(t), \quad \updelta \byjac(t) = \byjac(t) - \xpi(t), \quad \updelta \byk = \byk - \bxkpi, \quad \updelta \byjack = \byjack - \bxkpi.
\end{align*}
And the Jacobian error
\begin{align*}
\ejac(t) := \by(t) - \byjac(t), \quad \ejack := \ejac(t_k) = \byk - \byjack.
\end{align*}
The main challenge is recursively controlling $\|\ejack\|$. We begin with a computation which is immediate from \Cref{defn:ol_traj} (the first equality) and {} (the second equality):
\begin{lemma}[Curve Computations]\label{lem:y_curve_computations} For $t \in \cI_k$,
\begin{align*}
\ddt \by(t) &= \fdyn(\by(t), \bukpi + \deluk + \bKkpi(\updelta \byk))\\
\ddt \byjac(t) &= \ddt \xpi(t)  + \Api(t)\updelta \byjac(t) + \Bpi(t)\left(\deluk + \bKk(\updelta \byjack)\right)
\end{align*}
\end{lemma}

\paragraph{Computing the Jacobian Linearization.} The first step of the proof is a computation of the Jacobian linearization and a bound on its magnitude.
\begin{lemma}[Computation of Jacobian Linearization]\label{lem:JL_comp_stuff}
\begin{align}
\byjack = \sum_{j=1}^{k-1}\Phicldisc{k,j+1}\bBkpi[j]\deluk[j]  = \sum_{j=1}^{k-1}\Psicldisc{k,j}\deluk[j]\label{eq:yjack_decomp}
\end{align}
Therefore, 
\begin{align}
\max_{k \in [K+1]}\|\byjack \| \le \Lol\LF\min\{\Btwo\Kpitwo, \Binf\Kpione\} \label{eq:yjack_bound}
\end{align}
\end{lemma}
\begin{proof} \Cref{eq:yjack_decomp} follows from \Cref{lem:grad_compute_ct}. \Cref{eq:yjack_bound} follows from Cauchy Schwartz/Holder's inequality, and the bound $\|\bBkpi[j]\| \le \step\Lol\LF$ due to \Cref{lem:bound_on_bBkpi}.
\end{proof}
\paragraph{Recursion on proximity to Jacobian linearization.} Next, we argue that the true dynamics $\by(t)$ remain close to $\byjac(t)$. 
We establish a recursion under the following invariant:
\begin{equation}
\begin{aligned}\label{eq:feasability_invariant}
&\|\byjack\| \vee \|\byk\| \vee \|\bukpi + \deluk + \bKkpi(\updelta \byk)\| \vee \|\bukpi + \deluk + \bKkpi(\updelta \byjack)\| \le \frac{3}{4}\Rfeas\\
&\step \le \min\left\{\frac{1}{16\LF \Lpi},\frac{1}{8\KF}\right\}
\end{aligned}
\end{equation}

We prove the following recursion:
\begin{lemma}[Recursion on Error of Linearization]\label{lem:ejack_recurse} Suppose\Cref{eq:feasability_invariant} holds. Let $\tilde{\bPhi}_{\mathrm{cl},\pi}(t,t_k) = \Phiolpi(t,t_k)+(\int_{s=t_k}^t  \Phiolpi(t,s)\Bpi(s)\rmd s)\bKkpi$. Then, the following bound holds:
\begin{align*}
\sup_{t \in \cK}\|\ejac(t) - \tilde{\bPhi}_{\mathrm{cl},\pi}(t,t_k) \ejack[k]\| \le \MF\step \left(4\|\deluk\|^2 + 20\Lpi^2 \|\ejack[k]\|^2 +20\Lpi^2  \|\byjack - \byk\|^2 \right).
\end{align*}
In particular, we have
\begin{align*}
\|\ejack[k+1] - \bAclkpi \ejack[k]\| \le \MF\step \left(4\|\deluk\|^2 + 20\Lpi^2 \|\ejack[k]\|^2 +20\Lpi^2  \|\byjack - \byk\|^2 \right).
\end{align*}
\end{lemma}
To do prove \Cref{lem:ejack_recurse}, we introduce another family of curves $\bchyjack(t)$, defined for $t \ge t_k$, which begin at $\by(t_k)$ but evolve according to the Jacobian linearization:
\begin{align*}
\ddt \bchyjack(t) &= \ddt \xpi(t)  + \Api(t) \updelta\bchyjack(t) + \Bpi(t)\left(\deluk + \bKk(\updelta \byk)\right)\\
\bchyjack(t_k) &= \by(t_k) = \byk, \quad \updelta \byjack(t) =  \byjack(t) - \xpi(t).
\end{align*}
We begin by establishing feasibility of all relevant continuous-time curves on the interval $\cI_k$.
\begin{lemma}\label{lem:feasibility_maintained} Suppose that \Cref{eq:feasability_invariant} holds. Then, for all $t \in \cI_k$, 
\begin{align*}
\|\bchyjack(t)\| \vee \|\by(t)\| \vee \|\byjac(t)\| \le \Rfeas.
\end{align*}
\end{lemma}
\begin{proof} Let us start with $\by(t)$. Define the shorthand $\bbaruk := \bukpi + \deluk + \bKkpi(\updelta \byk)$, so that $\ddt \by(t) = \fdyn(\by(t),\buk)$. Under \Cref{eq:feasability_invariant}, $\|\bbaruk\| \le \frac{3}{4}\Rfeas$. At $t = t_k$, $\|\by(t)\| \le \frac{3}{4}\Rfeas$. Moreover, if at a given $t$, $\|\by(t)\| \le \Rfeas$, then $(\by(t),\buk)$ is feasible, so $\|\ddt \by(t)\| = \|\fdyn(\by(t),\bbaruk)\| \le \KF$. Thus, letting $t_\star := \sup\{t \in \cI_K: \|\by(t)\| \le \Rfeas\}$,  we see that if $\step \le \frac{\Rfeas}{4\KF}$, $t_{\star} = t_{k+1}$. 

The arguments for $\bchyjack(t)$ and $\byjac(t)$ are similar: if, say, $\|\bchyjack(t)\| \le \Rfeas$ for a given $t \in \cI_k$, then by 
\begin{align*}
\|\ddt \bchyjack(t)\| &= \|\ddt \xpi(t)  + \Api(t) \updelta\bchyjack(t) + \Bpi(t)\buk\| \\
&\le \KF + \LF(\|\updelta\bchyjack(t)\| + \|\buk\|) \\
&\le \KF + 2\Rfeas\LF.
\end{align*}
where above we use \Cref{eq:feasability_invariant} to bound $\|\buk\| \le \Rfeas$, feasibility of $\pi$. As $\|\updelta\bchyjack(t)\| \le \frac{3}{4}\Rfeas$, integrating (specifically, again considering $t_\star := \sup\{t \in \cI_K: \|\bchyjack(t)\| \le \Rfeas\}$) shows that as long as $\step(\KF + 2\LF\Rfeas) \le \Rfeas/4$, $\|\bchyjack(t)\| \le \Rfeas$ for all $t \in \cI_k$. For this, it suffices that $\step \le \min\{\frac{1}{16\LF},1/8\KF\}$.
\end{proof}

We continue with a crude bound on the difference  $\updelta \bchyjack(t) =  \bchyjack(t) - \xpi(t)$.
\begin{lemma}\label{lem:first_taylor_pert} Suppose \Cref{eq:feasability_invariant} holds all $t \in \cI_k$
\begin{align*}
\|\updelta \bchyjack(t)\| \le  (\step\Lol\LF\|\deluk\| + \Lol(1+\step\Lpi \LF)\|\updelta \byk\|)
\end{align*}
Similarly,
\begin{align*}
\|\updelta \by(t)\| \le  (\step\Lol\LF\|\deluk\| + \Lol(1+\step\Lpi \LF)\|\updelta \byk\|).
\end{align*}
\end{lemma}
\begin{proof}
Then, Picard's Lemma (\Cref{lem:picard}), feasibiliy of $\pi$  and \Cref{asm:max_dyn_final} implies that, for any $t \in \cI_k$
\begin{align*}
&\|\updelta \bchyjack(t)\| \le \exp((t-t_k) \LF)\epsilon_1
\end{align*}
where 
\begin{align*}
\epsilon_1 &:= \|\updelta\byk\| + \int_{s=t_k}^{t} \|\Bpi(t)\left(\deluk + \bKk(\updelta \byk)\right)\|\rmd s\\
&\le \|\updelta\byk\| + \int_{s=t_k}^{t_{k+1}}  \LF(\|\deluk\| + \|\bKkpi \updelta \byk\|) \tag{\Cref{asm:max_dyn_final}}\\
&\le \step (\LF\|\deluk\|\|\deluk\| + \Lpi \LF\|\updelta \byk\|) \tag{\Cref{cond:pi_cond}},
\end{align*}
Bounding $\exp((t-t_k) \LF) \le \exp(\step \LF) = \Lol$ concludes the first part. The second part follows from a similar argument, using Lipschitzness of $\fdyn$ in accordance with \Cref{asm:max_dyn_final}, and the feasibility of $(\by(t),\bukpi + \deluk + \bKkpi(\updelta \byjack))$ for $t \in \cI_k$, as ensured by \Cref{eq:feasability_invariant} and \Cref{lem:feasibility_maintained}.
\end{proof}
We are now ready to prove \Cref{lem:ejack_recurse}.
\begin{proof}[Proof of \Cref{lem:ejack_recurse}]
Observe that, for $t \in \cI_k$
\begin{align*}
\ddt (\bchyjack(t) - \byjac(t)) = \Api(t)(\bchyjack(t) - \byjac(t)) + \Bpi(t)\bKkpi(\bchyjack(t_k) - \byjac(t_k))
\end{align*}
By solving the affine ODE define $\bchyjack(t) - \byjac(t)$ (applying \Cref{lem:solve_affine_ode}), and recalling all various defintions,
\begin{align}
\bchyjack(t) - \byjac(t) = \tilde{\bPhi}_{\mathrm{cl},\pi}(t,t_k)(\by(t_k) - \byjac(t_k)).\label{eq:bchyjack_sum}
\end{align}
We now bound $\bchyjack(t) - \by(t)$.
By applying Picard's Lemma (\Cref{lem:picard})  and \Cref{asm:max_dyn_final} with $\Lol = \exp(\step \LF)$ to control the Lipschitz constant contribution, and using the agreement of initial conditions $\bchyjack(t_k) = \by(t_k)$, 
\begin{align}
\|\bchyjack(t) - \by(t)\| \le \Lol\int_{s=t_k}^t \|\Delta(s)\|\rmd s. \label{eq:bchyjack_bound}
\end{align}
where
\begin{align*}
\Delta(s) &= \fdyn(\bchyjack(s), \bukpi + \deluk + \bKkpi \updelta \byk)-\dds \bchyjack(s).
\end{align*}
By a Taylor expansion, we have
\begin{align*}
\Delta(s) &= \fdyn(\xpi(s),\bukpi) -\dds \bchyjack(s) + \partial_x \fdyn(\xpi(s),\bukpi)\updelta \bchyjack(s) + \partial_u \fdyn(\xpi(s),\bukpi)(\deluk + \bKkpi \updelta \byk)\\
&\quad + \frac{1}{2}\mathrm{remainder},
\end{align*}
where we bound
\begin{align*}
\|\mathrm{remainder}(s)\| 
&\le \sup_{\alpha \in [0,1]}\|\nablatwo \fdyn(\alpha \xpi(s) + (1-\alpha)\byjack(s), \alpha \bukpi + (1-\alpha)(\bukpi + \deluk + \bKkpi \updelta \byk) ) \left(\| \updelta \byjack(s)\|^2 + \|\deluk + \bKkpi\updelta\byk\|^2\right).
\end{align*}
From by feasibility of $\pi$, $\|\bukpi\| \vee \|\xpi(s)\| \le \Rfeas$. Moreover, $\|\bukpi +  \deluk + \bKkpi \updelta \byk\| \le \Rfeas$ by \Cref{eq:feasability_invariant} and $\|\byjack(s)\| \le \Rfeas$ by \Cref{lem:feasibility_maintained}. Thus, for $\alpha \in [0,1]$,
\begin{align*}
\|\alpha \xpi(s) + (1-\alpha)\byjack(s)\| \vee \|\alpha \bukpi + (1-\alpha)(\bukpi + \deluk + \bKkpi \updelta \byk)\| \le \Rfeas.
\end{align*}
Hence, as $\|\nabla^{\,2}\fdyn(x,u)\| \le \MF$ for feasible $(x,u)$, \Cref{asm:max_dyn_final} implies
\begin{align*}
\|\mathrm{remainder}(s)\|  &\le \MF \left(\| \updelta\byjac(s)\|^2 +  2\|\deluk\|^2  + 2\Lpi^2\|\updelta\byk\|^2\right) \tag{AM-GM and \Cref{cond:pi_cond}}\\
&\le \MF \left(\| (\Lol \LF \step\|\deluk\| + \Lol(1+\step\Lpi \LF)\|\updelta \byk\|)\|^2 +  2\|\deluk\|^2  + 2\Lpi^2\|\updelta\byk\|^2\right) \tag{A\Cref{lem:first_taylor_pert}}\\
&\le 2\MF \left( (\step^2\Lol^2\LF^2\|\deluk\|^2) + \Lol^2(1+\step\LF\Lpi)^2 \|\updelta\byk\|^2 +  \|\deluk\|^2  + \Lpi^2\|\updelta\byk\|^2\right)\tag{AM-GM}\\
&= 2\MF \left( (1+\step^2\Lol^2\LF^2)\|\deluk\|^2) + (\Lpi^2 +\Lol^2(1+\step\LF\Lpi)^2) \|\updelta\byk\|^2 \right).
\end{align*}
Finally, we conclude by noting that
\begin{align*}
&\fdyn(\xpi(s),\bukpi)  \\
&\qquad+ \partial_x \fdyn(\xpi(s),\bukpi)\updelta \bchyjack(s) + \partial_u \fdyn(\xpi(s),\bukpi)(\deluk + \bKkpi \updelta \byk)\\
&=\ddt \xpi(s)  + \Api(s)\updelta \bchyjack(s) + \Bpi(s) \fdyn(\xpi(s),\bukpi)(\deluk + \bKkpi \updelta \byk) = \dds \bchyjack(s),
\end{align*}
so that
\begin{align*}
\|\Delta(s)\| = \frac{1}{2}\|\mathrm{remainder}(s)\| &\le \MF \left( (1+\step^2\Lol^2\LF^2)\|\deluk\|^2) + (\Lpi^2 +\Lol^2(1+\step\LF\Lpi)^2) \|\updelta\byk\|^2 \right)\\
&\le \MF \left( (1+2\step^2\LF^2)\|\deluk\|^2) + (\Lpi^2 +2(1+\step\LF\Lpi)^2) \|\updelta\byk\|^2 \right) \tag{$\step \le 1/4\LF$, so $\Lol^2 \le 2$}\\
&\le \MF \left( 2\|\deluk\|^2) + (\Lpi^2 + 4) \|\updelta\byk\|^2 \right), \tag{again $\step \le 1/4\LF$, and when as $\step \le 4/\LF\Lpi$}\\
&\le \MF \left( 2\|\deluk\|^2) + 5\Lpi^2 \|\updelta\byk\|^2 \right) \tag{$\Lpi \ge 1$}
\end{align*}
where in the last line, we use $\step \le 1/4\LF$, so $\Lol^2 \le 2$.
Hence, from \Cref{eq:bchyjack_bound}, for all $t \in \cI_k$,
\begin{align*}
\|\bchyjack(t) - \by(t)\| &\le  \step\Lol\MF \MF \left( 2\|\deluk\|^2) + (\Lpi^2 + 4) \|\updelta\byk\|^2 \right). \\
&\le 2\MF\step   \left( 2\|\deluk\|^2) + (\Lpi^2 + 4) \|\updelta\byk\|^2 \right)\\
&= \MF\step   \left( 4\|\deluk\|^2) + 10\Lpi^2 \|\updelta\byk\|^2 \right).
\end{align*}
where above we bound $\Lol \le 2$ again.
And thus, from \Cref{eq:bchyjack_sum},
\begin{align*}
\|\by(t) - \byjac(t) - \tilde{\bPhi}_{\mathrm{cl},\pi}(t,t_k)(\by(t_k) - \byjac(t_k))\| &\le \MF\step   \left( 4\|\deluk\|^2) + 10\Lpi^2 \|\updelta\byk\|^2 \right).,\\
&\le \MF\step   \left( 4\|\deluk\|^2) + 10\Lpi^2 \|\|\byjack - \byk\|^2\| \right)\\
&\le 2\MF\step \left(4\|\deluk\|^2 + 20\Lpi^2 \|\updelta\byjack\|^2 +20\Lpi^2  \|\byjack - \byk\|^2 \right).
\end{align*}
Substituing in $\ejack := \byk - \byjack$ concludes.
\end{proof}

\paragraph{Solving the recursion. } To upper bound the recursion, assume an inductive hypothesis that, for some $R$ to be chosen 
\begin{align}
\max_{j \le k} \|\ejack\| \le R, \quad \text{ and } \quad \forall j \le k,~ \Cref{eq:feasability_invariant} \text{ holds}. \label{eq:recur_hypoth}
\end{align}
Note this hypothesis is true for $k = 1$, where $\ejack[1] = 0$, and all terms in \Cref{eq:feasability_invariant} coincide with $(\bxkpi[1],\bukpi[1])$, which is feasible. Now assume \Cref{eq:recur_hypoth} holds. Define
\begin{align*}
\bvk := \ejack[k+1] - \bAclkpi \ejack[k],
\end{align*}
and note that \Cref{lem:ejack_recurse}, followed by our induction hypothesis, implies that for $c_1 = 4$ and $c_2 = 20\Lpi^2$,
\begin{align*}
\|\bvk\| \le \step \MF\left(4\|\deluk\|^2 + 20\Lpi^2 \|\ejack[k]\|^2 + c_2\|\byjack\|^2\right) \\
\le \step \MF\left(4\|\deluk\|^2 + 20\Lpi^2 R^2 + 20\Lpi^2\|\byjack\|^2\right).
\end{align*}
By unfolding the recursion for $\bvk := \ejack[k+1] - \bAclkpi \ejack[k]$, we have 
\begin{align*}
\ejack[k+1] &= \bvk + \bAclkpi \ejack[k]\\
&= \underbrace{\eye}_{=\Phicldisc{k+1,k+1}}\bvk + \underbrace{\bAclkpi}_{=\Phicldisc{k+1,k}} \bvk[k-1]+ \underbrace{\bAclkpi\bAclkpi[k-1]}_{=\Phicldisc{k+1,k-1}} \ejack[k-1]\\
&= \sum_{j=1}^{k+1} \Phicldisc{k+1,j} \bvk + \underbrace{\Phicldisc{k+1,1} \ejack[1]}_{=0}.
\end{align*}
Thus, under our inductive hypothesis, recalling $B_2 := \step\|\deluk[1:K]\|_{\ell_2}^2$, and $B_{\infty} := \step  \max_k\|\deluk\|^2$, 
\begin{align*}
\|\ejack[k+1]\| &\le \MF\sum_{j=1}^{k+1} \|\Phicldisc{k+1,j} \| \step (c_1 \|\deluk\|^2 +  c_2 R^2 + c_2\|\byjack\|^2)\\
&\le 4\MF \min\{\Kpiinf B_2^2,  \Kpione B_{\infty}^2\} + 20\MF\Lpi^2\Kpione (R^2 + \max_k \|\byjack\|^2)\\
&\le 4\MF\min\{\Kpiinf B_2^2,  \Kpione B_{\infty}^2\} + 20\MF\Lpi^2\Kpione(R^2 + \Lol^2\LF^2\min\{\Btwo\Kpitwo, \Binf\Kpione\}^2 ) \tag{\Cref{lem:JL_comp_stuff}}\\
&\le 4\MF\min\{B_2^2 ( \Kpiinf + 5 \Lpi^2\Lol^2\LF^2\Kpitwo^2\Kpione), B_{\infty}^2 (\Kpione + 5\Lpi^2\Lol^2\LF^2\Kpione^3)\} + 20\MF \Lpi^2 \Kpione R^2.\\
&\le 4\MF\min\{B_2^2 ( \Kpiinf + 10\Lpi^2\LF^2\Kpitwo^2\Kpione), B_{\infty}^2 (\Kpione + 10\Lpi^2\LF^2\Kpione^3)\} + 20\MF \Lpi^2 \Kpione R^2,
\end{align*} 
where in the last step we use $\step \le 1/4\LF$ to bound $\Lol^2 \le 2$.  Hence, if we select
\begin{align*}
R &= 8\MF\min\{B_2^2 ( \Kpiinf + 10\Lpi^2\LF^2\Kpitwo^2\Kpione), B_{\infty}^2 (\Kpione + 10\Lpi^2\LF^2\Kpione^3)\}\\
&:= \min\{B_2^2\Mtaypitwo, B_{\infty}^2\Mtaypiinf\},
\end{align*}
where we recall
\begin{align*}
\Mtaypitwo &:= 8\MF( \Kpiinf + 10\Lpi^2\LF^2\Kpitwo^2\Kpione)\\
\Mtaypiinf &:= 8\MF(\Kpione + 10\Lpi^2\LF^2\Kpione^3,
\end{align*}
we get
\begin{align*}
\|\ejack[k+1]\| &\le \frac{R}{2} + 20\MF \Lpi^2 \Kpione R^2.
\end{align*}

Thus, if $R \le \frac{1}{40\MF \Lpi^2 \Kpione}$, it holds that 
\begin{align*}
\|\ejack[k+1]\| \le \min\{B_2^2\Mtaypitwo, B_{\infty}^2\Mtaypiinf\}.
\end{align*} 
Lastly, for the condition $R \le \frac{1}{40\MF\Lpi^2\Kpione}$ to hold, it suffices
\begin{align}
B_2^2 \le \frac{1}{40\MF\Lpi^2\Kpione\Mtaypitwo} , \quad \text{ or } B_{\infty}^2 \le \frac{1}{40\Lpi^2\Kpione\Mtaypiinf} \label{eq:prev_Bq_cond}
\end{align}
Notince that these conditions are met for $B_q^2 \le \Btaypiq^2$. Moreover,
\begin{align*}
\| \byk[k+1] - \bxkpi[k+1]\| \vee \|\byjack[k+1]-\bxkpi[k+1]\| &\le \|\byk[k+1] - \byjack[k+1]\| + \|\byjack[k+1]-\bxkpi[k+1]\\
&= \| \ejack[k+1] \|  + \|\sum_{j=1}^{k}\Phicldisc{k+1,j+1}\bBkpi[j]\deluk[j]\| \\
&\le \min_{q \in \{2,\infty\}}\Mtaypiq B_q^2  + \|\sum_{j=1}^{k-1}\Phicldisc{k+1,j+1}\bBkpi[j]\deluk[j]\| \\
&\le \min_{q \in \{2,\infty\}}\Mtaypiq B_q^2 +  \Lol\LF\min\{\Btwo\Kpitwo, \Binf\Kpione\} \tag{\Cref{lem:ejack_recurse}}\\
&\le  \min\{\Btwo(\Lol\LF\Kpitwo + \Mtaypitwo B_2), \Binf(\Lol\LF\Kpione + \Mtaypiinf \Binf)\}\\
&\le  \min\{\Btwo(1.5\LF\Kpitwo + \Mtaypitwo B_2), \Binf(1.5\LF\Kpione + \Mtaypiinf \Binf)\} \tag{$\Lol \le \exp(1/4) \le 1.5$}
\end{align*}
Hence, $B_2 \le \frac{1}{2}\LF\Kpitwo/\Mtaypitwo$ implies $\| \byk[k+1] - \bxkpi[k+1]\| \vee \|\byjack[k+1]-\bxkpi[k+1]\|  \le 2\LF\Kpitwo B_2 = \Ltaypitwo B_2$, and $B_\infty \le \frac{1}{2}\LF\Kpione/\Mtaypiinf $ implies  $\| \byk[k+1] - \bxkpi[k+1]\| \vee \|\byjack[k+1]-\bxkpi[k+1]\|  \le 2\LF\Kpione B_\infty = \Ltaypiinf B_2$.
Combinining these conditions with \Cref{eq:prev_Bq_cond} implies we require  $B_q \le \Btaypiq$, where
\begin{align*}
\Btaypitwo &= \min\left\{\frac{1}{\sqrt{40\MF\Lpi^2\Kpione\Mtaypitwo}} ,\frac{\LF\Kpitwo}{2\Mtaypitwo}\right\} \\
\Btaypiinf &\le \min\left\{\frac{1}{40\Lpi^2\Kpione\Mtaypiinf}, \frac{\LF\Kpione}{2\Mtaypiinf}\right\}
\end{align*}

Lastly, we need to check that the feasibility invariant \Cref{eq:feasability_invariant} for $k= k+1$ is maintained. Under the above conditions, it was shown that 
\begin{align*}
\|\updelta \byk[k+1] \| \vee \|\updelta \byjack[k+1]\| = \| \byk[k+1] - \bxkpi[k+1]\| \vee \|\byjack[k+1]-\bxkpi[k+1]\| \le \Ltaypiq B_q.
\end{align*}
Hence, if $B_q \le \frac{\Rfeas}{8\Lpi\Ltaypiq} \le \frac{\Rfeas}{4\Ltaypiq}$ for either $q \in \{2,\infty\}$,
\begin{align*}
\|\byjack[k+1]\| \vee \|\byk[k+1]\| &\le \|\bxkpi[k+1]\| + \Ltaypiq B_q \le \frac{\Rfeas}{2} + \Ltaypiq B_2 \le \frac{3}{4}\Rfeas
\end{align*}
Moreover, if  $B_q \le \frac{\Rfeas}{8\Lpi\Ltaypiq}$ for either $q \in \{2,\infty\}$, and $B_{\infty} \le \frac{\Rfeas}{8}$,
\begin{align*}
&\|\bukpi[k+1] + \deluk[k+1] + \bKkpi[k+1](\updelta \byk[k+1])\| \vee \|\bukpi[k+1] + \deluk[k+1] + \bKkpi[k+1](\updelta \byjack[k+1])\|\\
 &\le \|\bukpi[k+1]\| + \|\deluk[k+1]\| \|\bKkpi[k+1]\|(\|\updelta \byk[k+1]\| \vee \|\updelta \byjack[k+1]\|)\\
&\le \frac{\Rfeas}{2} + \|\deluk[k+1]\| + \Lpi \Ltaypiq B_2 \le \frac{\Rfeas}{2} + B_{\infty} + \Lpi \Ltaypiq B_2 \le \frac{3\Rfeas}{4}.
\end{align*}
This concludes the demonstration of  \Cref{eq:feasability_invariant} for $k= k+1$. Collecting our conditions, and recalling $\Ltaypitwo = 2\LF\Kpitwo$, and $\Ltaypiinf = 2\LF\Kpione$ we have show that if we take $B_q \le \Btaypiq$, where
\begin{align*}
\Btaypitwo &= \min\left\{\frac{1}{\sqrt{40\MF\Lpi^2\Kpione\Mtaypitwo}} ,\frac{\LF\Kpitwo}{2\Mtaypitwo},\frac{\Rfeas}{16\Lpi\LF\Kpitwo}\right\} \\
\Btaypiinf &= \min\left\{\frac{1}{40\Lpi^2\Kpione\Mtaypiinf}, \frac{\LF\Kpione}{2\Mtaypiinf}, \frac{\Rfeas}{16\Lpi\LF\Kpione}\right\} 
\end{align*}
and $B_{\infty} \le \frac{\Rfeas}{8}$, then 
\begin{align*}
\| \byk[k+1] - \bxkpi[k+1]\| \vee \|\byjack[k+1]-\bxkpi[k+1]\| \le \Ltaypiq B_2,
\end{align*}
and 
\begin{align*}
\|\ejack[k+1]\| = \|\byk[k+1] - \byjack[k+1]\|  &\le \Mtaypiq B_q^2.
\end{align*}
In addition, we have show that $\|\byk[k+1]\| \vee \|\bukpi[k+1] + \deluk[k+1] + \bKkpi[k+1](\updelta \byk[k+1])\| \le \frac{3}{4}\Rfeas$. This concludes the induction.

Substituing  in  $\byk = \btilxpik\,(\buk[1:K])$ and using the computation of $\byjack[k+1]$ in \Cref{lem:JL_comp_stuff} concludes the proof of the pertubation bounds. Moreover, the fact that \Cref{eq:feasability_invariant} holds for all $k$, and consequently the conclusion of \Cref{lem:feasibility_maintained} establishes the norm bounds $\|\btilxpik\,(\buk[1:K])\| \vee \|\btilupik\,(\buk[1:K])\| \le \frac{3\Rfeas}{4}$, and $\|\btilxpi(t \mid \buk[1:K])\| \le \Rfeas$.  

\qed

\subsection{Taylor Expansion of the Cost (\Cref{lem:taylor_expansion_of_cost})}\label{sec:lem:taylor_expansion_of_cost}

\begin{proof} Recall the definition
\begin{align*}
\Jpidisc\discbrak{\buvec} &:=  V(\btilxpik[K+1]\,\discbrak{\buvec}) + \step \sum_{k=1}^K Q( \btilxpik\,\discbrak{\buvec}, \btilupik\,\discbrak{\buvec}, t_k).
\end{align*}
\newcommand{\bek}[1][k]{\discfont{e}_{#1}}
\newcommand{\delxtilk}[1][k]{\updelta \btilxk[#1]}

Define the shorthand
\begin{align*}
\delxtilk &:= \btilxpik[k]\,(\deluk[1:K] + \bukpi[1:K]) - \bxkpi\\
\delxtilkjac &:= \btilxpijack\,(\deluk[1:K] + \bukpi[1:K])
\end{align*}
\begin{align*}
&\Jpidisc\discbrak{\buvec} - \Jpidisc\discbrak{\deluk[1:K] + \bukpi[1:K]}  \\
&:=  V(\bxkpi[K+1] + \delxtilk[K+1]) -  V(\bxkpi[K+1]) + \step \sum_{k=1}^K Q( \bxkpi + \delxtilk, \deluk + \bKkpi \delxtilk + \bukpi, t_k) - Q( \bxkpi, \bukpi, t_k)
\end{align*}
Notice that \Cref{prop:taylor_exp_dyn} and feasibility of $\pi$  implies that 
\begin{align}\|\bxkpi + \delxtilk\| \|\bxkpi\| \vee \|\deluk + \bKkpi \bek\| \vee \|\bukpi\| \le \Rfeas \label{eq:Jtay_fas}
\end{align}
Hence a Taylor expansion and \Cref{asm:cost_asm} imply
\begin{align*}
&|V(\bxkpi[K+1] + \delxtilk[K+1]) -  V(\bxkpi[K+1]) - \langle \partx V (\bxkpi[K+1]), \delxtilkjac \rangle| \\
&\le \frac{1}{2}\sup_{\alpha \in [0,1]} \|\partxx V(\bxkpi[K+1]+\alpha\delxtilk[K+1])\|\|\delxtilk[K+1]\|^2 + |\langle \partx V (\bxkpi[K+1]), \delxtilk[K+1] - \delxtilkjac[K+1] \rangle|\\
&\le \frac{\MQ}{2}\|\delxtilk[K+1]\|^2 + \LQ \|\delxtilk[K+1] - \delxtilkjac[K+1]\|.
\end{align*}
Similarly,
\begin{align*}
&\left| Q( \bxkpi + \delxtilk, \deluk + \bKkpi \delxtilk, t_k) - Q( \bxkpi, \bukpi, t_k) - \langle\partx Q( \bxkpi, \bukpi, t_k), \delxtilkjac\rangle -\langle \partu Q( \bxkpi, \bukpi, t_k), \deluk + \bKk \delxtilkjac \rangle\right| \\
&\le\left| \langle\partx Q( \bxkpi, \bukpi, t_k), \delxtilkjac\rangle| + | \langle \partu Q( \bxkpi, \bukpi, t_k), \bKk \delxtilkjac \rangle \right| \\
&\quad + \frac{1}{2}\MQ \left(\|\delxtilk\|^2 + \|\deluk+\bKk \delxtilkjac\|^2  \right)\\
&\le \LQ(1+\Lpi)\| \delxtilk - \delxtilkjac\| +  \frac{1}{2}\MQ \left((1+2\Lpi^2)\|\delxtilkjac\|^2 + 2\|\deluk\|^2  \right)\\
&\le2 \Lpi\LQ \| \delxtilk - \delxtilkjac\| +  \frac{1}{2}\MQ \left(3\Lpi^2\|\bek\|^2 + 2\|\deluk\|^2  \right) \tag{$\Lpi \ge 1$}.
\end{align*}
Then, from \Cref{lem:grad_compute_dt},
\begin{align*}
&\langle \deluk[1:K] ,  \nabla \Jpidisc(\bu) \rangle \\
&= \langle \partx V (\bxkpi[K+1]), \delxtilkjac \rangle + \step \sum_{k=1}^K\langle\partx Q( \bxkpi, \bukpi, t_k), \delxtilkjac\rangle +\langle \partu Q( \bxkpi, \bukpi, t_k), \deluk + \bKk \delxtilkjac \rangle.
\end{align*}
Therefore, we have
\begin{align*}
&\|\Jpidisc\discbrak{\deluk[1:K] + \bukpi[1:K]} -  \Jpidisc\discbrak{\bukpi[1:K]} - \langle \deluk[1:K] ,  \nabla \Jpidisc(\pi) \rangle \| \\
&\le \frac{\MQ}{2}\|\delxtilk[K+1]\|^2 + \LQ \|\delxtilkjac[K+1] - \delxtilk[K+1]\| +  \step\sum_{k=1}^K 2 \Lpi\LQ \| \delxtilkjac - \delxtilk\| +  \frac{1}{2}\MQ \left(3\Lpi^2\|\delxtilk\|^2 + 2\|\deluk\|^2  \right)\\
&\le \frac{\MQ}{2}(1+3\Lpi^2 T)\max_{k \in [K+1]}\|\delxtilk\|^2 + \LQ(1+2\Lpi T) \max_{k \in [K+1]}\| \delxtilkjac - \delxtilk\| + \Lpi \LQ \underbrace{\step \sum_{t=1}^T\|\deluk\|^2}_{=B_2}\\
&= \frac{\MQ}{2}(1+3\Lpi^2 T)\max_{k \in [K+1]}\|\delxtilk\|^2 + \LQ(1+2\Lpi T) \max_{k \in [K+1]}\| \delxtilk - \delxtilkjac\| + 2\Lpi \LQ B_2^2.
\end{align*}
From \Cref{prop:taylor_exp_dyn}, 
\begin{align*}
\max_{k \in [K+1]}\| \delxtilk - \delxtilk\| &\le \Mtaypitwo B_2^2\\
\max_{k \in [K+1]}\|\bek\|^2 &\le 4\LF^2 \Kpitwo^2 B_2^2. 
\end{align*}
Thus, 
\begin{align*}
&\|\Jpidisc\discbrak{\deluk[1:K] + \bukpi[1:K]} -  \Jpidisc\discbrak{\bukpi[1:K]} - \langle \deluk[1:K] ,  \nabla \Jpidisc(\pi) \rangle \| \\
&\le \underbrace{(2\MQ\LF^2 \Kpitwo^2(1+3\Lpi^2 T) \Mtaypitwo + \LQ(1+2\Lpi T)\Mtaypitwo + 2\Lpi \LQ)}_{:=\Mtayjpi}  B_2^2.
\end{align*}

\end{proof}

\subsection{Proof of \Cref{prop:Kpi_bounds_stab}}\label{sec:prop_Kpi_bounds_stab}
We begin with the following lemma, which we show shortly below.
    \begin{lemma}\label{lem:change_in_phi_perturb} Consider the setting of \Cref{prop:taylor_exp_dyn}, with $B_{\infty} \le \min\{\Btaypiinf,\Rfeas/8\}$. Let $\pi'$ denote the policy with gains $\bKkpi$ and inputs $\buk^{\pi'} = \buk = \bukpi + \deluk$. Then, 
    \begin{align*}
    \step\sum_{k=1}^K\|\bAclkpipr - \bAclkpi\| \le 12T\MF\Lpi  (1+\LF\Kpi) B_{\infty}.
    \end{align*}
    \end{lemma}
With this lemma, we turn to the proof of \Cref{prop:Kpi_bounds_stab}. Notice that, as $\pi$ and $\pi'$ have the same gains, $L_{\pi} = L_{\pi'}$. Therefore, following the proof of \Cref{lem:Kpi_bounds} (see, specifically, the proof of \Cref{claim:j_ge_knot}), we have that for  $\|\bAclkpipr - \eye\|$  and $\|\bAclkpi - \eye\| $ are both at most $\kappa := 3\step \LF\Lpi$ for $\step \le 1/6\LF\Lpi$. 

Let us now construct an interpolating curves $\bXk(s)$ with $\bXk(0) = \bAclkpi$ and $\bXk(1) = \bAclkpipr$, and define the interpolating Lyapunov function 
\begin{align*}    
    \bLamk[K+1](s) = \eye, \quad \bLamk(s) = \bXk(s)^\top \bLamk[k+1] \bXk(s) + \step \eye,
    \end{align*}
    Define
    \begin{align*}
    \Delta &= 3\sup_{s \in [0,1]}\max\{1,2\kappa\}\sum_{k=k_0}^K\|\bXk'(s)\|\\
    &= \max\{1,2\kappa\}\sum_{k=k_0}^K\|\bAclkpipr - \bAclkpi\|\\
    &= \max\{3,18\LF\Lpi\}\sum_{k=k_0}^K\|\bAclkpipr - \bAclkpi\|.
    \end{align*}
    and recall the  the shorthand $\|\bLamk[k_0:K+1](s)\|_{\maxop} := \max_{k \in [k_0:K+1]}\|\bLamk(s)\|$. 
    Then, as long as $\|\bLamk[k_0:K+1](0)\|_{\maxop}\Delta < 1$, \Cref{lem:Lyap_avg_perturb} (re-indexing to terminate the backward recursion at $k_0$ instead of $1$) implies
    \begin{align*}
    \|\bLamk[k_0:K+1](1)\|_{\maxop} \le (1-\|\bLamk[k_0:K+1](0)\|_{\maxop}\Delta)^{-1}\|\bLamk[k_0:K+1](0)\|_{\maxop}.
    \end{align*}
    We see that $\|\bLamk[k_0:K+1](0)\|_{\maxop} = \|\bLamk[k_0:K+1]^\pi\|_{\maxop} = K_{\pi,\star}$, and $\|\bLamk[k_0:K+1](1)\|_{\maxop} = \|\bLamk[k_0:K+1]^{\pi'}\|_{\maxop} = \Kpiprst$.  Thus, combining with the inequality $(1-x)^{-1} \le 1+2x$ for $x \in [0,1/2]$, we have that as long as $\Delta\Kpist \le 1/2$, 
    \begin{align*}
    \Kpiprst \le(1 + 2\Delta\Kpist )\Kpist.
    \end{align*}
    Lastly, we can bound  
    \begin{align*}
    2\Delta\Kpist &= \max\{6,36\LF\Lpi\}\Kpist\sum_{k=k_0}^K\|\bAclkpipr - \bAclkpi\| \\
    &\le \underbrace{\max\{6,36\LF\Lpi\}\Kpist \cdot 12T\MF\Lpi  (1+\LF\Kpi) B_{\infty}}_{:= 1/\Bstabpi} B_2 \tag{\Cref{lem:change_in_phi_perturb}}.
    \end{align*}
    In sum, for $ B_{\infty} \le \Bstabpi$, we have $\Kpiprst \le(1 +  B_\infty/\Bstabpi )\Kpist$, which concludes the proof. 


\newcommand{\Apipr}{\bA_{\mathrm{ol}}^{\pi'}}
\newcommand{\Bpipr}{\bB_{\mathrm{ol}}^{\pi'}}

    \begin{proof}[Proof of \Cref{lem:change_in_phi_perturb}] 

    Due to \Cref{prop:taylor_exp_dyn}, and the fact that $\bx^{\pi'}(t) = \btilxpi(t \mid \buk[1:K])$ and $\bu^{\pi'}(t) = \btilupik[k(t)]\,(\buk[1:K])$, we have that
    \begin{align}\label{eq:dyn_pert_feas}
    \forall t\in[0,T], \quad \|\bx^{\pi'}(t)\| \vee  \|\bu^{\pi'}(t)\| \le \Rfeas. 
    \end{align}
    Moreover each initial condition $\xi$ with norm $\|\xi\| = 1$, we have that from \Cref{lem:solve_affine_ode} and the definitions of $\bAclkpipr, \bAclkpi$ from \Cref{defn:disc_Jac_linz_dyns} that
    \begin{align*}
    (\bAclkpipr - \bAclkpi) \xi = \bz_2(\step)-\bz_1(\step), 
    \end{align*}
    where 
    \begin{align*}
    \bz_2(0) = \bz_1(0) = \xi, 
    \end{align*}
    and where $\ddt \bz_2(t) = \Apipr(t_k+t)\bz_2(t) + \Bpipr(t_k+t)\bKk \xi$, and where $\ddt \bz_1(t) = \Api(t_k+t)\bz_2(t) + \Bpi(t_k+t)\bKk \xi$. 

    By the Picard Lemma, \Cref{lem:picard}, and by bounding $\|\Apipr(t)\| \le \LF$ by \Cref{eq:dyn_pert_feas} and \Cref{asm:max_dyn_final}, it follows that 
    \begin{align*}
    \|(\bAclkpipr - \bAclkpi) \xi \| &\le \exp(\step \LF)\int_{0}^\step (\|\Apipr(t_k+t)-\Api(t_k + t)\|\|\bz_1(t)\| + \|\Bpipr(t_k+t)-\Bpi(t_k + t)\|\|\bKk\|\|\xi\|)\rmd t\\
    &\le \exp(\step \LF)\int_{0}^\step (\|\Apipr(t_k+t)-\Api(t_k + t)\|\|\bz_1(t)\| + \|\Bpipr(t_k+t)-\Bpi(t_k + t)\|\Lpi) \rmd t
    \end{align*}
    Set $\Lol = \exp(\step \LF)$.
    Following the computation in , we can bound $\sup_{t \in [0,\step}\|\bz_1(t)\| \le \|\bAclkpi\| = \|\Phicldisc{k+1,k}\| \le 5/3$ provided $\step \le 1/4\LF\Lpi$ (recall we assume $\Lpi \ge 1$). Hence,
    \begin{align*}
    \|(\bAclkpipr - \bAclkpi) \xi \| \le\Lol\int_{0}^\step (\frac{5}{3}\|\Apipr(t_k+t)-\Api(t_k + t)\| + \|\Bpipr(t_k+t)-\Bpi(t_k + t)\|\Lpi)\rmd t
    \end{align*}
    Finally, by the smoothness on the dynamics \Cref{asm:max_dyn_final} and invoking \Cref{eq:dyn_pert_feas} and feasibility of $\pi$ to ensure all relevant $(x,u)$ pairs are feasible, we have 
    \begin{align*}
    \|\Apipr(t_k+t)-\Api(t_k + t)\| &= \|\partx \fdyn(\bx^{\pi'}(t),\bu^{\pi'}(t)) - \partx \fdyn(\bx^{\pi}(t),\bu^{\pi}(t))\|\\
    &\le \MF \left(\|\bx^{\pi'}(t) - \bx^\pi(t)\| + \|\bu^{\pi'}(t) - \bu^{\pi}(t))\|\right)\\
    &\le \MF \left(\|\bx^{\pi'}(t) - \bx^\pi(t)\| + \|\deluk\|\right).
    \end{align*}
    Applying a similar bound to the term $\|\Bpipr(t_k+t)-\Bpi(t_k + t)\|$, we conclude 
    \begin{align*}
    \|\bAclkpipr - \bAclkpi \| &\le \sup_{\xi:\|\xi\|} = \|(\bAclkpipr - \bAclkpi) \xi \| \\
    &\le \Lol\MF(\frac{5}{3}+\Lpi)\int_{0}^\step \left(\|\bx^{\pi'}(t) - \bx^\pi(t)\| + \|\deluk\|\right)\rmd t\\
    &\le \step\Lol\MF(\frac{5}{3}+\Lpi) \left((1+\step\Lol \LF)\|\deluk\| +  \Lol(1+\step \Lpi \LF)\|\bx^{\pi'}(t_k) - \bx^\pi(t_k)\right) \tag{\Cref{lem:first_taylor_pert}}\\
    &\le \step\Lol\MF(\frac{5}{3}+\Lpi) \left(\frac{5\Lol}{4}\|\deluk\| +  \frac{5}{4}\Lol\|\bx^{\pi'}(t_k) - \bx^\pi(t_k)\right) \tag{$\Lol \ge 1$, $\step \le 1/4\LF\Lpi \le 1/4\LF$}\\
    &= \step\frac{5\Lol^2}{4}\MF(\frac{5}{3}+\Lpi) \left(\|\deluk\| + \|\bx^{\pi'}(t_k) - \bx^\pi(t_k)\|\right)\\
    &= \step\MF \cdot \frac{5e^{1/2}}{4}\cdot\frac{8}{3}\cdot\Lpi \left(\|\deluk\| + \|\bx^{\pi'}(t_k) - \bx^\pi(t_k)\|\right) \tag{$\Lpi \ge 1, \step\LF \le 1/4$}\\
    &\le 6\step\MF\Lpi \left(\|\deluk\| + \|\bx^{\pi'}(t_k) - \bx^\pi(t_k)\|\right)\\
    &\le 6\step\MF\Lpi \left(\|\deluk\| + 2\LF\Kpione B_\infty\right).\tag{\Cref{prop:taylor_exp_dyn}}
    \end{align*}
    Summing the bound, and using $K\step = T$, we have
    \begin{align*}
    \sum_{k=1}^K\|\bAclkpipr - \bAclkpi \| &\le 6\MF\Lpi \left(\step \sum_{k=1}^K\|\deluk\| + 2T\LF\Kpione B_{\infty}\right)\\
    &\le 6\MF\Lpi \left(K\step B_{\infty} + 2T\LF\Kpione B_{\infty}\right)\\
    &\le 12T\MF\Lpi  (1+\LF\Kpione) B_{\infty}.
    \end{align*}
    \end{proof}

\subsection{Proof of \Cref{lem:grad_bound}}\label{sec:lem:grad_bound}
\begin{proof} Using \Cref{asm:cost_norm} and and $1 \vee \|\bKk\| \le \Lpi$, 
\begin{align*}
\|(\nabla \Jdisc(\pi))_k\| &\le \step \| Q_u( \bxkpi, \bukpi, t_k)^2\|^2 + \|\Psicldisc{K+1,k})^\top V_x(\bxkpi[K+1]\| \\
&\qquad+ \left\|\step\sum_{j=k+1}^K (\Psicldisc{j,k})^\top(Q_x( \bxkpi[j], \bukpi[j], t_j) + (\bKkpi[j])^\top Q_u( \bxkpi[j], \bukpi[j], t_j))\right\|\\
&\le \step \LQ  + \step \|\Psicldisc{K+1,k}\| \LQ^2 + 2\Lpi^2 \LQ^2 \left(\step\sum_{j=k+1}^{K} \|\Psicldisc{j,k}\|\right)^2.
\end{align*}
Using $\step \le 1/4\LF$ and \Cref{lem:bound_on_bBkpi}, we can bound $\|\Psicldisc{j,k}\| = \|\Phicldisc{j,k+1}\bBkpi\| \le \step\LF\exp(\step \LF)\|\Phicldisc{j,k+1}\| \le \step\LF\exp(1/4)\|\Phicldisc{j,k+1}\|$. As $\exp(1/2) \le 3/4$, we conclude that 
\begin{align*}
\|(\nabla \Jdisc(\pi))_k\| 
&\le \step \LQ  + \frac{3}{2}\step^2 \LF \LQ + 3\step\Lpi \LQ \left(\step\sum_{j=k+1}^{K} \|\Phicldisc{j,k}\|\right)^2. \\
&\le \step \LQ  + \frac{3}{2}\step \LF \LQ \|\Phicldisc{K+1,k+1}\| + 3\Lpi \LQ  \Kpione.
\end{align*}
Using $\|\Phicldisc{K+1,k+1}\| \le \Kpiinf$gives the boudn $\|(\nabla \Jdisc(\pi))_k\|  \le \step \LQ(1 + \frac{3\LF}{2}\Kpiinf + 3\Lpi \Kpione) =: \Lnabpiinf$.
\end{proof}
\newcommand{\btilAk}[1][k]{\tilde{\discfont{A}}_{#1}}
\newcommand{\Lcl}{L_{\mathrm{cl}}}

\section{Estimation Proofs}\label{app:estimation}

\subsection{Estimation of Markov Parameters: Proof of \Cref{prop:markov_est}}\label{sec:lem:prop:markov_est}
We begin with two standard concentration inequalities. 
\begin{lemma}\label{lem:Gau_conc} Let $(y_i,x_i,w_i)_{i=1}^n$ be an independent sequence of triples of random vectors in with $y_i,x_i \in \R^{d}$, $w \in \R^{d'}$ and suppose that $y_i \mid x_i,w_i \sim \cN(x_i,\sigma^2 \eye_d)$ and $\max_{i} \|w_i\| \le R$. Then, 
\begin{align*}
\Pr\left[ \left\|\frac{1}{N}\sum_{i=1}^N (y_i - x_i) w_i\right\| \le R\sigma\sqrt{2\cdot\frac{d\log 5 + \log((d'+1)/\delta)}{N}}\right] \ge 1-\delta.
\end{align*}
\begin{proof}[Proof of \Cref{lem:Gau_conc}] By a standard covering argument (see, e.g. \citet[Chapter 4]{vershynin2018high}), there exists a finite covering $\cT$ of unit vectors $z \in \R^d$ such that (a) $\log |\cT| \le d \log 5$ and (b), for all vectors $v \in \R^d$,
\begin{align*}
\|v\| \le 2 \sup_{z \in \cT}\langle v, z \rangle. 
\end{align*} 
Hence, 
\begin{align}
\left\|\frac{1}{N}\sum_{i=1}^N (y_i - x_i)w_i\right\|_{\op} \le 2\sup_{z \in \cT}\left\|\frac{1}{N}\sum_{i=1}^N \langle z, y_i - x_i \rangle \cdot w_i\right\| = 2\sigma\sup_{z \in \cT}\left\|\frac{1}{N}\sum_{i=1}^N \xi_{i}(z) \cdot w_i\right\| \label{eq:Gauss_conc},
\end{align}
where above we define $\xi_i(z) := \sigma^{-1}\langle z, y_i - x_i \rangle$. Notice that $\xi_i(z) \mid w_i$ are standard Normal random variables. Thus, by standard Gaussian concentration (e.g. \citet[Chapter 2]{boucheron2013concentration}),
\begin{align*}
\Pr\left[ \left\|\frac{1}{N}\sum_{i=1}^N \xi_i(z) w_i\right\| \le R\sqrt{2\frac{\log((d'+1)/\delta)}{N}}\right] \ge 1-\delta.
\end{align*}
Hence, union bounding over $z \in \cT$, bounding $|\cT| \le d\log 5$,   \Cref{eq:Gauss_conc} implies the desired bound.
\begin{align*}
\Pr\left[ \left\|\frac{1}{N}\sum_{i=1}^N (y_i - x_i) w_i\right\| \le R\sigma\sqrt{2\frac{d\log 5 + \log((d'+1)/\delta)}{N}}\right] \ge 1-\delta.
\end{align*}
\end{proof}
\end{lemma}
\begin{lemma}[Assymetric Matrix Hoeffding]\label{lem:asym_mat_hoeff} Let $X_1,\dots,X_n$ be an independent sequence of matrices in $\R^{d_1 \times d_2}$ with $\|X_i\| \le R$. Then, 
\begin{align*}
 \Pr\left[\frac{1}{N}\left\|\sum_{i} X_i - \Exp[X_i]\right\| \le  4R\sqrt{\frac{\log(\frac{d_1+d_2}{\delta})}{N}}\right] \ge 1- \delta.
\end{align*}
\end{lemma}
\begin{proof}[Proof of \Cref{lem:asym_mat_hoeff}]
By recentering $X_i \gets X_i - \Exp[X_i]$, we may assume $\Exp[X_i] = 0$ and $\|X_i\| \le 2R$.  Define the Hermitian dilation
\begin{align*}
    Y_i = \begin{bmatrix}
        0 & X_i \\
        X_i^\top & 0
    \end{bmatrix}.
\end{align*}
Then
\begin{align*}
    Y_i^2 = \begin{bmatrix}
        X_iX_i^\top & 0 \\
        0 & X_i^\top X_i
    \end{bmatrix} \preceq \|X_i\|^2\eye_{d_1 + d_i} \le 4R^2\eye_{d_1 + d_2}
\end{align*}
Applying standard Matrix Hoeffding \citet[Theorem 1.4]{tropp2012user} for Hermitian matrices to the $Y_i$'s yields
\begin{align*}
    \mathbb{P} \left[\left\|\sum_{i} Y_i\right\| \geq t\right] \leq (d_1 + d_2) e^{-\frac{t^2}{32NR^2}}.
\end{align*}
Hence, by rearranging,
\begin{align*}
    \Pr\left[\left\|\sum_{i} Y_i\right\| \le  4R\sqrt{2N\log(\frac{d_1+d_2}{\delta})}\right] \ge 1- \delta
\end{align*}
As $\left\|\sum_{i} Y_i\right\| = \sqrt{2}\left\|\sum_{i} X_i\right\|$, we conclude 
\begin{align*}
    \Pr\left[\frac{1}{N}\left\|\sum_{i} X_i\right\| \le  4R\sqrt{\frac{\log(\frac{d_1+d_2}{\delta})}{N}}\right] \ge 1- \delta.
\end{align*}
\end{proof}
We now turn to concluding the proof of \Cref{prop:markov_est}. We begin with a claim which bounds $\max_{k} \| \bhatxk - \bxkpi\|$. 
Throughout, $\dst := \max\{\dimx,\dimu\}$. 
\begin{claim}\label{claim:xhat_err} With probability at least $1 - \delta/3$, the following bound holds
\begin{align*}
\max_{k} \| \bhatxk - \bxkpi\| \le \sigorac\sqrt{2\frac{\dst \log 5 + \log(6(K+1)/\delta)}{N}} \le \errx(\delta). 
\end{align*}
\end{claim}
\begin{proof} The result follows directly from \Cref{lem:Gau_conc}, with $w_i = 1 \in \R$ for each $i$. 
\end{proof}

\begin{proof}[Proof of \Cref{prop:markov_est}] 
Throughout, suppose the event of \Cref{claim:xhat_err} holds. We also note that
\begin{align}
\|\bwk[j]^{(i)}\| \le \sigw\sqrt{\dimu } \le \sigw\sqrt{\dst} \text{ a.s.} . \label{eq:bwkj} 
\end{align}

This covers the first inequality of the proposition. To bound the error on the transition operator, tet us fix indices $j,k$; we perform a union bound at the end of the proof.
For each perturbation sampled perturbation  $\bwk[1:K]^{(i)}$, define a perturbed control input 
\begin{align*}
\bchuk^{(i)} = \bukpi + \bwk^{(i)} + \bKk^{\pi} (\bxk^\pi - \bhatxk), \quad \btiluk^{i} = \bukpi + \bwk^{(i)} - \bKk^{\pi} \bhatxk.
\end{align*}
and observe that
\begin{align}\label{eq:oracle_id}
\xoffpik\,(\btiluk[1:K]^{(i)}) = \btilxk\,(\bchuk[1:K]^{(i)}), \quad \forall k \in [K].
\end{align}
Hence, we have that $\byk$ defined in \Cref{line:byk} satisfies
\begin{align*}
\byk^{(i)} \sim \cN( \btilxk\,(\bchuk[1:K]^{(i)}),\sigorac^2 \eye_{\dimx}).
\end{align*}
\newcommand{\Eorac}{\Exp_{\mathrm{orac}}}
Now, define the terms
\begin{align*}
\bzk^{(i)} := \byk^{(i)} - \bxkpi
\end{align*}
Lastly, let $\Eorac[\cdot]$ denote expectations with respect to the Gaussian noise of the oracle, (conditioning on $\bwk[1:K]^{(i)}$) while $\Exp[\cdot]$ denotes total expectation.
 We  argue an error bound on 
\begin{align*}
\left\| \frac{\sigw^{-2}}{N}\sum_{i=1}^N \bzk[k]^{(i)}\bwk[j]^{(i)\top} - \Psicldisc{k,j} \right\|_{\op} &\le   
\underbrace{\frac{\sigw^{-2}}{N}\left\| \sum_{i=1}^N \Eorac[\bzk[k]^{(i)}\bwk[j]^{(i)\top}] - \bzk[k]^{(i)}\bwk[j]^{(i)\top} \right\|_{\op}}_{=\Term_{1}} \\
&\qquad+ \underbrace{\frac{\sigw^{-2}}{N}\left\| \sum_{i=1}^N \Exp[\bzk[k]^{(i)}\bwk[j]^{(i)\top}] - \Eorac[\bzk[k]^{(i)}\bwk[j]^{(i)\top}] \right\|_{\op}}_{=\Term_{2}}\\
&\qquad + \underbrace{\left\| \frac{\sigw^{-2}}{N}\sum_{i=1}^N \Exp[\bzk[k]^{(i)}\bwk[j]^{(i)\top}] - \Psicldisc{k,j} \right\|_{\op}}_{=\Term_{3}},
\end{align*}
which essentially bounds the estiation error of $\Psicldisc{k,j}$ in the absence of observation noise. 

\paragraph{Bounding $\Term_1$. } Applying \Cref{lem:Gau_conc} with $\|\bwk[j]^{(i)}\| \le \sqrt{\dst}\sigw$, we have that with probability $1 - \delta/3$,
\begin{align*}
\Term_1 &\le \frac{\sigorac}{\sigw}\sqrt{2\dst\cdot\frac{\dst\log 5 + \log(6(\dst+1)/\delta)}{N}}\\
&\le \frac{\sigorac}{\sigw}\sqrt{2\dst\cdot\frac{\dst\log 5 + \log(12\dst/\delta)}{N}}.
\end{align*}

\paragraph{Bounding $\Term_2$.} On the event of \Cref{claim:xhat_err}, then as long as $\errx(\delta) \le \sigw\sqrt{\dst/\Lpi}$
\begin{align*}
\|\bchuk^{(i)} - \bukpi\| = \|\bwk^{(i)} + \bKk^{\pi} (\bxk^\pi - \bhatxk)\| \le \|\bwk^{(i)}\| + \Lpi \errx(\delta) \le \sigw\sqrt{\dst} + \Lpi \errx(\delta) \le 2\sigw\sqrt{\dst}.
\end{align*}
Notice that as $\errx(\delta) = \sigorac\sqrt{2\dst\iota(\delta)/N}$, $\errx(\delta) \le \sigw\sqrt{\dst/\Lpi}$ holds for $N \ge (\sigorac/\sigw)^2 2\Lpi\iota(\delta)$, i.e. which holds for when $\pi$ is estimation-friendly.
 
Moroever, when $\pi$ is estimation-friendly, $\Btaypiinf \ge 2\sigw\sqrt{\dst }$, so  the conditions of \Cref{prop:taylor_exp_dyn} hold.  Therefore,
\begin{align*}
\|\bzk^{(i)}\| \le 2\sigw\Ltaypiinf \sqrt{\dst}.
\end{align*}
and thus
\begin{align*}
\|\bwk[j]^{(i)}\bzk^{(i)}\| \le 2\dst\sigw^2\Ltaypiinf.
\end{align*}
Applying \Cref{lem:asym_mat_hoeff} with $X_i \gets \bzk[k]^{(i)}\bwk[j]^{(i)\top}$ with $R \gets 2\dst\Ltaypiinf\sigw^2$, it holds that with probability $1 - \delta/3$ that
\begin{align*}
\Term_{2} \le  \sigw^{-2}\cdot 8\Ltaypiinf\sigw^2\dimu\sqrt{\frac{\log(\frac{3(\dimu + \dimx)}{\delta})}{N}} \le  8\Ltaypiinf\dst\sqrt{\frac{\log(6\dst/\delta)}{N}}.
\end{align*}

\paragraph{Bounding $\Term_3$}. As establish in the bound on $\Term_2$, the conditions of \Cref{prop:taylor_exp_dyn} hold, and $\|\bchuk^{(i)} - \bukpi\|\le 2\sigw\sqrt{\dst}.$. Therefore, 
\begin{align*}
\|\bzk^{(i)} - \sum_{\ell=1}^k \Psicldisc{k,\ell}(\bchuk[\ell]^{(i)} - \bukpi[\ell]) \| \le 4\sigw^2\Mtaypitwo \dst.
\end{align*}
Consequently, bounding $\|\bwk[j]^{(i)}\| \le \sigw \dst$,
\begin{align*}
\frac{1}{\sigw^{2}}\|\bzk^{(i)}\bwk[j]^{(i)} - \sum_{\ell=1}^k \Psicldisc{k,\ell}(\bchuk[\ell]^{(i)} - \bukpi[\ell])\bwk[j]^{(i)} \| \le 4\sigw\Mtaypitwo \dst^{3/2}
\end{align*}
and thus, by Jensen's inequality,
\begin{align*}
4\sigw\Mtaypitwo \dst^{3/2} &\ge \frac{1}{N\sigw^{2}}\|\sum_{i=1}\Exp[\bzk^{(i)}\bwk[j]^{(i)}] - \sum_{\ell=1}^k \Exp[\Psicldisc{k,\ell}(\bchuk[\ell]^{(i)} - \bukpi[\ell])\bwk[j]^{(i)}] \| \\
&= \frac{1}{N\sigw^{2}}\|\sum_{i=1}\Exp[\bzk^{(i)}\bwk[j]^{(i)}] - \sum_{\ell=1}^k \Psicldisc{k,\ell}\Exp[( \bwk^{(i)}[\ell] + \bKk^{\pi} (\bxk[\ell]^\pi - \bhatxk[\ell]))\bwk[j]^{(i)}] \|\\
&= \frac{1}{N\sigw^{2}}\|\sum_{i=1}\Exp[\bzk^{(i)}\bwk[j]^{(i)}] - \sigw^{2}\sum_{\ell=1}^k \Psicldisc{k,\ell} \I_{\ell = j}] \|\\
&= \frac{1}{N\sigw^{2}}\|\sum_{i=1}\Exp[\bzk^{(i)}\bwk[j]^{(i)}] - \sigw^2 \Psicldisc{k,j}]  = \Term_3.
\end{align*}
In sum, with probability at least $1 - 3\delta/4$, the following bound holds
\begin{align*}
&\left\| \frac{\sigw^{-2}}{N}\sum_{i=1}^N \bzk[k]^{(i)}\bwk[j]^{(i)\top} - \Psicldisc{k,j} \right\|_{\op} \\
&\le \Term_1 + \Term_2 + \Term_3\\
&\le \frac{\sigorac}{\sigw}\sqrt{2\dst\cdot\frac{\dst\log 5 + \log(\frac{12\dst}{\delta})}{N}} + 8\Ltaypiinf\dst\sqrt{\frac{\log(6\dst/\delta)}{N}} + 4\sigw\Mtaypitwo \dst^{3/2}\\
&\le \sqrt{\frac{\log \frac{12 \dst}{\delta}}{N}}\left(\frac{2\sigorac}{\sigw}\dst^{3/2} + 8\Ltaypiinf \dst\right)+ 4\sigw\Mtaypitwo \dst^{3/2}.
\end{align*}
The final bound follows from a union bound over all $\binom{K}{2} \le K^2$ pairs, and replacing $\delta$ with $\delta/2\nfin$.

\end{proof}

\subsection{Error in the Gradient (Proof of \Cref{lem:grad_err})}\label{sec:lem:grad_err}
Recall the definitions
\begin{align*}
(\nabla \Jdisc(\pi))_k &=  \step Q_u( \bxkpi, \bukpi, t_k)  + (\Psicldisc{K+1,k})^\top V_x(\bxkpi[K+1]) + \\
&+ \step\sum_{j=k+1}^K (\Psicldisc{j,k})^\top(Q_x( \bxkpi[j], \bukpi[j], t_j) + (\bKkpi[j])^\top Q_u( \bxkpi[j], \bukpi[j], t_j))
\end{align*}
and 
\begin{align*}
 \bnabhatk &=  Q_u(\bbarxk^\pi,\buk^\pi,t_j) +
    \bhatPsi{K+1}{k}^\top V_x(\bbarxk[K+1]) \\
    &+ \step \sum_{j=k+1}^K \bhatPsi{j}{k}^\top \left( Q_x(\bhatxk[j],\buk[j]^\pi,t_j) +  (\bKkpi[j])^\top   Q_u(\bhatxk[j],\buk[j]^\pi,t_j)\right)
    \end{align*}
    Using $V_x(\cdot)$ and $Q_x(\cdot),Q_u(\cdot)$ are all at most $\LQ$ in magnitude, that the gradients of the cost are $\MQ$-Lipschitz, and $1 \vee \|\bKkpi[j]\| \le \Lpi$ we can bound 
    \begin{align*}
    &\|(\nabla \Jdisc(\pi))_k - \bnabhatk\|\\
     &\le \LQ \|
    \bhatPsi{K+1}{k} -\Psicldisc{K+1,k}\| + 2\Lpi\LQ \step \sum_{j=k+1}^K  \|\bhatPsi{j}{k} - \Psicldisc{j,k}\|\\
    &\quad+\MQ  \left(\|\bxkpi - \bbarxk^\pi \|  + \|\Psicldisc{K+1,k}\| \cdot\|\bxkpi[K+1] - \bhatxk[K+1] \| + 2\step\Lpi\sum_{j=k+1}^K \|\Psicldisc{j,k}\|\|\bxkpi[j] - \bhatxk[j] \|\right)\\
    &\le \LQ \errpsipi(\delta) (1+2\Lpi \step K) +\MQ \errx(\delta) \left(1  + \|\Psicldisc{K+1,k}\|  + 2\step\Lpi\sum_{j=k+1}^K \|\Psicldisc{j,k}\|\right)\\
    &\le \LQ \errpsipi(\delta) (1+2\Lpi \step K) +\MQ \errx(\delta) \left(1  + \Kpiinf  + 2K\step\Lpi\Kpiinf\right)\\
    &= \LQ \errpsipi(\delta) (1+2T\Lpi) +\MQ \errx(\delta) \left(1  + (1+ 2T\Lpi)\Kpiinf\right)\\
    &\le \underbrace{(\LQ \errpsipi(\delta) + (1+\Kpiinf)\MQ \errx(\delta))(1+2T\Lpi)}_{:= \errnabpi(\delta)}
    \end{align*}
\qed

\subsection{Discrete-Time Closed-Loop Controllability (\Cref{prop:control_disc})}\label{sec:prop:ctr_disc}

We begin by lower bounding the continuous-time controllability Grammian under a policy $\pi$, and then turn to lower bounding its discretization. At the end of the proof, we remark upon how the bound can be refined.  The first part of the argument follows \cite{chen2021black}. 

\paragraph{Equivalent characterization of controllability Gramian smallest singular value.} The following is a continuous-time analogue of \citet[Lemma 15]{chen2021black}. 
\begin{lemma}[Characterization of Controllability Gramian smallest singular value]\label{lem:ctr_char} Let $\bPsi(t,s) \in \R^{\dimx \times \dimu}$ be an arbitrary ( locally square integrable), and set
\begin{align*}
\Lambda := \int_{s=t-\tcont}^{t}\bPsi(t,s)\bPsi(t,s)^\top \rmd s.
\end{align*}
Then, $\lambda_{\min}(\Lambda) \ge \nu$ if and only if for all unit vectors $\xi$, there exists some $\bu_{\xi}(s)$ such that $\xi = \int_{s=t-\tcont}^{t}\bPsi(t,s)\bu_{\xi}(s)$ and $\int_{s=t-\tcont}^{t}\|\bu_{\xi}(s)\|^2 \le \nu^{-1}$.
\end{lemma}
\begin{proof}[Proof of \Cref{lem:ctr_char}] Fix any unit vector $\xi \in \R^n$, define.  First, suppose $\int_{s=t-\tcont}^{t}\bPsi(t,s)\bPsi(t,s)^\top \rmd s \ge \nu$.  
\begin{align*}
\bu_{\xi}(s) := \bPsi(t,s)^\top\Lambda^{-1} \xi. 
\end{align*}
One can verify then that
\begin{align*}
&\int_{s=t-\tcont}^{t}\bPsi(t,s)\bu_{\xi}(s)\rmd s = \Lambda \Lambda^{-1}\xi = \xi
&\int_{s=t-\tcont}^t\|\bu_{\xi}(s) \|^2 \rmd s = \xi \Lambda^{-1}\cdot \Lambda  \cdot \Lambda^{-1} \xi = \xi \Lambda^{-1} \xi \le \lambda_{\min}(\Lambda)^{-1}.
\end{align*}
On the other hand, suppose that there exists a control $\bu_{\xi}(s)$ with $\int_{s=t-\tcont}^t\|\bu_{\xi}(s) \|^2 \le \lambda_{\min}(\Lambda)^{-1}$ such that $\int_{s=t-\tcont}^{t}\bPsi(t,s)\bu_{\xi}(s)\rmd s = \xi$. As $\xi$ is a unit vector, i.e. $\xi^\top \xi = 1$,
\begin{align*}
1 &= \left(\xi^\top\int_{s=t-\tcont}^{t}\bPsi(t,s)\bu_{\xi}(s)\rmd s\right)^2\\
&= \left(\int_{s=t-\tcont}^{t}\xi^\top\bPsi(t,s)\bu_{\xi}(s)\rmd s\right)\\
&\le \left(\int_{s=t-\tcont}^{t}\|\xi^\top\bPsi(t,s)\|^2 \rmd s \right) \cdot \left(\int_{s=t-\tcont}^{t}\|\bu_{\xi}(s)\|^2\rmd s\right)\\
&\le \xi^\top \Lambda \xi \cdot \left(\int_{s=t-\tcont}^{t}\|\bu_{\xi}(s)\|^2\rmd s\right)\\
&\le \xi^\top \Lambda \xi \cdot \lambda_{\min}(\Lambda)^{-1}.
\end{align*}
The bound follows.
\end{proof}

\paragraph{Lower bounding the controllability Gramian until algernative policies.} This next part is the continuous-time analogue of \cite[Lemma 16]{chen2021black}, establishing controllability of the closed-loop linearized system in feedback with policy $\pi$.
\newcommand{\nucontcl}{\nu_{\mathrm{ctrl},\mathrm{cl}}}

\begin{lemma}[Controllabiity of Closed-Loop Transitions, Continuous-Time]\label{lem:ctr_closed_loop} Recall $\Lpi \ge 1$, and  $\gamcont := \max\{1,\LF \tcont\}$. Then, under \Cref{asm:ctr},
\begin{align*}
\int_{s=t-\tcont}^{t}\Psiclpi(t,s)\Psiclpi(t,s)^\top \rmd s \succeq \frac{\nucont}{4\Lpi^2\gamcont^2\exp(2\gamcont)}.
\end{align*}
\end{lemma}
\begin{proof}[Proof of \Cref{lem:ctr_closed_loop}] Fix any $\xi \in \R^{\dimx}$ of unit norm. \Cref{lem:ctr_char,asm:ctr} guarantee the existence of an input $\bu_{\xi}(s)$ for which
\begin{align*}
\int_{s=t-\tcont}^{t}\Phiolpi(t,s)\Bpi(s)\bu_{\xi}(s)\rmd s = \xi, \quad \int_{s=t-\tcont}^{t}\|\bu_{\xi}(s)\|^2\rmd s \le \nucont^{-1}.
\end{align*}
Let
\begin{align*}
\bz_{\xi}(s') = \int_{s=t-\tcont}^{s'}\Phiolpi(t,s)\Bpi(s)\bu_{\xi}(s)\rmd s.
\end{align*}
 Define now the input 
\begin{align*}
\tilde{\bu}_{\xi}(s) := \bu_{\xi}(s) - \I\{t_k(s) > t-\tcont\} \bKkpi[k(s)]\bz_{\xi}(t_k(s)).
\end{align*}
It can be directly verified (by induction on $k$) that 
\begin{align*}
\forall s' \in [t-\tcont,t], \quad \int_{s=t-\tcont}^{s'}\Phiolpi(t,s)\Bpi(s)\bu_{\xi}(s)\rmd s = \int_{s=t-\tcont}^{s'}\Psiclpi(t,s)\tilde\bu_{\xi}(s)\rmd s,
\end{align*}
so in particular
\begin{align*}
 \xi = \int_{s=t-\tcont}^{t}\Psiclpi(t,s)\tilde\bu_{\xi}(s)\rmd s.
\end{align*}
We may now bound
\begin{align*}
\int_{s=t-\tcont}^{t}\|\tilde{\bu}_{\xi}(s)\|^2 &= \int_{s=t-\tcont}^{t}\left(\|\bu_{\xi}(s) - \I\{t_k(s) > t-\tcont\} \bKkpi[k(s)]\bz_{\xi}(t_k(s))\|^2\right) \rmd s\\
&\le 2\int_{s=t-\tcont}^{t}\left(\|\bu_{\xi}(s)\|^2 + \|\bKkpi[k(s)]\|\|\bz_{\xi}(t_k(s))\|^2\right)\rmd s\\
&\le 2(\nucont^{-1} + \Lpi^2 \int_{s=t-\tcont}^{t}\|\bz_{\xi}(t_k(s))\|^2 \rmd s), \numberthis \label{eq:utilxis_int}
\end{align*}
We now adopt the following claim, mirroring the proof of \citet[Lemma 16]{chen2021black}. 
\begin{claim}\label{claim:zxi_bound} The following bound holds:
\begin{align*}
\forall s' \in [t-\tcont,t], \quad \|\bz_{\xi}(s')\|^2 \le \tcont\nucont^{-1}\LF^2\exp(2\tcont \LF).
\end{align*}
\end{claim}
\begin{proof}[Proof of \Cref{claim:zxi_bound}]
We bound
\begin{align*}
\|\bz_{\xi}(s')\|^2 &= \left\|\int_{s=t-\tcont}^{s'}\Phiolpi(t,s)\Bpi(s)\bu_{\xi}(s)\rmd s\right\|^2\\
&\le \left(\int_{s=t-\tcont}^{s'}\|\Phiolpi(t,s)\|\|\Bpi(s)\|\|\bu_{\xi}(s)\|\rmd s\right)^2\\
&\le \LF^2\left(\int_{s=t-\tcont}^{s'}\|\Phiolpi(t,s)\|\|\bu_{\xi}(s)\|\rmd s\right)^2 \tag{\Cref{asm:max_dyn_final}}\\
&\le \LF^2\left(\int_{s=t-\tcont}^{s'}\|\Phiolpi(t,s)\|^2 \rmd s\right)^2 \left(\int_{s=t-\tcont}^{s'}\|\bu_{\xi}(s)\|\rmd s\right)^2 \tag{Cauchy-Schwartz}\\
&\le\LF^2\left(\int_{s=t-\tcont}^{t}\|\Phiolpi(t,s)\|^2 \rmd s\right) \left(\int_{s=t-\tcont}^{t}\|\bu_{\xi}(s)\|\rmd s\right)\\
&\le\nucont^{-1}\LF^2\left(\int_{s=t-\tcont}^{t}\|\Phiolpi(t,s)\|^2\rmd s\right).
\end{align*}
By \Cref{lem:bound_on_open_loop}, we can bound can bound $\|\Phiolpi(t,s)\| \le \exp(\LF(t-s))\le \exp(\LF \tcont)$ for $s \in [t-\tcont,t]$, yielding $\int_{s=t-\tcont}^{t}\|\Phiolpi(t,s)\|^2 \le \tcont\exp(2\LF \tcont)$. The bound claim. 
\end{proof}
Combining \Cref{eq:utilxis_int} and \Cref{claim:zxi_bound},
\begin{align*}
\int_{s=t-\tcont}^{t}\|\tilde{\bu}_{\xi}(s)\|^2  &\le 2\nucont^{-1}\left(1 + \tcont^2\LF^2\Lpi^2 \exp(2\LF \tcont)\right)\\
&\le 2\nucont^{-1}\Lpi^2\left(1 + \tcont^2\LF^2 \exp(2\LF \tcont)\right) \tag{$\Lpi \ge 1$}\\
&\le 2\nucont^{-1}\Lpi^2\left(1 + \gamcont^2 \exp(2\gamcont)\right) \tag{$\gamcont = \max\{1,\tcont \LF\}$}\\
&\le 4\nucont^{-1}\Lpi^2 \gamcont^2 \exp(2\gamcont) \tag{$\gamcont\ge 1$},
\end{align*}
which concludes the proof.
\end{proof}

\paragraph{Discretizing the Closed-Loop Gramian.} To conclude the argument, we relate the controllability of the closed-loop Gramian in continuous-time to that in discrete-time. 
\begin{lemma}[Discretization of Controllability Grammian]\label{lem:ctr_disc} Suppose \Cref{asm:ctr} holds and $\step \le \LF/4$, then following holds:
\begin{align*}
&\left\|\int_{s=t_k-\tcont}^{t_k}\Psiclpi(t_k,s)\Psiclpi(t_k,s)^\top \rmd s - \frac{1}{\step}\sum_{j=k-\kcont}^{k-1} \Psicldisc{k,j}(\Psicldisc{k,j})^\top \right\|_{\op} \\
&\quad\le 4\step \gamcont \Kpiinf^2\left( \KF\MF + 2\LF^2\right)
\end{align*}
\end{lemma}
\begin{proof} Recall the shorthand $\Lol := \exp(\step \LF)$, used in the discretization arguments in \Cref{app:dt_args}. We can write
\begin{align*}
&\left\|\int_{s=t_k-\tcont}^{t_k}\Psiclpi(t_k,s)\Psiclpi(t_k,s)^\top \rmd s - \step^{-1}\sum_{j=k-\kcont}^{k-1} \Psicldisc{k,j}(\Psicldisc{k,j})^\top\right\|\\
&= \left\|\sum_{j=k-\kcont}^{k-1} \int_{s =t_j}^{t_{j}+1} \Psiclpi(t,s)\Psiclpi(t,s)^\top \rmd s  - \frac{1}{\step} \Psicldisc{k,j}(\frac{1}{\step}\Psicldisc{k,j})^\top\right\|\\
&\le \step\sum_{j=k-\kcont}^{k-1} \max_{s \in \cI_j} \left\|\Psiclpi(t_k,s)\Psiclpi(t_k,s)^\top   - \frac{1}{\step} \Psicldisc{k,j}(\frac{1}{\step}\Psicldisc{k,j})^\top\right\|\\
&\le 2\step\sum_{j=k-\kcont}^{k-1} \max_{s \in \cI_j} \left\|\Psiclpi(t_k,s) - \frac{1}{\step} \Psicldisc{k,j}\right\| \cdot \max\left\{ \|\Psiclpi(t_k,s)\|, \frac{1}{\step} \|\Psicldisc{k,j}\|\right\}\\
&\le 2\LF\Lol\step \Kpiinf \sum_{j=k-\kcont}^{k-1} \max_{s \in \cI_j} \left\|\Psiclpi(t_k,s) - \frac{1}{\step} \Psicldisc{k,j}\right\| \tag{\Cref{lem:correct_Psiclpi_disc}(d)}\\
&\le 2\LF\Lol\Kpiinf^2\left( \KF\MF + 2\LF^2\right)\sum_{j=k-\kcont}^{k-1}  \step^2  \tag{\Cref{lem:correct_Psiclpi_disc}(b)}\\
&\le 2\LF\Lol^2\Kpiinf^2\left( \KF\MF + 2\LF^2\right)\sum_{j=k-\kcont}^{k-1}  \step^2 \\
&\le 2\LF\Lol^2\Kpiinf^2\left( \KF\MF + 2\LF^2\right) \kcont\step^2 \\
&=  2\step \tcont\LF\Lol^2\Kpiinf^2\left( \KF\MF + 2\LF^2\right). 
\end{align*}
As $\step \le \LF/4$, $\Lol^2 \le \exp(1/2) \le 2$ , so that the above is at most $4\step \tcont\LF\Kpiinf^2\left( \KF\MF + 2\LF^2\right) $. Recalling $\gamcont := \max\{1,\tcont \LF\}$ concludes. 
\end{proof}

\paragraph{Concluding the proof.} 
\begin{proof}[Proof of \Cref{prop:control_disc}] The proof follows by combining the bounds in \Cref{lem:ctr_disc,lem:ctr_closed_loop}. These yield
\begin{align*}
\frac{1}{\step}\lambda_{\min}\left(\sum_{j=k-\kcont}^{k-1} \Psicldisc{k,j}(\Psicldisc{k,j})^\top\right) \succeq \frac{\nucont}{4\Lpi^2\gamcont^2\exp(2\gamcont)} - \Kpiinf^2 \cdot 4\step \gamcont\left( \KF\MF + 2\LF^2\right)
\end{align*}
Recall $\gamcont = \tcont \LF$. Hence, if 
\begin{align*}
\step \le \frac{\nucont}{8\Lpi^2\Kpiinf^2 \gamcont^3\exp(2\gamcont)\left( \KF\MF + 2\LF^2\right)}, 
\end{align*}
it holds that 
\begin{align*}
\lambda_{\min}\left(\sum_{j=k-\kcont}^{k-1} \Psicldisc{k,j}(\Psicldisc{k,j})^\top\right) \succeq \frac{\nucont}{8\Lpi^2\gamcont^2\exp(2\gamcont)}.
\end{align*}
\end{proof}

\subsection{Recovery of State-Transition Matrix (\Cref{prop:A_est})}\label{sec:prop:A_est}
\newcommand{\cgramtarg}{\cgram_{k \mid k-1,j}}
\newcommand{\cgramden}{\cgram_{k-1 \mid k-1,j}}
\newcommand{\cgramhattarg}{\cgramhat_{k \mid k-1,j}}
\newcommand{\cgramhatden}{\cgramhat_{k-1 \mid k-1,j}}

The analysis is based on the Ho-Kalman scheme. We begin with the observation that 
\begin{align*}
\Psicldisc{k,j} = \bAclkpi\Psicldisc{k-1,j}, \quad \forall j < k. 
\end{align*}
To this end, define the matrices 
\begin{align*}
\cgram_{k \mid j_2,j_1} := [\Psicldisc{k+1,j_2} \mid \Psicldisc{k+1, j_2 - 1} \mid \dots \Psicldisc{k+1,j_1}], 
\end{align*}
Then, we have the identity
\begin{align*}
\cgramtarg  = \bAclkpi\cgramden,
\end{align*}
so that if $\rank(\cgramden) = \dimx$, we have $\bAclkpi = \cgramtarg\cgramden^{\dagger}$, where $(\cdot)^{\dagger}$ denotes the Moore-Penrose pseudoinverse. We now state and prove a more-or-less standard perturbation bound.
\begin{lemma}\label{lem:ls_pert} Suppose $\rank(\cgramden) = \dimx$, and consider any  estimates $\cgramhattarg,\cgramhatden$ of $\cgramtarg,\cgramden$. Define 
\begin{align*}
\Delta &:= \max\{\|\cgramtarg - \cgramhattarg\|,\|\cgramden - \cgramhatden\|\}\\
M &:= \max\{\|\cgramtarg\|,\|\cgramden\|\}.
\end{align*}
Then, if $\Delta \le \sigma_{\min}(\cgramden)/2$, the estimate $\btilAk := \cgramhattarg\cgramhatden^{\dagger}$ satisfies
\begin{align*}
\|\btilAk - \bAclkpi\| \le 6\Delta M\sigma_{\min}(\cgramden)^{-2}.
\end{align*}
\end{lemma}
\begin{proof}[Proof of \Cref{lem:ls_pert}]
Then, we have (provided $\rank(\cgramhatden) = \rank(\cgramden) = \dimx$), we have
\begin{align*}
\|\btilAk - \bAclkpi\| &= \|\cgramtarg\cgramden^{\dagger} - \cgramhattarg\cgramhatden^{\dagger}\|\\
&\le \|\cgramhatden^{\dagger}\|\|\cgramtarg- \cgramhattarg\| +  \|\cgramhatden^{\dagger} - \cgramden^{\dagger}\|\|\cgramtarg\|\\
&\le \|\cgramhatden^{\dagger}\|\|\cgramtarg- \cgramhattarg\| \\
&\qquad+  \frac{1+\sqrt{5}}{2}\|\cgramhatden^{\dagger}\|\cdot\|\cgramden^{\dagger}\|\cdot\|\cgramden- \cgramhatden\| \cdot \|\cgramhattarg\| \tag{cite \cite{stewart1977perturbation}, and also \cite{xu2020perturbation}}\\
&\overset{(i)}{\le} \Delta\|\cgramhatden^{\dagger} (1 + \frac{1+\sqrt{5}}{2}\|\cgramden^{\dagger}\|\|\cgramtarg\|)\\
&\overset{(ii)}{\le} 3\Delta M\|\cgramhatden^{\dagger} \|\cgramden^{\dagger}\|,
\end{align*} 
where in $(i)$ we use $\Delta := \max\{\|\cgramtarg - \cgramhattarg\|,\|\cgramden - \cgramhatden\|\}$, and in $(ii)$, we use $M= \max\{\|\cgramtarg\|,\|\cgramden\|\}$, which admits the simplification in $(ii)$ because $\|\cgramden\|\|\cgramden^{\dagger}\| \ge 1$. In particular, if $\rank(\cgramden) = \dimx $, and $\Delta \le \sigma_{\min}(\cgramden)/2$, then $\|\cgramhatden^{\dagger}\| \le 2/\sigma_{\min}(\cgramden)$, and we obtain
\begin{align*}
\|\btilAk - \bAclkpi\| \le 6\Delta M\sigma_{\min}(\cgramden)^{-2}.
\end{align*}
\end{proof}
Next, restricting our attention to $k \ge \kcont +2$, we specialize the above analysis to  
\begin{align*}
&\cgramin := \cgram_{k-1 \mid k-1,k-\knot+1}, \quad \cgramout := \cgram_{k \mid k-1,k-\knot+1}\\
&\cgramhatin := \cgramhat_{k-1 \mid k-1,k-\knot+1}, \quad \cgramhatout := \cgramhat_{k \mid k-1,k-\knot+1}
\end{align*}
where $\cgramhat_{(\cdot)}$ arises from the plug-in estimates 
\begin{align*}
\cgramhat_{k \mid j_2,j_1} := [\bhatPsi{k+1}{j_2} \mid \bhatPsi{k+1}{j_2 - 1} \mid \dots \bhatPsi{k+1}{j_1}]
\end{align*}
Define further
\begin{align*}
\btilAk := \cgramhatout\cgramhatin^{\dagger}, \quad \text{ so that } \bhatAk = \btilAk - \bhatBk \bKk^\pi.
\end{align*}
Recall $\tnot = \knot/\step$.
We can now bound, recalling $\Lol := \exp(\step \LF) \le 2$ for $\step \le \LF/4$ and  $\gamcont = \max\{1,\tcont \LF\} = \max\{1, \step \kcont \LF\}$, 
\begin{align*}
\max\{\|\cgramin\|,\|\cgramout\|\} &\le \sqrt{\knot}\max_{j < k} \|\Psicldisc{k,j}\|\\
&\le \sqrt{\kcont}\step  \LF \Lol \Kpiinf\tag{\Cref{lem:correct_Psiclpi_disc}(d)}\\
&\le \sqrt{\step t_0}\step  \LF \Lol \Kpiinf\tag{$t_0 = k_0 \step$ }\\
&\le 2 \Kpiinf \gamcont\sqrt{\step t_0}  \tag{$\gamcont \ge 1$}.
\end{align*}
Invoking \Cref{prop:control_disc}, we also have that provided $\step \le \min\{\stepdyn,\stepctrlpi\}$, since $\knot \ge \kcont + 2$,
\begin{align*}
\sigma_{\min}(\cgramden)^2 &= \lambda_{\min}\left(\sum_{j=k-k_0+1}^{k-1} \Psicldisc{k-1,j}(\Psicldisc{k-1,j})^\top\right) \\
&\ge \lambda_{\min}\left(\sum_{j=k-\kcont-1}^{k-1} \Psicldisc{k-1,j}(\Psicldisc{k-1,j})^\top\right)  \succeq \step \cdot \frac{\nucont}{8\Lpi^2\gamcont^2\exp(2\gamcont)}.
\end{align*}
Therefore, as long as 
\begin{align*}
\Delta := \max\{\|\cgramin - \cgramhatin\|, \|\cgramout -\cgramhatout\|\} \le \frac{\sqrt{\step \nucont}}{2\sqrt{2}\Lpi \gamcont \exp(\gamcont) },
\end{align*}
we have 
\begin{align*}
\|\btilAk - \bAclkpi\| \le \sqrt{\frac{\tnot}{\step}}\frac{96 \Delta}{\nucont}\cdot \Kpiinf\Lpi^2\gamcont^3\exp(2\gamcont). 
\end{align*}
Lastly, we can upper bound $\Delta \le \sqrt{\kcont} \errpsipi(\delta) = \sqrt{ \tnot/\step}\errpsipi(\delta)$, from which we cconlude that as long as $\errpsi(\delta) \le \step \frac{\sqrt{\nucont/\tnot}}{2\sqrt{2}\Lpi \gamcont \exp(\gamcont) }$, we have
\begin{align*}
\|\btilAk - \bAclkpi\| \le \tnot\Kpiinf\Lpi^2 \cdot \frac{96 \errpsipi(\delta)}{\step\nucont}\cdot \gamcont^3\exp(2\gamcont). 
\end{align*}
Now to wrap up. We observe that $\bBkpi = \Psicldisc{k+1,k}$, so 
\begin{align*}
\|\bhatBk - \bBkpi\| = \|\Psicldisc{k+1,k}-\bhatPsi{k+1}{k}\| \le \errpsipi(\delta).
\end{align*}
Therefore,
\begin{align*}
\|\bhatAk - \bAkpi\| &= \|\btilAk - (\bAclkpi -  \bBkpi \bKk^\pi)\|\\
&\le \|\btilAk - \bAclkpi\|+ \|\bhatBk \bKk^\pi- \bBkpi \bKk^\pi\|\\
&\le \|\btilAk - \bAclkpi\|+ \|\bhatBk - \bBkpi \|\Lpi\\
&\le \Lpi\errpsipi(\delta) + \tnot\Kpiinf\Lpi^2 \cdot \frac{96 \errpsipi(\delta)}{\step\nucont}\cdot \gamcont^3\exp(2\gamcont). 
\end{align*}
Lastly, we notice this upper bound on $\|\bhatAk - \bAkpi\|$ is larger than that on $\|\bhatBk - \bBkpi\|$, as $\Lpi \ge 1$ by assumption, and that for $\step \le \stepctrlpi$,  $\Lpi\errpsipi(\delta) \le \tnot\Kpiinf\Lpi^2 \cdot \frac{96 \errpsipi(\delta)}{\step\nucont}\cdot \gamcont^3\exp(2\gamcont)$.
Thus,
\begin{align*}
\|\bhatAk - \bAkpi\| \vee \|\bhatBk - \bBkpi\| \le \frac{\errpsipi(\delta)}{\step} \cdot t_0\Kpiinf\Lpi^2\gamcont^3\exp(2\gamcont) \cdot \frac{192}{\nucont}. 
\end{align*}

\qed

\newcommand{\Phidisc}{\discfont{\Phi}}
\newcommand{\Phice}{\discfont{\Phi}^{\mathrm{ce}}}

\section{Certainty Equivalence}
\label{append:certainty_equiv}
In this section, we establish a general certainty equivalence bound for linear time-varying discrete-time systems;  we apply this in the proof \Cref{prop:ce_bound} in \Cref{sec:prop:ce_bound}.

Let $\bstTheta := (\bstAk[1:K],\bstBk[1:K])$ denote ground-truth system parameters, and let $\bhatTheta := (\bhatAk[1:K],\bhatBk[1:K])$ denote estimates. We work with a slightly different discretization parameterization, where the dynamics are given by $\bxk[h+1] = \bAk\bxk[h] + \step \bBk\buk[h]$. This parametrization ensures that the norms of $\bBk$ scale like constants independent of $\step$ when instantiated with $\bAk \gets \bAkpi$ and $\bBk \gets \frac{1}{\step}\bBkpi$. 
\begin{definition}Given cost matrices $\bQ,\bR$, step $\step$, and parameters $\bTheta = (\bAk[1:K],\bBk[1:K])$, we define $\Poptk(\bTheta)$ as the solution to the following program
\begin{equation}\label{eq:popt}
\begin{aligned}
x^\top \Poptk(\bTheta) x &= \min_{\bu_{k:H}} \bxk[K+1]^\top \bQ\bxk[K+1] + \step \sum_{h = k}^{K} (\bxk[h]^\top \bQ \bxk[h] + \buk[h]^\top \bQ \buk[h])  \\\
&\text{ s.t. } \bxk[h+1] = \bAk[h]\bxk[h] + \step \bBk[h]\buk[h], \quad \bxk = x.
\end{aligned}
\end{equation}
\end{definition}
The closed form for $\Poptk$ is given by the follow standard computation, modified with the reparametrized dynamics $\bxk[h+1] = \bAk[h]\bxk[h] + \step \bBk[h]\buk[h]$)
\begin{lemma}\label{lem:Poptk_dp} The optimal Riccati cost-to-go $\Poptk[1:K+1] = \Poptk(\bTheta)$ is given by the solution to the following recursion with final condition $\Poptk[K+1] = \bQ$ and
\begin{align*}
\Poptk[k] = \bAk^\top \Poptk[k+1]\bAk - \step\left(\bBk\Poptk[k+1]\bAk\right)^\top( \bR + \step\bBk^\top \Poptk[k+1]\bBk)^{-1} \left(\bBk\Poptk[k+1]\bAk\right) + \step \bQ
\end{align*}
Moreover, defining $\Koptk = \Koptk(\Theta) := -(\bR + \step \bBk^\top \Poptk \bBk)^{-1}\bBk^\top \Poptk\bAk$, the optimal control for \Cref{eq:popt} is given by $\bxk = \Koptk \buk$.
\end{lemma}

\begin{proof} This follows by reparametrizing the standard discrete-time Ricatti update (see e.g. \citet[Section 2.4]{anderson2007optimal}), with $\bBk \gets \step \bBk$, $\bQ \gets \step \bQ$, and $\bR \gets \step \bR$, and simplifying dependence on $\step$.
\end{proof}
The following identity is also standard (again, consult \citet[Section 2.4]{anderson2007optimal}, albeit again with the reparamerizations $\bBk \gets \step \bBk$, $\bQ \gets \step \bQ$,and $\bR \gets \step \bR$):
\begin{align}
\Poptk = (\bAk + \step \bBk \Koptk)^\top \Poptk[k+1](\bAk + \step \bBk \Koptk) + \step(\bQ + (\Koptk)^\top\bR(\Koptk)) \label{eq:Poptk_lyapunov_form}
\end{align}

Next, we define the cost-to-go functions associated for arbitrary sequences of feedback matrices, and from the optimal feedback matrices from another instance $\bTheta'$.
\begin{definition}[Feedback and Certainty Equivalent Cost-to-go]\label{defn:Pfeed} Given a sequence of feedback gains $\bKk[1:K]$, we define the induced cost-to-go
\begin{align*}
\Pfeedk(\bTheta;\bKk[1:K]) &:= \bxk[K+1]^\top \bQ\bxk[K+1] + \step \sum_{h = k}^{K} (\bxk[h]^\top \bQ \bxk[h] + \buk[h]^\top \bQ \buk[h]) \\\
&\text{ s.t. } \bxk[h+1] = (\bAk[h]+\step \bBk[h] \bKk)\bxk[h] \quad \bxk = x.
\end{align*}
And define the \emph{certainty equivalent} cost-to-go as $\Pcek(\bTheta;\bTheta') = \Pfeedk(\bTheta;\Koptk[1:K](\bTheta'))$ as the feedback cost-to-go for $\bTheta$ using the optimal gains for $\bTheta'$. 
\end{definition}
In particular, $\Pcek(\bTheta;\bTheta) = \Poptk(\bTheta)$.  We now present upper bounds on $\Pcek(\bTheta;\bTheta')$. We assume bounds on the various parameters of interest. 

\begin{condition}\label{asm:par_bounds} We have that there are constants $\KB,\KA \ge 1$ such that, for all $k \in [K]$, 
\begin{align*}
\|\bstBk\|\vee\|\bhatBk\| \le \KB \quad 
\|\bstAk\|\vee\|\bhatAk\| \le \KA,
\end{align*}
\end{condition}
\begin{condition}\label{asm:par_diffs} We assume that there exists $\DelA,\DelB > 0$,
\begin{align*}
\forall k, \quad \|\bhatBk-\bstBk\| \le \DelB \text{ and } \|\step^{-1}(\bhatAk-\bstAk)\| \le \DelA.
\end{align*}
\end{condition}

\begin{condition}\label{asm:cost_norm}  We assume the a normalization on the cost matrices satisfy $\bR \succeq \eye$, $\bQ \succeq \eye$ and $\|\bQ\| \ge \|\bR\|$. As a special case, $\bQ = \eye$ and $\bR = \eye$ suffices. 
\end{condition}
Lastly, the following assumption is needed to derive an upper bound on the closed-loop transition operator.
\newcommand{\kapA}{\kappa_{A}}
\begin{condition}\label{asm:small_a_diff} We assume that $\max_{k}\|\bAk - \eye\| \le \step\kapA$.
\end{condition}
\begin{theorem}[Main Perturbation]\label{thm:main_pert} Suppose \Cref{asm:par_diffs,asm:par_bounds,asm:cost_norm} hold. 
    Define  the terms 
    \begin{align*}
    \Delce := 80 C^4\KA^3\KB^3(1+\step C\KB)(\DelA+ \DelB), \quad C := \max_{k \in [K+1]}\|\Poptk(\bTheta)\|.
    \end{align*}  Then, as long as  $\Delce < 1$, we have
    \begin{itemize}
        \item[(a)] 
    \begin{align*}
    \max_{k \in [K+1]}\|\Pcek(\bTheta;\bhatTheta)\| \le \left(1 - \Delce \right)^{-1}\max_{k \in [K+1]}\|\Poptk(\bTheta)\|
    \end{align*}
    \item[(b)] $\max_{k \in [K+1]}\|\Koptk(\bhatTheta)\|\le \frac{5}{4}\KB\KA C$. 
    \item[(c)] Moreover, if \Cref{asm:small_a_diff} holds, then the transition operators defined as
    \begin{align*}
    \Phice_{j,k} := (\bAk[j-1] + \step \bBk[j-1]\Koptk[j-1](\bhatTheta))\cdot (\bAk[j-2] + \step \bBk[j-2]\Koptk[j-2](\bhatTheta)) \cdot \dots \cdot (\bAk[k] + \step \bBk[k]\Koptk[k](\bhatTheta)),
    \end{align*}
    with the convention $\Phice_{k,k+} = \eye$ 
    satisfy, for all $1\le j \le k \le K$,
    \begin{align*}
    \|\Phice_{j,k}\|^2 \le 2\kappa(1 - \step \gamma)^{j-k}, \quad \text{where } \kappa = \kappa_{A}+ \frac{5}{4}\KB^2\KA C , \quad \gamma = \frac{1-\Delce}{C},
    \end{align*}
    provided $\kappa \le 1/2\step$.
\end{itemize}
    \end{theorem}
    The proof of the \Cref{thm:main_pert} is outlined in  \Cref{sec:thm:main_pert_overview}, and the supporting lemmas are proved in the subsequent sections. We now use this guarantee to establish \Cref{prop:ce_bound}.

\subsection{Proof Overview of \Cref{thm:main_pert}}\label{sec:thm:main_pert_overview}

\paragraph{Step 0: Notation \& Interpolating segments.}
 To simplify notation, introduce the maximal operator norms, such that for an $H$-tuple of matrices $\bXk[1:H] = (\bXk[1],\dots,\bX[H])$,
\begin{align*}
\|\bXk[1:H]\|_{\maxop} := \max_{k \in [H]}\|\bXk\|_{\op}.
\end{align*}
Let us a consider the line segment joining these the parameters
\begin{align}
\bTheta(s) =  (\bAk[1:K](s),\bBk[1:K](s)) =  (1-s)\bstTheta + s\bhatTheta \label{eq:interpolator}
\end{align}
For fixed cost matrices $\bQ$ and $\bR$, let $\bPk[1:K+1](s)$ and $\bPk[1:K](s)$ denote the solution to the finite-time Riccati recursion with parameters $\bTheta(s)$, where here $\bQ$ also serves as a terminal cost at step $K+1$. We let $\bstPk[1:K+1],\bstKk[1:K]$ be the solution for the truth $\bstTheta$ and $\bhatP_{1:H},\bhatK_{1:H}$ the solution to the Riccati equation with $\bhatTheta$; i.e. the certainty equivalent solution. By construction,
\begin{align*}
(\bPk[1:K+1](0), \bKk[1:K](0)) = (\bstPk[1:K+1],\bstKk[1:K]), \quad (\bPk[1:K+1](1), \bKk[1:K](1)) = (\bhatPk[1:K+1],\bhatKk[1:K+1])
\end{align*}
For all quantity $\bX(s)$ paramterized by $s \in [0,1]$, adopt the shorthand $\bX'(s) := \dds \bX(s)$. 

\paragraph{Step 1. Self-Bounding ODE Method.}  We use an interpolation argument to study the certainty equivalence controller. Our main tool is the following interpolation bound, which states that if the norm of the $s$-derivative of a quantity is bounded by the norm of the quantity its self, then that quantity is uniformly bounded on a small enough range.
\begin{lemma}[Self-Bounding ODE Method, variant of Corollary 3 in \cite{simchowitz2020naive}]\label{lem:cor_polynomial_comparison} Fix dimensions $d_1,d_2 \ge 1$, let $\cV \subset \R^{d_1}$, let $f: \cV \to \R^{d_2}$ be a $\cC^1$ map and let $\bv(s):[0,1] \to \cV$ be a $\cC^1$ curve defined on $[0,1]$. Finally,  let $\|\cdot\|$ be an arbitrary norm on $\R^{d_2}$  and suppose that $c > 0$ and $p \ge 1$ satisfy
\begin{align}
\|\dds f(\bv(s))\| \le c\max\{\|f(\bv(s))\|, \|f(\bv(0))\|\}^p \quad \forall s \in [0,1]. \label{eq:self_bounding_condition}
\end{align}
Then, if $p > 1$ and if $\alpha = c(p-1)\|f(\bv(0))\|^{p-1}$ satisfies $\alpha < 1$, the following bound holds for all $s \in [0,1]$:
\begin{align*}
\|f(\bv(s)\| \le (1-\alpha)^{-\frac{1}{p-1}}\|f(\bv(0))\|, \quad \|\dds f(\bv(s)\| \le c(1-\alpha)^{-\frac{p}{p-1}}\|f(\bv(0))\|^p
\end{align*}
\end{lemma}

\paragraph{Step 2. Perturbation of $\bPk[1:K+1](t)$} First, we show that the Riccati-updates obey the structure of \Cref{lem:cor_polynomial_comparison}. 
\begin{lemma}\label{lem:main_self_bounding} Suppose (for simplicity) that $\lambda_{\min}(\bQ),\lambda_{\min}(\bR) \ge 1$. Then, for all $s \in [0,1]$
\begin{align*}
\|\bPk[1:K+1]'(s)\|_{\maxop} \le 2(\DelA + \KA\KB\DelB)\|\bPk[1:K+1](s)\|_{\maxop}^3,
\end{align*}
\end{lemma}

Our next result gives uniform bounds on $\bPk[1:K+1]$ and its derivative by invoking \Cref{lem:cor_polynomial_comparison}.

\begin{lemma}\label{lem:bounds_on_P} Suppose (for simplicity) that $\lambda_{\min}(\bQ),\lambda_{\min}(\bR) \ge 1$, and that $(\DelA + \KA\KB\DelB) \le 1/8\|\Poptk[1:K+1](\bTheta)\|_{\maxop}^2$. Then, for all $s \in [0,1]$,
\begin{align*}
\|\bPk[1:K+1](s)\|_{\maxop} &\le 1.8\|\bstPk[1:K+1]\|_{\maxop}\\
\|\bPk[1:K+1]'(s)\|_{\maxop} &\le  12(\DelA + \KA\KB\DelB)\|\bstPk[1:K+1]\|_{\maxop}^3
\end{align*}
\end{lemma}
As the gains $\bPk[1:K](s)$ are explicit function of $\bPk[1:K+1](s)$, we optain the following perturbation bound for the gains. 
\begin{lemma}\label{lem:K_perturb}Under the assumptions of \Cref{lem:bounds_on_P}, the following holds:
\begin{align*}
\|\Koptk[1:K](\bTheta) - \Koptk[1:K](\bhatTheta)\|_{\maxop}  \le 20 C^3\KA^3\KB^2(1+\step C\KB)(\DelA+ \DelB) , \quad C := \|\Poptk[1:K+1](\bTheta)\|
\end{align*}
\end{lemma}

\subsubsection{Proof of \Cref{thm:main_pert} }
 \paragraph{Proof of part (a). }Consider the curve 
    \begin{align*}
    \bKk(s) = (1-s)\Koptk(\bstTheta) + s\Koptk(\bhatTheta).
    \end{align*}
    We then note that the curve  
    \begin{align*}
    \Pcek(s) = \Pfeedk(\bTheta;\bKk[1:K](s))
    \end{align*}
    satisfies $\Pcek(0) = \Pcek(\bstTheta;\bstTheta) = \Poptk(\bstTheta) = \bstPk$ and $\Pcek(0) = \Pfeedk(\bstTheta;\Koptk(\bhatTheta)) = \Pcek(\bstTheta;\bhatTheta)$

    By \Cref{defn:Pfeed}, we can write $\Pcek(s) = \bLamk(s)$, where $\bLamk$ solve the following Lyapunov equation 
    \begin{equation}\label{eq:lyap_ce_solve}
    \begin{aligned}
    &\Pcek[K+1](s) = \bQ, \quad \Pcek(s) = \bXk(s)^\top \Pcek[k+1](s) \bXk(s) + \step \bQ(s) + \bYk(s) \quad \text{where }\\
    &\bXk(s) := \bAk + \step \bBk \bKk(s)\quad \text{and} \quad \bYk(s) := \step\bK(s)^\top\bR\bK(s).
    \end{aligned}
    \end{equation}
    As $\bXk'(s) = \step \bBk\bKk'(s)$ and $\bYk'(s) = \Sym(\bK(s)^\top\bR\bK'(s)$), salient term from \Cref{prop:lyap_pert_max_not_factored} evaluates to  
    \begin{align*}
    &\Delta(s) = \max_{j \in [k]} \step^{-1}\left(2\|\bXk[j](s)'\| +  \|\Pcek[1:K+1](s)\|_{\maxop}^{-1}\|\bYk[j](s)'\|\right) \\
    &=\max_{j \in [k]}\step^{-1}\left(2\step\KB\|\bK'(s)\| +  \step\|\Pcek[1:K+1](s)\|_{\maxop}^{-1}\|\bKk(s)\|\|\bR\|\bKk'(s)\|\right) \\
    &=\max_{j \in [k]} \left(2\KB+  \|\Pcek[1:K+1](s)\|_{\maxop}^{-1}\|\bKk(s)\|\right) \|\bKk'(s)\|  \tag{$\bR = \eye$}\\
    \end{align*}
    We further bound
    \begin{align*}
    \|\bKk(s)\| &= \|(1-s)\Koptk(\bTheta) + s\Koptk(\bhatTheta)\|\\
    &\le  \|\Koptk(\bstTheta)\| \vee \|\Koptk(\bhatTheta)\|\\
    &\le  \|(\bR+\step\bBk^\top \Poptk(\bstTheta)\bBk)^{-1}(\bBk^\top \Poptk(\bstTheta) \bAk) \| \vee \|(\bR+\step\bhatBk^\top \Poptk(\bhatTheta)\bhatB)^{-1}(\bhatBk^\top \Poptk(\hat{\bTheta}) \bhatAk)\|\\
    &\le  \|\bBk^\top \Poptk(\bstTheta) \bAk \| \vee \|\bhatBk^\top \Poptk(\hat{\bTheta}) \bhatAk\| \tag{$\bR = \eye$}\\
    &\le  \KB\KA(\| \Poptk(\bstTheta) \| \vee \|\Poptk(\hat{\bTheta})\|)\\
    &\le  2\KB\KA\| \Poptk[1:K+1](\bstTheta) \|_{\maxop} \tag{\Cref{lem:K_perturb}}\\
    &=  2\KB\KA\| \Pcek[1:K+1](0) \|_{\maxop} \tag{definition of $\Pcek$}
    \end{align*}
    Thus,
    \begin{align*}
    \Delta(s) &\le\left(2\KB+  2\KB\KA\|\Pcek[1:K+1](s)\|_{\maxop}^{-1}\| \Pcek[1:K+1](0) \|_{\maxop} \right)  \max_{j \in [k]}\|\bKk'(s)\|  \tag{$\bR = \eye$}\\
    &\le 4\KB\KA\left(1 \vee \|\Pcek[1:K+1](s)\|_{\maxop}^{-1}\| \Pcek[1:K+1](0) \|_{\maxop} \right)  \max_{j \in [k]}\|\bKk'(s)\|  \tag{$\KA \ge 1$}.
    \end{align*}
    Hence, setting $\DelK := \sup_{s \in [0,1]}\max_{j \in [k]}\|\bKk'(s)\|$, \Cref{prop:lyap_pert_max_not_factored} implies
    \begin{align*}
     \Pcek[k](s)' &\le \|\Pcek[1:K+1](s)\|_{\maxop}^2\Delta(s)\\
     &\le 4\KB\KA \|\Pcek[1:K+1](s)\|_{\maxop}^2 \left(1 \vee \|\Pcek[1:K+1](s)\|_{\maxop}^{-1}\| \Pcek[1:K+1](0) \|_{\maxop} \right)  \max_{j \in [k]}\DelK\\
     &\le 4\KB\KA  \left(\|\Pcek[1:K+1](s)\|_{\maxop}^2 \vee \| \Pcek[1:K+1](0) \|_{\maxop}^2 \right)  \DelK
    \end{align*}
    Hence, \Cref{lem:cor_polynomial_comparison} implies that as long as $4\KB\KA  \DelK\| \Pcek[1:K+1](0) \|_{\maxop} < 1$, we have 
    \begin{align*}
     \sup_{s \in [0,1]}\| \Pcek[1:K+1](s) \|_{\maxop} \le (1 - 4\KB\KA  \DelK\| \Pcek[1:K+1](0) \|_{\maxop})^{-1}\| \Pcek[1:K+1](0) \|_{\maxop}
    \end{align*}
    Using $\Pcek[1:K+1](0) = \Poptk[1:K+1](\bstTheta)$, $\Pcek[1:K+1](1) = \Pcek[1:K+1](\bstTheta;\bhatTheta)$, and defining the shorthand
    \begin{align*}
     C := \|\Poptk[1:K+1](\bstTheta)\|,
    \end{align*}
     we conclude that for any upper bound $\Delta \ge 4\KB\KA  \DelK C$ satisfying $\Delta < 1$, 
    \begin{align*}
     \|\Pcek[1:K+1](\bstTheta;\bhatTheta) \|_{\maxop} \le (1 - \Delta)^{-1}\| \Poptk[1:K+1](\bstTheta) \|_{\maxop}.
    \end{align*}
    By \Cref{lem:K_perturb}, it holds thats if $8\|\Poptk[1:K+1](\bstTheta)\|_{\maxop}^2(\DelA + \KA\KB\DelB) < 1$, we can bound. we can take $\DelK \le 20 C^3\KA^3\KB^2(1+\step C\KB)(\DelA+ \DelB)$. Hence, we can bound
    \begin{align*}
    4\KB\KA  \DelK C \le 80 C^4\KA^4\KB^3(1+\step C\KB)(\DelA+ \DelB) := \Delce,
    \end{align*}
    which concludes the proof of part (a).

    \paragraph{Proof of part (b).}
    We bound
    \begin{align*}
    \|\Koptk[1:K](\bhatTheta)\|_{\maxop} &\le (\|\Koptk[1:K](\bstTheta)\|_{\maxop} + 1/4\KB) \tag{\Cref{lem:K_perturb}, definition of $\Delce$, and using $\KB,\KA,C \ge 1$}\\
    &= (1/4\KB) +\|(\bR + \step \bBk^\top \Poptk(\bstTheta) \bBk)^{-1}\bBk^\top \Poptk(\bstTheta) \bAk\|_{\maxop}  \tag{Definition of $\Koptk$, \Cref{defn:discrete_riccati}}\\
    &= \frac{1}{4\KB} \KB\KA\|\Poptk[1:K+1](\bstTheta)\|_{\maxop}  \tag{Definition of $\KA,\KB$ in \Cref{asm:par_bounds}, $\bR \succeq \eye$}\\
    &= \frac{1}{4\KB} +  \KB \KA C  \tag{Definition of $C$}\\
    &\le \frac{5}{4}\KB\KA C\tag{$C,\KA,\KB \ge 1$}.
    \end{align*}

    \paragraph{Proof of part (c).} We aim to bound the square of the operator norm of the following term
    \begin{align*}
    \Phice_{j,k} := (\bAk[j-1] + \step \bBk[j-1]\Koptk[j-1](\bhatTheta))\cdot (\bAk[j-2] + \step \bBk[j-2]\Koptk[j-2](\bhatTheta)) \cdot \dots \cdot (\bAk[k] + \step \bBk[k]\Koptk[k](\bhatTheta)).
    \end{align*}
    Using the fact that $\Pcek(\bTheta;\bhatTheta)$ solves the Lyapunov equation \Cref{eq:lyap_ce_solve}, it follows from \Cref{lem:Lyap_lb_disc} that if 
    \begin{align*}
    \kappa_0 &:= \step^{-1}\max_k\|\eye - (\bAk + \step \bBk \bKk(1))\|_{\op} = \step^{-1}\max_k\|\eye - (\bAk + \step \bBk \Koptk(\bhatTheta)\|_{\op} \le 1/2\step,
    \end{align*}
    then
    \begin{align}
    \|\Phice_{j,k}\|^2 \le \max\{1,2\kappa_0\}(1 - \step \gamma_0)^{j-k}, \quad \gamma_0 := \frac{1}{\|\Pcek[1:K+1](\bTheta;\bhatTheta) \|_{\maxop}}  \label{eq:Phidisc_bound}
    \end{align}

    From part (a), we can lower bound $\gamma_0 \ge \gamma := \frac{1-\Delce}{C}$. Moreover, we can bound $\kappa_0$
    \begin{align*}
    \kappa_0 &\le  \step^{-1}\max_{k}\|\eye - \bAk\| + \max_{k}\|\bBk\|\|\Koptk(\bhatTheta)\|\\
    &\le  \kappa_{A}+ \KB\|\Koptk[1:K](\bhatTheta)\|_{\maxop} \tag{\Cref{asm:par_bounds,asm:small_a_diff}}\\
    &\le \kappa_{A}+ \frac{5}{4}\KB^2\KA C := \kappa \tag{\Cref{thm:main_pert}(b)}
    \end{align*}
    As $\kappa \ge 1$, we conclude via \Cref{eq:Phidisc_bound} that
    \begin{align*}
    \|\Phice_{j,k}\|^2 \le 2\kappa(1 - \step \gamma)^{j-k}, \quad \kappa = \kappa_{A}+ \frac{5}{4}\KB^2\KA C , \quad \gamma = \frac{1-\Delce}{C}.
    \end{align*}

    \qed

\subsection{Proof of \Cref{lem:main_self_bounding}}
    To apply the self-bounding ODE method, we bound $\bP'_{1:H}(s)$ in terms of $\bPk[1:K+1](s)$.
    To prove \Cref{lem:main_self_bounding}
    Let us first introduce some notation. Further, for simplicity, we shall suppress $s$ in equations and let $(\cdot)\big{|}_{s}$ to denote evaluation at $s$. With this convention, define the matrices 
    \begin{align}
    \bXik(s) &:= (\bAk' + \step \bK_k  \bBk')^\top \bPk[k+1] (\bAk + \step\bBk \bK_k)
     \label{eq:Xik}
    \end{align}
    and define the operator
    \begin{align}
    \cT_{k+1}(\cdot\,;s) &:=  \{(\bAk+\bBk \bK_k)^\top(\cdot) (\bAk+\bBk \bK_k)\} \big{|}_{s}, \quad \text{with the convention } \cT_{k+1}(\cdot)\big{|}_{s} = \cT_{k+1}(\cdot;s),
    \end{align}
    Lastly, we define their compositions
    \begin{align*}
    \cT_{k+i,k} := \cT_{k}(\cdot) \circ \cT_{k+1}(\cdot) \circ \dots \circ \cT_{k+i}(\cdot) \big{|}_{s}, 
    \end{align*}
    with the convention $\cT_{k;k}$ is the identity map. These operators give an expression for the derivatives $\bP_{k-1}'(s)$.
    \begin{lemma}\label{lem:Pprime_comp} For all $s$, it holds that 
    \begin{align*}
    \bP_{k}'(s) = \sum_{j=k}^{K+1} \cT_{k;j}(\bXik + \bXik^\top) \big{|}_{s}
    \end{align*}
    \end{lemma}
    \begin{proof} Let $\Sym(\bX) = \bX + \bX^\top$.
    The Ricatti update (backwards in time) is
    \begin{align*}
    \bP_{k} &= \bAk^\top \bPk[k+1] \bAk - \step(\bAk^\top \bPk[k+1] \bBk)(\bR + \step \bBk^\top \bPk[k+1] \bBk)^{-1}(\bAk^\top \bPk[k+1] \bBk)^\top + \step\bQ
    \end{align*}
    Let $\mathsf{Sym}(\bX) := \bX + \bX^\top$. Then, we compute
    \begin{align*}
    &\bP_{k}(s)'\\
    &= \mathsf{Sym}\left((\bAk')^\top \bPk[k+1] \bAk \right)- \step\mathsf{Sym}\left((\bAk')^\top \bPk[k+1] \bBk + \bAk^\top\bPk[k+1](\bBk') )(\bR + \step\bB^\top \bPk[k+1] \bB)^{-1}(\bAk^\top \bPk[k+1] \bBk)^\top  \right)\\
    &\quad+ \step^2(\bAk^\top \bPk[k+1] \bBk) (\bR + \step\bBk^\top \bPk[k+1] \bBk)^{-1}(\mathsf{Sym}((\bBk')^\top \bPk[k+1] \bBk) )(\bR + \step\bBk^\top \bPk[k+1] \bBk)^{-1}(\bAk^\top \bPk[k+1] \bBk)^\top \\
    &\quad+ \bAk^\top (\bPk[k+1]') \bAk - \step\mathsf{Sym}\left((\bAk^\top (\bPk[k+1]') \bBk)(\bR + \step\bBk^\top \bPk[k+1] \bBk)^{-1}(\bAk^\top \bPk[k+1] \bBk)^\top\right)\\
    &\quad + \step^2(\bAk^\top \bPk[k+1] \bBk)(\bR + \bBk^\top \bPk[k+1] \bBk)^{-1}(\bBk^\top \bPk[k+1]' \bBk^\top) (\bR + \bBk^\top \bPk[k+1] \bBk)^{-1}(\bAk^\top \bPk[k+1] \bBk)^\top \\
    &= \mathsf{Sym}\left( (\bAk')^\top \bPk[k+1] \bAk + \step((\bAk')^\top \bPk[k+1] \bBk + \bAk^\top\bPk[k+1](\bBk') )\bK_k   + \step^2\bK_k^\top (((\bBk')^\top \bPk[k+1] \bBk) )\bK_k\right) \\
    &\quad+ \bAk^\top (\bPk[k+1]') \bAk + \step \mathsf{Sym}\left((\bAk^\top (\bPk[k+1]') \bBk)\bK_k\right) +\step^2(\bBk\bK_k)
    (\bPk[k+1]')(\bBk\bK_k)\\
    &= \mathsf{Sym}\left( (\bAk')^\top \bPk[k+1] \bAk + \step((\bAk')^\top \bPk[k+1] \bBk + \bAk^\top\bPk[k+1](\bBk') )\bK_k   + \step^2\bK_k^\top ((\bBk')^\top \bPk[k+1] \bBk) )\bK_k\right) \\
    &\quad+ (\bAk + \step\bBk \bK_k)^\top (\bPk[k+1]') (\bAk + \step\bBk \bK_k)^\top
    \end{align*}
    where above we the fact that $\bK_k = -(\bR + \step\bBk^\top \bPk[k+1] \bBk)^{-1}(\bAk^\top \bPk[k+1] \bBk)^\top$.  Noting that
    \begin{align*}
    &\mathsf{Sym}\left( (\bAk')^\top \bPk[k+1] \bAk + \step((\bAk')^\top \bPk[k+1] \bBk + \bAk^\top\bPk[k+1](\bBk') )\bK_k   + \step^2\bK_k^\top ((\bBk')^\top \bPk[k+1] \bBk) )\bK_k\right)\\
    &\overset{(i)}{=}\mathsf{Sym}\left( (\bAk')^\top \bPk[k+1] (\bAk + \step\bBk \bK_k) + \step\bAk^\top\bPk[k+1](\bBk') \bK_k   + \step^2\bK_k^\top ((\bBk')^\top \bPk[k+1] \bBk) )\bK_k\right)\\
    &\overset{(ii)}{=}\mathsf{Sym}\left( (\bAk')^\top \bPk[k+1] (\bAk + \step\bBk \bK_k) + \step\bK_k^\top (\bBk') \bPk[k+1] \bAk   + \step^2\bK_k^\top ((\bBk')^\top \bPk[k+1] \bBk) )\bK_k\right) \\
    &=\mathsf{Sym}\left( (\bAk')^\top \bPk[k+1] (\bAk + \step\bBk \bK_k) + \step\bK_k^\top (\bBk') \bPk[k+1] (\bAk  + \step \bBk\bK_k)\right) \\
    &=\mathsf{Sym}\left( (\bAk' + \step \bK_k  \bBk')^\top \bPk[k+1] (\bAk + \step\bBk \bK_k)\right)  := \Sym(\bXik)
    \end{align*}
    Therefore, we have
    \begin{align*}
     \bP_{k}' =  \cT_{k+1}(\bPk[k+1]')+ \mathsf{Sym}(\bXik)
    \end{align*}
    Thus, unfolding the recursion, we have
    \begin{align*}
    \bP_{k}'  &= \cT_{k+1}(\bPk[k+1]') + \mathsf{Sym}(\bXik)\\
    &= \cT_{k+1}(\cT_{k+2}(\bPk[k+1]') + \mathsf{Sym}(\bXik[k+1])) + \mathsf{Sym}(\bXik)\\
    &= \cT_{k+1}(\cT_{k+2}(\bPk[k+1]')) +  \cT_{k+1}(\mathsf{Sym}(\bXik[k+1])) + \mathsf{Sym}(\bXik)\\
    &= \dots \\
    &=   \sum_{j=k}^{K+1} \cT_{k;j}(\mathsf{Sym}(\bXik[k+j])).
    \end{align*}
    \end{proof}   
    
    Using this fact, a standard Lyapunov argument gives a generic upper bound on sums of these operators. 
    \begin{lemma}\label{lem:operator_sum_bound} The operators $\cT_{j,k}(\cdot;s)$ are matrix monotone. Hence, if $\bXk[1:K]$ are any sequence of $\R^{n \times n}$ matrices, 
    \begin{align*}
    \|\sum_{j=k}^{K+1} \cT_{j,k}(\bXk[j]+\bXk[j]^\top;s)\|_{\op} \le \frac{2\|\bPk\|_{\op}\max_{j \ge k}\|\bXk\|_{\op}}{\step} \big{|}_{s}
    \end{align*}
    Consequently, by \Cref{lem:Pprime_comp},
    \begin{align}
    \|\bPk'(s)\|_{\op} \le \frac{2\|\bPk\|_{\op}\max_{j \ge k}\|\bXik\|_{\op}}{\step} \big{|}_{s}. \label{eq:Pprime}
    \end{align}
    \end{lemma}
    \begin{proof} 
    This is a direct consequence of rewriting $\bPk$ as in \Cref{eq:Poptk_lyapunov_form}, applying \Cref{lem:Lyap_lb_disc}(a) with $\bXk = \bAk + \bBk\bKk$, and upper bounding $\|\bXk[j]+\bXk[j]^\top\| \le \|\bXk[j]\|_{\op}$. 
    \end{proof}
    Finally, let us upper bound the norm of the matrices $\bXik$

    \begin{lemma}\label{lem:Xibound}
    \begin{align*}
    \step^{-1}\|\bXik(s)\| \le \left(\DelA +  \DelB  \KA\KB \right)\|\bPk[1:K+1]\|_{\maxop}^2.
    \end{align*}
    \end{lemma}
    \begin{proof}[Proof of \Cref{lem:Xibound}]  Recall
    \begin{align*}
    \bXik := (\bAk' + \step \bK_k  \bBk')^\top \bPk[k+1] (\bAk + \step\bBk \bK_k),
    \end{align*}
    \begin{align*}
    \|\bK_k\| &= \|(\bR + \step\bBk^\top \bPk[k+1] \bBk)^{-1}\bBk^\top\bPk[k+1] \bAk\|\\
    &= \lambda_{\min}(\bR)^{-1}\|\bBk^\top\bPk[k+1] \bAk\|\\
     &\le \lambda_{\min}(\bR)^{-1}\|\bBk\|\|\bAk\| \|\bPk[k+1]\|\\
     &\le \lambda_{\min}(\bR)^{-1}\KA\KB \|\bPk[k+1]\|\\
     &\le \KA\KB \|\bPk[k+1]\| \numberthis \label{eq:K_bound}
    \end{align*}
    Next,
    \begin{align*}
    &\step^{-1}\|(\bAk' + \step \bK_k  \bBk')^\top \bPk[k+1] (\bAk + \step\bBk \bK_k)\|\\
     &\le \step^{-1}\left(\|\bAk'\| + \step \|\bBk'\|\|\bK_k\|\right)\|\bPk[k+1]^{\half}\|\|\bPk[k+1]^{\half}(\bAk + \bBk\bKk)\| \\
    &\overset{(i)}{\le} \step^{-1}\left(\|\bAk'\| + \step \|\bBk'\|\|\bK_k\|\right)\|\bPk[k+1]^{\half}\|\|\bPk^{\half}\|\\
    &\le \step^{-1}\left(\|\bAk'\| + \step \|\bBk'\|  \KA\KB \|\bPk[k+1]\|\right)\|\bPk[k+1]^{\half}\|\|\bPk^{\half}\|\\ 
    &\le \step^{-1}\left(\|\bAk'\| + \step \|\bBk'\|  \KA\KB \|\bPk[1:K+1]\|_{\maxop}\right)\|\bPk[1:K+1]\|_{\maxop}\\ 
    &\le \step^{-1}\left(\|\bAk'\| + \step \|\bBk'\|  \KA\KB \right)\|\bPk[1:K+1]\|_{\maxop}^2 \tag{$\|\bPk[1:K+1]\|_{\maxop} \ge \|\bQ\| \ge 1$}\\ 
    &\le \left(\DelA +  \DelB  \KA\KB \right)\|\bPk[1:K+1]\|_{\maxop}^2,
    \end{align*}
    where in $(i)$ we use \Cref{eq:Poptk_lyapunov_form}, which under the present notation gives
    \begin{align*}
    \bPk = (\bAk + \bBk\bKk)\bPk[k+1](\bAk + \bBk\bKk) + \step \bQ,
    \end{align*}
    and since $\bPk \succeq \step \bQ$, 
    \begin{align*}
    \|\bPk[k+1]^{\half}(\bAk + \bBk\bKk)\|^2 &=\|(\bAk + \bBk\bKk)\bPk[k+1](\bAk + \bBk\bKk)\| =  \|\bPk - \step \bQ\| \le \|\bPk\|.
    \end{align*}

    \end{proof}

    \begin{proof}[Proof of \Cref{lem:main_self_bounding}] From \Cref{eq:Pprime} in \Cref{lem:operator_sum_bound}, followed by \Cref{lem:Xibound}, we have for $k \in [K]$ that
    \begin{align*}
    \|\bPk'(s)\|_{\op} &\le \step^{-1} 2\|\bP_{k}\|_{\op}\max_{j \ge k}\|\bXi_{k}\|_{\op} \big{|}_{s}\\
    &\le 2\left(\DelA +  \DelB  \KA\KB \right)\|\bPk[1:K+1](0)\|_{\maxop}^3. 
    \end{align*}


    \end{proof}


\subsection{Proof of \Cref{lem:bounds_on_P}}
Let us apply the \Cref{lem:main_self_bounding} with $\bv(s) = \bTheta(s) = (\bA_{1:K}(s),\bB_{1:K}(s))$ as in \Cref{eq:interpolator} and $f$ as the mapping from $(\bA_{1:K},\bB_{1:K}) \to \bPk[1:K+1]$. This map is algebraic and thus $\cC^1$, and $\bv(s)$ is also $\cC^1$ as it is linear. Finally, take $\|\cdot\|$ to be $\|\cdot\|_{\maxop}$, take $g(z) = cz^p$, where $p = 3$ and $c = 2(\DelA + \KA\KB\DelB)$. The corresponding $\alpha$ in \Cref{lem:main_self_bounding} is  $\alpha  = c(p-1)\|f(\bv(0))\|^{p-1} = 2(\DelA + \KA\KB\DelB)\|\bPk[1:K+1]\|^{2}$, then, if $\alpha \le 1/4$, i.e. $(\DelA + \KA\KB\DelB) \le /8\|\bPk[1:K+1](0)\|^{2}$,
 we have by \Cref{lem:main_self_bounding} that 

\begin{align*}
\|\bPk[1:K](s)\|_{\maxop} \le (1-\alpha)^{-\frac{1}{p-1}} \|\bstP_{1:K}\|_{\maxop} \le (4/3)^2\|\bstP_{1:K}\|_{\maxop} \le 1.8\|\bstP_{1:K}(0)\|_{\maxop}.
\end{align*}
and 
\begin{align*}
\|\bPk[1:K](s)'\|_{\maxop} &\le 2(4/3)^{6}(\DelA + \KA\KB\DelB)\|\bstP_{1:K}\|_{\maxop}^3 \le 12(\DelA + \KA\KB\DelB)\|\bstP_{1:K}(0)\|_{\maxop}^3
\end{align*}
\subsection{Perturbation on the gains (\Cref{lem:K_perturb})}

\begin{proof}
Observe that
\begin{align*}
\bK_k = -(\bR + \step\bBk^\top \bPk[k+1] \bBk)^{-1}\bBk^\top \bPk[k+1]\bAk
\end{align*}
Therefore, 
\begin{align}
\bK_k'  &=(\bR + \step\bBk^\top \bPk[k+1] \bBk)^{-1}\cdot(\bR + \step\bBk^\top \bPk[k+1] \bBk)' \cdot \underbrace{(\bR + \bBk^\top \bPk[k+1] \bBk)^{-1}\bBk^\top \bPk[k+1]\bAk}_{=\bK_k}\label{eq:Kfirst_line}\\
&\quad -(\bR + \step\bBk^\top \bPk[k+1] \bBk)^{-1}(\bBk^\top \bPk[k+1]\bAk)'.
 \label{eq:Ksecondline}
\end{align}
Introduce the constant $C := \|\bstPk[1:K+1]\|_{\maxop}$. 
Using $\bR \succeq \eye$, we have
\begin{align*}
\|\bK_k'\| &= \|(\bR + \step\bBk^\top \bPk[k+1] \bBk)'\|\|\bK\| + \|(\bBk^\top \bPk[k+1]\bAk)'\|\\
&= \step\|(\bBk^\top \bPk[k+1] \bBk)'\|\|\bK_k\| + \|\bBk^\top \bPk[k+1]\bAk)'\|\\
&\le \step(2\|\bBk\|\|\bPk[k+1]\|\|\bB'_k\| + \|\bBk\|^2\|\bPk[k+1]'\|)  \|\bK_k\| + \left(\|\bBk'\|\|\bAk\| + \|\bAk'\|\|\bBk\|\|\right)\|\bPk[k+1]\| + \|\bPk[k+1]'\|\|\bAk\| )\\
&\le \step(2\KB\DelB\|\bPk[k+1]\| + \KB^2\|\bPk[k+1]'\|)  \|\bK_k\| + \left(\DelB\KA + \step\DelA\KB\|\right)\|\bPk[k+1]\| + \|\bPk[k+1]'\|\KA\KB )\\
&\le \step(2\KB^2\KA\DelB\|\bPk[k+1]\|^2 + \KB^3\KA\|\bPk[k+1]'\|\|\bPk[k+1]\| + \DelA\KB\|\|\bPk[k+1]\|)  \\
&\qquad+ \left(\|\bPk[k+1]\|\DelB\KA +\|\bPk[k+1]'\|\KA\KB \right)\\
&\le \step(2\cdot 1.8^2 \KB^2\KA\DelB C^2 + 12(\DelA + \KA\KB\DelB)\KB^3\KA C^4 + 1.8\DelA\KB C)  \\
&\qquad+ \left(1.8 C\DelB\KA +12(\DelA + \KA\KB\DelB)\KA\KB C\right) \tag{\Cref{lem:bounds_on_P}}\\
&\le \step C^4(2\cdot 1.8^2\KB^2\KA\DelB + 12(\DelA + \KA\KB\DelB)\KB^3\KA + 1.8\DelA\KB)  \\
&\qquad+ C^3\left(1.8 \DelB\KA +12(\DelA + \KA\KB\DelB)\KA\KB \right)\\
&\le \step C^4\KB^3\KA^2(1.8^2 \cdot 2\DelB + 12(\DelA + \DelB) + 1.8\DelA)  \\
&\qquad+ C^3\KA^2\KB^2\left(1.8 \DelB +12(\DelA + \DelB) \right)\\
&\le  C^3\KA^2\KB^2(1+\step C\KB)(1.8^2 \cdot 2\DelB + 12(\DelA + \DelB) + 1.8\DelA) \\
&\le 20  C^3\KA^3\KB^2(1+\step C\KB)(\DelA+ \DelB).
\end{align*}
It follows from Taylors theorem that $\|\Koptk(\bTheta) - \Koptk(\bhatTheta)\| \le 20 C^3\KA^3\KB^2(1+\step C\KB)(\DelA+ \DelB) $. The result follows. 
\end{proof}

\subsection{Proof of \Cref{lem:cor_polynomial_comparison}}\label{sec:lem:cor_polynomial_comparison}
   
    \Cref{lem:cor_polynomial_comparison} is a special case of \citet[Corollary 3]{simchowitz2020naive}. To check this, we first establish the following special case of Theorem 13 in \cite{simchowitz2020naive}.
    \begin{lemma}[Comparison Lemma ]\label{lem:comparison} Fix dimensions $d_1,d_2 \ge 1$, let $\cV \subset \R^{d_1}$, let $f: \cV \to \R^{d_2}$ be a $\cC^1$ map and let $\bv(s):[0,1] \to \cV$ be a $\cC^1$ curve defined on $[0,1]$. Suppose that $g(\cdot): \R \to \R$ is non-negative and non-decreasing scalar function, and $\|\cdot\|$ be an arbitrary norm on $\R^{d_2}$ such that
    \begin{align}
    \|\dds f(\bv(s))\| \le g(\|f(\bv(s))\|) \label{eq:self_bounding_tuple}
    \end{align}
    Finally, let $\eta > 0$ and  $\tilde g: \R \to \R$ be a scalar function such that (a) for all $z \ge \|v(0)\|$, $\tilde{g}(z) \ge \eta + g(z)$ and (b) the following scalar ODE has a solution on $[0,1]$: 
    \begin{align*}
    z(0) = \|v(0)\| + \eta, \quad z'(s) = \tilde g(z(s))
    \end{align*}
    Then, it holds that
    \begin{align*}
    \forall s\in [0,1], \quad \|f(\bv(s))\| \le z(s).
    \end{align*}
    \end{lemma}
    \begin{proof} Theorem 13 in \cite{simchowitz2020naive} proves a more general result for implicit ODEs, such as those that arise in infinite-horizon Riccati equations. We do not need these complication here, so we specialize their result. Define the function $\cF: \R^{d_1} \times \R^{d_2} \to \R^{d_2}$ via $\cF(\bv,\bw) = f(\bv(s)) - \bw$. It is then clear that $\tilde \bw(s) = f(\tilde \bv(s))$ is the unique solution to $\cF(\tilde \bv(s),\tilde \bw(s)) = 0$ for any $\cC^1$ curve $\tilde \bv(z)$; since $f$ is $\cC^1$, any such solution $\tilde \bw$ is also $\cC^1$. Thus, $\cF$ is a ``valid implicit function'' on $\cV$ in the sense of \citet[Definition 3.2]{simchowitz2020naive} with . Moreover, by \Cref{eq:self_bounding_tuple}, $(\cF,\cV,g,\|\cdot\|,\bv(\cdot))$ form a self-bounding tuple in the sense of \citet[Definition 3.3]{simchowitz2020naive}. The result now follows from \citet[Theorem 13]{simchowitz2020naive}.
    \end{proof}

    \begin{proof}[Proof of \Cref{lem:cor_polynomial_comparison}] Take $g(z) = cz^p$. Define $h_{\eta} = c(z+\eta)^p$. For any $\eta > 0$, there exist an $\eta'$ such that $h_{\eta'}(z) \le g(z) + \eta$ for $z \ge z_0$. Moreover, as $\eta$ approaches $0$, we can take $\eta' \to 0$ as wel. Solving the ODE $z(0) = \|f(\bv(0))\| + \eta$ and $z'(s) = c(z+\eta)^p$, we see the solution is given by 
    \begin{align*}
    \frac{\rmd z}{(z+\eta')^p} = c\rmd s.
    \end{align*}
    As $z'(s) \ge 0$, it suffices to bound $z(1)$. For $p > 1$, the  solution to this ODE when it exists satisfies
    \begin{align*}
    \frac{1}{(p-1)(\|f(\bv(0))\| + \eta + \eta')^{p-1}} - \frac{1}{(p-1)(z(1) + \eta')^{p-1}} = c
    \end{align*}
    Rearranging and setting $\eta,\eta ' \to 0$ lets check that, as long as $\frac{1}{(p-1)\|f(\bv(0))\|^{p-1}} > c$,  \Cref{lem:comparison} yields
    \begin{align*}
    \max_{s \in [0,1]}\|f(\bv(s))\| \le \left(\frac{1}{\|f(\bv(0))\|^{p-1}} - (p-1)c\right)^{\frac{1}{-(p-1)}} = (1-\alpha)^{\frac{1}{-(p-1)}}\|f(\bv(0))\|. 
    \end{align*}
    For $p = 1$, we instead get 
     \begin{align*}
    \ln(z(1) + \eta') - \ln ( \|f(\bv(0))\| + \eta + \eta')  = c
    \end{align*}
    Again, taking $\eta',\eta \to 0$, \Cref{lem:comparison} yields
    \begin{align*}
     \max_{s \in [0,1]}\|f(\bv(s))\| \le \exp(c + \ln ( \|f(\bv(0))\|) = \|f(\bv(0))\|e^c.
    \end{align*}
    \end{proof}

\subsection{Perturbation bounds for Lyapunov Solutions}

\begin{lemma}[Basic Lypaunov Theory]\label{lem:Lyap_lb_disc} Let $\bXk[1:K]$ and $\bYk[1:K]$ be a sequence of matrices of appropriate dimension. Suppose that $\bYk \succeq 0$, and let $\bQ \succeq \eye$. Define $\bLamk$ as via the solution to the Lyapunov recursion
\begin{align*}
\bLamk[K+1] = \bQ, \quad \bLamk = \bXk^\top \bLamk[k+1] \bXk + \step \bQ + \bYk.
\end{align*}
and define the matrix $\Phidisc_{j+1,k} := (\bX_{j}\cdot\bX_{j-1}\dots \cdot \bX_{k+1}\cdot \bX_k)$, with the convention $\Phidisc_{k,k} = \eye$, let and define the operator $\cT_{j,k}(\cdot) = \Phidisc_{j,k}^\top(\cdot)\Phidisc_{j,k}$. Then
\begin{itemize}
    \item[(a)] For any symmetric matrices $\bZk[j]$, we have
    \begin{align*}
    \step\sum_{j=k}^{K}\cT_{k,j}(\bZk[j]) \preceq \max_{j = k}^{K}\|\bZk[j]\| \bLamk.
    \end{align*}
    \item[(b)] If, in addition, $\max_k\|\eye - \bXk\|_{\op} \le \kappa\step$ for some $\kappa \le 1/2\step$, $\lambda_{\min}(\bLamk)  \ge \min\{\frac{1}{2\kappa},1\}$
    \item[(c)] Under the condition in part (b), we have
    \begin{align*}
    \|\Phidisc_{j,k}\|^2 \le \max\{1,2\kappa\}\|\bLamk[1:K+1]\|_{\maxop}(1 - \step \gamma)^{j-k}, \quad \gamma := \frac{1}{\|\bLamk[1:K+1]\|_{\maxop}}.
    \end{align*}
    In particular, $\|\cT_{k,j}(\bZk[j])\| \le\max\{1,2\kappa\}\|\bLamk[1:K+1]\|_{\maxop} \|\bZk(j)\|$.
\end{itemize}
\end{lemma}
    \begin{proof} We begin with part (a). By unfolding the recusion, we get
    \begin{align*}
    \bLamk &= \bXk^\top \bLamk[k+1] \bXk + \step (\bQ + \step^{-1}\bYk) \\
    &= \left(\step\,\cT_{k,k}(\bQ + \step^{-1}\bYk) + \cT_{k+1,k}(\bLamk[k+1])\right)\\
    &= \step\sum_{j=k}^{K}\cT_{k,j}(\bQ + \step^{-1}\bYk) + \cT_{K+1,k}(\bLamk[K+1])\\
    &\succeq \step\sum_{j=k}^{K}\cT_{k,j}(\cI)
    \end{align*}
    where in the last line, we use $\bQ + \step^{-1}\bYk \succeq \bQ \succeq \eye$.
    AS $\cT_{k,j}(\cdot)$ is a matrix monotone operator, we have that symmetric matrix $\bZ$, we have
    \begin{align*}
    \step\sum_{j=k}^{K}\cT_{j,k}(\bZk[j]) \preceq \step\max_{j=k}^{K}\|\bZk[j]\|\sum_{j=k}^{K}\cT_{j,k}(\eye) \preceq \bLamk,
    \end{align*}
    and similarly for $-\step\sum_{j=k}^{K}\cT_{j,k}(\bZk[j])$.

    \textbf{Part (b).} We argue part (b) by induction backwards on $k$, noting that $k = K+1$ is immediate. We have
    \begin{align*}
    \bLamk &\succeq \bXk^\top \bLamk[k+1] \bXk + \step \bQ \succeq \bXk^\top \bLamk[k+1] \bXk + \step
    \end{align*}
    Hence,
    \begin{align*}
    \lambda_{\min}(\bLamk) &\ge \lambda_{\min}(\bLamk[k+1])\sigma_{\min}(\bXk)^2 + \step \\
    &\ge \lambda_{\min}(\bLamk[k+1])(1-\|\bXk-\eye\|)^2 + \step\\
    &\ge \lambda_{\min}(\bLamk[k+1])(1 -  \kappa\step)^2 + \step\\
    &\ge \lambda_{\min}(\bLamk[k+1])(1 - 2 \kappa\step) + \step
    \end{align*}
    Applying the inductive hypothesis, we see that the above is at least
    \begin{align*}
    \lambda_{\min}(\bLamk) \ge \frac{1}{2\kappa}(1 - 2 \kappa\step)+ \step \ge \frac{1}{2\kappa}, \text{ as needed.}
    \end{align*}
    
    \textbf{Part (c).} We have
    \begin{align*}
    \bXk[j]^\top \bLamk[j+1] \bXk[j] &= \bLamk[j] - \step (\bQ + \bYk) \preceq \bLamk[j](1 - \step \bLamk[j]^{-1/2}\bQ\bLamk[j]^{-1/2}) \preceq \bLamk[j](1 - \step \gamma),
    \end{align*}
    where we recall $\gamma = 1/ \|\bLamk[1:K+1]\|_{\maxop}$ and use $\bQ \succeq \eye$. By unfolding the bound, we find
    \begin{align*}
    \|\Phidisc_{j+1,k}^\top\bLamk[j+1]\Phidisc_{j+1,k}\| \le(1 - \step \gamma)^{j+1-k}\|\bLamk\|. 
    \end{align*} 
    Hence, 
    \begin{align*}
    \|\Phidisc_{j+1,k}\|^2 \le \lambda_{\min}(\bLamk[j+1])^{-1}\|\Phidisc_{j+1,k}^\top\bLamk[j+1]\Phidisc_{j+1,k}\| \le 2\|\bLamk[j+1]\|\kappa(1 - \step \gamma)^{j+1-k}. 
    \end{align*}
    Moreover, as $\kappa \ge 1$, the bound also applies to $\Phidisc_{j+1,j+1} = \eye$.
    \end{proof}

    \begin{lemma}[Formula for Lyapunov Curve Derivatives]\label{lem:Lyap_der_comp} Consider curves $\bXk[1:K](s),\bYk[1:K](s)$, let $\bQ \succeq I$, and define
    \begin{align*}
    \bLamk[K+1](s) = \bQ, \quad \bLamk(s) = \bXk(s)^\top \bLamk[k+1] \bXk(s) + \step \bQ + \bYk(s)
    \end{align*}
    Again, let $\Phidisc_{k,j} := (\bX_{j-1}\cdot\bX_{j-2}\dots \cdot \bX_{k+1}\cdot \bX_k)$, with the convention $\Phidisc_{k,k} = \eye$, define the operator $\cT_{k,j}(\cdot) = \Phidisc_{k,j}^\top(\cdot)\Phidisc_{j,j}$. Then,
    \begin{align*}
    \bLamk[k]' = \sum_{j=k}^K \cT_{k,j}(\bOmegak), \quad  \bOmegak := \Sym(\bXk^\top \bLamk[k+1] \bXk') + \bYk'.
    \end{align*}
    \end{lemma}
    \begin{proof}
    We compute
    \begin{align*}
    \bLamk[k]' = \underbrace{\Sym(\bXk^\top \bLamk[k+1] \bXk') + \bYk'}_{=\bOmegak} + \bXk(s)^\top \bLamk[k+1] \bXk(s).
    \end{align*}
    The result follows by unfolding the recursion, with the base case $\bLamk[K+1]' = \dds \bQ = 0$.
    \end{proof}

    We now state our Lyapunov perturbation bound:
    \begin{proposition}[Lyapunov Function Perturbation]\label{prop:lyap_pert_max_not_factored} Consider curves $\bXk[1:K](s),\bYk[1:K](s)$, and define for $\bQ \succeq \eye$
    \begin{align*}    
    \bLamk[K+1](s) = \bQ, \quad \bLamk(s) = \bXk(s)^\top \bLamk[k+1] \bXk(s) + \step \bQ + \bYk(s)
    \end{align*}
    Then, 
    \begin{align*}
    \bLamk[k](s)'  &\le \|\bLamk[1:K+1](s)\|_{\maxop}^2\Delta(s), \quad \text{where } \\
    \Delta(s) &= \max_{j \in [k]} \step^{-1}\left(2\|\bXk[j](s)'\| +  \|\bLamk[1:K+1](s)\|_{\maxop}^{-1}\|\bYk[j](s)'\|\right)
    \end{align*}    
    Moreover,  as long as $\|\bLamk[1:K+1](0)\|_{\maxop} \sup_{s \in [0,1]}\Delta(s)  < 1$, 
    \begin{align*}
     \max_{s \in [0,1]}\|\bLamk[1:K+1](s)\|_{\maxop}  \le \left(1 - \|\bLamk[1:K+1](0)\|_{\maxop}\sup_{s \in [0,1]}\Delta(s) \right)^{-1}\|\bLamk[1:K+1](0)\|_{\maxop},
    \end{align*}
    The above bound also holds when $\Delta(s)$ is replaced by the simpler term
    \begin{align}
    \tilde \Delta(s) := \max_{j \in [k]} \step^{-1}\left(2\|\bXk[j](s)'\| +  \|\bYk[j](s)'\|\right) \label{eq:Deltil_suffices}
    \end{align}
    \end{proposition}
    \begin{proof}
    We write
    \begin{align}
    \bLamk[k]' = \sum_{j=k}^K \cT_{k,j}(\bOmegak), \quad  \bOmegak := \Sym(\bXk^\top \bLamk[k+1] \bXk' )  + \bYk' \label{eq:bLamkform_b}
    \end{align}
    We have 
    \begin{align*}
    \|\bOmegak\| &\le 2\|\bXk^\top \bLamk[k+1]\| \|\bXk'\| + \|\bYk'\|\\
    &\le 2\|\bXk^\top \bLamk[k+1]^{\half}\|\|\bLamk[k+1]^{\half}\| \|\bXk'\| + \|\bYk'\|.
    \end{align*}
    Observe that $\|\bXk^\top \bLamk[k+1]^{\half}\| = \|\bXk^\top \bLamk[k+1] \bXk \|^{\half}$. As $0 \preceq \bXk^\top \bLamk[k+1] \bXk^\top = \bLamk - \bYk \preceq \bLamk$, we conclude,
    \begin{align*}
    \|\bOmegak\| &\le 2\|\bLamk\|^{\half}\|\bLamk[k+1]^{\half}\| \|\bXk'\| + \|\bYk'\|\\
    &\le \|\bLamk[1:K+1]\|_{\maxop}2\|\bXk'\| + \|\bLamk[1:K+1](s)\|_{\maxop}^{-1}\|\bYk'\|\\
    &\le \|\bLamk[1:K+1]\|_{\maxop}\left(2\|\bXk'\| +  \|\bLamk[1:K+1](s)\|_{\maxop}^{-1}\|\bYk'\|\right)\\
    &\le  \|\bLamk[1:K+1]\|_{\maxop}\max_{j \in [K]}\left(2\|\bXk[j]'\| +  \|\bLamk[1:K+1](s)\|_{\maxop}^{-1}\|\bYk[j]'\|\right) \le \step \|\bLamk[1:K+1]\|_{\maxop}\Delta(s).
    \end{align*}
     Thus, from \Cref{eq:bLamkform_b} and \Cref{lem:Lyap_lb_disc}, we conclude 
    \begin{align*}
    \bLamk[k](s)' &\le \step^{-1}\cdot \step\|\bLamk\|\|\bLamk[1:K+1](s)\|_{\maxop}\Delta(s)\\
    &\le \|\bLamk[1:K+1](s)\|_{\maxop}^2\Delta(s)
    \end{align*}
    The final result follows by applying \Cref{lem:cor_polynomial_comparison} with $p =2$,  $  c = \max_{s \in [0,1]}\Delta$, 
    and $\alpha = \Delta(s)\|\bLamk[1:K+1](0)\|_{\maxop}$.  That \Cref{eq:Deltil_suffices} follows from the fact that if $\bQ \succeq \eye$, $\|\bLamk[1:K+1](s)\|_{\maxop} \ge 1$.
    \end{proof}

\newcommand{\Delsum}{\Delta_{\mathrm{sum}}}
    \begin{lemma}[Average Perturbation]\label{lem:Lyap_avg_perturb} Let $\kappa \le 1/2\step$. Consider a curve $\bXk[1:K](s)$ such $\max_k\sup_{s \in [0,1]}\|\eye - \bXk(s)\|_{\op} \le \kappa\step$. Then, 
    \begin{align*}    
    \bLamk[K+1](s) = \eye, \quad \bLamk(s) = \bXk(s)^\top \bLamk[k+1] \bXk(s) + \step \eye.
    \end{align*}
    Fix
    \begin{align*}
    \Delsum = 3\sup_{s \in [0,1]}\max\{1,2\kappa\}\sum_{k=1}^K\|\bXk'(s)\|.
    \end{align*}
    Then, as long as $\|\bLamk[1:K+1](0)\|_{\maxop}\Delsum < 1$, we have
    \begin{align*}
    \|\bLamk[1:K+1](1)\|_{\maxop} \le (1-\|\bLamk[1:K+1](0)\|_{\maxop}\Delsum)^{-1}\|\bLamk[1:K+1](0)\|_{\maxop}.
    \end{align*}
    \end{lemma}
    \begin{proof} We have that
    \begin{align}
    \bLamk[k]' = \sum_{j=k}^K \cT_{k,j}(\bOmegak), \quad  \bOmegak := \Sym(\bXk^\top \bLamk[k+1] \bXk' ). \label{eq:bLamkform_c},
    \end{align}
    so by \Cref{lem:Lyap_lb_disc}(c),
    \begin{align*}
    \|\bLamk[k]'\|_{\op} &\le \|\bLamk[1:K+1]\|_{\maxop}\max\{1,2\kappa\} \sum_{j=k}^K \|\bOmegak[j]\|\\
    &\le 2\|\bLamk[1:K+1]\|_{\maxop}\max\{1,2\kappa\} \sum_{j=k}^K \|\bXk\|\|\bXk'\|\|\bLamk[k+1]\|\\
    &\le 2\|\bLamk[1:K+1]\|_{\maxop}^2\max\{1,2\kappa\} \sum_{j=k}^K \|\bXk'\|\|\bXk\|\\
    &\le 2\|\bLamk[1:K+1]||_{\maxop}^2\max\{1,2\kappa\}\underbrace{(1+\kappa \step)}_{\le \frac{3}{2}} \sum_{j=k}^K \|\bXk'\| \\
    &\le \Delsum\|\bLamk[1:K+1]\|_{\maxop}^2.
    \end{align*}
    The result now follows from \Cref{lem:cor_polynomial_comparison}.
    \end{proof}

\section{Instantiantions of Certainty Equivalence Bound}\label{app:ce}

\begin{definition}\label{defn:discrete_riccati}
Given a sequence of gains $\btilKk[1:K] \in (\R^{\dimu \times \dimx })^K$, we define the discrete cost-to-go matrix as 
\begin{align*}
\bPpik[K+1]\, = \eye, \quad \bPpik\,[{\btilKk[k:K]}] = (\bAkpi +\bBkpi \btilKk)^\top \bPpik[k+1]\,[\btilKk[k+1:K]](\bAkpi +\bBkpi \btilKk) + \step(\eye + \btilKk^\top \btilKk).
\end{align*}
\end{definition}
The follow is standard (see, e.g. \citet[Section 2.4]{anderson2007optimal}).
\begin{lemma} There exists a unique minimizer sequence $\bKk[1:K]^{\pi,\star}$ such that, for all other $\btilKk[1:K]$, $\bPpik\,[{\bKk[k:K]^{\pi,\star}}] \preceq \bPpik\,[{\btilKk[k:K]}]$. We denote this minimize P-matrix $\bPpikst := \bPpik\,[{\bKk[k:K]^{\pi,\star}}]$. 
\end{lemma}

\begin{proposition}\label{prop:Ricatti_value_disc}  Recall the definition  of $\stepric$ from \Cref{defn:step_sizes},
\begin{align*}
\stepric:= \frac{1}{4\LFP^2\left(3\MF \KF\LFP\LF +  13\LF^2(1+\LF\LFP)^2 \right)} = \frac{1}{\Bigohst[1]}
\end{align*}
Then, as long as $\step \le \min\{\stepricpi,1/4\LF\}$, it holds that for any feasible policy $\pi$, $\max_{k \in [K+1]}\|\bPpikst \| \le 2\LFP$.  
\end{proposition}
The following lemma bounding the constant $\Kpi$ for the initial policy $\pi$ can be estabilished along the same lines of \Cref{prop:Ricatti_value_disc}. Its proof is given in \Cref{sec:riccati_disc}.
\begin{lemma}\label{lem:init_policy} Suppose that $\step \le \stepric$.
For $\pi = \pione$, $\Kpist \le 2\LFP$ and $\Lpi = 1$.
\end{lemma}

\begin{proposition}[Certainty Equivalence Bound]\label{prop:ce_bound_app} Let $\bhatAkpi$ and $\bhatBkpi$ be estimates of $\bAkpi$ and $\bBkpi$, and let $\bhatKk$ denote the corresponding certainty equivalence controller sythesized by solving the following recursion given by $\bhatPk[K+1] = \eye$, and for $k \in [k_0:K]$, setting
\begin{align*}
\bhatPk[k] &= (\bhatAkpi)^\top \bhatPk[k+1]\bhatAkpi - \left(\bhatBkpi\bhatPk[k+1]\bhatAkpi\right)^\top( \step^{-1}\eye + (\bhatBkpi)^\top \bhatPk[k+1]\bhatBkpi)^{-1} \left(\bhatBkpi\bhatPk[k+1]\bhatAkpi\right) + \step \eye\\
\bhatKk &= -(\step^{-1}\eye +  (\bhatBkpi)^\top \bhatPk \bhatBkpi)^{-1}(\bhatBkpi)^\top \bhatPk\bhatAkpi,
\end{align*}
Then, as long as $\max_{k \in [k_0:K]}\|\bhatAkpi - \bAkpi\|_{\op} \vee \|\bhatBkpi - \bBkpi\|_{\op} \le (2^{17}\step\LFP^4\max\{1,\LF^3\})^{-1}$, and $\step \le \min\{\stepric,1/4\LF\Lpi\}$, we have
\begin{align*}
\max_{k \ge k_0} \|\bPpik\,[{\bhatKk[k:K]}]\| \le 4\LFP, \quad \text{and} \quad \max_{k \ge k_0}\|\bhatKk\| \le 6\max\{1,\LF\}\LFP.
\end{align*}
\end{proposition}
We can now prove \Cref{prop:ce_bound}.
\begin{proof}  Let $\bhatKk[1:K]$ be the gains synthesized according to \Cref{alg:learn_mpc_feedback}(\Cref{line:control_synth_start}-\ref{line:control_synth_end}), and $\pi' = (\buk[1:K]^\pi,\bhatKk[1:K])$ be the policy with the same inputs as $\pi$ but with these new gains. Definining the shorthand $\btilPk := \|\bPpik\,[{\bhatKk[k:K]}]\|$, \Cref{prop:ce_bound_app} then implies that 
\begin{align*}
\max_{k \ge k_0} \|\btilPk\| \le 4\LFP, \quad \text{and} \quad \max_{k \ge k_0}\|\bhatKk\| \le 6\max\{1,\LF\}\LFP.
\end{align*}
Since $\bhatKk = 0$ for $k < k_0$, we conclude $\max_{k \in [K+1]}\|\bhatKk\| \le 6\max\{1,\LF\}\LFP$, which we note is $\ge 1$ as $\LFP \ge 1$. Thus, we can take $L_{\pi'} = 6\max\{1,\LF\}\LFP$. Moreover, for this policy $\pi$, we have $\bAclk^{\pi'} = (\bAkpi + \bBkpi \bhatKk)$, so that the matrices $\btilPk$ are given by the recursion
\begin{align*}
\btilPk[K+1]\, = \eye, \quad \btilPk = (\bAclk^{\pi'})^\top \btilPk[k+1](\bAclk^{\pi'}) + \step(\eye + \bhatKk^\top \bhatKk).
\end{align*}
Hence, $\btilPk \succeq \bLamk^{\pi'}$, where   we recal that $\bLamk^{\pi'}$ satisfy the recursion
\begin{align*}
\bLamk[K+1]^{\pi'}\, = \eye, \quad \bLamk^{\pi'} = (\bAclk^{\pi'})^\top \bLamk[k+1]^{\pi'}(\bAclk^{\pi'}) + \step \eye.
\end{align*}
Thus, $\mu_{\pi',\star} := \max_{k \in [k_0:K]}\|\bLamk^{\pi'}\| \le \max_{k \in [k_0:K]}\|\btilPk\| \le 4\LFP$. 
\end{proof}

\subsection{Proof of \Cref{prop:ce_bound_app}}\label{sec:prop:ce_bound}

    \newcommand{\epsA}{\epsilon_A}
    \newcommand{\epsB}{\epsilon_B}

    Essentially, we instantiate  \Cref{thm:main_pert} with appropriate bounds on parameters, and using the last part of the recursion for $k \ge k_0$. Fix an index $k_0 \in [K]$, let $K_0 = K  - k_0-1$, and recall $[k_0:j] := \{k_0,\dots,j\}$. Throughout, we suppose
    \begin{align}
    \step \le 1/4\LF\max\{1,\Lpi,\LFP\}
    \end{align}
    Suppose we have givens estimates  $\bhatAkpi$ and $\bhatBkpi$ satisfying 
    \begin{align*}
    \max_{k \in [k_0:K]}\step^{-1}\|\bhatAkpi - \bAkpi\|_{\op} \le \epsA, \quad \max_{k \in [k_0:K]}\step^{-1}\|\bhatBkpi - \bBkpi\|_{\op} \quad \le \epsilon_{B} 
    \end{align*}
    We apply \Cref{thm:main_pert} with the substitutions
    \begin{align*}
    K \gets K_0, \quad \bhatBk \gets \step  \bhatBkpi[k+k_0-1], \quad \bBk \gets \step \bBkpi[k+k_0-1],\quad \bhatAk \gets  \bhatAkpi[k+k_0-1], \quad \bAk \gets \step \bAkpi[k+k_0-1],
    \end{align*} 
    So that, with $\bTheta = (\bAk[j],\bBk[j])_{j \in [K_0]}$ and $\bhatTheta = (\bhatAk[j],\bhatBk[j])_{j \in [K_0]}$, we have
    \begin{align*}
    \bhatKk = \begin{cases} 0 & k < k_0 \\
    \Koptk[k-k_0](\bhatTheta) & k \ge k_0
    \end{cases}
    \end{align*}
    and thus,
    \begin{align*}
    \Poptk[j](\bTheta) = \bPpikst, \quad \Pcek(\bTheta;\bhatTheta) = \bPpik[j+k_0-1]\,[\bhatKk[k_0:K]].
    \end{align*}
    With the above substitutions, we can 
    apply \Cref{prop:Ricatti_value_disc}  as long as $\step$ satisfies the condition stipulated in that proposition, we have
    \begin{align}
    \max_{j \in [K_0+1]}\|\Poptk[j](\bTheta)\| \le 2\LFP. \label{eq:Poptk_bound}
    \end{align}
    Moreover, we have that by \Cref{lem:bound_on_bBkpi,lem:Phicldisc_onestep_bound,lem:bound_on_open_loop}, the following holds for  $\step \le 1/4\LF\max\{1,\Lpi\}$,
    \begin{equation}\label{eq:par_bounds_cl_disc}
    \begin{aligned}
    &\max_{k}\step^{-1}\|\bBkpi\| \le \exp(1/4)\LF\\
    &\max_{k}\|\bAkpi\| = \max_{k}\|\Phicldisc{k+1,k}\| \le \frac{5}{3}\\
    &\max_{k \in [K]}\|\bAk - \eye\| = \max_{k \in [K]} \|\Phiolpi(t_{k+1},t_k) - \eye\| \le \exp(1/4)\step \LF 
    \end{aligned}
    \end{equation}

    Hence \Cref{asm:par_bounds,asm:small_a_diff} hold for 
    \begin{equation}\label{eq:KA_bounds}
    \begin{aligned}
    \KA &= 1 \vee \max_{k \in [k_0:K]}\|\bAkpi\| \vee \|\bhatAkpi\| \le \frac{5}{3}\\
    \KB &= 1 \vee \max_{k \in [k_0:K]}\step^{-1}(\|\bBkpi\| \vee \|\bhatBkpi\|) \le \exp(1/4)\max\{1,\LF\}\\
    \kapA &:= \step^{-1}\max_{k \in [k_0:K]}\|\bAk - \eye\| = \exp(1/4) \LF.
    \end{aligned}
    \end{equation}
    Moreover, \Cref{asm:par_diffs} holds with $\Delta_A = \epsA$, $\Delta_B = \epsB$. We can now apply \Cref{thm:main_pert}. We take  and for $\epsilon_A \le 1/3$ and $\epsilon_B \le \LF/2$, we may take
    \begin{align*}
    \Delce &:= 80 C^4\KA^3\KB^3(1+\step C\KB)(\DelA+ \DelB)\\
    C &:= \max_{j \in [K_0+1]}\|\Poptk[j](\bTheta)\| \le 2\LFP \tag{by \Cref{eq:Poptk_bound}}
    \end{align*}
    And we can bound (recalling $\step \le 1/4\LF\LFP$ and $\LFP \ge 1$)
    \begin{align*}
    \Delce &\le  80 \LFP^4 \cdot (16 \cdot (5/3)^3 \cdot \exp(3/4) ) \max\{1,\LF^3\}(1+4\step \LF\LFP)(\epsilon_A + \epsilon_B)\\
    &\le  2^{14} \LFP^4\cdot \max\{1,\LF^3\}(\epsilon_A + \epsilon_B)(1+4\step \LF\LFP) \\
    &\le  2^{15} \LFP^4\cdot \max\{1,\LF^3\}(\epsilon_A + \epsilon_B).
    \end{align*}
    Hence, as long as 
    \begin{align*}2^{16}\LFP^4\max\{1,\LF^3\}(\epsilon_A + \epsilon_B) \le 1,
    \end{align*} we have 
    \begin{align}
    \Delce \le 1/2 \label{eq:Delce_bound}
    \end{align} and therefore, by \Cref{thm:main_pert}(a),
    \begin{align*}
    \max_{k \in [k_0:K+1]}\|\bPpik\,[\bhatKk[k_0:K]]\| &= \max_{j \in [K_0 +1]}\|\Pcek[j](\bTheta;\bhatTheta)\| \le 2 \max_{j \in [K_0+1]}\|\Poptk[j](\bTheta)\| \le 4\LFP.
    \end{align*}
    Next, \Cref{thm:main_pert}(b), we can take $L_{\pi'} = 6\max\{1,\LF\}\LFP$:
    \begin{align*}
    \max_{k \in [K]}\|\bhatK^{\pi'}\| 
    &= \max_{j \in [K_0]}\|\Koptk[j](\bhatTheta)\|\\
    &\le \frac{5}{4}\KB\KA C \le \frac{5}{4}(5/3)\exp(1/1)\max\{1,\LF\}\cdot 2\LFP  \le L_{\pi'} := 6\max\{1,\LF\}\LFP. 
    \end{align*}
    \qed

\subsection{Proof of \Cref{prop:Ricatti_value_disc}}\label{sec:riccati_disc}

\subsubsection{Preliminaries. } 
    We recall the following, standard definition of continuous-time cost to-go matrices (see, e.g. \cite[Section 2.3]{anderson2007optimal}):
    \begin{definition}[Cost-to-Go Matrices]\label{defn:cont_cost_to_go}
    Given a \emph{policy} $\pi$, and a sequence of controls $\btilu(\cdot) \in \cU$, let $\bPpi(\cdot \mid \btilu)$ as the cost-to-go matrix satisfying $\xi^\top\bPpi(t \mid \btilu)\xi = \int_{s=t}^T (\|\btilx(s)\|^2 +\|\btilu(s)\|^2)\rmd s + \|\btilx(T)\|^2$, under the dynamics  $\dds \btilx(s) = \Api(s)\btilx(s) + \Bpi(s)\btilu(s), \quad \btilx(t) = \xi$. We let $\bPpist(t) $ denote the optimal cost-to-go matrix, i.e., the matrix satisfying $\xi^\top\bPpist(t) \xi = \min_{\btilu \in \cU}\xi^\top\bPpi(t \mid \btilu)\xi := V^\pi(t \mid \btilu, \xi)$. 
\end{definition}
Recall that \Cref{asm:LFP} implies $V^\pi(t \mid \btilu, \xi) \le \LFP\|\xi\|^2$, so that $\|\bPpist(t)\| \le \LFP$. In what follows, we supress superscript dependence on $\pi$, assume $\pi$ is feasible, and adopt the shorthand $\bP(t) = \bPpist(t)$, $\bA(t) = \Api(t)$, $\bB(t) = \Bpi(t)$, $\bx(t) = \xpi(t)$, and $\bu(t) = \upi(t)$. We also use the shorthand
    \begin{align}
    \Lcl := \LF(1+\LF \LFP). \label{eq:LCL_def}
    \end{align}

    The optimal input defining $\bP(t)$ in \Cref{asm:LFP} selects $\btilu(t) = \bK(t)\btilx(t)$, where  $\bK(t) = \bB(t)^\top\bP(t)$ (again, \citet[Section 2.3]{anderson2007optimal}). 
     Introduce the evaluations of  the \emph{continuous} value function $\bP(t)$ and $\bK(t)$ at the time steps $t_k$:
    \begin{align}
    \bPctk := \bP(t_k), \quad \bKctk := \bK(t_k) \label{eq:Pct_def}
    \end{align} 
    We also define an \emph{suboptimal} discrete-time value function by taking $\bPsubk = \bPpik\,[\bKctk[1:K]]$, defined in \Cref{defn:discrete_riccati},
    which satisfies
    \begin{align*}
    \bPsubk \succeq \bPpikst.
    \end{align*}
    by optimality of $\bPpikst$. 
    Hence, it suffices to bound $\bPsubk$. To do this, first express both $\bPsubk$ and $\bPctk$ as discrete Lyapunov recusions. To do so, we require the relevant transition operators.
    \begin{definition}[Relevant Transitions Operators] For $k \in [K]$ and $s \in \cI_k$, let $\bPhi_1(s,t_k)$ and $\bPhi_2(s,t_k)$ denote the solution to the ODEs 
    \begin{align*}
    \dds \bPhi_1(s,t_k) &= (\bA(s)+\bB(s)\bK(s))\bPhi_1(s,t) \numberthis \label{eq:phione_rel_trans}\\
    \dds \bPhi_2(s,t_k) &= \bA(s)\bPhi_2(s,t) + \bB(s)\bKk(t_k). 
    \end{align*}
    with initial conditions $\bPhi_1(t_k,t_k) = \bPhi_2(t_k,t_k) = \eye$. We define
    \begin{align*}
    \bXctk :=  \bPhi_1(t_{k+1},t_k), \quad \bXsubk :=  \bPhi_2(t_{k+1},t_k)
    \end{align*}
    \end{definition}

    \begin{definition}[Relevant Cost Matrices] For $k \in [K]$, define
    \begin{align*}
    \bYctk &:=  \int_{s=t_k}^{t_k} \bPhi_1(s,t_k)^\top (\eye + \bK(s)^\top\bK(s)) \bPhi_1(s,t_k)\rmd s\\
    \bYsubk &:= \step(\eye + \bK(t_k)^\top \bK(t_k))
    \end{align*}
    \end{definition}
    \begin{lemma}\label{lem:cost_to_go_char} The cost-to-go matrices $\bPctk$ and $\bPsubk$ are given by the following Lyapunov recursions, with initial conditions $\bPctk[K+1] = \bPsubk[K+1] = \eye$:
    \begin{align*}
    \bPctk &= (\bXctk)^\top \bPctk[k+1]\bXctk + \bYctk\\
    \bPsubk &= (\bXsubk)^\top \bPsubk[k+1]\bXsubk +  \bYsubk
    \end{align*}
    \end{lemma}
    \begin{proof}[Proof of \Cref{lem:cost_to_go_char}] The recursion for $\bPsubk$ is directly from \Cref{defn:discrete_riccati}, and the fact that $\bXsubk = \bAclkpi$ due to \Cref{lem:solve_affine_ode}. To verify the recursion for $\bPctk$, we note that we can express $\bP(t) = \bPpist(t)$ in \Cref{defn:cont_cost_to_go} as satisfying the following ODE (see \citet[Section 2.3]{anderson2007optimal}):
    \begin{align*}
    \bP(T) = \eye, \quad -\ddt \bP(t)  = (\bA(t)+\bB(t)\bK(t))^\top \bP(t) (\bA(t)+\bB(t)\bK(t)) + \eye + \bK(t)^\top \bK(t)
    \end{align*}
    It can be checked then by computing derivatives and using existence and uniqueness of ODEs that
    \begin{align*}
     \bP(s,t) = \bPhi_1(s,t)^\top\bP(s)\bPhi_1(s,t) + \int_{s' = t}^s \bPhi_1(s',t)^\top(\eye + \bK(s')^\top \bK(s'))\bPhi(s',t)\rmd s'
    \end{align*}
    Specializing to $s = t_{k+1}$ and $t = t_k$ verifies the desired recursion.
    \end{proof}

    As $\bPctk = \bP(t_k)$, the terms $\bPctk$ are bounded whenever $\bP(\cdot)$ is. Therefore, we use a Lyapunov perturbation bound to bound $\bPsubk$ in terms of $\bPctk$. This requires reasoning about the differences $\bXctk - \bXsubk$ and $\bYctk - \bYsubk$, which we do in just below.  

\subsubsection{Controlling the rate of change of $\bK(t)$.}
    \paragraph{} Our first step in controlling the perturbation term is to argue that the optimal controller $\bK(t)$ does not change too rapidly. As $\bK(t) = \bB(t)^\top \bP(t)$, we begin by bounding the change in $\bB(t)$.
    \begin{claim}[Change in $\bB(t)$]\label{claim:B_change}  $\bB(t)$ is differentiable in $t$ on for $t \in \interior(\cI_k)$, and satisfies $\|\ddt \bB(t)\| \le \MF \KF$
    \end{claim}
    \begin{proof} Recall that $\bB(t) = \partial_u f(\bx(t),\bu(t))$. For $t \in \interior(\cI_k)$, $\bu(t)$ is constant, and $\bx(t)$, being the solution to an ODE, is also $t$-differentiable. We now bound $\|\ddt \bB(t)\|$.  We have
    \begin{align*}
    \|\ddt \bB(t)\| &= \|\ddt \partial_u f(\bx(t),\bu(t))\| \\
    &= \| \partial_{uu} f(\bx(t),\bu(t)) \ddt \bu(t) + \partial_{xu} f(\bx(t),\bu(t)) \ddt \bx(t)\|\\
    &\le \MF\|\ddt \bx(t)\| \le \MF\KF
    \end{align*}
    where the second-to-last inequality is the limiting consequence holds from \Cref{asm:max_dyn_final}, and where the term $ \partial_{uu} f(\bx(t),\bu(t)) \ddt \bu(t)$ vanishes vanishes because $\bu(t) = \upi(t)$ is constant on $\cI_k$.
    \end{proof}
    Next, we bound the change in $\bP(t)$:
    \begin{claim}[Change in $\bP(t)$]\label{claim:P_change}  $\bP(t)$ is differentiable in $t$, and  $\|\ddt \bP(t)\| \le  (\Lcl/\LF)^2$.
    \end{claim}
    \begin{proof} Note that $\bP(t)$ is given by the ODE
    \begin{align*}
    &\bP(T) = \eye, \quad -\ddt \bP(t) = \bA(t)^\top \bP(t) + \bP(t) \bA(t) - \bP(t) \bB(t)\bB(t)^\top \bP(t) + \eye,
    \end{align*}
    which ensures differentiability. Thus, as $\|\bA(t)\| \vee \|\bA(t)\| \vee 1 \le \LF$  by \Cref{asm:max_dyn_final} and  \Cref{asm:LFP},
    \begin{align*}
    \|\ddt \bP(t)\| \le 1 + 2\LF\|\bP(t)\| + \LF^2\|\bP(t)\|^2 \le (1 + 2\LF\LFP + \LF^2 \LFP^2) \le (1+ \LF\LFP)^2,
    \end{align*}
    which is precisely $(\Lcl/\LF)^2$.
    \end{proof}

    We now establish a bound on the change in $\bK(t)$.
    \begin{claim}[Continuity of Optimal Controller]\label{lem:K_cont}For all $t \in \cI_k$,
    \begin{align*}
    \|\ddt \bK(t)\| \le \MF \KF\LFP +  \LF^{-1}\Lcl^2. 
    \end{align*}
    \end{claim}
    \begin{proof}[Proof of \Cref{lem:K_cont}] By \Cref{claim:B_change,claim:P_change}, we have
    \begin{align*}
    \|\ddt\bK(t)\| &\le \|\ddt \bB(t)\|\|\bP(t)\| + \|\bB(t)\|\|\ddt\bP(t)\|\\
    &\le \|\ddt \bB(t)\|\LFP + \LF (\Lcl/\LF)^2\\
    &\le \MF \KF\LFP +  \LF^{-1}\Lcl^2.
    \end{align*}
    \end{proof}
    By integrating, we arrive at the next claim. 
    \begin{claim}\label{claim:K_diff} The following bound holds
    \begin{align*}
    \sup_{s \in \cI_k}\|\bK(s) - \bK(t_k)\| \le \step\left(\MF \KF\LFP +  \LF^{-1}\Lcl^2 \right).
    \end{align*}
    \end{claim}
    \begin{proof}[Proof of \Cref{claim:K_diff}]
    Directly from \Cref{lem:K_cont}.
    \end{proof}

\subsubsection{Controlling differences in $\|\bXctk - \bXsubk\|$ and $\|\bYctk - \bYsubk\|$}

    We first state a bound on the magnitudes of various quantities of interest.
    \begin{claim}\label{lem:K_Acl_bounds} $\|\bK(t)\| \le \LFP\LF$ and $\|\bA(t) + \bB(t)\bK(t)\| \le \Lcl$, where we recall $\Lcl := \LF(1+\LF \LFP)$.
    \end{claim}
    \begin{proof} Recall that $\bK(t) = \bB(t)^\top\bP(t)$. From \Cref{asm:LFP}, $\|\bP(t)\| \le \LFP$, and $\|\bB(t)\| \le \LF$ by \Cref{asm:max_dyn_final}, which gives $\|\bK(t)\| \le \LFP\LF$. Bounding $\|\bA(t)\|\vee \|\bB(t)\| $ by $\LF$ (again, invoking \Cref{asm:max_dyn_final}), concludes the demonstration. 
    \end{proof}
    Next, we show that $\bPhi_1(s,t_{k})$ is close to the identity for sufficiently small $\step$.
    \begin{claim}\label{claim:yone_xi} Suppose that $\step \Lcl \le 1/2$. Then, 
    \begin{align*}
    \|\eye - \bPhi_1(s,t_{k})\| \le \step\Lcl\exp(1/2) \le \min\{1, 2\step \Lcl\}
    \end{align*}
    \end{claim}
    \begin{proof}[Proof of \Cref{claim:yone_xi}]
    It suffices to bound, for all $\xi \in \R^{\dimx}:\|\xi\| = 1$ the differences $\|\by_1(s) - \xi\|$ where  $\by_1 = \Phi_1(s,t_{k})\xi$.
    We do this via Picard's lemma.

    Specifically, write $\dds \by_1(s) = \tilde{f}(\by_1(s),s)$, where $\tilde{f}(y,s) = (\bA(s) + \bB(s)\bK(s))y$, and $\bz(s) = \xi$. As $\tilde{f}(y,s)$ is $\sup_{s \in \cI_k}\|\bA(s) + \bB(s)\bK(s)\| \le \Lcl$ Lipchitz is $y$ (here, we use \Cref{lem:K_Acl_bounds}) and as $\dds \xi = 0$, and  the Picard Lemma (\Cref{lem:picard}) gives
    \begin{align*}
    \|\xi - \by_1(s)\|  &\le \exp((s-t_k)(2\LF^2\LFP)) \int_{s'=t_k}^{s}\|(\bA(s') + \bB(s')\bK(s'))\xi\| \rmd s'\\
    &\le \exp((s-t_k)\Lcl)\int_{s'=t_k}^{s}\|(\bA(s') + \bB(s')\bK(s'))\|\rmd s' \tag{$\|\xi\|\le 1$}\\
    &\le\exp((s-t_k)\Lcl) \cdot (s-t_k)\Lcl,\\
    &\le\exp(\step \Lcl) \cdot \step \Lcl,\\
    &\le\exp(1/2)\step \Lcl
    \end{align*}
    where we assume $\step \Lcl \le 1/2$. 
    \end{proof}
    We can now bound the differences between $\|\bXctk - \bXsubk\| = \|\bPhi_2(t_{k+1},t_k) - \bPhi_1(t_{k+1},t_k)\|$.
    \begin{lemma}\label{lem:perturb_of_xs} For $k \in [K]$ and $s \in \cI_k$, let $\bPhi_1(s,t_k)$ and $\bPhi_2(s,t_k)$ denote the solution to the ODEs 
    \begin{align*}
    \dds \bPhi_1(s,t_k) = (\bA(s)+\bB(s)\bK(s))\bPhi_1(s,t), \quad \dds \bPhi_2(s,t_k) = \bA(s)\bPhi_2(s,t) + \bB(s)\bKk(t_k).
    \end{align*}
    with initial conditions $\bPhi_1(t_k,t_k) = \bPhi_2(t_k,t_k) = \eye$. Then, if $\step \Lcl \le 1/2$,
    \begin{align*}
    &\|\bXctk - \bXsubk\|= \|\bPhi_2(t_{k+1},t_k) - \bPhi_1(t_{k+1},t_k)\|\le  2\step^2\left(\LF\LFP\MF\KF + 3\Lcl^2 \right).
    \end{align*}
    \end{lemma}
    \begin{proof} It suffices to bound, for all initial conditions, $\xi \in \R^{\dimx}$ with $\|\xi\|= 1$, the solutions  $\by_i(s) = \bPhi_i(s)\xi$. We apply the Picard Lemma, with $\bz(s) \gets \by_1(s)$, and express $\by_2(s) = \tilde f(\by_2(s), s)$, where $\tilde{f}(y,s) = \bA(s)y + \bB(s)\bK(t)$. As $\|\bA(s)\| \le \LF$, the Picard Lemma (\Cref{lem:picard}) yields
    \begin{align*}
    \|\by_1(t_{k+1}) - \by_2(t_{k+1})\| &\le \exp(\LF (t-s)) \int_{s = t}^{t_{k+1}} \|\bA(s)\by_1(s) + \bB(s)\bK(t_k)\xi - \dds \by_1(s)\|\rmd s\\
    &\le \exp(\LF \step ) \int_{s = t}^{t_{k+1}}  \|\bA(s)\by_1(s) + \bB(s)\bK(t_k)\xi - (\bA(s)+ \bB(s)\bK(s))\by_1(s) \|\rmd s\\
    &\le \exp(\LF \step) \int_{s = t}^{t_{k+1}} \|\bB(s)(\bK(s)\by_1(s) - \bK(t_k)\xi) \|\rmd s\\
    &\le \LF\exp(\LF \step) \int_{s = t}^{t_{k+1}}\|\bK(s)\by_1(s) - \bK(t_k)\xi \|\rmd s\\
    &\le \LF\exp(\LF \step) \int_{s = t}^{t_{k+1}}  (\|(\bK(s) - \bK(t_k))\xi \| + \|\bK(s)(\xi - \by_1(s)) \| )\rmd s\\
    &\le \LF\exp(\LF \step) \int_{s = t}^{t_{k+1}}(\|\bK(s) - \bK(t_k)\| + \LF\LFP\|\xi - \by_1(s) \| )\rmd s\\
    &\le \exp(1/2)\LF\step \max_{s \in \cI_k} (\|\bK(s) - \bK(t_k)\| + \LF\LFP\|\xi - \by_1(s) \| )
    \end{align*}
    where the second-to-last line uses $\|\xi\| = 1$ and $\|\bK(s)\| \le \LFU\LFP$, and the last uses $\step \le 1/2\LFX$ and well as a bound of an integral by a maximum. By claims \Cref{claim:yone_xi,claim:K_diff},
    \begin{align*}
    &\max_{s \in \cI_k} (\|\bK(s) - \bK(t_k)\| + \LF\LFP\|\xi - \by_1(s) \| ) \\
    &\le  \step\left(\MF \KF\LFP +  \LF^{-1}\Lcl^2 \right) + \LF\LFP\exp(1/2)\step\Lcl\\
    &=  \step \left(\LFP\MF\KF + \Lcl(\LF\LFP\exp(1/2)+\LF^{-1}\Lcl) \right)\\
    &\le  \step \left(\LFP\MF\KF + 3\LF^{-1}\Lcl^2 \right),
    \end{align*}
    where in the last inequality we use $\exp(1/2) \le 2$ and $\LF\LFP \le (1+\LF\LFP) = \Lcl \LF^{-1}$. Therefore, again using $\exp(1/2) \le 2$,
    \begin{align*}
    \|\by_1(t_{k+1}) - \by_2(t_{k+1})\| \le 2\step^2\left(\LF\LFP\MF\KF + 3\Lcl^2 \right).
    \end{align*}
    Quantifying over all unit-norm initial conditions $\xi$ concludes the proof.
    \end{proof}
    We now establish a qualitatively similar bound on $\step\|\bYctk - \bYsubk \|$.
    \begin{lemma}\label{lem:perturb_of_Y's} $\|\bYctk - \bYsubk \| \le 2\step^2\left(\MF \KF\LFP^2\LF +  7\LFP\Lcl^2 \right)$. 
    \end{lemma}
    \begin{proof} Recall the definitions
    \begin{align*}
    \bYctk &:=  \int_{s=t_k}^{t_k} \bPhi_1(s,t_k)^\top (\eye + \bK(s)^\top\bK(s)) \bPhi_1(s,t_k)\rmd s\\
    \bYsubk &:= \step(\eye + \bK(t_k)^\top \bK(t_k))
    \end{align*}
    We can then express
    \begin{align*}
    &\bYctk - \bYsubk = \bYctk - \step(\eye + \bK(t_k)\bK(t_k)^\top) =  \int_{s=t_k}^{t_k} \bZ_k(s),\\
    & \bZ_k(s) := \left\{\bPhi_1(s,t_k)^\top (\eye + \bK(s)^\top\bK(s)) \bPhi_1(s,t_k) - (\eye + \bK(t_k)^\top\bK(t_k))\right\} \rmd s
    \end{align*}
    Thus, 
    \begin{align}
    \|\bYctk - \bYsubk \| \le \step \max_{s \in \cI_k}\|\bZ_k(s)\|. \label{eq:bY_Z_bound}
    \end{align}
    With numerous applications of the triangle inequality,
    \begin{align*}
    \|\bZ_k(s)\| &\le \|\eye - \bPhi_1(s,t_k)\| \|\eye + \bK(s)^\top\bK(s)\| \|\bPhi_1(s,t_k)\| \\
    &\quad+\|\eye + \bK(s)^\top\bK(s)\| \|\eye - \bPhi_1(s,t_k)\| + \|\bK(s) - \bK(t_k)\|(\|\bK(s)\|+\|\bK(t_k)\|).
    \end{align*}
    Using $\|\bK(s)\|\vee\|\bK(t_k)\| \le \LF\LFP$ due to \Cref{lem:K_cont}, we have
    \begin{align*}
    \|\bZ_k(s)\| &\le (1+\LF^2\LFP^2)(1+ \|\bPhi_1(s,t_k)\|)\|\eye - \bPhi_1(s,t_k)\| + 2\LFP\LF\|\bK(s) - \bK(t_k)\|\\
    &\le 3(1+\LF^2\LFP^2)\|\eye - \bPhi_1(s,t_k)\| + 2\LFP\LF\|\bK(s) - \bK(t_k)\|  \tag{\Cref{claim:yone_xi}}\\
    &\le 6\step \Lcl (1+\LF^2\LFP^2) + 2\LFP\LF\|\bK(s) - \bK(t_k)\| \tag{\Cref{claim:yone_xi}}\\
    &\le 12\step \LF^2\LFP (1+\LF^2\LFP^2) + 2\step\LFP\LF\left(\MF \KF\LFP +  \LF^{-1}\Lcl^2 \right) \tag{\Cref{claim:K_diff}}.
    \end{align*}
    We can upper bound $\LF^2(1+\LF^2\LFP^2) \le \LF^2(1+\LF\LFP)^2 = \Lcl^2$, and simplify $2\step\LFP\LF\left(\MF \KF\LFP +  \LF^{-1}\Lcl^2 \right) = 2\step\left(\MF \KF\LFP^2\LF +  \LFP\Lcl^2 \right)$. This gives
    \begin{align*}
    \|\bZ_k(s)\| \le 12\step \LFP\Lcl^2 + 2\step\left(\MF \KF\LFP^2\LF +  \LFP\Lcl^2 \right) = 2\step\left(\MF \KF\LFP^2\LF +  7\LFP\Lcl^2 \right).
    \end{align*}
    Plugging the above bound into \Cref{eq:bY_Z_bound} concludes. 
    \end{proof}

\subsubsection{Concluding the proof of \Cref{prop:Ricatti_value_disc}}
    From \Cref{lem:perturb_of_xs,lem:perturb_of_Y's}, we have
    \begin{align*}
    \|\bXctk - \bXsubk\|&\le  2\step^2\left(\LF\LFP\MF\KF + 3\Lcl^2 \right), \quad \|\bYctk - \bYsubk \| \le 2\step^2\LFP\left(\MF \KF\LFP\LF +  7\Lcl^2 \right)
    \end{align*}
    Therefore, using $\LFP \ge 1$ 
    \begin{align}
    2\|\bXctk - \bXsubk\| + \|\bYctk - \bYsubk \| \le 2\step^2\LFP\left(3\MF \KF\LFP\LF +  13\Lcl^2 \right). \label{eq:Deltil_prebound}
    \end{align}
    Now, we invoke \Cref{prop:lyap_pert_max_not_factored}. We construct linear interpolation (here, $s \in [0,1]$ parametrizes the interpolation and not time)
    \begin{align*}
    \bXk(s) = (1-s)\bXctk + s\bXsubk, \quad \bYk(s) = (1-s)\bYctk + s\bYsubk.
    \end{align*}
    Then, by \Cref{lem:cost_to_go_char}, the interpolator $\bLamk(s)$ defined in \Cref{prop:lyap_pert_max_not_factored}  satisfies $\bLamk(0) = \bPctk$ and $\bLamk(1) = \bPsubk$. In particular, 
    \begin{align*}
    \|\bLamk[1:K+1](0)\|_{\maxop} &= \max_{k \in [K+1]}\|\bPctk\| \tag{since $\bLamk = \bPctk$ and definition of $\|\cdot\|_{\maxop}$ } \\
    &= \max_{k \in [K+1]}\|\bP(t_k)\|  \tag{by \Cref{eq:Pct_def} }\\
    &\le \sup_{t \in [T]}\|\bP(t)\|  \\
    &\le \LFP \numberthis \label{eq:Lamkzero_bound},
    \end{align*}
    where the last inequality is by \Cref{asm:LFP}. Moreover, the term $\tilde \Delta(s)$ defined in \Cref{eq:Deltil_suffices} satisfies 
    \begin{align*}
    \forall s \in [0,1], \tilde \Delta(s) &= \step^{-1}(2\|\bXctk - \bXsubk\| + \step\|\bYctk - \bYsubk \| \\
    &\le \step \cdot 2\LFP\left(3\MF \KF\LFP\LF +  13\Lcl^2 \right) \tag{by \Cref{eq:Deltil_prebound}}.
    \end{align*}
    Hence, recalling $\|\bLamk[1:K+1](0)\|_{\maxop} \le \LFP$ due to \Cref{eq:Lamkzero_bound}, it holds that long as $2\step\LFP^2\left(3\MF \KF\LFP\LF +  13\Lcl^2 \right) \le \frac{1}{2}$, it holds that
    \begin{align*}
    \max_{k \in [K+1]}\|\bPsubk\| = \|\bLamk[1:K+1](1)\|_{\maxop} \le 2\|\bLamk[1:K+1](0)\|_{\maxop} \le 2\LFP.
    \end{align*}
    Lastly, we note the condition $2\step\LFP^2\left(3\MF \KF\LFP\LF +  13\Lcl^2 \right) \le \frac{1}{2}$ is equivalent to 
    \begin{align*}
    \step &\le \frac{1}{4\LFP^2\left(3\MF \KF\LFP\LF +  13\Lcl^2 \right)}\\
    &\le \frac{1}{4\LFP^2\left(3\MF \KF\LFP\LF +  13\LF^2(1+\LF\LFP)^2 \right)} := \stepric \tag{Definition of $\Lcl$ in \Cref{eq:LCL_def}}
    \end{align*}
    This concludes the proof of \Cref{prop:Ricatti_value_disc}. \qed

\subsection{Proof of \Cref{lem:init_policy}}
The proof is similar to \Cref{prop:Ricatti_value_disc}. Let $\pi = \pione$.  As $\bKkpi = 0$ for all $k$, we that $\Lpi = 1$, and that $\Lampik[K+1] = \eye$, and 
\begin{align}
\Lampik = (\bAkpi)^\top\Lampik[k+1]\bAkpi + \step \eye \label{eq:Lampik_reminder}
\end{align} On the other hand, following the arguments of \Cref{prop:Ricatti_value_disc}, we it can be shown that $V^\pi(t_k;\bu = 0,\xi) = \xi^\top \bPctk \xi$, where $\bPctk$ satisfies the recursion $\bPctk[K+1] = \eye$ and 
    \begin{align*}
    \bPctk &= \bPhi_1(t_{k+1},t_k)^\top \bPctk[k+1]\bPhi_1(t_{k+1},t_k) + \bYctk, \quad \bYctk :=  \int_{s=t_k}^{t_k} \bPhi_1(s,t_k)^\top\bPhi_1(s,t_k)\rmd s,
    \end{align*}
    where  $\Phi_1(s,s) = \eye$ and where (using that we consider $V^\pi(t_k;\bu = 0,\xi)$ with $\bu = 0$, so the corresponding $\bK(t)$ in \Cref{eq:phione_rel_trans} vanishes)
    \begin{align*}
    \ddt\bPhi_1(t,s) = \Api(s)\bPhi_1(t,s). 
    \end{align*}
    Hence, $\bPhi_1(t_{k+1},t_k) = \bAkpi$, so that 
    \begin{align*}
    \bPctk &= (\bAkpi)^\top \bPctk[k+1]\bAkpi + \bYctk, \quad \bYctk :=  \int_{s=t_k}^{t_k} \bPhi_1(s,t_k)^\top\bPhi_1(s,t_k)\rmd s
    \end{align*}
    Along the lines of \Cref{lem:perturb_of_Y's}, it can be shown that for $\step \le \stepric$, $\bPhi_1(s,t_k)^\top\bPhi_1(s,t_k) \succeq \frac{1}{2}\eye$ for all $t \in \cI_k$. Thus,
    \begin{align*}
    \bPctk &\succeq (\bAkpi)^\top \bPctk[k+1]\bAkpi + \frac{1}{2}\eye.
    \end{align*} 
    Comparing to \Cref{eq:Lampik_reminder}, we find that $\bPctk  \succeq \frac{1}{2}\Lampik$. As $\|\bPctk\| = \sup_{\xi:\|\xi\| = 1} V^\pi(t_k;\bu = 0,\xi)\le \LFP$, we conclude $\|\Lampik\| \le 2\LFP$, as neeeded.

\subsection{Proof of \Cref{lem:Kpi_bounds}}\label{sec:proof:lem:Kpi_bounds}
     We recall the definitions $\Kpiinf := \max_{1 \le j \le k \le K+1}\|\Phicldisc{k,j}\|$,
     and
     \begin{align*}
    \Kpione &:=  \max_{k \in [K+1]}\step \left(\sum_{j=1}^{k}\|\Phicldisc{k,j}\| \vee \sum_{j=k}^{K+1}\|\Phicldisc{j,k}\|\right)\\
    \Kpitwo^2 &:= \max_{k\in [K+1]}\step \left(\sum_{j=1}^{k}\|\Phicldisc{k,j}\|^2 \vee  \sum_{j=k}^{K+1}\|\Phicldisc{j,k}\|^2\right),
    \end{align*}
    and recall the definitions  $\Lampik[K+1] = \eye$, and $\Lampik = (\bAclkpi)^\top\Lampik[k+1]\bAclkpi + \step \eye$, and $\Kpist := \max_{k \in \{k_0,k_0+1,\dots,K+1\}}\|\Lampik\|$. 

    Let us first bound $\|\Phicldisc{k,j}\| $ for $j \ge k_0$. 
    \begin{claim}\label{claim:j_ge_knot} For $j \ge k_0$, and $\step \le 1/6\LF\Lpi$, $\|\Phicldisc{k,j}\| \le \sqrt{\max\{1,6\LF\Lpi\}\Kpist(1 - \step/\Kpist)^{k-j}}$ 
    \end{claim}
    \begin{proof} We apply \Cref{lem:Lyap_lb_disc} with $\bX_k \gets \bAkpi$, $\bQ = \eye$, and $\bY_k = 0$, and only take the recursion back to $k = k_0$. That lemma shows that, as long as $\|\bAkpi - \eye\| \le \kappa \step$ for some $\kappa \le 2/\step$, it holds that (for $j \ge k_0$)
    \begin{align*}
    \|\Phicldisc{k,j}\|^2 \le  \max\{1,2\kappa\}\Kpist(1 - \step/\Kpist)^{k-j}. 
    \end{align*}
    \Cref{lem:bound_on_bBkpi,lem:bound_on_open_loop}, and using $\Lpi \ge 1$, we have that for $\step \le 1/4\LF$, $\|\bAclkpi - \eye\| \le \|\bAkpi - \eye\| + \|\bBkpi \bKkpi\| \le \exp(1/4)\step\LF(1+\Lpi) \le 2\exp(1/4)\step\LF\Lpi \le 3\step \LF\Lpi$. Hence, for $\step \le 1/6\LF\Lpi$, we can take $\kappa := 3\LF\Lpi$ and have $\kappa \le 1/2\step$. For this choice of $\kappa$, we get
    \begin{align*}
    \|\Phicldisc{k,j}\|^2 \le  \max\{1,6\LF\Lpi\}\Kpist(1 - \step/\Kpist)^{k-j}. 
    \end{align*}
    \end{proof}
    Next, we bound $\|\Phicldisc{k,j}\| $ for $k \le k_0$. 
    \begin{claim}\label{claim:k_le_knot} For $k \ge k_0$, $\|\Phicldisc{k,j}\| \le \exp(k_0\step \LF)$.  
    \end{claim}
    \begin{proof} For $j \le k \le k_0$, we have $\bKkpi[j] = 0$. Hence $\bAclkpi[j] = \bAkpi[j]$, and from \Cref{lem:bound_on_open_loop}, we get $\|\bAclkpi[j]\| \le \exp(\step \LF)$.  Thus $\|\Phicldisc{k,j}\| \le \prod_{j=1}^{k}\|\bAkpi[j]\| \le \exp(k\step  \LF) \le \exp(k_0 \step \LF)$.
    \end{proof}
    Finally, we bound we bound $\|\Phicldisc{k,j}\| $ for $j \le k_0$.
    \begin{claim}\label{claim:j_le_knot} For $j \le k_0$, $\|\Phicldisc{k,j}\| \le \sqrt{\max\{1,6\LF\Lpi\}\Kpist}\exp(\step k_0 \LF)$.  
    \end{claim}
    \begin{proof} For $j < k_0$, we have $\|\Phicldisc{k,j}\| = \|\Phicldisc{k,k_0}\Phicldisc{k,j}\| \le \|\Phicldisc{k,k_0}\|\|\Phicldisc{k_0,j}\|$. The first term is at most $\max\{1,6\LF\Lpi\}\Kpist$  by \Cref{claim:j_ge_knot}, and the second term at most $\exp(\step k_0 \LF)$ by \Cref{claim:k_le_knot}. 
    \end{proof}
    We can now bound all terms of interest. Directly from the dichotmoty in \Cref{claim:j_ge_knot,claim:j_ge_knot}, we have $\Kpiinf \le \sqrt{\max\{1,6\LF\Lpi\}\Kpist}\exp(\step k_0 \LF)$. Next, for any $k$, we can bound via \Cref{claim:j_ge_knot,claim:j_le_knot}
    \begin{align*}
    \step \sum_{j=1}^{k}\|\Phicldisc{k,j}\|^2 &\le \step \sum_{j=1}^{k_0}\|\Phicldisc{k,j}\|^2 + \step \sum_{j=k_0}^{k}\|\Phicldisc{k,j}\|^2\\
    &\le \max\{1,6\LF\Lpi\}\Kpist\left((\step k_0)\exp(2\step k_0 \LF) + \step \sum_{j=k_0}^{k}(1 - \step/\Kpist)^{k-k_0}\right) \tag{\Cref{claim:j_ge_knot,claim:j_le_knot}}\\
    &\le \max\{1,6\LF\Lpi\}\Kpist\left((\step k_0)\exp(2\step k_0 \LF) + \step \underbrace{\sum_{n=0}^{\infty}(1 - \step/\Kpist)^{n}}_{= \frac{1}{1-(1-\step/\Kpist)} = \Kpist/\step}\right)\\
    &\le \max\{1,6\LF\Lpi\}\Kpist\left((\step k_0)\exp(2\step k_0 \LF) + \Kpist\right). 
    \end{align*}
    and show the same bound for $ \step \sum_{j=k}^{K+1}\|\Phicldisc{j,k}\|^2$,
    which yields the desired upper bound on $\Kpitwo$. Finally, to bound $\Kpione$, 
    \begin{align*}
    \step \sum_{j=1}^{k}\|\Phicldisc{k,j}\| &\le  \step \sum_{j=1}^{k_0}\|\Phicldisc{k,j}\| + \step \sum_{j=k_0}^{k}\|\Phicldisc{k,j}\| \\
    &\le  \sqrt{ \max\{1,6\LF\Lpi\}\Kpist}\left((\step k_0)\exp(\step k_0 \LF) + \step\sum_{j=k_0}^{k}\sqrt{(1 - \step/\Kpist)^{k-k_0}}\right) \tag{\Cref{claim:j_ge_knot,claim:j_le_knot}}\\
    &\le  \sqrt{ \max\{1,6\LF\Lpi\}\Kpist}\left((\step k_0)\exp(\step k_0 \LF) + \step\sum_{n=0}^{\infty}\sqrt{(1 - \step/\Kpist)^{n}}\right)\\
    &\overset{(i)}{\le}  \sqrt{ \max\{1,6\LF\Lpi\}\Kpist}\left((\step k_0)\exp(\step k_0 \LF) + 2\step\sum_{n=0}^{\infty}(1 - \step/\Kpist)^{n}\right)\\
    &= \sqrt{ \max\{1,6\LF\Lpi\}\Kpist}\left((\step k_0)\exp(\step k_0 \LF) + 2\Kpist\right),
    \end{align*}
    where in $(i)$, we use that $\sum_{n\ge0}\sqrt{(1-\gamma)^n} = \sum_{n \ge 0}\sqrt{(1-\gamma)^{2n}} + \sqrt{(1-\gamma)^{2n+1}} \le 2\sum_{n \ge 0}\sqrt{(1-\gamma)^{2n}} = 2\sum_{n \ge 0}(1-\gamma)^{n}$. One can establish the same bound for $\step \sum_{j=k}^{K+1}\|\Phicldisc{j,k}\| $, which gives the desired bound on $\Kpione$.
    \qed

\newpage
\section{Optimization Proofs}\label{app:opt_proofs}
\subsection{Proof of Descent Lemma (\Cref{lem:descent_lem})}\label{sec:lem:descent_lem}

	\begin{proof}[Proof of \Cref{lem:descent_lem}] 
	For simplicity, write $\errnab = \errnabpi[\pin](\delta)$, $\Lnabpiinf = \Lnabpiinf[\pin]$, and $\errx = \errx(\delta)$ and $\errnab(\delta)$ and $M \ge  \Mtayjpi[\pin]$. Note that if $\pi$ and $\tilde{\pi}$ have the same input sequence but possibly different gains, $\Jdisc(\pi) = \Jdisc(\tilde \pi)$. Therefore,
	\begin{align*}
	\Jdisc(\pi^{(n+1)}) = \Jdisc(\tilde{\pi}^{(n+1)}).
	\end{align*}
	Define the input 
	\begin{align*}
	\bchuk^{(n)} = \bukn - \eta \bnabhatkn + \bKk^{\pin}(\bxkpi - \bhatxk)
	\end{align*}
	Then, as in \Cref{eq:oracle_id},
	\begin{align*}
	\buk^{\tilde{\pi}^{(n)}} = \uoffpink\,(\btiluk[1:K]) = \btiluk^{\pin}\,(\bchuk[1:K]), \quad \bxk^{\tilde{\pi}^{(n)}} = \xoffpink\,(\btiluk[1:K]) = \btilxk^{\pin}\,(\bchuk[1:K]), \quad \forall k \in [K].
	\end{align*}
	Consequently, we have the quality
	\begin{align*}
	\Jdisc(\tilde{\pi}^{(n)}) := \Jpidisc[\tilde{\pi}^{(n+1)}](\buk[1:K]^{\tilde{\pi}^{(n)}}) = \Jpidisc[\pin](\bchuk[1:K]) = \Jpidisc[\pin](\buk[1:K]^{\pin} +\updelta\bchuk^{(n)}),
	\end{align*}
	where we introduced
	\begin{align*}
	\updelta\bchuk^{(n)} := \bchuk^{(n)}  - \bukn = \underbrace{  - \frac{\eta}{\step} \bnabhatkn}_{=\updelta\bchuk^{(n;1)}} + \underbrace{\bKk^{\pin}(\bxkpi - \bhatxk)}_{=\updelta\bchuk^{(n;2)}}
	\end{align*}

	\begin{claim} We have
	\begin{align*}
	\sqrt{\step}\|\updelta\bchuk[1:K]^{(n)}\|_{\ell_2} &\le \sqrt{T}(\eta( \Lnabpiinf + \frac{1}{\step}\errnab  ) +\errx)\\
	\max_k \|\updelta\bchuk^{(n)}\| &\le (\eta( \Lnabpiinf + \frac{1}{\step}\errnab  ) +\errx)
	\end{align*}
	\end{claim}
	\begin{proof} The first bound follows from the second. We have that 
	\begin{align*}
	\max_k \|\updelta\bchuk^{(n)}\| &\le \max_k \|\updelta\bchuk^{(n;1)}\| + \max_k \|\updelta\bchuk^{(n;2)}\|\\
	&\le \frac{\eta}{\step}\max_k \|\bnabhatkn\| +\errx\\
	&\le \frac{\eta}{\step}\left(\max_k \|\Jpidisc[\pin](\buk[1:K]^{\pin})\| +  \errnab\right) +\errx\\
	&\le \frac{\eta}{\step}(\step \Lnabpiinf + \errnab) +\errx \tag{\Cref{lem:grad_bound}}\\
	&= \eta( \Lnabpiinf + \frac{1}{\step}\errnab  ) +\errx
	\end{align*}
	\end{proof}

	As a consequence of the above claim, it holds that if 
	\begin{align*}
	(\eta( \Lnabpiinf + \frac{1}{\step}\errnab  ) +\errx) \le \min\left\{\frac{\Rfeas}{8},\Bstabpi,\Btaypiinf,\frac{\Btaypitwo}{\sqrt{T}}\right\},
	\end{align*} 
	then (a) \Cref{prop:Kpi_bounds_stab} implies stability of $\tilde \pi^{(n)}$:
	\begin{align*}
	\mu_{\tilde{\pi}^{(n)},\star} \le 2\mu_{\pin,\star}, \quad L_{\tilde{\pi}^{(n)}} = L_{\pin},
	\end{align*}
	and (b)  the Taylor expansion in \Cref{lem:taylor_expansion_of_cost} implies that
	\begin{align*}
	&\Jdisc(\tilde{\pi}^{(n)}) = \Jpidisc[\pin](\buk[1:K]^{\pin} +\updelta\bchuk^{(n)}) \\
	&=\Jdisc(\pin) + \left\langle  \updelta\bchuk[1:K]^{(n)}, \nabla \Jpidisc[\pin](\buk[1:K]^{\pin})\right\rangle + \frac{R\step}{2}\|\updelta\bchuk[1:K]^{(n)}\|_{\ell_2}^2 \\
	&\le \Jdisc(\pin) + \sum_{i=1}^2\underbrace{\left\langle  \updelta\bchuk[1:K]^{(n;i)}, \nabla \Jpidisc[\pin](\buk[1:K]^{\pin})\right\rangle + M\step\|\updelta\bchuk[1:K]^{(n;i)}\|_{\ell_2}^2}_{\Term_i} \tag{AM-GM}.
	\end{align*}
	It remains to massage the above display to obtain the descent descent guarantee:
	\begin{align*}
	\Term_1 &= \left\langle  \updelta\bchuk[1:K]^{(n;1)}, \nabla \Jpidisc[\pin](\buk[1:K]^{\pin})\right\rangle + M\step\|\updelta\bchuk[1:K]^{(n;1)}\|_{\ell_2}^2\\
	&\le \langle  \updelta\bchuk[1:K]^{(n;1)}, \bnabhatkn[1:K] \rangle +  \|\updelta\bchuk^{(n;1)}\|_{\ell_2}\sqrt{K}\errnab  + M\step\|\updelta\bchuk[1:K]^{(n;1)}\|_{\ell_2}^2\\
	&\le \langle  \updelta\bchuk[1:K]^{(n;1)}, \bnabhatkn[1:K]\rangle +  \frac{K}{4R\step}\errnab^2  + 2M\step\|\updelta\bchuk[1:K]^{(n;1)}\|_{\ell_2}^2\\
	&=  (-\frac{\eta}{\step} + 2M\frac{\eta^2}{\step})\|\bnabhatkn[1:K]\|_{\ell_2}^2 +  \frac{T}{4R\step^2}\errnab^2 \\
	&\ge  -\frac{\eta}{2\step}\|\bnabhatkn[1:K]\|_{\ell_2}^2 +  \frac{T}{4M\step^2}\errnab^2,
	\end{align*}
	where the last step uses $\eta \le \frac{1}{4M}$.
	Then, 
	\begin{align*}
	\Term_2 &= \left\langle  \updelta\bchuk[1:K]^{(n;2)}, \nabla \Jpidisc[\pin](\buk[1:K]^{\pin})\right\rangle + M\step\|\updelta\bchuk[1:K]^{(n;2)}\|_{\ell_2}^2\\
	&\le \|\updelta\bchuk[1:K]^{(n;2)}\|_{\ell_2} \|\nabla \Jpidisc[\pin](\buk[1:K]^{\pin})\|_{\ell_2} + M\step\|\updelta\bchuk[1:K]^{(n;2)}\|_{\ell_2}^2\\
	&\le \sqrt{\step}\|\updelta\bchuk[1:K]^{(n;2)}\|_{\ell_2} \cdot \sqrt{K}\step\Lnabpiinf + M\step\|\updelta\bchuk[1:K]^{(n;2)}\|_{\ell_2}^2 \tag{\Cref{lem:grad_bound}}\\
	&\le \sqrt{ K}\max_{k}\|\updelta\bchuk^{(n;2)}\| \cdot \sqrt{K}\step\Lnabpiinf + M\step K\max_{k}\|\updelta\bchuk\|^2 \\
	&= T(\errx \Lnabpiinf + M \errx^2).
	\end{align*}
	Thus,
	\begin{align*}
	&\Jdisc(\pi^{(n+1)}) - \Jdisc(\pi^{(n)}) \\
	&\le  \Term_1 + \Term_2 \\
	&\le -\frac{\eta}{2\step}\|\bnabhatkn[1:K]\|_{\ell_2}^2 +  T(\frac{1}{4M\step^2}\errnab^2 + \errx \Lnabpiinf + M \errx^2).
	\end{align*}

	\end{proof}

\subsection{Proof of \Cref{prop:Jpijac}}\label{sec:proof:Jpijac}
	The proof of \Cref{prop:Jpijac} makes liberal use of the definitions of the linearizations given in \Cref{defn:Jlin_prelim}, which we recall without further comment. Going forward, introduce the Jacobian linearization of the stabilized cost:
	\begin{align*}
	\Jpijac(\bbaru) &:= V(\btilxpijac(t \mid \bbaru)) + \int_{0}^T Q(\btilxpijac(t \mid \bbaru),\btilupijac(t \mid \bbaru),t)\rmd t.
	\end{align*}
	We now characterize some properties of $\Jpijac$.
	\begin{lemma}[Valid First-Order Approximation]\label{lem:valid_first_order} We have that $\nabla_{\bbaru}\Jpi(\bbaru) = \nabla_{\bbaru}\Jpijac(\bbaru)$.
	\end{lemma}
	\begin{proof} Immediate from the chain rule, and the fact that the Jacobian linearizations are defined as the first-order Taylor expansion of the true dynamics.
	\end{proof}

	\begin{lemma}[Congruence with the Open-Loop]\label{lem:ol_congruence}
	\begin{align*}
	\inf_{\btilu} \Jpijac(\btilu) = \inf_{\btilu} \Jjac_T(\btilu;\upi).
	\end{align*}
	\end{lemma}
	\begin{proof} We prove $\inf_{\btilu} \Jpijac(\btilu) \le \inf_{\btilu} \Jjac_T(\btilu;\upi)$; the converse can be proved similarly. Fix any $\bbaru_1 \in \cU$. It suffices to exhibit some $\bbaru_2 \in \cU$ such that, for all $t \in [0,T]$,
	\begin{align*}
	\btilxpijac(t\mid \bbaru_2) = \xjac(t \mid \bbaru_1;\upi), \quad \btilupijac(t\mid \bbaru_2) = \bbaru_1(t).
	\end{align*}
	By substracting off $\xpi(t)$ and $\upi(t)$, it suffices to show that 
	\begin{align*}
	\updelta\btilxpijac(t\mid \bbaru_2) = \updelta\xjac(t \mid \bbaru_1;\upi), \quad \updelta\btilupijac(t\mid \bbaru_2) = \updelta\bbaru_1(t).
	\end{align*}
	It can be directly checked from \Cref{lem:jac_cl_implicit,lem:jac_ol_implicit} that the input $\bbaru_2(t) = \bbaru_1(t) - \bKkpi[k(t)]\updelta\xjac(t \mid \bbaru_1;\upi)$ ensures the above display holds.
	\end{proof}

	\newcommand{\alphapi}{\alpha_{\pi}}
	The last lemma contains our main technical endeavor, and its proof is defered to \Cref{lem:strong_cvx_establish} just below.
	\begin{lemma}[Strong Convexity]\label{lem:strong_cvx_establish} Suppose $\step \le \min\{\frac{1}{4\LF},\frac{1}{16\Lpi\LF}\}$. Then, $\bbaru \mapsto \Jpijac(\bbaru) - \alphapi \|\bbaru\|_{\cL_2(\cU)}^2$ is convex, where $\alphapi:=\frac{\alpha}{64\max\{1,\Lpi^2\}}$.
	\end{lemma}
	We may now conclude the proof of our proposition.
	\begin{proof}[Proof of \Cref{prop:Jpijac}] Suppose that  $\pi$ satisfies \Cref{cond:pi_cond}, and suppose  $\step \le \min\{\frac{1}{4\LF},\frac{1}{16\Lpi\LF}\}$ and $\|\nabla_{\bbaru}\Jpi(\bbaru)\big{|}_{\bbaru = \upi}\|_{\cL_2(\cU))} \le \epsilon_0$. By \Cref{lem:valid_first_order}, $\|\nabla_{\bbaru}\Jpijac(\bbaru)\big{|}_{\bbaru = \upi}\|_{\cL_2(\cU))} \le \epsilon_0$. By \Cref{lem:strong_cvx_establish} and the fact that strong convex functions satisfy the PL-inequality (e.g. \citet[Theorem 2]{karimi2016linear}), we have
	\begin{align*}
	\Jpijac(\upi) \le \inf_{\bbaru} \Jpijac(\bbaru) + \frac{\epsilon_0^2}{\alpha_{\pi}} =  \inf_{\bbaru} \Jpijac(\bbaru) +  \frac{64\epsilon_0^2\max\{1,\Lpi^2\}}{\alpha}.
	\end{align*}
	Finally, by \Cref{lem:ol_congruence}, $ \inf_{\bbaru} \Jpijac(\bbaru) =  \inf_{\bbaru} \Jjac_T(\bbaru;\bu)$, which implies the proposition.
	\end{proof}

\subsubsection{Proof of \Cref{lem:strong_cvx_establish}}\label{sec:lem:strong_cvx_establish}
	
	\begin{proof} We claim that suffices to show the following PSD lower bound:
	\begin{align}
	\forall \bbaru \in \cU, \cQ^{\pi}(\bbaru) \ge \frac{\alphapi}{\alpha} \|\bbaru\|_{\cL_2}^2, \label{eq:Qpi_lb}
	\end{align}
	where we define
	\begin{align*}
	\cQ^{\pi}(\bbaru) := \int_{0}^T (\|\updelta\btilxpijac(t \mid \bbaru)\|^2 + \|\updelta\btilupijac(t \mid \bbaru)\|^2)\rmd t,
	\end{align*}
	and where we define the deviations $\updelta (\cdot)$ as in \Cref{lem:jac_cl_implicit}.
	\begin{claim}\label{claim:str_cvx} If \Cref{eq:Qpi_lb} holds, then \Cref{lem:strong_cvx_establish} holds.
	\end{claim}
	\begin{proof}[Proof of \Cref{claim:str_cvx}]
	Note that $\bbaru \mapsto \btilxpijac(t \mid \bbaru)$ and $\bbaru \mapsto \btilupijac(t \mid \bbaru)$  are affine, that $\updelta\btilxpijac(t \mid \bbaru) = \btilxpijac(t \mid \bbaru) -\xpi(t)$ and $\updelta\btilupijac(t \mid \bbaru) = \btilupijac(t \mid \bbaru) -\upi(t)$ are linear (no affine term), and that the diferences $\btilxpijac(t \mid \bbaru) - \updelta\btilxpijac(t \mid \bbaru) = \xpi(t)$ and $\btilupijac(t \mid \bbaru) - \updelta\btilupijac(t \mid \bbaru) = \upi(t)$ are independent of $\bbaru$. Hence, we cocnlude $\cQ^\pi(\bbaru)$ is a quadratic function with no linear term, and $\tilde\cQ^\pi(\bbaru) - \cQ^\pi(\bbaru)$ is linear, where we define
	\begin{align*}
	\tilde\cQ^{\pi}(\bbaru) := \int_{0}^T (\|\btilxpijac(t \mid \bbaru)\|^2 + \|\btilupijac(t \mid \bbaru)\|^2)\rmd t,
	\end{align*}
	\Cref{asm:convexity} implies that $\Jpijac(\bbaru) - \alpha\tilde \cQ^{\pi}(\bbaru)$ is convex, and since the difference $\tilde\cQ^{\pi}(\bbaru) - \cQ^{\pi}(\bbaru)$ is linear, that $\Jpijac(\bbaru) - \alpha\tilde \cQ^{\pi}(\bbaru)$ is also convex. Lastly, as $\cQ^{\pi}(\bbaru)$ is quadratic with no linear term, \Cref{eq:Qpi_lb} implies $\alpha\cQ^{\pi}(\bbaru) - \alphapi \|\bbaru\|_{\cL_2}^2$ is convex. Thus,  $\Jpijac(\bbaru) -  \alphapi \|\bbaru\|_{\cL_2}^2 = (\Jpijac(\bbaru) - \alpha\tilde \cQ^{\pi}(\bbaru)) - (\alpha\cQ^{\pi}(\bbaru) - \alphapi \|\bbaru\|_{\cL_2}^2)$ is convex.
	\end{proof}
	\newcommand{\bbarx}{\bar{\bx}}
	To verify \Cref{eq:Qpi_lb}, let us define a few salient operators. Let $\cX$ denote the space of $\cL_2$ bounded curves $\bx(t) \in \R^{\dimx}$. We define linear operators $\bbT_1:\cU \to \cX$ and $\bbT_2:\cU \to \cX$ and $\bbK:\cX \to \cU$ via
	\begin{align*}
	\bbT_1[\bbaru] (t)= \updelta\btilxpijac(\tkt \mid \bbaru), \quad \bbT_2[\bbaru] (t)= \updelta\btilxpijac(t \mid \bbaru), \quad \bbK[\bbarx](t) = \bKkpi[k(t)]\bbarx(t).
	\end{align*}
	Then, letting $\eye_{\cU}$ denote the identity operator of $\cU$, we can write
	\begin{align*}
	\cQ^{\pi}(\bbaru) = \| \bbT_2[\bbaru]\|_{\cL_2(\cX)}^2 + \|(\eye_{\cU}+ \bbK\bbT_1)[\bbaru]\|_{\cL_2(\cU)}^2.
	\end{align*}
	Next, we relate $\bbT_2$ and $\bbT_1$. Define the operators $\bbL: \cX \to \cX$ and $\bbW:\cU \to \cX$ by
	\begin{align*}
	\bbL[\bbarx](t) = \Phiolpi(t,\tkt)\bbarx(t), \quad \bbW[\bbaru](t) = \int_{s=\tkt}^t \Phiolpi(t,s)\Bpi(s)\bbaru(s)\rmd s.
	\end{align*}
	Then, it can be checked from \Cref{lem:jac_cl_implicit,lem:solve_affine_ode} that
	\begin{align*}
	\bbT_2[\bbaru] = \bbL \bbT_1[\bbaru] + \bbW[\bbaru].
	\end{align*}
	Hence,
	\begin{align*}
	\cQ^{\pi}(\bbaru) = \| (\bbL\bbT_1+\bbW)[\bbaru]\|_{\cL_2(\cX)}^2 + \|(\eye_{\cU}+ \bbK\bbT_1)[\bbaru]\|_{\cL_2(\cU)}^2.
	\end{align*}
	With this representation of $\cQ^\pi$, we establish a lower bound by applying the following lemma, whose proof we below:
	\begin{lemma}\label{lem:PSD_hilbert_space} Let $\cA,\cB$ be Hilbert spaces with norms $\|\cdot\|_{\cA}$ and $\|\cdot\|_{\cB}$, let $\bI_{\cA}$ denote the identity operator on $\bU$, and let $\bbT,\bbW: \cA \to \cB$, $\bbL: \cB \to \cB$, and $\bK: \cB \to \cA$ be linear operators, and let $\|\cdot\|_{\op}$ denote operator norms. Then, if $\|\bbW\|_{\op} \le \frac{\min\{1,\sigma_{\min}(\bbL)\}}{4\|\bbK\|}$, it holds for any $\ba \in \cA$,
	\begin{align*}
	\|(\bbL\bbT + \bbW) [\ba]\|_{\cB}^2 + \|(\bI_{\cA} + \bbK\bbT )[\ba]\|_{\cA}^2 \ge \|\ba\|^2 \cdot \frac{\min\{1,\sigma_{\min}(\bbL)^2\}}{16\max\{1,\|\bbK\|_{\op}^2\}},
	\end{align*} 
	where $\sigma_{\min}(\bbL) := \inf_{\ba:\|\ba\|_{\cA} = 1}\|\bbL \ba\|_{\cA}$.
	\end{lemma}
	To apply the lemma, we first perform a few computations. Throughout, we use $\|\cdot\|_{\op}$ to denote operator norm, and $\sigma_{\min}$ to denote minimal singular value as an operator, e.g. 
	\begin{align*}
	\|\bbK\|_{\op} = \sup_{\bbarx \ne 0} \frac{\|\bbK[\bbarx]\|_{\cL_2(\cU)}}{\|\bbarx\|_{\cL_2(\cX)}}, \sigma_{\min}(\bbL) = \inf_{\bbarx \ne 0} \frac{\|\bbL[\bbarx]\|_{\cL_2(\cU)}}{\|\bbarx\|_{\cL_2(\cX)}}.
	\end{align*}
	\begin{claim} Under \Cref{cond:pi_cond}, $\|\bbK\|_{\op} \le \Lpi$.
	\end{claim}
	\begin{proof}
	Since $\bbK$ is a (block-)diagonal operator in $t$, i.e. $\bbK[\bbarx](t) = \bKkpi[k(t)]\bbarx(t)$,  its $\cL_2(\cX) \to \cL_2(\cU)$ operator is bounded by $\max_{k}\|\bKkpi\|$, which is at most $\Lpi$ under \Cref{cond:pi_cond}.
	\end{proof} 
	\begin{claim} For $\step \le \LF/4$, $\sigma_{\min}(\bbL) \ge 1-\step \LF\exp(\step\LF) \ge \frac{1}{2}$.
	\end{claim}
	\begin{proof} Again, since $\bbL$ is a block-diagonal operator in $t$, i.e. $\bbL[\bbarx](t) = \Phiolpi(t,\tkt)\bbarx(t)$, $\sigma_{\min}(\bbL) = \inf_{t \in [0,T]}\sigma_{\min}(\Phiolpi(t,\tkt))$. By $\dots$, $\|\Phiolpi(t,\tkt)-\eye\| \le \step \LF \step(\step \LF) \le 1/2$. Thus,$ \inf_{t \in [0,T]}\sigma_{\min}(\Phiolpi(t,\tkt)) \ge 1/2$.
	\end{proof}
	\begin{claim} $\|\bbW\|_{\op} \le \step\LF\exp(\step \LF) $, which is at most $2\step \LF$ for $\step \le \LF/4$.
	\end{claim}
	\begin{proof} For any $\bbaru$, we bound
	\begin{align*}
	\|\bbW[\bbaru]\|_{\cL_2(\cU)}^2 &= \int_{t=0}^T \left\|\int_{s=\tkt}^t \Phiolpi(t,s)\Bpi(s)\bbaru(s)\rmd s\right\|^2\\
	&\le \int_{t=0}^T (t-\tkt)\int_{s=\tkt}^t \left\|\Phiolpi(t,s)\Bpi(s)\bbaru(s)\rmd s\right\|^2 \tag{Cauchy-Schwartz}\\
	&\le \step\exp(2\step \LF)\int_{t=0}^T \int_{s=\tkt}^t \|\Bpi(s)\|^2\left\|\bbaru(s)\rmd s\right\|^2 \tag{$t-\tkt \le \step$, \Cref{lem:bound_on_open_loop}}\\
	&\le \step\LF^2\exp(2\step \LF)\int_{t=0}^T \int_{s=\tkt}^t \left\|\bbaru(s)\rmd s\right\|^2 \tag{ \Cref{asm:max_dyn_final}}\\
	&= \step\LF^2\exp(2\step \LF)\int_{s=0}^T \int_{t=s}^{t_{k+1}(t)} \left\|\bbaru(s)\rmd s\right\|^2 \rmd s\\
	&\le \step^2\LF^2\exp(2\step \LF)\int_{s=0}^T  \left\|\bbaru(s)\rmd s\right\|^2 \rmd s \\
	&=(\step\LF\exp(\step \LF))^2 \|\bbaru\|_{\cL_2(\cU)}^2.
	\end{align*}
	\end{proof}
	With the above three claims, \Cref{lem:PSD_hilbert_space} implies that as long as $\step \le \min\{\frac{1}{4\LF},\frac{1}{16\Lpi\LF}\}$,
	\begin{align*}
	\cQ^{\pi}(\bbaru) = \| (\bbL\bbT_1+\bbW)[\bbaru]\|_{\cL_2(\cX)}^2 + \|(\eye_{\cU}+ \bbK\bbT_1)[\bbaru]\|_{\cL_2(\cU)}^2\ge \frac{1}{64\max\{1,\Lpi^2\}} = \frac{\alphapi}{\alpha}.
	\end{align*}
	\end{proof}

	\begin{proof}[Proof of \Cref{lem:PSD_hilbert_space}] Without loss of generality, assume $\|\ba\|_{\cA} = 1$, and define $\eta = \frac{\min\{1,\sigma_{\min}(\bbL)\}}{2\|\bbK\|_{\op}}$ . We consider two cases: First, if $\|\bbT\ba\| \le \eta$, then 
	\begin{align*}
	\|(\bbL\bbT + \bbW) [\ba]\|_{\cB}^2 + \|(\bI_{\cA} + \bbK\bbT )[\ba]\|_{\cA}^2 \ge \|(\bI_{\cA} + \bbK\bbT )[\ba]\|_{\cA}^2 \ge (\|\ba\| - \|\bbK\|_{\op}\eta)^2 = (1 - \frac{1}{2})^2 \ge \frac{1}{4}.
	\end{align*}
	Otherwise, suppose $\|\bbT\ba\| > \eta$. Then, 
	\begin{align*}
	\|(\bbL\bbT + \bbW) [\ba]\|_{\cB}^2 + \|(\bI_{\cA} + \bbK\bbT )[\ba]\|_{\cA}^2 \ge \|(\bbL\bbT + \bbW) [\ba]\|_{\cB}^2 \ge (\eta - \|\bbW\|_{\op}\|\ba\|)^2 = (\frac{\min\{1,\sigma_{\min}(\bbL)\}}{2\|\bbK\|_{\op}} - \|\bbW\|_{\op})^2.
	\end{align*}
	Hence, if $\|\bbW\|_{\op} \le \frac{\min\{1,\sigma_{\min}(\bbL)\}}{4\|\bbK\|_{\op}}$, the above is at least
	\begin{align*}
	\frac{\min\{1,\sigma_{\min}(\bbL)^2\}}{16\max\{1,\|\bbK\|_{\op}^2\}}.
	\end{align*}
	\end{proof}


\newcommand{\psipi}{\psi_{\pi}}
\newcommand{\Ktilpi}{\tilde{K}_{\pi}}

\newcommand{\Lclone}{L_{\mathrm{cl},1}}
\newcommand{\Lcltwo}{L_{\mathrm{cl},2}}
\newcommand{\Lclthree}{L_{\mathrm{cl},3}}
\section{Discretization Arguments}\label{app:dt_args}
In this section, we use discretizations of the Markov and transition operators. Again, we assume $\pi$ is feasible. We define the shorthand.
\begin{definition}[Useful Shorthand]\label{defn:shorthand} Define the short hand 
\begin{align}\psipi(k_2,k_1) := \|\Phicldisc{k_2,k_1+1}\|, \quad \Lol := \exp(\step \LF), \label{eq:shorthand}
 \end{align}
 and note that $\psipi(k_2,k_1) \le \Kpiinf$ due to \Cref{lem:Kpi_bounds}.
\end{definition}

\subsection{Discretization of Open-Loop Linearizations}
All lemmas in this section assume $\pi$ is feasible. 
\begin{lemma}[Continuity of $\xpi(\cdot)$]\label{lem:continuity_x} Then,
\begin{align*}
\|\xpi(t) - \xpi(t')\| \le \KF|t-t'|
\end{align*}
\end{lemma}
\begin{proof} Assume without loss of generality that $t' \ge t$. By \Cref{asm:max_dyn_final} feasibility of $\pi$,
\begin{align*}
\|\frac{\rmd}{\rmd s} \xpi(s)\| = \|\fdyn(\xpi(s),\upi(s)\| \le \KF. 
\end{align*}
Hence, as $\xpi(t') = \xpi(t) + \int_{s=t}^{t'} \fdyn(\xpi(s),\upi(s))\rmd s$, the bound follows.
\end{proof}

\begin{lemma}\label{lem:Bpi_Api}  For all $k,s \in \cI_k$, we have $\|\Bpi(t) - \Bpi(s)\| \vee \|\Api(t) - \Api(s)\| \le \step \KF\MF$. 
\end{lemma}
\begin{proof} We bound $\|\Bpi(t) - \Bpi(s)\| $ as $\|\Api(t) - \Api(s)\|$ is similar. Assume $s \le t$ without loss of generality. Then, 
\begin{align*} 
\|\Bpi(t) - \Bpi(s)\| &= \|\partu \fdyn(\xpi(t),\upi(t)) - \partu \fdyn(\xpi(s),\upi(s))\\
&= \|\partu \fdyn(\xpi(t),\upi(s)) - \partu \fdyn(\xpi(s),\upi(s))\| \tag{$\upi \in \Utrajstep$}\\
&= \|\partu \fdyn(\xpi(t),\upi(s)) - \partu \fdyn(\xpi(s),\upi(s))\| \tag{$\upi \in \Utrajstep$}\\
&\le \|\xpi(t)-\xpi(s)\| \max_{\alpha \in [0,1]}\|\partu \fdyn(\alpha\xpi(t)+ (1-\alpha)\xpi(s),\upi(s))\| \tag{Mean Value Theorem}\\
&\le \|\xpi(t)-\xpi(s)\| \MF \tag{\Cref{asm:max_dyn_final} and convexity of feasibility}\\
&\le (t-s) \KF\MF \le \step \KF\MF. \tag{\Cref{lem:continuity_x}}
\end{align*}
\end{proof}
\begin{lemma}[Bound on $\bBkpi$]\label{lem:bound_on_bBkpi} For any $k \in [K],$ $\|\bBkpi\| \le \step \Lol \LF = \step \LF\exp(\step \LF)$.
\end{lemma}
\begin{proof} $\|\bBkpi\|  = \int_{s=t_k}^{t_{k+1}} \Phiolpi(t_{k+1},s)\Bpi(s)\rmd s \le \step \max_{s \in \cI_k} \|\Phiolpi(t_{k+1},s)\|\|\Bpi(s)\|$. We bound $\|\Bpi(s)\| \le \LF$ by \Cref{asm:max_dyn_final} and $\|\Phiolpi(t_{k+1},s)\| \le \Lol$ by \Cref{lem:bound_on_open_loop} below.
\end{proof}

\subsection{Discretization of Transition and Markov Operators}

We begin by discretizing the open-loop transition operator. 
\begin{lemma}[Discretization of Open-Loop Transition Operator]\label{lem:bound_on_open_loop} Recall $\Lol = \exp(\step \Lol)$. $\|\Phiolpi(t',t) - \eye\| \le (t'-t)\LF \exp((t'-t)\LF)$. Moreover, $\Phiolpi(t,t') \le \exp((t'-t)\LF)$. In particular, if $t,t' \in \cI_k$, then 
\begin{align*}
\|\Phiolpi(t',t) - \eye\| \le \step \LF \Lol , \quad \Phiolpi(t',t) \le \Lol = \exp(\step \LF)
\end{align*}
\end{lemma}
\begin{proof}[Proof of \Cref{lem:bound_on_open_loop}] For the first part, it suffices to bound the ODE $\by(t') = \Phiolpi(t',t)\xi$, where $\xi \in \R^{\dimx}$ is an arbitrary initial condition with $\|\xi\| = 1$. Then $\by(t) = \Phiolpi(t,t)\xi = \xi$, and $\dds \by(s) = \Api(s)\by(s)$. Hence, $\|\dds \by(s)\| \le \LF\|\by(s)\|$. The result now follows by comparison to the constant ODE $\bz(t) = \by(t)$, $\dds \bz(s) = 0$, and Picard's Lemma (\Cref{lem:picard}). The second part follows from Picard's lemma with comparison the to the stationary curve $\bz(0) = 0$. 
\end{proof}
Next, we bound the difference between the operator $\Phitilclpi$ in the definition of $\Phiclpi$, and the identity matrix. 
\begin{lemma}\label{lem:bound_on_til} Recall the definition $\Lol = \exp(\step \LF)$ and 
\begin{align*}
\Phitilclpi(s,t_{k}) = \Phiolpi(s,t_{k}) + (\int_{s'=t_k}^{s}\Phiolpi(s,s')\Bpi(s')\rmd s)\bKk.
\end{align*}
Then,
\begin{align*}\|\Phitilclpi(t,\tkt) - \eye\| \le \step \LF\Lol(1+\Lpi).
\end{align*} 
Similary, 
\begin{align*}\|\Phitilclpi(t,\tkt) - \eye\| \le \step\LF\Lol(1+\Lpi).
\end{align*}
\end{lemma}
\begin{proof}[Proof of \Cref{lem:bound_on_til}] Let $t_k = \tkt$ for shorthand. We have
\begin{align*}
\|\Phitilclpi(t,t_k) - \eye\| &= \|\Phiolpi(t,t_k)  + \int_{s=t_k}^{t}\Phiolpi(t,s)\Bpi(s)\rmd s)\bKkpi - \eye\| \\
 &\le \|\Phiolpi(t,t_k)  - \eye\|  + |t-t_k|\|\bKkpi\| \max_{s \in [t_k,t]}\|\Phiolpi(t,s)\Bpi(s)\|\\
  &\le \|\Phiolpi(t,t_k)  - \eye\|  + |t-t_k|\Lpi\LF \max_{s \in [t_k,t]}\|\Phiolpi(t,s)\| \tag{\Cref{asm:max_dyn_final}}\\
  &\le |t-t_k|\LF \exp(\LF(t - t_k))  + |t-t_k|\Lpi\LF \max_{s \in [t_k,t]} \exp(\LF(t - t_k)) \tag{\Cref{lem:bound_on_open_loop}}\\
  &\le \step \LF(1+\Lpi) \exp(\LF \step) = \step \LF(1+\Lpi)\Lol.
\end{align*}
The bound on $\|\Phitilclpi(t,\tkt) - \eye\|$ can be derived similarly.

\end{proof}

\begin{lemma}[Discretization of Closed-Loop Transition Operator]\label{lem:bounds_on_cl}
 Let $s > t$ such that $\tks > \tkt$. Then, under \Cref{cond:pi_cond},
\begin{align*}
\|\Phiclpi(s,\tktpl) - \Phicldisc{k(s),k(t)+1}\| \le  2\LF\Lol\Lpi \psipi(k(s),k(t)).
\end{align*}
\end{lemma}
\begin{proof}[Proof of \Cref{lem:bounds_on_cl}]
As $\tks > \tkt$, we can write $s \in \cI_{k_2}$ and $t \in \cI_{k_1}$ for $k_2 > k_1$; then 
\begin{align*}
k_2 = k(s), \quad k_1 = k(t).
\end{align*}
 We now have
\begin{align*}
\|\Phiclpi(s,\tktpl) - \Phicldisc{k(s),k(t)+1}\| &= \| \Phitilclpi(s,t_{k_2}) - \eye\|\cdot \|\Phicldisc{k_2,k_1+1}\|\\
&= \| \Phitilclpi(s,t_{k_2}) - \eye\|\cdot\psipi(k_2,k_1)
\end{align*} 
Directly from { \Cref{lem:bound_on_til} and $k_2 = k(s)$, $\|\Phitilclpi(s,t_{k_2}) - \eye\| \le  \step \LF\Lol(1+\Lpi)$.} This yields, with $\Lpi \ge 1$, 
\begin{align*}
\|\Phiclpi(s,\tktpl) - \Phicldisc{k(s),k(t)+1}\| \le \step \LF\Lol(1+\Lpi)\psipi(k_2,k_1) \le 2\LF\Lol\Lpi \psipi(k_2,k_1).
\end{align*}
\end{proof}

\begin{lemma}\label{lem:Phicldisc_onestep_bound} Suppose \Cref{cond:pi_cond} holds and that $\step \le 1/4\LF\max\{1,\Lpi\}$. Then,  $\|\Phicldisc{k+1,k}\| \le 5/3$.
\end{lemma}
\begin{proof} We have
\begin{align*}
\|\Phicldisc{k+1,1}\| = \|\Phiolpi(t_{k+1},t_k) + \bBkpi \bKk\| \le \|\Phiolpi(t_{k+1},t_k)\| + \Lpi \|\bBkpi\|,
\end{align*}
where we use $\|\bKk\| \le \Lpi$ under \Cref{cond:pi_cond}. By \Cref{lem:bound_on_bBkpi}, $\|\bBkpi\| \le \step \LF\exp(\step \LF)$ and by \Cref{lem:bound_on_open_loop},$\|\Phiolpi(t_{k+1},t_k)\| \le \Lol = \exp(\step \LF)$. Then, $\|\Phicldisc{k+1,1}\| \le \exp(\step \LF)(1 + \step \LF \Lpi)$. For $\step \le 1/\LF\max\{1,\Lpi\}$, we have $\|\Phicldisc{k+1,1}\| \le \exp(1/4)(1+1/4) \le 5/3$.
\end{proof}
Finally, we turn to a discretization of the Markov operator:
\begin{lemma}[Discretization of Closed-Loop Markov Operator]\label{lem:correct_Psiclpi_disc} The following bounds hold:
\begin{itemize}
\item[(a)]For any $k_2 > k_1$, we have
\begin{align*}
\max_{t \in \cI_{k_1},s \in \cI_{k_2}}\|\step^{-1}\Psicldisc{k_2,k_1} - \Psiclpi(s,t)\| \le \step\psipi(k_2,k_1) \underbrace{\Lol\left( \KF\MF + 4\LF^2\Lpi \right)}_{\Lclone}.
\end{align*}
\item[(b)]  For any $k_2 > k_1$, we have 
\begin{align*}
\sup_{t \in \cI_{k_1}}\|\step^{-1}\Psicldisc{k_2,k_1} - \Psiclpi(k_2,t)\| &\le \step\psipi(k_2,k_1) \underbrace{\Lol\left( \KF\MF + 2\LF^2\right)}_{\Lcltwo}\\
&\le \step\Kpiinf \Lol\left( \KF\MF + 2\LF^2\right).
\end{align*}
\item[(c)] For any $(s,t)$ with $\tks = \tkt$, we have
\begin{align*}
\|\Psicldisc{\tks,\tkt} \| \le \Lol\LF
\end{align*}
\item[(d)] For any $1 \le k_1 < k_2 \le K$ and $t$ for which $\tkt = k_1$, 
\begin{align*}
 \frac{1}{\step}\|\Psicldisc{k_2,k_1}\| \vee\|\Psiclpi(t_{k_2},t)\|  \le  \Lol\LF\psipi(k_2,k_1) \le \Lol \LF \Kpiinf.
\end{align*}
\end{itemize}
\end{lemma}
\begin{proof} Let us start with the first bound. Set $k_1 = k(t)$, and $k_2 = k(s)$. Note that $\Psicldisc{k_2,k_1} := \Phicldisc{k_2,k_1+1}\bBkpi[k_1]$

\begin{align*}
&\|\step^{-1}\Psicldisc{k_2,k_1} - \Psiclpi(s,t)\| \\
&= \| \step^{-1}\Phicldisc{k_2,k_1+1}\bBkpi[k_1] - \Phiclpi(s,t_{k_1+1})\Phiclpi(t_{k_1+1},t)\Bpi(t)\| \\
&\le  \|\Phicldisc{k_2,k_1+1}(\Phiclpi(t_{k_1+1},t)\Bpi(t) - \step^{-1}\bBkpi[k_1])\| + \|(\Phiclpi(s,t_{k_1+1})-\Phicldisc{k_2,k_1+1})\Bpi(t)\|\\
&\le  \psipi(k_2,k_1) \|\Phiclpi(t_{k_1+1},t)\Bpi(t) - \step^{-1}\bBkpi[k_1]\| + \LF\|\Phiclpi(s,t_{k_1+1})-\Phicldisc{k_2,k_1+1}\| \tag{\Cref{asm:max_dyn_final,eq:shorthand}}\\
&\le \psipi(k_2,k_1) \|\Phiclpi(t_{k_1+1},t)\Bpi(t) - \step^{-1}\bBkpi[k_1]\| + 2\step\LF^2\Lol\Lpi \psipi(k_2,k_1) \tag{\Cref{lem:bounds_on_cl}}.
\end{align*}
Finally, we bound
\begin{align*}
&\|\Phiclpi(t_{k_1+1},t)\Bpi(t) - \step^{-1}\bBkpi[k_1]\| \\
&= \|\Phiolpi(t_{k_1+1},t)\Bpi(t) - \step^{-1}\int_{s=t_{k_1}}^{t_{k_1+1}}\Phiolpi(t_{k_1+1},s) \Bpi(s)\rmd s\|\\
&\le \|\Phiolpi(t_{k_1+1},t)\Bpi(t) - \step^{-1}\int_{s=t_{k_1}}^{t}\Phiolpi(t_{k_1+1},t) \Bpi(s)\rmd s\| + \|\step^{-1}\int_{s=t_{k_1}}^{t}\Phiolpi(t_{k_1+1},s) - \Phiolpi(t_{k_1+1},t) \Bpi(s)\rmd s\|\\
&\le \|\Phiolpi(t_{k_1+1},t)\|\max_{s \in \cI_{k_1}}\|\Bpi(t) - \Bpi(s)\| + \LF\max_{s \in \cI_{k_1}}\|\Phiolpi(t_{k_1+1},s) - \Phiolpi(t_{k_1+1},t)\|\|\Bpi(s)\|\\
&\le \step \Lol \KF\MF + \LF\max_{s \in \cI_{k_1}}\|\Phiolpi(t_{k_1+1},s) - \Phiolpi(t_{k_1+1},t)\|\\
&\le \step \Lol \KF\MF +  2\step \LF^2 \Lol
\end{align*}
where the second-tolast step uses \Cref{lem:bound_on_open_loop,lem:Bpi_Api} and \Cref{asm:max_dyn_final}, and the last step uses $\|\|\Phiolpi(t_{k_1+1},s) - \Phiolpi(t_{k_1+1},t)\| \le \|\Phiolpi(t_{k_1+1},s) - \eye\| + \|\Phiolpi(t_{k_1+1},t) - \eye\| \le 2\step \LF \Lol$ by \Cref{lem:bound_on_open_loop}. Combining with the previous display, 

\begin{align*}
&\|\step^{-1}\Psicldisc{k_2,k_1} - \Psiclpi(s,t)\| \\
&\le \step\psipi(k_2,k_1) \left( \Lol \KF\MF +  2 \LF^2 \Lol + 2\LF^2\Lol\Lpi \right)\\
&\le \step\Lol\psipi(k_2,k_1) \left( \KF\MF + 4\LF^2\Lpi \right) \tag{$\Lpi \ge 1$}
\end{align*}
This concludes the proof of (a).

For (b), the argument is the same, but the contribution of $2\LF^2\Lol\Lpi$ vanishes, as $\Phiclpi(\tks,t_{k_1+1}) = \Phicldisc{k_2,k_1+1}$.

For (c), we note that if $\tks = \tkt$, $\|\Psiclpi(s,t) \| = \|\Phiolpi(s,t)\Bpi(t) \| \le \Lol\LF$ by \Cref{asm:max_dyn_final,lem:bound_on_open_loop}. For the final inequality, we have by \Cref{eq:shorthand,lem:bound_on_bBkpi} that 
\begin{align*}
\|\Psicldisc{\tks,\tkt}\| &= \|\Phicldisc{k_2,k_1+1}\bBkpi[k_1]\| \le \step \Lol\LF\psipi(k_2,k_1).
\end{align*}
The bound on $\|\Psiclpi(t_{k_2},t)\| $ is similar.
\end{proof}

\subsection{Discretization of the Gradient (Proof of \Cref{prop:grad_disc})}\label{sec:prop:grad_disc}

We use the shorthand from \Cref{defn:shorthand}.  Recall that $\nabtil\cJ_T(\pi) = \frac{1}{\step}\step(\nabla \Jdisc(\pi))$ is the continuous-time inclusion of the discrete-time gradient, renormalized by $\step^{-1}$. Thus, from \Cref{lem:grad_compute_dt},
\begin{align*}
\nabtil \cJ_T( \pi)(t) &= \step^{-1}\Qupi(\tkt) +\step^{-1} (\Psicldisc{T,\tkt})^\top\left(\partx V(\xpi(T))\right)\\
&\quad+ \sum_{j = k(t)+1}^{K}(\Psicldisc{j,k(t)})^\top (\Qxpi(t_j) + \bKk[j]\Qupi(t_j)). 
\end{align*}
From \Cref{lem:grad_compute_ct}, we can write
\begin{align*}
\nabla \cJ_T( \pi)(t) &=  \Qupi(t) + \Psiclpi(T,t)^\top\left(\partx V(\xpi(T))\right)\\
&\quad+\int_{s=t_{k(t)}}^t \Psiclpi(s,t)^\top  \Qxpi(s)\rmd s  \\
&\qquad + \sum_{j = k(t)+1}^K \int_{s \in \cI_j} \left(\Psiclpi(s,t)\Qxpi(s) + \Psiclpi(t_{j},t)^\top \bKk[j]\top \Qupi(s)\right)\rmd s,
\end{align*}
and therefore decompose the error into five terms via the triangle inequality.
\begin{align*}
\|\nabtil \cJ_T( \pi)(t) - \nabla \cJ_T( \pi)(t)\| &\le \underbrace{\|\Qupi(t) -\Qupi(\tkt)\|}_{:=\Term_1}  + \underbrace{\| (\Psiclpi(T,t) - \step^{-1} (\Psicldisc{T,\tkt}))^\top\left(\partx V(\xpi(T))\right)\|}_{:=\Term_2}\\
&\quad+ \underbrace{\int_{s=t}^{t_{k(t)+1}} \|\Psiclpi(\tks,\tkt)^\top  \Qxpi(s)\|\rmd s}_{:=\Term_3}\\
&\quad+ \step \sum_{j = k(t)+1}^K \Term_{4,j} + \Term_{5,j},
\end{align*}
where
we further define
\begin{align*}
\Term_{4,j}  &:= \sup_{s \in \cI_j}\|\Psiclpi(s,t)\Qxpi(s) - \step^{-1}(\Psicldisc{j,k(t)})^\top (\Qxpi(t_j)\|\\
\Term_{5,j} &:= \sup_{s \in \cI_j}\|\Psiclpi(t_j,t)^\top (\bKk[j])^\top \Qupi(s) - \step^{-1}(\Psicldisc{j,k(t)})^\top(\bKk[j])^\top \Qupi(t_j)\|
\end{align*}
Before contitnuous, we apply the following lemma.
\begin{lemma}[Discretization of Cost-Gradient]\label{lem:cost_disc} For $z \in \{x,u\}$,
\begin{align*}
\|Q_{x}^{\pi}(t)  - Q_{x}^{\pi}(t_k(t))\| \vee \|Q_{u}^{\pi}(t)  - Q_{u}^{\pi}(t_k(t))\|  \le  \step\MQ(1+\KF).
\end{align*}
\end{lemma}
\begin{proof} We bound $\|Q_{x}^{\pi}(t)  - Q_{x}^{\pi}(t_k(t))\|$ as the bound on $\|Q_{u}^{\pi}(t)  - Q_{u}^{\pi}(t_k(t))\|$ is similar.
\begin{align*}
\|Q_{x}^{\pi}(t)  - Q_{x}^{\pi}(t_k(t))\| &= \|\partx Q(\xpi(t), \upi(t),t)) - \partx Q(\xpi(\tkt) \upi(\tkt),\tkt)\|\\
&= \|\partx Q(\xpi(t), \upi(\tkt),t)) - \partx Q(\xpi(\tkt), \upi(\tkt),\tkt)\| \tag{$\upi(\cdot)\in \Utrajstep$}\\
&\le \MQ|t-\tkt| + \MQ\|\xpi(\tkt) - \xpi(t)\| \tag{Integrating \Cref{asm:cost_asm}}\\
&\le \MQ\step + \MQ\KF\step \tag{\Cref{lem:continuity_x}} = \step\MQ(1+\KF).
\end{align*}
\end{proof}

From \Cref{lem:cost_disc}, we bound 
\begin{align*}\Term_1 \le \step\MQ(1+\KF).
\end{align*}
Next, using that $T/\step$ is integral by assumption, i.e. $t = t_{k(T)}$, we have
\begin{align*}
\Term_2 &= \| (\Psiclpi(T,t) - \step^{-1} (\Psicldisc{T,\tkt}))^\top\left(\partx V(\xpi(T))\right)\| \\
&\le \LQ\|\Psiclpi(T,t) - \step^{-1} (\Psicldisc{T,\tkt})\| \tag{\Cref{asm:cost_asm}}\\
&\le \step\LQ\Lcltwo\psipi(k(T), k(t)) = \step\LQ\Lcltwo\psipi(K+1, k(t)) \tag{\Cref{lem:correct_Psiclpi_disc}(b)},
\end{align*}

For the third term, we have
\begin{align*}
\Term_3 &= \int_{s=t}^{t_{k(t)+1}} \|\Psiclpi(\tks,\tkt)^\top  \Qxpi(s)\|\rmd s\\
&\le \LQ \int_{s=t}^{t_{k(t)+1}} \|\Psiclpi(\tks,\tkt)^\top\|\rmd s \tag{\Cref{asm:cost_asm}}\\
&\le \step\LQ \max_{s \in [t,t_{k(t)+1})} \|\Psiclpi(\tks,\tkt)^\top\| \tag{ignoring interval endpoint due to integration}\\
&= \LQ \Lol \LF,
\end{align*}
where in the last step we use \Cref{lem:correct_Psiclpi_disc}(c). Summarizing these the bounds on the first and third term, 
\begin{align}
\Term_1 + \Term_2 + \Term_3  \le \step(\LQ \Lol \LF + \MQ(1+\KF) + \LQ\Lcltwo\psipi(K+1, k(t)) \label{eq:Terms_one_two_three}
\end{align}
Next, we turn to the fourth and and fifth terms. We bound
\begin{align*}
\Term_{4,j}  &:= \sup_{s \in \cI_j}\|\Psiclpi(s,t)\Qxpi(s) - \step^{-1}(\Psicldisc{j,k(t)})^\top (\Qxpi(t_j)\|\\
&\le \|\step^{-1}\Psicldisc{j,k(t)}\| \cdot \sup_{s \in \cI_j}\|\Qxpi(s) - \Qxpi(t_j)\| + \|\step^{-1}\Psicldisc{j,k(t)} - \Psiclpi(s,t)\|\|\Qxpi(t_j)\|\\
&\le \Lol\LF\psipi(j,k(t)) \cdot \sup_{s \in \cI_j}\|\Qxpi(s) - \Qxpi(t_j)\| + \step \Lclone\psipi(j,k(t))\|\Qxpi(t_j)\|\tag{\Cref{lem:correct_Psiclpi_disc}(a\&c)}\\
&\le \Lol\LF\psipi(j,k(t)) \cdot \step\MQ(1+\KF)) + \step \Lclone\psipi(j,k(t))\LQ\tag{\Cref{asm:cost_asm,lem:cost_disc}}\\
&=\step\psipi(j,k(t))(\Lol\LF \MQ(1+\KF) + \LQ\Lclone).
\end{align*}
Moreover,
\begin{align*}
\Term_{5,j} &= \sup_{s \in \cI_j}\|\Psiclpi(t_{j},t)^\top (\bKk[j])^\top \Qupi(s) - \step^{-1}(\Psicldisc{j,k(t)})^\top(\bKk[j])^\top \Qupi(t_j)\|\\
&\le \|\step^{-1}(\Psicldisc{j,k(t)})\| \|\bKk[j]\| \sup_{s \in \cI_j}\| \Qupi(s) -\Qupi(t_j)\|\\
&\quad + \sup_{s \in \cI_{j}}\|\step^{-1}(\Psicldisc{j,k(t)}) - \Psiclpi(t_{j},t)\| \|\bKk[j]\| \|\Qupi(t_j)\|\\
&\le \Lpi \|\step^{-1}(\Psicldisc{j,k(t)})\| \sup_{s \in \cI_j}\| \Qupi(s) -\Qupi(t_j)\|\\
&\quad + \Lpi\sup_{s \in \cI_{j}}\|\step^{-1}(\Psicldisc{j,k(t)}) - \Psiclpi(t_{j},t)\| \|\Qupi(t_j)\| \tag{\Cref{cond:pi_cond}}\\
&\le \Lpi \Lol\LF \psipi(j,k(t)) \sup_{s \in \cI_j}\| \Qupi(s) -\Qupi(t_j)\| \tag{\Cref{lem:correct_Psiclpi_disc}(d)}\\
&\quad + \Lpi \Lcltwo\psipi(j,k(t)) \|\Qupi(t_j)\| \tag{\Cref{lem:correct_Psiclpi_disc}(b)}\\
&\le \step \Lpi \Lol\LF \psipi(j,k(t))\MQ(1+\KF)  + \Lpi \Lcltwo\psipi(j,k(t)) \LQ \tag{\Cref{asm:cost_asm,lem:cost_disc}}\\
&= \step \psipi(j,k(t)) \left(\Lpi \Lol\LF \MQ(1+\KF)  + \Lpi \Lcltwo \LQ\right).
\end{align*}
Hence,
\begin{align*}
\Term_{4,j} + \Term_{5,j} \le \step \psipi(j,k(t)) \left((1+\Lpi)\Lol\LF \MQ(1+\KF)  + \LQ(\Lclone + \Lpi \Lcltwo) \right)
\end{align*}
Using the definitions of $\Lclone,\Lcltwo$ in \Cref{lem:correct_Psiclpi_disc} and $\Lpi \ge 1$, we have 
\begin{align*}
(\Lclone + \Lpi \Lcltwo) &= \Lol(\KF\MF + 4\LF^2\Lpi  + \Lpi(\KF\MF + 2\LF^2))\\
&= \Lol((1+\Lpi)\KF\MF + 6\LF^2\Lpi \le \Lpi(2\KF\MF + 6\LF^2)) 
\end{align*}
Substituing into the the previous display and again using $\Lpi, \ge 1$, 
\begin{align*}
\Term_{4,j} + \Term_{5,j} &\le \step \psipi(j,k(t)) \Lol\left(2\Lpi\LF \MQ(1+\KF)  + \LQ \Lpi(2\KF\MF + 6\LF^2 ) \right)\\
&\le 2\Lpi\step \psipi(j,k(t)) \Lol\left(\LF \MQ(1+\KF)  + \LQ (\KF\MF + 3\LF^2 ) \right)\\
&\le \step \psipi(j,k(t)) \underbrace{2\Lpi\Lol\left(\LF \MQ(1+\KF)  + \LQ (\KF\MF + 3\LF^2 ) \right)}_{:= \Lclthree}.
\end{align*}
Hence,
\begin{align*}
\step \sum_{j=k(t)+1}^K  \Term_{4,j} + \Term_{5,j}  &\le \step \Lclthree \cdot (\step \sum_{j=k(t)+1}^K\psipi(j,k(t)))\\
&\le \step \Lclthree \Kpione. \tag{ $\step\sum_{j=k(t)+1}^K\psipi(j,k(t))= \step\sum_{j=k(t)+1}^{K}\|\Phicldisc{j,k}\| \le \Kpione$}
\end{align*}
In sum, we conclude that 
\begin{align*}
&\|\nabtil \cJ_T( \pi)(t) - \nabla \cJ_T( \pi)(t)\| \\
&\le \step\left(\LQ \Lol \LF + \MQ(1+\KF) + \LQ\Lcltwo\psipi(K+1, k(t)) +  \Kpione\Lclthree\right)\\
&\le \step\left(\LQ \Lol \LF + \MQ(1+\KF) + \LQ\Lcltwo\Kpiinf +  \Kpione\Lclthree\right)\\
&\le \step\max\{\Kpiinf,\Kpione,1\}\left(\LQ \Lol \LF + \MQ(1+\KF) + (\LQ\Lcltwo+\Lclthree)\right).
\end{align*}
Finally, using the definition of $\Lcltwo := \Lol\left( \KF\MF + 2\LF^2\right)$ in \Cref{lem:correct_Psiclpi_disc}(b) and the definition of $\Lclthree := 2\Lpi\Lol\left(\LF \MQ(1+\KF)  + \LQ (\KF\MF + 3\LF^2 ) \right)$ defined above, and $\Lpi \ge 1$, 
\begin{align*}
\MQ(1+\KF) + \LQ \Lol \LF + \LQ\Lcltwo+\Lclthree \le \Lpi\Lol\left((1+\LF) \MQ(1+\KF)  + \LQ (3\KF\MF + 8\LF^2 +\LF) \right).
\end{align*}
 Thus, recalling $\Lol = \exp(\step \LF)$,
\begin{align*}
\|\nabtil \cJ_T( \pi)(t) - \nabla \cJ_T( \pi)(t)\| \le \step e^{\step \LF}\max\{\Kpiinf,\Kpione,1\}\Lpi \left((1+\LF) \MQ(1+\KF)  + \LQ (3\KF\MF + 8\LF^2 +\LF) \right),
\end{align*}
as needed.

\qed

\part{Experiments}

\section{Experiments Details}
\label{append:experimentdetails}

\subsection{Implementation Details}
\texttt{trajax}~\cite{trajax2021github} is used for nonlinear iLQR trajectory optimization and \texttt{haiku}+\texttt{optax}~\cite{haiku2020github,deepmind2020jax} for training neural network dynamics models.

\subsection{Environments}

\subsubsection{Pendulum}

We consider simple pendulum dynamics with state $(\theta, \dot{\theta})$ and input $u$:
\begin{align*}
     \ddot{\theta} &= \sin(\theta) + u.
\end{align*}
To integrate these dynamics, we use a standard forward Euler approximation with step size $\tau=0.1$,
applying a zero-order hold to the input.
The goal is to swing up the pendulum to the origin state $(0, 0)$.
We consider the cost function:
\begin{align*}
    c((\theta, v), u) = \theta^2 + v^2 + u^2.
\end{align*}

\paragraph{Evaluation details.} All methods were evaluated over a horizon of length $T = 50$ on initial states sampled from $\mathrm{Unif}([-1 + \pi, 1 + \pi] \times \{0\})$.

\paragraph{Random state sampling distribution.}
For learning from random states and actions, we sample 
the initial condition from $\mathrm{Unif}([-5, 5]^2)$
and random actions from $\mathrm{Unif}([-1, 1])$.

\subsubsection{Quadrotor}

The 2D quadrotor is described by the state vector:
$$
    (x, z, \phi, \dot{x}, \dot{z}, \dot{\phi}),
$$
with input $u = (u_1, u_2)$
and dynamics:
\begin{align*}
   \ddot{x} &= -u_1 \sin(\phi) / m, \\
   \ddot{z} &= u_1 \cos(\phi)/m - g, \\
   \ddot{\phi} &= u_2 / I_{xx}.
\end{align*}
The specific constants we use are 
$m = 0.8$, $g = 0.1$, and $I_{xx} = 0.5$.
Again, we integrate these dynamics using forward Euler with step size $\tau=0.1$.
The task is to move the quadrotor to the origin state.
The cost function we use is:
$$
    c((x, z, \phi, \dot{x}, \dot{z}, \dot{\phi}), (u_1, u_2)) = 
    x^2 + z^2 + 10\phi^2 + 0.1(\dot{x}^2 + \dot{z}^2 + \dot{\phi}^2) + 0.1 (u_1^2 + u_2^2).
$$

\paragraph{Evaluation details.} All methods were evaluated over a horizon of length $T = 50$ on initial states sampled from $\mathrm{Uniform}([-0.5, 0.5]^2 \times \{0\}^4)$

\paragraph{Random state sampling distribution.}
For learning from random states and actions, we sample the initial condition
from $\mathrm{Unif}([-3, 3]^6)$ and random actions from $\mathrm{Unif}([-1.5, 1.5])$.

\subsection{Neural network training}

For modeling environment dynamics, we consider three layer fully connected neural networks 
For pendulum, we set the width to $96$, the learning rate to $10^{-3}$, and the
activation to $\mathrm{swish}$.
For quadrotor, we set the width to $128$, the learning rate to $5 \times 10^{-3}$,
and the activation to $\mathrm{gelu}$.
We use the Adam optimizer with $10^{-4}$ additive weight decay and
a cosine decay learning schedule.

\subsection{Least Squares}




While \Cref{alg:learn_mpc_feedback} features a method of moments estimator to estimate
the Markov transition operators, our implementation relies on using regularized least squares.
Specifically, we solve:
\begin{align}
\underbrace{\begin{bmatrix} \bhatPsi{j}{1} \\ \bhatPsi{j}{2} \\ \dots \\ \bhatPsi{j}{j-1}
\end{bmatrix}^\top }_{\in \R^{\dimx \times (j-1)\dimu}} =  \left(\sum_{i=1}^N \underbrace{\begin{bmatrix} \bwki[1]\\ \bwki[2]\\ \dots\\ \bwki[j-1]\end{bmatrix}}_{\in \R^{(j-1)\dimu}}\begin{bmatrix} \bwki[1]\\
\bwki[2]\\ \dots\\ \bwki[j-1] \end{bmatrix}^\top + \lambda \eye\right)^{-1}
\cdot\sum_{i=1}^N
\begin{bmatrix} \bwki[1]\\ \bwki[2]\\ \dots \\ \bwki[j-1] \end{bmatrix}
(\byk[j]^{(i)} - \bhatxk[j])^\top \label{eq:group_ls}
\end{align}

\subsection{Scaling the gain matrix}

In order to stabilize the gain computation during 
gain estimation (\Cref{alg:gain_est}), we scale the update to the $\bhatPk$ matrix as:
\begin{align*}
    \bhatPk \gets \bhatPk \cdot \frac{1}{1+0.01\| \bhatPk \|_F}.
\end{align*}

\end{document}